\documentclass{article}

\usepackage[preprint]{neurips_2026}

\usepackage[utf8]{inputenc} % allow utf-8 input
\usepackage[T1]{fontenc}    % use 8-bit T1 fonts
\usepackage{hyperref}       % hyperlinks
\usepackage{url}            % simple URL typesetting
\usepackage{booktabs}       % professional-quality tables
\usepackage{amsfonts}       % blackboard math symbols
\usepackage{nicefrac}       % compact symbols for 1/2, etc.
\usepackage{microtype}      % microtypography
\usepackage{xcolor}         % colors

\usepackage{comment}
\usepackage{graphicx}
\usepackage{overpic}
\usepackage{amsthm}
\usepackage{multirow} 
\usepackage{enumitem}
\usepackage{tabulary}
\usepackage{colortbl}
\usepackage{wrapfig}
\usepackage{amsmath, amssymb}
\usepackage{algorithm}
\usepackage{algpseudocodex}
\usepackage{diagbox}

\newcommand{\R}{\mathbb{R}}
\let \bs=\mathbf
\let \set=\mathcal

\def \prctile {\textup{prctile}}

\def \deform {\textup{def}}

\def \D {\mathcal{D}}

\def \diag {\mathrm{diag}}

\def \data {\textup{data}}
\def \reg {\textup{reg}}
\def \sync {\textup{sync}}

\def \path {\mathit{path}}

\def \exp {\textup{exp}}

\let \set = \mathcal
\let \bs = \boldsymbol

\usepackage{tikz}
\usetikzlibrary{calc} % <--- 必须添加这一行

\newcommand{\redbox}[1]{%
\begin{tikzpicture}[inner sep=0pt, baseline=(img.south)]
    \node[anchor=south west] (img) at (0,0) {#1};
    \draw[red, line width=1.5pt] ($(img.south west)-(1.5pt,1.5pt)$) rectangle ($(img.north east)+(1.5pt,1.5pt)$);
\end{tikzpicture}%
}

\newcommand{\bluepart}[5]{%
\begin{tikzpicture}[inner sep=0pt, baseline=(img.south)]
    \node[anchor=south west] (img) at (0,0) {#1};
    \begin{scope}[x={(img.south east)}, y={(img.north west)}]
        \draw[blue, line width=1.5pt] (#2, #3) rectangle (#4, #5);
    \end{scope}
\end{tikzpicture}%
}

\newcommand{\metricdown}{\mkern-5mu\downarrow}

\newcommand{\gr}{\rowcolor{gray!10}}

\usepackage{arydshln}

\usepackage{subcaption}

\title{ARAPDiffusion: Geometry–Distribution Feedback for Deformable 3D Shape Generation}

\author{%
  Haibo Liu$^1$ \And
  Jinghan Ke$^1$ \And
  Haitao Yang$^1$ \And
  Xiangru Huang$^2$ \And
  Georgios Pavlakos$^1$ \And
  Qixing Huang$^1$ \\
  \\
  $^1$University of Texas at Austin \quad $^2$Westlake University
}

\begin{document}

\maketitle

\begin{abstract}
Deformable 3D shapes can undergo large changes in pose and morphology while still preserving plausible local geometry, yet generative models often produce samples with unnatural stretching or distortion. In latent generative models, such failures can arise from two coupled sources: the decoder determines the geometry produced at a latent code, while the generative distribution determines which latent regions are actually visited. We introduce ARAPDiffusion, a latent generative framework that establishes feedback between deformation geometry and the generative distribution. First, latent codes sampled from the current diffusion model determine where as-rigid-as-possible (ARAP) constraint is applied, concentrating geometric supervision in regions that the generator actually visits. Second, generated shapes are evaluated by their deformation relative to the observed shape family, and these scores modulate the influence of generated samples during diffusion refinement. In this way, the generative distribution determines where geometry is enforced, while deformation quality in turn reshapes what the model generates. ARAPDiffusion applies to both consistently meshed shapes and unordered point clouds, and the learned prior can also be used for conditional reconstruction. Experiments across human, animal, and bone shape families show consistent improvements in generated-shape quality, distribution matching, and reconstruction accuracy.
\end{abstract}
\section{Introduction}

A central challenge in deformable 3D shape generation is to capture the diversity of a shape family without allowing geometrically implausible deformations. Humans, animals, and anatomical structures can differ substantially in pose and morphology, yet not every geometric variation is plausible: a shape with a reasonable global configuration may still contain unnatural local stretching or distortion. A successful generator must therefore model not only the range of possible shapes, but also how valid shapes can vary from one another. This requirement is important for both unconditional generation and applications that use the learned generator as a shape prior, such as reconstructing complete geometry from incomplete observations.

This issue becomes more subtle in latent generative models because shape generation is split between a learned representation and a learned distribution over that representation. Reconstruction losses constrain the decoder at latent codes corresponding to observed shapes, but provide limited guidance on how the decoder should behave at newly generated codes. A latent diffusion model, meanwhile, learns which regions of the encoded shape space are likely to be sampled. As a result, geometric failures can arise in two different ways: the decoder may behave poorly in regions not sufficiently constrained by reconstruction, or the generative model may place probability on such regions. Addressing either component alone can leave the other source of error unresolved.

Deformation geometry provides a natural prior for this problem. Even under substantial global changes, local regions of deformable shapes often move approximately rigidly. ARAP captures this structure by factoring out local rigid motion and penalizing the remaining distortion, making it well suited to articulated shape families where large valid motions should not themselves be suppressed; previous approaches~\cite{Huang_2021_ICCV,DBLP:conf/iclr/0005HSBH24} use such geometry to regularize learned shape representations. Latent diffusion, however, exposes a mismatch that representation-level regularization alone does not address. Geometric constraints are typically applied at latent locations chosen independently of the learned generative distribution, even though some regions are far more relevant to generation than others. Conversely, improving the decoder at a latent code does not change how much probability the diffusion model assigns to that region. This leads to two questions: where should deformation geometry be enforced in a generative latent space, and how can deformation quality influence the generative distribution itself?

\begin{wrapfigure}{r}{0.43\textwidth}
\vspace{-0.15in}
\begin{overpic}[width=0.42\textwidth]{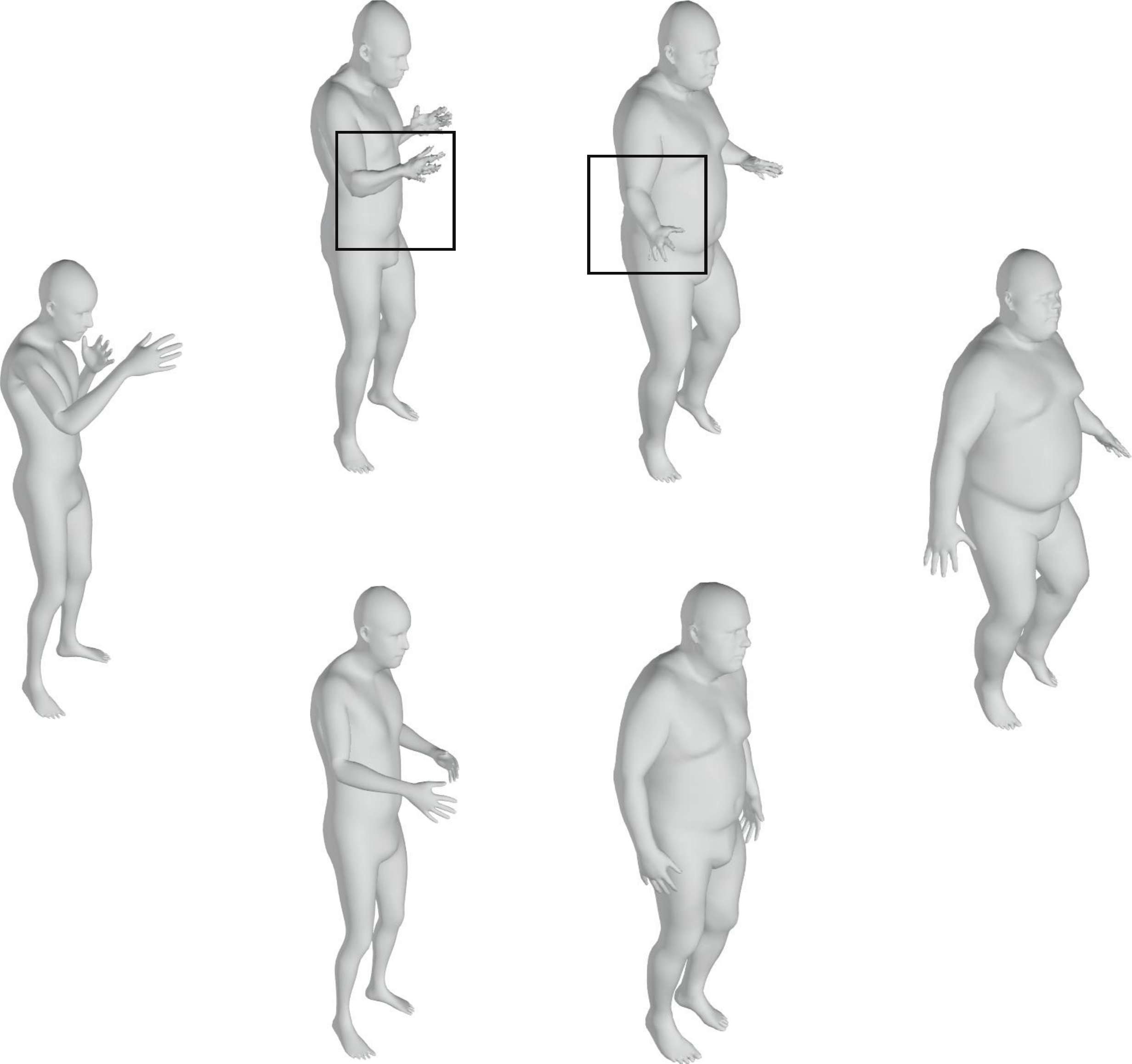}
\put(1, 23){Source}
\put(82, 23){Target}
\put(34, 49){ARAPReg}
\put(27, -5.5){ARAPDiffusion}
\end{overpic}
\vspace{0.2in}
\caption{Intermediate shapes between the same source and target. ARAPReg~\cite{Huang_2021_ICCV} produces visible geometric distortions in intermediate shapes, while ARAPDiffusion better preserves local geometry throughout the deformation.}
\vspace{-0.1in}
\label{Fig:ARAPDiffusion:Teaser}
\end{wrapfigure}

We introduce ARAPDiffusion as a geometry--distribution feedback framework whose defining feature is a bidirectional dependency between decoder geometry and the latent generative distribution. First, in \emph{generation-aware geometric refinement}, the current diffusion model supplies latent codes at which ARAP constraints are applied to the decoder. Geometric supervision is therefore concentrated in latent regions that the learned generator actually visits, rather than at locations sampled independently of the generative distribution. Second, in \emph{geometry-aware diffusion refinement}, we evaluate generated shapes by their deformation relative to the observed shape family and use this geometric quality to refine the diffusion model. Samples with better deformation quality receive greater influence during this update. Together, these two state-dependent directions define one feedback step: the current generative distribution determines where geometry is refined, while the resulting deformation quality determines the update to the next generative state.

The key methodological distinction is the dependency structure between the two refinements rather than either refinement in isolation. Geometric supervision is conditioned on the current generative distribution, whereas the subsequent generative update is conditioned on deformation quality induced by the refined representation. Thus, the supervision received by each component evolves with the state of the other instead of being specified once under a fixed sampling or weighting rule. Alternating these updates realizes a closed state-dependent feedback loop. Fig.~\ref{Fig:ARAPDiffusion:Teaser} illustrates the decoder-side effect by comparing intermediate shapes between the same endpoints, where ARAPDiffusion better preserves local geometry. The same framework applies to consistently meshed shapes and unordered point clouds through mesh and implicit decoders, respectively, and the learned deformation-aware prior can also be used for conditional reconstruction. Across human and animal datasets, with additional results on bone shapes, ARAPDiffusion improves generated-shape quality and distribution matching while also improving reconstruction accuracy. Ablations directly test the two cross-dependencies by retaining the corresponding refinement procedure while replacing its state-dependent signal with a frozen or deformation-independent counterpart.
\section{Related Work}

\noindent\textbf{Generative models for deformable shapes.} For consistently meshed data, COMA~\cite{DBLP:conf/eccv/RanjanBSB18}, Neural3DMM~\cite{DBLP:conf/iccv/BouritsasBPZB19}, SP-Disentangle~\cite{conf/eccv/ZhouBP20}, and MeshConv~\cite{DBLP:conf/nips/ZhouWLCYSLS20} learn low-dimensional representations of non-rigid shape variation, while implicit generative models~\cite{DBLP:conf/iclr/AtzmonL21,DBLP:journals/corr/abs-2108-08931} remove the need for shared mesh connectivity. Diffusion-based methods have also been applied to humans and articulated shapes~\cite{DBLP:conf/iclr/TevetRGSCB23,DBLP:conf/cvpr/StathopoulosHM24,DBLP:conf/cvpr/ZhangBXWK0B24,DBLP:journals/corr/abs-2410-03665,lu2025dposer,ho2025phd}, often in parametric spaces such as SMPL~\cite{loper2015smpl} or SMAL~\cite{Zuffi:CVPR:2017}. Latent diffusion~\cite{DBLP:conf/cvpr/RombachBLEO22} performs generation in learned representation spaces and has been widely adopted for 3D generation~\cite{10.1145/3592442,cae613698fd84e4c8ec6c13b2d2e0f7c,Xiang_2025_CVPR}, with joint learning of latent representations and generative models also explored~\cite{DBLP:conf/nips/VahdatKK21}. GPLD3D~\cite{dong2024gpld3d} further incorporates geometric and physical quality priors into latent diffusion for 3D shape generation. Our focus is instead on coupling the geometry induced by a learned shape representation with its generative distribution, so that generation determines where geometric supervision is applied and the resulting geometry in turn influences subsequent generation, for both consistently meshed shapes and unordered point clouds.

\noindent\textbf{Deformation geometry for generative shape models.} ARAPReg~\cite{Huang_2021_ICCV} regularizes mesh generators using as-rigid-as-possible deformation geometry, while subsequent work develops related geometric constraints, equivariant formulations, and Riemannian descriptions of learned shape spaces~\cite{doi:10.1137/24M1644730,Atzmon_2022_CVPR,DBLP:conf/cvpr/MuralikrishnanC22,DBLP:journals/cgf/FotiKSC23,10.1145/3610548.3618192,DBLP:conf/iclr/AtzmonHWL24}; GenCorres~\cite{DBLP:conf/iclr/0005HSBH24} extends deformation-based regularization to implicit generators without known correspondence. Related work in computational shape analysis studies geometric metrics and geodesics for deformable shape spaces~\cite{DBLP:journals/pami/KlassenSMJ03,JEMS_2006_8_1_a0,DBLP:journals/pami/SrivastavaKJJ11,DBLP:journals/ijcv/HartmanSKCB23,10.1145/3618371,DBLP:conf/iccv/HartmanP0CD23,hartman2024selfsupervisednetworkslearning,DBLP:journals/ijcv/HartmanPBDC25,10.1145/1276377.1276457,DBLP:conf/cvpr/EisenbergerNKLN21}; GeoLatent~\cite{10.1145/3618371}, for example, derives a Riemannian metric from ARAP deformation geometry. These methods establish deformation geometry as an effective prior for learning shape variation, but geometric regularization is typically defined independently of where a learned generative model places probability, and improving the decoder does not by itself alter that distribution. ARAPDiffusion instead couples the two: the current generative distribution determines where geometric supervision is applied, while the resulting deformation quality provides feedback for subsequent generation.

GPLD3D~\cite{dong2024gpld3d} incorporates geometric and physical priors into latent diffusion, while reward-weighted regression applies quality-weighted denoising to generated samples~\cite{black2024training}. Our diffusion update follows this general principle, retaining observed-data supervision and restricting generated-sample supervision to low noise levels.

ARAPDiffusion makes geometric representation learning and generative refinement mutually state-dependent: the current generator determines where ARAP regularization is applied, while the refined decoder provides deformation scores for the next generative update. Our ablations test both dependencies using frozen sampling and mean-matched uniform controls.
\section{Problem Statement and Approach Overview}

\noindent\textbf{Problem statement.} Let $\overline{\set{S}}$ denote a deformable shape family and $\set{S}=\{S_1,\ldots,S_N\}\subset\overline{\set{S}}$ a collection of observed shapes. We aim to learn a generative model that produces new, geometrically plausible shapes from the same family, primarily considering unconditional generation and later extending the learned prior to conditional reconstruction. The framework supports both consistently meshed shapes, such as SMPL~\cite{loper2015smpl} and SMAL~\cite{Zuffi:CVPR:2017}, and unordered point clouds without shared connectivity.

\noindent\textbf{Approach overview.} As shown in Fig.~\ref{Figure:Overview}, ARAPDiffusion builds on latent diffusion~\cite{DBLP:conf/cvpr/RombachBLEO22}, with an encoder $e^{\psi}$ mapping observed shapes to a latent space $\set{Z}\cong\R^d$, a decoder $g^{\phi}:\set{Z}\rightarrow\overline{\set{S}}$, and a diffusion model $\D^{\theta}$ over the encoded shapes. Starting from standard latent-diffusion initialization, ARAPDiffusion alternates two feedback directions: samples from the current generative distribution determine where ARAP supervision is applied, while deformation quality of generated shapes modulates their influence during diffusion refinement. Their alternation forms the proposed geometry--distribution feedback loop. The framework applies to both fixed-topology mesh decoders and implicit decoders for unordered point clouds.

\begin{figure}
\begin{overpic}[width=1.0\textwidth]{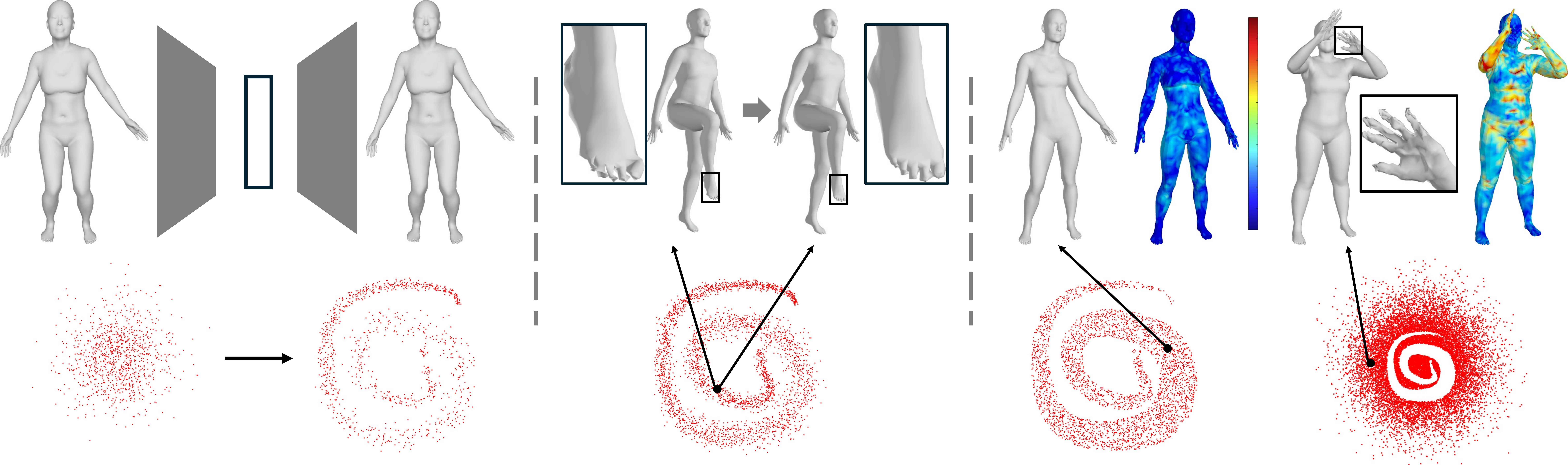}
\put(11,13.2){$e^{\psi}$}
\put(15.5,15.5){$\set{Z}$}
\put(20,13.2){$g^{\phi}$}
\put(15.5, 4){$\D^{\theta}$}
\put(4,-1){$\set{N}(0,I_d)$}
\put(22,-1){$\{e^{\psi}(S_i)\}$}
\put(46, -1) {$p_{\sync}$}
\put(63, -1) {$\{\bs{z}|r_{\deform}(\bs{z}) \leq 2\eta\}$}
\put(83, -1) {$\{\bs{z}|r_{\deform}(\bs{z}) > 2\eta\}$}
\put(73, 12.6) {nn in $\set{S}$}
\put(93.5, 12.6) {nn in $\set{S}$}
\put(80.5, 15){\small{$0$}}
\put(80.5, 27.5){\small{$1$}}
\end{overpic}
\vspace{0.01in}
\caption{Overview of ARAPDiffusion. Starting from standard latent-diffusion initialization, ARAPDiffusion alternates two feedback directions: samples from the current generative distribution determine where ARAP supervision is applied, while deformation quality of generated shapes modulates their influence during diffusion refinement. Because each direction changes the state used to construct the supervision signal for the other, their alternation forms the proposed geometry--distribution feedback loop.}
\label{Figure:Overview}
\vspace{-0.2in}
\end{figure}
\section{Method}
\label{Section:Approach}

ARAPDiffusion realizes geometry--distribution feedback through alternating state-dependent updates, where each update changes the sampling or weighting rule used by the other.

\subsection{Standard Latent-Diffusion Initialization}
\label{Subsec:StageI}

We first train an autoencoder on the observed shapes. A shape encoder $e^{\psi}: \overline{\set{S}}\rightarrow \set{Z}$ maps each shape to a latent code, and a decoder $g^{\phi}:\set{Z}\rightarrow \overline{\set{S}}$ reconstructs the shape:
\begin{equation}
\min\limits_{\psi, \phi} \quad \sum\limits_{S\in \set{S}} d^2(S, g^{\phi}(e^{\psi}(S))).
\end{equation}
For consistently meshed shapes, $d^2$ is the cumulative squared distance between corresponding vertices, using the encoder and decoder architectures of ARAPReg. For unordered point clouds, we use PointNet++~\cite{10.5555/3295222.3295263} as the encoder and the GenCorres implicit decoder, with squared algebraic distance as the reconstruction loss.

We then fit a diffusion model to the encoded training shapes. Let $p_{\data}$ denote the empirical distribution over $\{e^{\psi}(S)\mid S\in\set{S}\}$. Following EDM~\cite{DBLP:conf/nips/KarrasAAL22}, we train
$$
\min\limits_{\theta}
\underset{\bs{z}\sim p_{\data}}{\mathbb{E}}
\underset{\sigma}{\mathbb{E}}
\underset{\bs{n}}{\mathbb{E}}
\lambda(\sigma)
\|\D^{\theta}(\bs{z}+\bs{n};\sigma)-\bs{z}\|^2 .
$$
The denoising architecture follows DiffuseSDF~\cite{chou2022diffusionsdf} and contains attention~\cite{10.5555/3295222.3295349}, fully connected, and normalization layers. This provides the initial latent-diffusion model used by the two refinement stages below.

\subsection{Generation-Aware Geometric Refinement}
\label{Subsec:StageII}

The first refinement improves decoder geometry in latent regions relevant to generation. Standard ARAP regularization samples latent locations from a fixed prior, which need not reflect where the learned generator places probability. We instead sample latent codes from the current diffusion model itself. Let $p_{\sync}$ denote this sampling distribution. Holding the diffusion model fixed, we refine the encoder and decoder using
\begin{align}
\min\limits_{\psi, \phi} & \ \ \frac{1}{|\set{S}|}\sum\limits_{S\in \set{S}} d^2(S, g^{\phi}(e^{\psi}(S))) +\underset{\bs{z}\sim p_{\sync}}{\mathbb{E}}l_{\reg}(\bs{m}^{\phi}(\bs{z}))
\label{Eq:AE:Finetuning}
\end{align}
The reconstruction term keeps the representation anchored to the observed shapes, while the second term realizes distribution-to-geometry feedback: drawing $\bs{z}$ from $p_{\sync}$ rather than a fixed Gaussian concentrates ARAP supervision in latent regions emphasized by the current generator.

ARAP should measure geometric distortion rather than deformation magnitude. Following ARAPReg~\cite{Huang_2021_ICCV} and GeoLatent~\cite{10.1145/3618371}, let
$H(\bs{m}^{\phi}(\bs{z}))$ denote the ARAP pull-back metric in latent space, which measures the local geometric distortion induced by latent perturbations; its construction is reviewed in Sec.~\ref{App:ARAPReg:Details}. For example, a limb may undergo a large displacement while remaining locally close to rigid. We therefore normalize the ARAP energy by the magnitude of the induced deformation. For a tangent direction $\bs{v}$ at $\bs{z}$, we define
$$
r_{\reg}(\bs{z},\bs{v}) :=
\frac{\bs{v}^TH(\bs{m}^{\phi}(\bs{z}))\bs{v}}
{\bs{v}^T E(\bs{m}^{\phi}(\bs{z}))\bs{v}},
\qquad
E(\bs{m}^{\phi}(\bs{z})) =
{\frac{\partial \bs{m}^{\phi}(\bs{z})}{\partial \bs{z}}}^T
\frac{\partial \bs{m}^{\phi}(\bs{z})}{\partial \bs{z}}
+\delta I_d ,
$$
where $\delta=10^{-10}$ prevents numerical instability for infinitesimal displacements. This yields the normalized metric
\begin{equation}
H_n(\bs{m}^{\phi}(\bs{z})) =
{E(\bs{m}^{\phi}(\bs{z}))}^{-\frac{1}{2}}
H(\bs{m}^{\phi}(\bs{z}))
{E(\bs{m}^{\phi}(\bs{z}))}^{-\frac{1}{2}} .
\end{equation}
We use
$$
l_{\reg}(\bs{m}^{\phi}(\bs{z}))
=
\|H_n(\bs{m}^{\phi}(\bs{z}))\|_{\star},
$$
where $\|\cdot\|_{\star}$ denotes the trace norm. This normalization makes the regularizer sensitive to distortion relative to deformation magnitude, rather than penalizing large but geometrically valid shape changes.

For implicit decoders, $\bs{m}^{\phi}(\bs{z})$ denotes a mesh discretization of the implicit surface, and we evaluate the same pull-back metric using the local deformation construction of GenCorres~\cite{DBLP:conf/iclr/0005HSBH24}; details are provided in Sec.~\ref{App:GenCorres:Details}.

\subsection{Geometry-Aware Diffusion Refinement}
\label{Subsec:StageIII}

Improving the decoder does not determine how probability should be distributed across latent space. A latent region may decode smoothly yet still produce shapes that are poor members of the observed family. The reverse geometry-to-distribution direction closes the feedback loop by using deformation quality to refine the diffusion model on generated samples.

For a generated latent code $\bs{z}$, we define a family-relative deformation score $r_{\deform}(\bs{z})$. A small value indicates that the decoded shape can be related to the observed shape family through relatively low-distortion deformation. We convert this score into

$$
w(\bs{z}) =
\exp\bigg(-\frac{r^2_{\deform}(\bs{z})}{2\eta^2}\bigg),
$$
where $\eta$ is the median leave-one-out deformation score of the encoded training shapes; when calibrating this scale on a training shape, the query shape itself is excluded from the nearest-neighbor search. At each alternation, $p_{\data}$ is formed from the training shapes encoded by the current encoder after the distribution-to-geometry update. We then fix the encoder and decoder and refine $\D^\theta$ by augmenting the original EDM objective with weighted generated samples:

\begin{equation}
\min\limits_{\theta}\quad
\underset{\bs{z}\sim p_{\data}}{\mathbb{E}}
\underset{\sigma}{\mathbb{E}}
\underset{\bs{n}}{\mathbb{E}}
\lambda(\sigma)
\|\D^{\theta}(\bs{z}+\bs{n};\sigma)-\bs{z}\|^2
+
\underset{\bs{z}\sim p_{\sync}}{\mathbb{E}}
\underset{\sigma\leq \sigma_0}{\mathbb{E}}
\underset{\bs{n}}{\mathbb{E}}
\lambda(\sigma) w(\bs{z})
\|\D^{\theta}(\bs{z}+\bs{n};\sigma)-\bs{z}\|^2 .
\label{Eq:DM:Loss2}
\end{equation}

The first term anchors diffusion to the encoded observations. The second uses deformation quality as a confidence signal for generated samples: geometrically better generated samples receive greater influence, while poorly deformed samples are suppressed. Viewed in isolation, this is quality-weighted refinement on generated samples. Within the alternating procedure, however, these weights are produced by decoder geometry refined at locations sampled from the current generator; updating $\D^\theta$ then changes that generator, and hence the latent locations used by the next geometry update. Because the weights are not normalized across generated samples, they affect both relative contributions among generated samples and the aggregate strength of generated-sample refinement with respect to the data term. We intentionally retain both effects, treating $w(\bs{z})$ as confidence rather than as a normalized importance ratio, and therefore do not interpret Eq.~\ref{Eq:DM:Loss2} as defining an explicit normalized target distribution. In Sec.~\ref{Subsec:Ablation:Study}, a mean-matched uniform control separately tests the sample-dependent component while preserving aggregate generated-sample weight. We restrict this update to small noise levels $\sigma\leq\sigma_0$, so the feedback remains local around generated codes while the large-noise regime remains anchored by $p_{\data}$. We set $\sigma_0$ to half of the average distance from each training latent code to its nearest other training code.

\paragraph{Deformation score.}
For consistently meshed shapes, we evaluate deformation quality relative to the closest observed shape in latent space. Because mesh connectivity is shared, vertex correspondence is known. At each vertex, we factor out the best local rotation and measure the residual deformation; small residuals indicate locally near-rigid changes, while large residuals indicate stronger distortion.

Formally, for two meshes with vertex vectors $\bs{p},\bs{q}\in\R^{3n}$ and vertex neighborhood $\set{N}(i)$, we define
\begin{equation}
r_{\deform}^i(\bs{p},\bs{q}) =
\Bigg(
\min\limits_{R\in\mathrm{SO}(3)}
\sum\limits_{j\in \set{N}(i)}
\|R(\bs{p}_i-\bs{p}_j)-(\bs{q}_i-\bs{q}_j)\|^2
\Bigg)^{\frac{1}{2}} .
\label{Eq:Deformation:i}
\end{equation}
The optimal rotation is obtained by singular-value decomposition of
$M_i=\sum_{j\in\set{N}(i)}(\bs{p}_i-\bs{p}_j)(\bs{q}_i-\bs{q}_j)^T$~\cite{arun1987least}.
Let $S_{i_{\bs{z}}}$ denote the closest training shape to $\bs{m}^{\phi}(\bs{z})$. We compute
$r_{\deform}^i(\bs{m}^{\phi}(\bs{z}),S_{i_{\bs{z}}})$
over all vertices and use their $p_c$-percentile as $r_{\deform}(\bs{z})$, with $p_c=80$ in our experiments.

We use a percentile rather than the maximum because articulated joints can naturally exhibit larger deformation than most of the surface, causing a small number of vertices to dominate as $p_c$ approaches $100$. The score correlates with SMPL/SMAL reconstruction error and remains stable over a broad range of $p_c$; we provide the corresponding analysis in Appendix~\ref{App:DeformationScoreValidation}.

\noindent\textbf{Implicit generative model.}
The same principle applies to implicit generators, but explicit vertex correspondence is unavailable. We therefore use GenCorres~\cite{DBLP:conf/iclr/0005HSBH24} to deform the closest training shape toward the generated implicit surface, providing a temporary correspondence through which the same local deformation error can be evaluated.

Let $S_{i_{\bs{z}}}(\bs{z})$ denote the training shape after deformation toward $g^\phi(\bs{z})$. We define
$$
r_{\deform}(\bs{z}) =
\prctile\!\left(
r_{\deform}^{j}
\big(
S_{i_{\bs{z}}},
S_{i_{\bs{z}}}(\bs{z})
\big),
p_c
\right)
+
d_{\textup{Haus}}
\big(
S_{i_{\bs{z}}}(\bs{z}),
g^{\phi}(\bs{z})
\big).
$$
The first term measures the deformation required from the observed shape, while the second measures how well the deformed shape reaches the generated surface. Validation of the deformation score for implicit generators is provided in Appendix~\ref{App:DeformationScoreValidation}. For the implicit datasets considered here, we additionally set $w(\bs{z})=0$ for topologically invalid generated meshes whose Euler characteristic $|\set{V}|-|\set{E}|+|\set{F}|\neq2$. Additional score variants are discussed in the supplementary material.

\subsection{Extension to Conditional Reconstruction}
\label{Subsec:CG}

The deformation-aware prior learned above can also be used for reconstruction from partial observations. We first learn the decoder using the unconditional procedure in Secs.~\ref{Subsec:StageI}--\ref{Subsec:StageIII}, then fix it and train a conditional diffusion model $\D^{\theta}(\bs{z}+\bs{n},\sigma;c)$, where $c$ denotes the observation.

A valid conditional sample must satisfy both the observed condition and the deformation prior. We therefore draw latent codes from the current conditional diffusion model, denoting the resulting distribution by $p_{\sync}^c$, and weight their contribution using the same deformation score:

\begin{comment}
\begin{equation}
\hspace{-0.7cm}
\begin{aligned}
\min_{\theta}&\quad
\underset{\bs{z}\sim p_{\data}}{\mathbb{E}}
\underset{\sigma}{\mathbb{E}}
\underset{\bs{n}}{\mathbb{E}}
\lambda(\sigma)
\|\D^{\theta}(\bs{z}+\bs{n},\sigma,c)-\bs{z}\|^2
\\
\quad+
\mu
\underset{\bs{z}\sim p_{\sync}^c}{\mathbb{E}}
&\underset{\sigma\leq \sigma_0}{\mathbb{E}}
\underset{\bs{n}}{\mathbb{E}}
\lambda(\sigma)w(\bs{z})
\|\D^{\theta}(\bs{z}+\bs{n},\sigma,c)-\bs{z}\|^2 .
\end{aligned}
\label{Eq:DM:Loss3}
\end{equation}
\end{comment}

\begin{equation}
\min_{\theta}\quad
\underset{\bs{z}\sim p_{\data}}{\mathbb{E}}
\underset{\sigma}{\mathbb{E}}
\underset{\bs{n}}{\mathbb{E}}
\lambda(\sigma)
\|\D^{\theta}(\bs{z}+\bs{n},\sigma,c)-\bs{z}\|^2
+
\mu
\underset{\bs{z}\sim p_{\sync}^c}{\mathbb{E}}
\underset{\sigma\leq \sigma_0}{\mathbb{E}}
\underset{\bs{n}}{\mathbb{E}}
\lambda(\sigma)w(\bs{z})
\|\D^{\theta}(\bs{z}+\bs{n},\sigma,c)-\bs{z}\|^2 .
\label{Eq:DM:Loss3}
\end{equation}
\section{Experimental Results}

\subsection{Experimental Setup}
\label{Subsec:Exp:Setup}

We evaluate ARAPDiffusion on the Human and Animal datasets used by ARAPReg~\cite{Huang_2021_ICCV} and GenCorres~\cite{DBLP:conf/iclr/0005HSBH24}, which are generated from SMPL~\cite{loper2015smpl} and SMAL~\cite{Zuffi:CVPR:2017}, respectively. For mesh generators, we follow the training/test splits of ARAPReg; for implicit generators, we follow those of GenCorres. Additional experiments on the Bone datasets from ARAPReg are provided in the supplementary material.

We evaluate four complementary aspects of the learned model. The reconstruction error $e_r$, following ARAPReg and GenCorres, measures reconstruction on held-out shapes. For generated-shape quality, $e_t$ is the Hausdorff distance to the corresponding SMPL/SMAL fit (Sec.~\ref{Sec:Experimental:Details}); because fitting is initialized and regularized using the nearest training shape, we treat $e_t$ as a fitting-based proxy rather than an exact distance to the parametric manifold. The distribution metric $d_{\textup{W}}$ is the empirical Wasserstein distance between 2000 generated and training samples in ambient space, using Chamfer distance as the pairwise ground cost. We therefore use $e_t$ as a generated-shape quality proxy and rely on $d_W$ as a complementary distribution-level measure, rather than interpreting $e_t$ alone as a direct estimate of distributional fidelity. Finally, $d_n$ is the Hausdorff distance to the nearest training shape, serving as a sanity check on training-set proximity rather than a standalone coverage metric. For implicit generators, we extract triangular meshes before evaluation. Additional evaluation details and timing are provided in Sec.~\ref{Sec:Experimental:Details}.

\subsection{Unconditional Mesh Generators}
\label{Subsec:UC:Mesh:Gen}

\begin{figure*}[t]
\centering
\begin{overpic}[width=\textwidth]{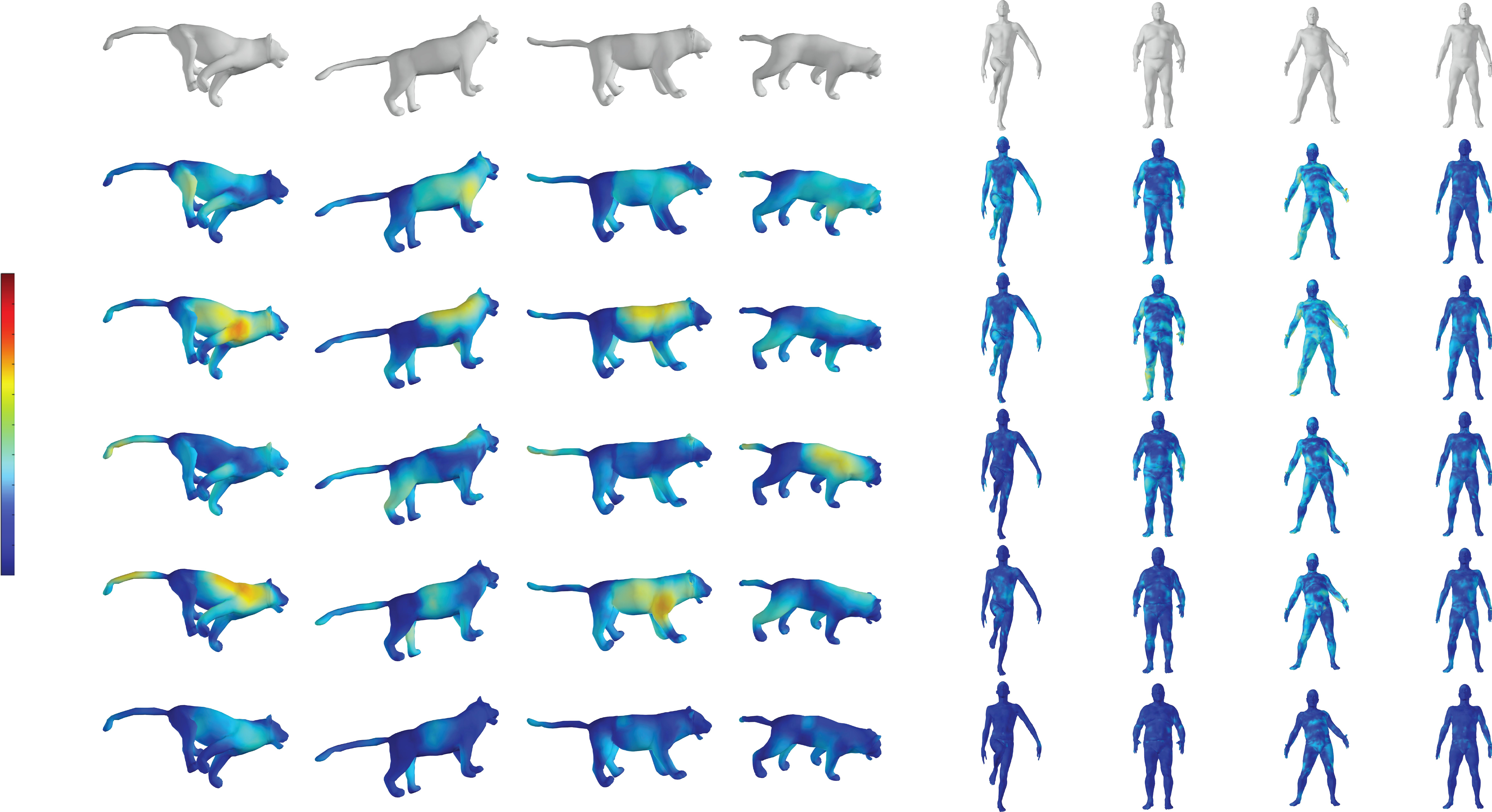}
\put(3, 0){\rotatebox{90}{\small{ARAPDiff.}}}
\put(3, 11){\rotatebox{90}{\small{BRESA}}}
\put(3, 19.5){\rotatebox{90}{\small{GeoLatent}}}
\put(3, 30){\rotatebox{90}{\small{FrameAve}}}
\put(3, 38.5){\rotatebox{90}{\small{ARAPReg}}}
\put(3, 48.5){\rotatebox{90}{\small{Test}}}
\end{overpic}
\caption{Color-coded reconstruction errors on held-out test shapes. We compare ARAPDiffusion with the strongest deformation-aware baselines.}
\label{Fig:Qualitative:Results}
\vspace{-0.2in}
\end{figure*}

We first evaluate unconditional mesh generation against SP-Disent.~\cite{conf/eccv/ZhouBP20}, COMA~\cite{DBLP:conf/eccv/RanjanBSB18}, Neural3DMM~\cite{DBLP:conf/iccv/BouritsasBPZB19}, MeshConv~\cite{DBLP:conf/nips/ZhouWLCYSLS20}, ARAPReg~\cite{Huang_2021_ICCV}, FrameAve~\cite{Atzmon_2022_CVPR}, GeoLatent~\cite{10.1145/3618371}, and BRESA~\cite{DBLP:journals/ijcv/HartmanPBDC25}. Tab.~\ref{Table:Quantitative:Mesh} reports the quantitative results, while Fig.~\ref{Fig:Qualitative:Results} provides a spatial comparison of reconstruction errors on held-out test shapes.

ARAPDiffusion outperforms all prior baselines in $e_r$, $e_t$, and $d_W$ on both Human and Animal. The largest gains occur in $e_t$ and $d_W$, the metrics most directly associated with the two goals of our framework: improving the geometry of individual generated shapes and improving distribution-level matching. The $d_n$ values remain comparable to those of the baselines, providing no indication that these gains are obtained simply by moving generated samples closer to the training set.

NoReg is already a strong latent-diffusion baseline and substantially outperforms prior methods in generated-shape quality $e_t$ and distribution matching $d_W$. Its advantage does not extend uniformly to reconstruction, however: its $e_r$ is worse than several deformation-aware baselines on both datasets. ARAPDiffusion further reduces $e_r/e_t/d_W$ over NoReg by 19.3\%/32.0\%/17.9\% on Human and 15.6\%/16.1\%/20.4\% on Animal, showing that deformation geometry provides substantial complementary gains on top of the strong latent-diffusion backbone. Additional random samples from all methods are provided in the supplementary material.

\begin{wraptable}{r}{0.56\textwidth}
\vspace{-0.2in}
\setlength\tabcolsep{1.5pt}
    \centering
    \begin{tabular}{c|cccc|cccc}
         &\multicolumn{4}{c}{Human} & \multicolumn{4}{c}{Animal} \\\hline
         &$e_r\metricdown$ & $e_t\metricdown$ & $d_W\metricdown$ & $d_{n}$& $e_r\metricdown$& $e_t\metricdown$ & $d_W\metricdown$ &$d_{n}$ \\\hline
SP-Disent.  &10.0 &24.1 & 23.1& 14.2 &15.2 & 14.3&26.1 & 15.3\\
COMA       &8.80  & 22.3 & 20.2& 10.7 & 14.5 & 13.2 & 23.2 & 12.9\\
3DMM        & 7.39& 19.1& 21.3 & 9.81 &17.3 & 13.6 & 24.7 & 11.1\\
MeshConv             & 5.43& 16.8 & 19.3 & 9.76 &8.01 & 11.9  & 20.1 & 10.8\\
ARAPReg & 4.52 & 19.7& 18.9 & 6.81 &9.68 & 12.5& 22.5 & 10.9\\
FrameAve &4.94 & 20.6& 16.7 & 9.54 & 12.5& 10.6& 20.2 & 11.1\\
GeoLatent &4.31 & 17.2& 19.1 & 9.81  & 5.98 &11.2 & 21.2 & 11.2\\
BRESA & 4.41& 22.4 & 20.1 & 9.71 & 6.12 & 12.6 & 23.4  &10.7 \\
\hline
\gr ARAPDiff. & \textbf{4.18}& \textbf{7.21} &\textbf{12.4} & 9.85 & \textbf{5.75} & \textbf{7.35} & \textbf{14.4} & 10.9 \\
\cdashline{1-9}[2pt/3pt]
NoReg & 5.18 & 10.6 & 15.1 & 9.86& 6.81 & 8.76 & 18.1 &10.9\\
NoDeReg & 4.74 &8.05 & 13.9 &9.77 & 6.16& 7.99 & 16.1 &11.0\\
NoLaReg &4.34 & 7.27 &13.4 &9.87 & 6.03 &7.71 & 15.3 & 11.1 \\
FrozenInit  & 4.32 & 7.89 & 13.6 & 9.88 & 6.01 & 7.84 & 15.8 &  11.0\\
MeanMatch& 4.21 & 7.26 & 13.2 & 9.86 & 5.82 & 7.61& 15.2 & 10.8 \\
AlterDiff & 4.16 & 7.19 & 12.3 & 9.86 & 5.77 & 7.34 & 14.3 & 10.8
\end{tabular}
\caption{Comparison with prior mesh-generation methods and variants of ARAPDiffusion. We report $e_r$ and $e_t$ in mm and $d_W$ and $d_n$ in cm. ARAPDiffusion is highlighted in bold; ablations and the alternative diffusion backbone are reported below the dashed line.}
\label{Table:Quantitative:Mesh}
\vspace{-0.15in}
\end{wraptable}

\subsection{Unconditional Implicit Generators}
\label{Subsec:UC:Implicit:Gen}

We next test whether the same geometry--generation coupling extends beyond meshes with known correspondence. Fig.~\ref{Figure:Implicit:Qualitative} and Tab.~\ref{Table:Quantitative:Implicit} compare ARAPDiffusion with SALD~\cite{DBLP:conf/iclr/AtzmonL21} and GenCorres~\cite{DBLP:conf/iclr/0005HSBH24}. We additionally include HY-LoRA, a strong reference obtained by fine-tuning the pretrained Hunyuan3D 2.5 model~\cite{DBLP:journals/corr/abs-2506-16504} on the training data using LoRA~\cite{DBLP:conf/iclr/HuSWALWWC22}.

ARAPDiffusion achieves the lowest $e_r$, $e_t$, and $d_W$ on both Human and Animal. More importantly, the improvement pattern mirrors the mesh setting: the gains are substantially larger in generated-shape quality and distribution matching than in reconstruction. Compared with GenCorres, $e_r$ decreases by 2.06\% and 5.95\% on Human and Animal, respectively, while the larger reductions in $e_t$ and $d_W$ show that incorporating deformation geometry into generation benefits both decoded samples and the learned distribution. This consistency indicates that the proposed coupling does not rely on explicit mesh correspondence.

\begin{wraptable}{r}{0.54\textwidth}
\setlength\tabcolsep{1.5pt}
\centering
\begin{tabular}{c|cccc|cccc}
         &\multicolumn{4}{c}{Human} & \multicolumn{4}{c}{Animal} \\\hline
         &$e_r\metricdown$ & $e_t\metricdown$ & $d_W\metricdown$ & $d_n$ & $e_r\metricdown$& $e_t\metricdown$ & $d_W\metricdown$ & $d_n$ \\\hline
SALD  & 7.21 &16.9 & 21.8 & 10.2 & 8.12&18.9 & 20.5 & 11.2\\
GenCorres        &  6.31&  16.4 & 22.1 & 9.89 & 7.56& 19.1 & 21.2&11.3\\
HY-LoRA & 6.23& 8.84 & 15.9 & 9.8 & 7.23 & 9.25 & 16.9  &8.5 \\ \hline
\gr ARAPDiff. & \textbf{6.18} & \textbf{8.39} & \textbf{14.8} & 10.1 & \textbf{7.11} & \textbf{8.91} & \textbf{16.5} & 11.1
\end{tabular}
\caption{Comparison with implicit-generation baselines under $e_r$, $e_t$, $d_W$, and $d_n$. The best result for $e_r$, $e_t$, and $d_W$ is bold-faced.}
    \label{Table:Quantitative:Implicit}
\vspace{-0in}
\vspace{-0.4in}
\end{wraptable}

ARAPDiffusion also improves $e_r$, $e_t$, and $d_W$ relative to HY-LoRA, despite the latter starting from a large-scale pretrained 3D model. HY-LoRA has a smaller $d_n$, indicating that its samples lie closer to the training set. Fig.~\ref{Figure:Implicit:Qualitative} provides qualitative comparisons of the reconstructed shapes.

\begin{wrapfigure}{r}{0.55\textwidth}
\vspace{0in}
\begin{overpic}[width=0.54\textwidth]{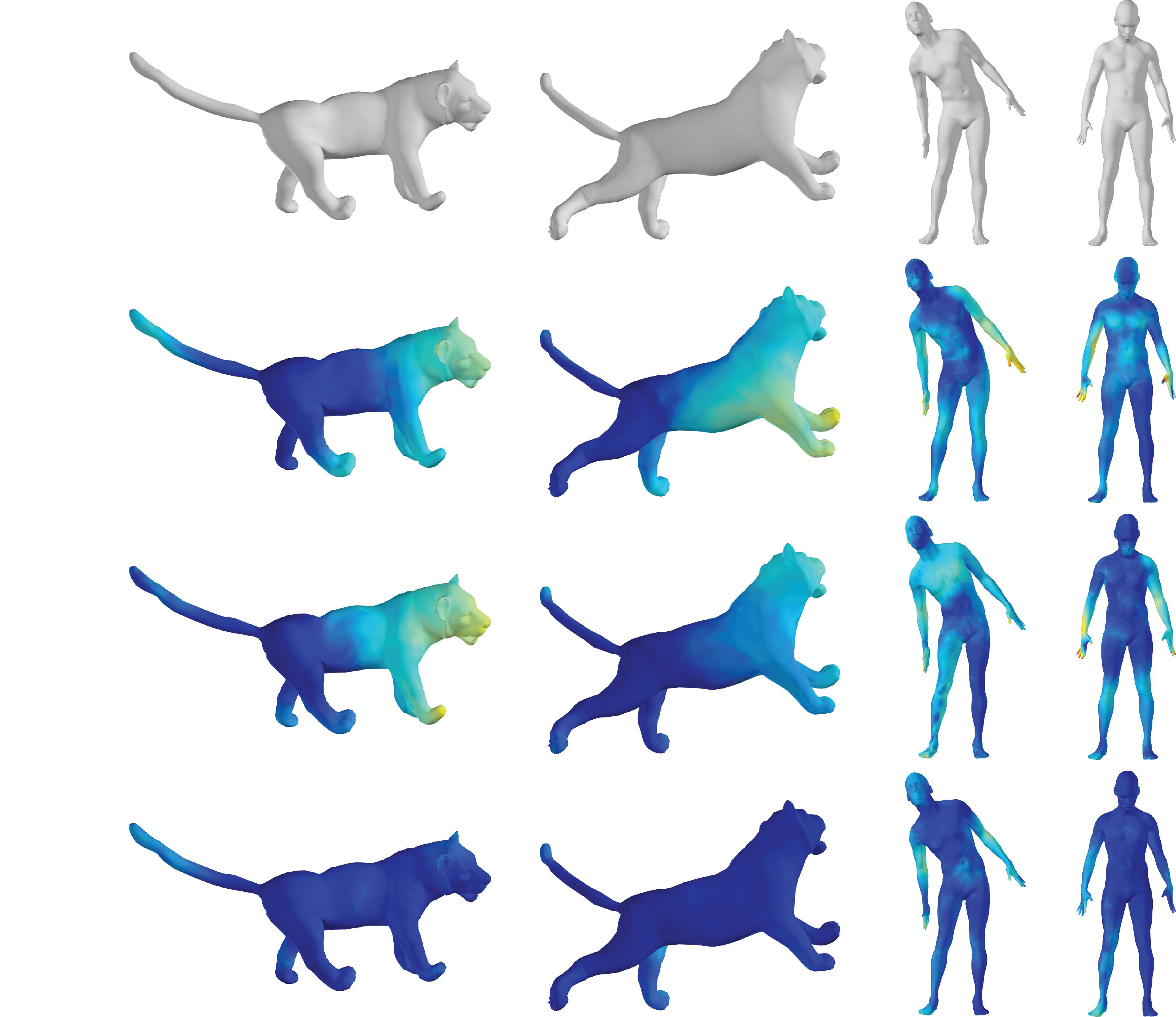}
\put(1, 2){\rotatebox{90}{\small{ARAPDiff.}}}
\put(1, 25){\rotatebox{90}{\small{GenCorres}}}
\put(1, 50){\rotatebox{90}{\small{SALD}}}
\put(1, 74){\rotatebox{90}{\small{Test}}}
\end{overpic}
\caption{Qualitative reconstruction comparison for implicit generators. ARAPDiffusion more closely reproduces the held-out test shapes.}
\label{Figure:Implicit:Qualitative}
\vspace{-0.2in}
\end{wrapfigure}

\subsection{Ablation Study}
\label{Subsec:Ablation:Study}

We ablate the feedback loop at two levels: NoDeReg and NoLaReg remove one refinement direction entirely, while FrozenInit and MeanMatchedUniform retain the corresponding refinement procedure but replace its state-dependent signal with a static or deformation-independent counterpart, respectively. FrozenInit tests the distribution-to-geometry dependency by sampling geometric-refinement locations throughout training from a frozen copy of the diffusion model obtained after standard initialization, rather than from the evolving generator. MeanMatchedUniform isolates the sample-dependent component of geometry-to-distribution feedback. For each generated minibatch $\mathcal{B}$, we first compute the deformation-dependent weights $w(\bs{z}_i)$ used by the full model and their mean $\bar{w}_{\mathcal{B}}=\frac{1}{|\mathcal{B}|}\sum_{\bs{z}_i\in\mathcal{B}}w(\bs{z}_i)$. We then assign the same weight $\bar{w}_{\mathcal{B}}$ to every generated sample in that minibatch. Consequently, the aggregate generated-sample weight is preserved, while the deformation-dependent relative weighting among samples is removed. All other training and evaluation settings are kept unchanged across each controlled comparison, so that each variant alters only the feedback component under study. Results for mesh generators are reported in Tab.~\ref{Table:Quantitative:Mesh}.

\noindent\textbf{NoReg (standard latent diffusion).} Removing both geometry-based refinements increases $e_r$, $e_t$, and $d_W$ by 23.9\%/18.4\%, 47.0\%/19.1\%, and 21.7\%/25.7\% on Human/Animal, respectively. Thus, while latent diffusion already provides strong generation, deformation-aware refinement contributes substantial additional gains across representation and generation metrics.

\noindent\textbf{NoDeReg (without geometric decoder regularization).} Removing generation-aware geometric regularization, while retaining the reconstruction update and geometry-aware diffusion refinement, increases $e_r$, $e_t$, and $d_W$ by 13.4\%/7.13\%, 11.6\%/8.71\%, and 12.1\%/11.8\%, respectively. The degradation relative to the full model shows that geometric supervision of the decoder in regions sampled by the current generator provides a complementary contribution.

\noindent\textbf{NoLaReg (geometric refinement only).} Removing geometry-aware diffusion refinement increases $e_r$, $e_t$, and $d_W$ by 3.83\%/4.87\%, 0.83\%/4.90\%, and 8.06\%/6.25\% on Human/Animal, respectively. The largest and most consistent degradation occurs in $d_W$, which is consistent with geometry-to-distribution feedback acting primarily on the generative model rather than directly on decoder geometry. These results show that geometry-aware diffusion refinement provides a complementary benefit beyond decoder-side geometric refinement. Together with NoDeReg, they support the contribution of both directions to the full geometry--distribution feedback loop.

\noindent\textbf{FrozenInit (frozen-initial-generator refinement).} To assess whether distribution-to-geometry feedback benefits from tracking the evolving generative distribution, we freeze a copy of the diffusion model immediately after standard latent-diffusion initialization. At every subsequent geometry-refinement update, we draw fresh latent samples from this fixed initial generator rather than from the current diffusion model. The trainable diffusion model, geometry-aware diffusion refinement, optimization schedule, and all other settings remain unchanged. Compared with FrozenInit, the full model reduces $e_t$ from 7.89 to 7.21 on Human and from 7.84 to 7.35 on Animal, with corresponding reductions in $d_W$ from 13.6 to 12.4 and from 15.8 to 14.4. Since both variants use latent locations sampled from a learned diffusion model and differ in whether this sampling distribution follows the evolving generator, this comparison tests the benefit of updating geometric supervision with the current generative state rather than fixing it at initialization.

\noindent\textbf{MeanMatchedUniform (mean-matched uniform refinement).} To isolate deformation-dependent sample modulation from the aggregate strength of geometry-aware diffusion refinement, for each generated minibatch $\mathcal{B}$ we replace the individual weights $w(\bs{z}_i)$ with their batch mean $\bar{w}_{\mathcal{B}}$. This preserves the total generated-sample weight within the minibatch while removing all deformation-dependent relative weighting among its samples. Compared with MeanMatchedUniform, the full model reduces $d_W$ from 13.2 to 12.4 on Human and from 15.2 to 14.4 on Animal, while $e_t$ changes from 7.26 to 7.21 and from 7.61 to 7.35. The larger effect on $d_W$ is consistent with the role of this intervention: decoder-side geometric refinement is unchanged, while deformation-dependent relative weighting is removed specifically from the generative-model update. Because the two weighting rules match the aggregate weight on the same generated minibatch, the comparison separates sample-dependent geometric feedback from changes in overall generated-sample weight. Together with FrozenInit, these controls support the two state dependencies underlying the feedback loop: geometric supervision tracks the evolving generative distribution, while generated-sample influence depends on deformation quality.

\noindent\textbf{AlterDiff (alternative generative backbone).} Finally, we replace EDM with the rectified flow model~\cite{DBLP:conf/iclr/LiuG023}. Performance remains nearly unchanged, and in several entries differs only marginally from the EDM version. This suggests that the observed improvements are not specific to EDM and persist under an alternative generative formulation.

\subsection{Conditional Reconstruction}
\label{Subsec:C:Mesh:Gen}

\begin{wrapfigure}{r}{0.50\textwidth}
\vspace{-0.5in}
\begin{overpic}[width=0.50\textwidth]{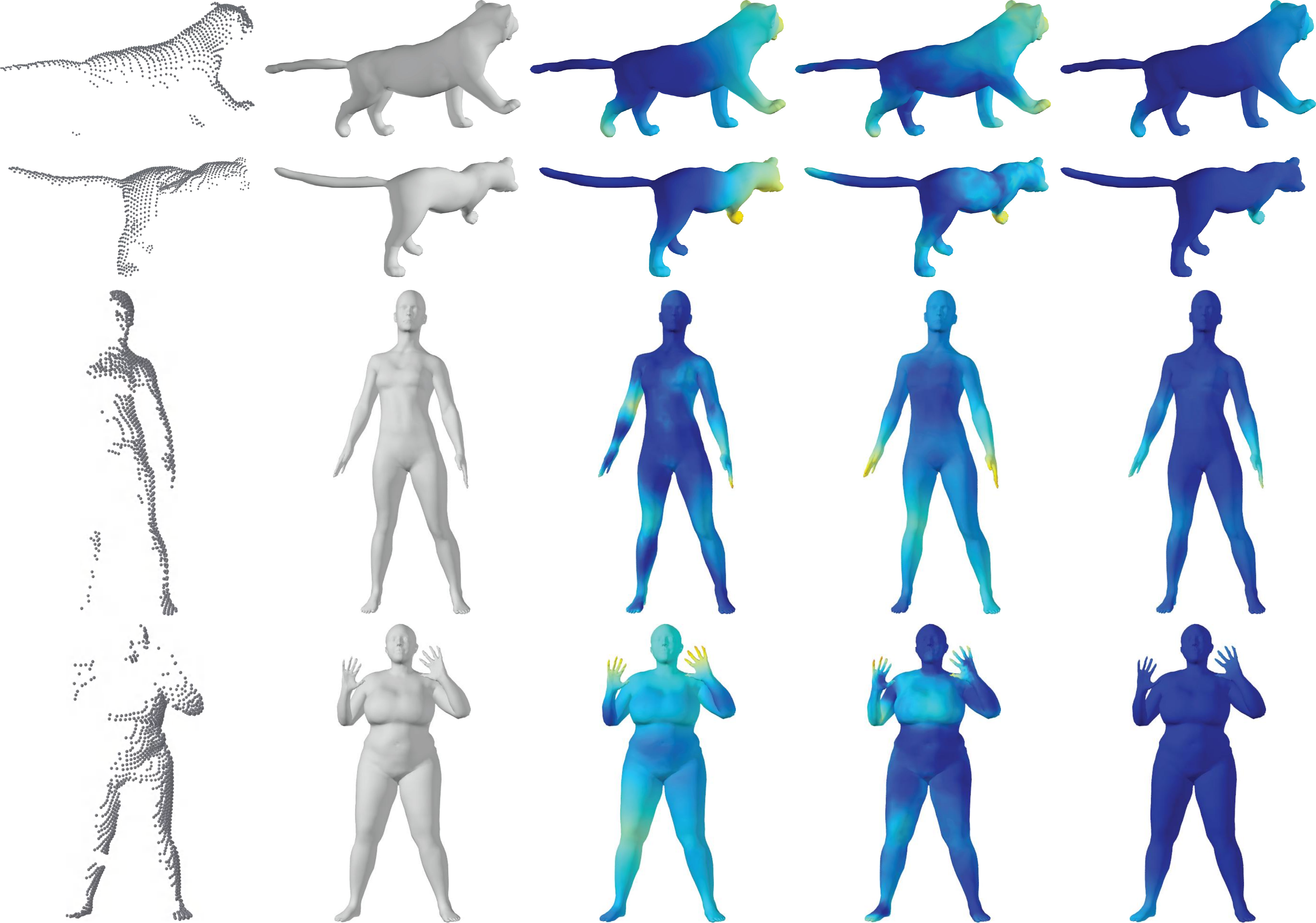}
\put(6,-6){\small{Input}}
\put(28,-6){\small{GT}}
\put(42,-6){\small{GeoLatent}}
\put(64,-6){\small{BRESA}}
\put(82,-6){\small{ARAPDiff.}}
\end{overpic}
\vspace{0in}
\caption{Conditional reconstruction from partial observations. ARAPDiffusion uses the learned shape prior together with the observed points to recover the complete shape.}
\label{Figure:Conditional:Generation}
\vspace{-0.5in}
\end{wrapfigure}

Fig.~\ref{Figure:Conditional:Generation} qualitatively evaluates the extension to reconstruction from partial observations against GeoLatent and BRESA. Each method uses a conditional diffusion model in its own latent space to infer a complete shape from the observed points. Across the shown examples, ARAPDiffusion recovers plausible unobserved geometry while remaining consistent with the partial input, illustrating that the learned deformation-aware prior can also be used for conditional reconstruction.

\section{Conclusions and Future Work}

We presented ARAPDiffusion, which couples decoder geometry and latent generation through a closed state-dependent loop. The current generator determines where decoder geometry is refined, while the refined decoder determines how generated samples update the next generator. Because each step changes the state conditioning the other, the contribution lies in this bidirectional dependency rather than either refinement or alternation alone. Our ablations test both directions and their state dependence against static or deformation-independent controls. The resulting coupling improves generated-shape quality and distribution matching, applies to mesh and implicit representations, and yields a deformation-aware prior for conditional reconstruction from partial observations.

Future work includes integrating deformation quality more directly into diffusion training or sampling, extending the framework to broader shape representations and families, and exploring dynamic settings such as 4D shape and motion generation. More broadly, our results suggest that geometric structure can provide useful feedback not only for learning shape representations, but also for shaping the generative distributions built on top of them.

\section*{Acknowledgments}

This project was supported by NSF-2047677, 2413161, 2504906, and 2515626; GIFTs from Adobe and Google; and computing support on the Vista GPU Cluster through the Center for Generative AI (CGAI) and the Texas Advanced Computing Center (TACC) at UT Austin.

% \clearpage
\bibliographystyle{abbrv}
\bibliography{main}

@inproceedings{dong2024gpld3d,
  title={Gpld3d: Latent diffusion of 3d shape generative models by enforcing geometric and physical priors},
  author={Dong, Yuan and Zuo, Qi and Gu, Xiaodong and Yuan, Weihao and Zhao, Zhengyi and Dong, Zilong and Bo, Liefeng and Huang, Qixing},
  booktitle={2024 IEEE/CVF Conference on Computer Vision and Pattern Recognition (CVPR)},
  pages={56--66},
  year={2024},
  organization={IEEE}
}

@inproceedings{black2024training,
  title={Training Diffusion Models with Reinforcement Learning},
  author={Black, Kevin and Janner, Michael and Du, Yilun and Kostrikov, Ilya and Levine, Sergey},
  booktitle={The Twelfth International Conference on Learning Representations},
  year={2024},
  url={https://openreview.net/forum?id=nMkdvj7BF8}
}

@article{loper2015smpl,
  title={{SMPL}: A Skinned Multi-Person Linear Model},
  author={Loper, Matthew and Mahmood, Naureen and Romero, Javier and Pons-Moll, Gerard and Black, Michael J},
  journal={ACM Transactions on Graphics},
  volume={34},
  number={6},
  year={2015},
  publisher={Association for Computing Machinery}
}

@inproceedings{DBLP:conf/iclr/HuSWALWWC22,
  author       = {Edward J. Hu and
                  Yelong Shen and
                  Phillip Wallis and
                  Zeyuan Allen{-}Zhu and
                  Yuanzhi Li and
                  Shean Wang and
                  Lu Wang and
                  Weizhu Chen},
  title        = {LoRA: Low-Rank Adaptation of Large Language Models},
  booktitle    = {The Tenth International Conference on Learning Representations, {ICLR}
                  2022, Virtual Event, April 25-29, 2022},
  publisher    = {OpenReview.net},
  year         = {2022},
  url          = {https://openreview.net/forum?id=nZeVKeeFYf9},
  bibsource    = {dblp computer science bibliography, https://dblp.org}
}

@article{DBLP:journals/corr/abs-2506-16504,
  author       = {Zeqiang Lai and
                  Yunfei Zhao and
                  Haolin Liu and
                  Zibo Zhao and
                  Qingxiang Lin and
                  Huiwen Shi and
                  Xianghui Yang and
                  Mingxin Yang and
                  Shuhui Yang and
                  Yifei Feng and
                  Sheng Zhang and
                  Xin Huang and
                  Di Luo and
                  Fan Yang and
                  Fang Yang and
                  Lifu Wang and
                  Sicong Liu and
                  Yixuan Tang and
                  Yulin Cai and
                  Zebin He and
                  Tian Liu and
                  Yuhong Liu and
                  Jie Jiang and
                  Linus and
                  Jingwei Huang and
                  Chunchao Guo},
  title        = {Hunyuan3D 2.5: Towards High-Fidelity 3D Assets Generation with Ultimate
                  Details},
  journal      = {CoRR},
  volume       = {abs/2506.16504},
  year         = {2025},
  url          = {https://doi.org/10.48550/arXiv.2506.16504},
  doi          = {10.48550/ARXIV.2506.16504},
  eprinttype   = {arXiv},
  eprint       = {2506.16504},
  bibsource    = {dblp computer science bibliography, https://dblp.org}
}

@article{doi:10.1137/24M1644730,
author = {Dummer, Sven and Strisciuglio, Nicola and Brune, Christoph},
title = {RDA-INR: Riemannian Diffeomorphic Autoencoding via Implicit Neural Representations},
journal = {SIAM Journal on Imaging Sciences},
volume = {17},
number = {4},
pages = {2302-2330},
year = {2024},
doi = {10.1137/24M1644730},
URL = {https://doi.org/10.1137/24M1644730},
eprint = {https://doi.org/10.1137/24M1644730},
}

@inproceedings{10.5555/3295222.3295349,
author = {Vaswani, Ashish and Shazeer, Noam and Parmar, Niki and Uszkoreit, Jakob and Jones, Llion and Gomez, Aidan N. and Kaiser, \L{}ukasz and Polosukhin, Illia},
title = {Attention is all you need},
year = {2017},
isbn = {9781510860964},
publisher = {Curran Associates Inc.},
address = {Red Hook, NY, USA},
booktitle = {Proceedings of the 31st International Conference on Neural Information Processing Systems},
pages = {6000–6010},
numpages = {11},
location = {Long Beach, California, USA},
series = {NIPS'17}
}

@inproceedings{chou2022diffusionsdf,
title={Diffusion-SDF: Conditional Generative Modeling of Signed Distance Functions},
author={Gene Chou and Yuval Bahat and Felix Heide},
booktitle={Proceedings of the IEEE International Conference on Computer Vision (ICCV)},
year={2023}
}

@article{DBLP:journals/cgf/FotiKSC23,
  author       = {Simone Foti and
                  Bongjin Koo and
                  Danail Stoyanov and
                  Matthew J. Clarkson},
  title        = {3D Generative Model Latent Disentanglement via Local Eigenprojection},
  journal      = {Comput. Graph. Forum},
  volume       = {42},
  number       = {6},
  year         = {2023},
  url          = {https://doi.org/10.1111/cgf.14793},
  doi          = {10.1111/CGF.14793},
  bibsource    = {dblp computer science bibliography, https://dblp.org}
}

@inproceedings{DBLP:conf/cvpr/MuralikrishnanC22,
  author       = {Sanjeev Muralikrishnan and
                  Siddhartha Chaudhuri and
                  Noam Aigerman and
                  Vladimir G. Kim and
                  Matthew Fisher and
                  Niloy J. Mitra},
  title        = {{GLASS:} Geometric Latent Augmentation for Shape Spaces},
  booktitle    = {{IEEE/CVF} Conference on Computer Vision and Pattern Recognition,
                  {CVPR} 2022, New Orleans, LA, USA, June 18-24, 2022},
  pages        = {470--479},
  publisher    = {{IEEE}},
  year         = {2022},
  url          = {https://doi.org/10.1109/CVPR52688.2022.01800},
  doi          = {10.1109/CVPR52688.2022.01800},
  bibsource    = {dblp computer science bibliography, https://dblp.org}
}

@inproceedings{10.1145/3610548.3618192,
author = {Maesumi, Arman and Guerrero, Paul and Aigerman, Noam and Kim, Vladimir and Fisher, Matthew and Chaudhuri, Siddhartha and Ritchie, Daniel},
title = {Explorable Mesh Deformation Subspaces from Unstructured 3D Generative Models},
year = {2023},
isbn = {9798400703157},
publisher = {Association for Computing Machinery},
address = {New York, NY, USA},
url = {https://doi.org/10.1145/3610548.3618192},
doi = {10.1145/3610548.3618192},
booktitle = {SIGGRAPH Asia 2023 Conference Papers},
articleno = {68},
numpages = {11},
location = {Sydney, NSW, Australia},
series = {SA '23}
}

@inproceedings{10.5555/3295222.3295263,
author = {Qi, Charles R. and Yi, Li and Su, Hao and Guibas, Leonidas J.},
title = {PointNet++: deep hierarchical feature learning on point sets in a metric space},
year = {2017},
isbn = {9781510860964},
publisher = {Curran Associates Inc.},
address = {Red Hook, NY, USA},
booktitle = {Proceedings of the 31st International Conference on Neural Information Processing Systems},
pages = {5105–5114},
numpages = {10},
location = {Long Beach, California, USA},
series = {NIPS'17}
}

@inproceedings{DBLP:conf/iclr/AtzmonHWL24,
  author       = {Matan Atzmon and
                  Jiahui Huang and
                  Francis Williams and
                  Or Litany},
  title        = {Approximately Piecewise {E(3)} Equivariant Point Networks},
  booktitle    = {The Twelfth International Conference on Learning Representations,
                  {ICLR} 2024, Vienna, Austria, May 7-11, 2024},
  publisher    = {OpenReview.net},
  year         = {2024},
  url          = {https://openreview.net/forum?id=aKJEHWmBEf},
  bibsource    = {dblp computer science bibliography, https://dblp.org}
}

@inproceedings{DBLP:conf/cvpr/RombachBLEO22,
  author       = {Robin Rombach and
                  Andreas Blattmann and
                  Dominik Lorenz and
                  Patrick Esser and
                  Bj{\"{o}}rn Ommer},
  title        = {High-Resolution Image Synthesis with Latent Diffusion Models},
  booktitle    = {{IEEE/CVF} Conference on Computer Vision and Pattern Recognition,
                  {CVPR} 2022, New Orleans, LA, USA, June 18-24, 2022},
  pages        = {10674--10685},
  publisher    = {{IEEE}},
  year         = {2022},
  url          = {https://doi.org/10.1109/CVPR52688.2022.01042},
  doi          = {10.1109/CVPR52688.2022.01042},
  bibsource    = {dblp computer science bibliography, https://dblp.org}
}

@InProceedings{Atzmon_2022_CVPR,
    author    = {Atzmon, Matan and Nagano, Koki and Fidler, Sanja and Khamis, Sameh and Lipman, Yaron},
    title     = {Frame Averaging for Equivariant Shape Space Learning},
    booktitle = {Proceedings of the IEEE/CVF Conference on Computer Vision and Pattern Recognition (CVPR)},
    month     = {June},
    year      = {2022},
    pages     = {631-641}
}

@inproceedings{DBLP:conf/iclr/AtzmonL21,
  author       = {Matan Atzmon and
                  Yaron Lipman},
  title        = {{SALD:} Sign Agnostic Learning with Derivatives},
  booktitle    = {9th International Conference on Learning Representations, {ICLR} 2021,
                  Virtual Event, Austria, May 3-7, 2021},
  publisher    = {OpenReview.net},
  year         = {2021},
  url          = {https://openreview.net/forum?id=7EDgLu9reQD},
  bibsource    = {dblp computer science bibliography, https://dblp.org}
}

@article{DBLP:journals/corr/abs-2108-08931,
  author       = {Matan Atzmon and
                  David Novotn{\'{y}} and
                  Andrea Vedaldi and
                  Yaron Lipman},
  title        = {Augmenting Implicit Neural Shape Representations with Explicit Deformation
                  Fields},
  journal      = {CoRR},
  volume       = {abs/2108.08931},
  year         = {2021},
  url          = {https://arxiv.org/abs/2108.08931},
  eprinttype    = {arXiv},
  eprint       = {2108.08931},
  bibsource    = {dblp computer science bibliography, https://dblp.org}
}

@inproceedings{DBLP:conf/nips/ZhouWLCYSLS20,
  author       = {Yi Zhou and
                  Chenglei Wu and
                  Zimo Li and
                  Chen Cao and
                  Yuting Ye and
                  Jason M. Saragih and
                  Hao Li and
                  Yaser Sheikh},
  editor       = {Hugo Larochelle and
                  Marc'Aurelio Ranzato and
                  Raia Hadsell and
                  Maria{-}Florina Balcan and
                  Hsuan{-}Tien Lin},
  title        = {Fully Convolutional Mesh Autoencoder using Efficient Spatially Varying
                  Kernels},
  booktitle    = {Advances in Neural Information Processing Systems 33: Annual Conference
                  on Neural Information Processing Systems 2020, NeurIPS 2020, December
                  6-12, 2020, virtual},
  year         = {2020},
  url          = {https://proceedings.neurips.cc/paper/2020/hash/68dd09b9ff11f0df5624a690fe0f6729-Abstract.html},
  bibsource    = {dblp computer science bibliography, https://dblp.org}
}

@inproceedings{DBLP:conf/iccv/BouritsasBPZB19,
  author       = {Giorgos Bouritsas and
                  Sergiy Bokhnyak and
                  Stylianos Ploumpis and
                  Stefanos Zafeiriou and
                  Michael M. Bronstein},
  title        = {Neural 3D Morphable Models: Spiral Convolutional Networks for 3D Shape
                  Representation Learning and Generation},
  booktitle    = {2019 {IEEE/CVF} International Conference on Computer Vision, {ICCV}
                  2019, Seoul, Korea (South), October 27 - November 2, 2019},
  pages        = {7212--7221},
  publisher    = {{IEEE}},
  year         = {2019},
  url          = {https://doi.org/10.1109/ICCV.2019.00731},
  doi          = {10.1109/ICCV.2019.00731},
  bibsource    = {dblp computer science bibliography, https://dblp.org}
}

@inproceedings{DBLP:conf/eccv/RanjanBSB18,
  author       = {Anurag Ranjan and
                  Timo Bolkart and
                  Soubhik Sanyal and
                  Michael J. Black},
  editor       = {Vittorio Ferrari and
                  Martial Hebert and
                  Cristian Sminchisescu and
                  Yair Weiss},
  title        = {Generating 3D Faces Using Convolutional Mesh Autoencoders},
  booktitle    = {Computer Vision - {ECCV} 2018 - 15th European Conference, Munich,
                  Germany, September 8-14, 2018, Proceedings, Part {III}},
  series       = {Lecture Notes in Computer Science},
  volume       = {11207},
  pages        = {725--741},
  publisher    = {Springer},
  year         = {2018},
  url          = {https://doi.org/10.1007/978-3-030-01219-9\_43},
  doi          = {10.1007/978-3-030-01219-9\_43},
  bibsource    = {dblp computer science bibliography, https://dblp.org}
}

@misc{hartman2024selfsupervisednetworkslearning,
      title={Self Supervised Networks for Learning Latent Space Representations of Human Body Scans and Motions}, 
      author={Emmanuel Hartman and Nicolas Charon and Martin Bauer},
      year={2024},
      eprint={2411.03475},
      archivePrefix={arXiv},
      primaryClass={cs.CV},
      url={https://arxiv.org/abs/2411.03475}, 
}

@inproceedings{DBLP:conf/iclr/TevetRGSCB23,
  author       = {Guy Tevet and
                  Sigal Raab and
                  Brian Gordon and
                  Yonatan Shafir and
                  Daniel Cohen{-}Or and
                  Amit Haim Bermano},
  title        = {Human Motion Diffusion Model},
  booktitle    = {The Eleventh International Conference on Learning Representations,
                  {ICLR} 2023, Kigali, Rwanda, May 1-5, 2023},
  publisher    = {OpenReview.net},
  year         = {2023},
  url          = {https://openreview.net/forum?id=SJ1kSyO2jwu},
  bibsource    = {dblp computer science bibliography, https://dblp.org}
}

@inproceedings{DBLP:conf/iclr/LiuG023,
  author       = {Xingchao Liu and
                  Chengyue Gong and
                  Qiang Liu},
  title        = {Flow Straight and Fast: Learning to Generate and Transfer Data with
                  Rectified Flow},
  booktitle    = {The Eleventh International Conference on Learning Representations,
                  {ICLR} 2023, Kigali, Rwanda, May 1-5, 2023},
  publisher    = {OpenReview.net},
  year         = {2023},
  url          = {https://openreview.net/forum?id=XVjTT1nw5z},
  bibsource    = {dblp computer science bibliography, https://dblp.org}
}

@inproceedings{DBLP:conf/cvpr/ZhangBXWK0B24,
  author       = {Siwei Zhang and
                  Bharat Lal Bhatnagar and
                  Yuanlu Xu and
                  Alexander Winkler and
                  Petr Kadlecek and
                  Siyu Tang and
                  Federica Bogo},
  title        = {RoHM: Robust Human Motion Reconstruction via Diffusion},
  booktitle    = {{IEEE/CVF} Conference on Computer Vision and Pattern Recognition,
                  {CVPR} 2024, Seattle, WA, USA, June 16-22, 2024},
  pages        = {14606--14617},
  publisher    = {{IEEE}},
  year         = {2024},
  url          = {https://doi.org/10.1109/CVPR52733.2024.01384},
  doi          = {10.1109/CVPR52733.2024.01384},
  bibsource    = {dblp computer science bibliography, https://dblp.org}
}

@InProceedings{Xiang_2025_CVPR,
    author    = {Xiang, Jianfeng and Lv, Zelong and Xu, Sicheng and Deng, Yu and Wang, Ruicheng and Zhang, Bowen and Chen, Dong and Tong, Xin and Yang, Jiaolong},
    title     = {Structured 3D Latents for Scalable and Versatile 3D Generation},
    booktitle = {Proceedings of the Computer Vision and Pattern Recognition Conference (CVPR)},
    month     = {June},
    year      = {2025},
    pages     = {21469-21480}
}

@inproceedings{cae613698fd84e4c8ec6c13b2d2e0f7c,
  title={Sdfusion: Multimodal 3d shape completion, reconstruction, and generation},
  author={Cheng, Yen-Chi and Lee, Hsin-Ying and Tulyakov, Sergey and Schwing, Alexander G and Gui, Liang-Yan},
  booktitle={Proceedings of the IEEE/CVF conference on computer vision and pattern recognition},
  pages={4456--4465},
  year={2023}
}

@article{10.1145/3592442,
author = {Zhang, Biao and Tang, Jiapeng and Nie\ss{}ner, Matthias and Wonka, Peter},
title = {3DShape2VecSet: A 3D Shape Representation for Neural Fields and Generative Diffusion Models},
year = {2023},
issue_date = {August 2023},
publisher = {Association for Computing Machinery},
address = {New York, NY, USA},
volume = {42},
number = {4},
issn = {0730-0301},
url = {https://doi.org/10.1145/3592442},
doi = {10.1145/3592442},
journal = {ACM Trans. Graph.},
month = jul,
articleno = {92},
numpages = {16}
}

@inproceedings{DBLP:conf/cvpr/StathopoulosHM24,
  author       = {Anastasis Stathopoulos and
                  Ligong Han and
                  Dimitris N. Metaxas},
  title        = {Score-Guided Diffusion for 3D Human Recovery},
  booktitle    = {{IEEE/CVF} Conference on Computer Vision and Pattern Recognition,
                  {CVPR} 2024, Seattle, WA, USA, June 16-22, 2024},
  pages        = {906--915},
  publisher    = {{IEEE}},
  year         = {2024},
  url          = {https://doi.org/10.1109/CVPR52733.2024.00092},
  doi          = {10.1109/CVPR52733.2024.00092},
  bibsource    = {dblp computer science bibliography, https://dblp.org}
}

@article{DBLP:journals/corr/abs-2410-03665,
  author       = {Brent Yi and
                  Vickie Ye and
                  Maya Zheng and
                  Lea M{\"{u}}ller and
                  Georgios Pavlakos and
                  Yi Ma and
                  Jitendra Malik and
                  Angjoo Kanazawa},
  title        = {Estimating Body and Hand Motion in an Ego-sensed World},
  journal      = {CoRR},
  volume       = {abs/2410.03665},
  year         = {2024},
  url          = {https://doi.org/10.48550/arXiv.2410.03665},
  doi          = {10.48550/ARXIV.2410.03665},
  eprinttype    = {arXiv},
  eprint       = {2410.03665},
  bibsource    = {dblp computer science bibliography, https://dblp.org}
}

@inproceedings{DBLP:conf/iccv/HartmanP0CD23,
  author       = {Emmanuel Hartman and
                  Emery Pierson and
                  Martin Bauer and
                  Nicolas Charon and
                  Mohamed Daoudi},
  title        = {BaRe-ESA: {A} Riemannian Framework for Unregistered Human Body Shapes},
  booktitle    = {{IEEE/CVF} International Conference on Computer Vision, {ICCV} 2023,
                  Paris, France, October 1-6, 2023},
  pages        = {14135--14145},
  publisher    = {{IEEE}},
  year         = {2023},
  url          = {https://doi.org/10.1109/ICCV51070.2023.01304},
  doi          = {10.1109/ICCV51070.2023.01304},
  bibsource    = {dblp computer science bibliography, https://dblp.org}
}

@article{DBLP:journals/ijcv/HartmanPBDC25,
  author       = {Emmanuel Hartman and
                  Emery Pierson and
                  Martin Bauer and
                  Mohamed Daoudi and
                  Nicolas Charon},
  title        = {Basis Restricted Elastic Shape Analysis on the Space of Unregistered
                  Surfaces},
  journal      = {Int. J. Comput. Vis.},
  volume       = {133},
  number       = {4},
  pages        = {1999--2024},
  year         = {2025},
  url          = {https://doi.org/10.1007/s11263-024-02269-3},
  doi          = {10.1007/S11263-024-02269-3},
  bibsource    = {dblp computer science bibliography, https://dblp.org}
}

@article{DBLP:journals/ijcv/HartmanSKCB23,
  author       = {Emmanuel Hartman and
                  Yashil Sukurdeep and
                  Eric Klassen and
                  Nicolas Charon and
                  Martin Bauer},
  title        = {Elastic Shape Analysis of Surfaces with Second-Order Sobolev Metrics:
                  {A} Comprehensive Numerical Framework},
  journal      = {Int. J. Comput. Vis.},
  volume       = {131},
  number       = {5},
  pages        = {1183--1209},
  year         = {2023},
  url          = {https://doi.org/10.1007/s11263-022-01743-0},
  doi          = {10.1007/S11263-022-01743-0},
  bibsource    = {dblp computer science bibliography, https://dblp.org}
}

@article{JEMS_2006_8_1_a0,
     author = {Peter W. Michor and David B. Mumford},
     title = {Riemannian geometries on spaces of plane curves},
     journal = {Journal of the European Mathematical Society},
     pages = {1--48},
     publisher = {mathdoc},
     volume = {8},
     number = {1},
     year = {2006},
     doi = {10.4171/jems/37},
     url = {https://geodesic-test.mathdoc.fr/articles/10.4171/jems/37/}
}

@article{DBLP:journals/pami/KlassenSMJ03,
  author       = {Eric Klassen and
                  Anuj Srivastava and
                  Washington Mio and
                  Shantanu H. Joshi},
  title        = {Analysis of Planar Shapes Using Geodesic Paths on Shape Spaces},
  journal      = {{IEEE} Trans. Pattern Anal. Mach. Intell.},
  volume       = {26},
  number       = {3},
  pages        = {372--383},
  year         = {2004},
  url          = {https://doi.org/10.1109/TPAMI.2004.1262333},
  doi          = {10.1109/TPAMI.2004.1262333},
  bibsource    = {dblp computer science bibliography, https://dblp.org}
}

@article{DBLP:journals/pami/SrivastavaKJJ11,
  author       = {Anuj Srivastava and
                  Eric Klassen and
                  Shantanu H. Joshi and
                  Ian H. Jermyn},
  title        = {Shape Analysis of Elastic Curves in Euclidean Spaces},
  journal      = {{IEEE} Trans. Pattern Anal. Mach. Intell.},
  volume       = {33},
  number       = {7},
  pages        = {1415--1428},
  year         = {2011},
  url          = {https://doi.org/10.1109/TPAMI.2010.184},
  doi          = {10.1109/TPAMI.2010.184},
  bibsource    = {dblp computer science bibliography, https://dblp.org}
}

@inproceedings{DBLP:conf/cvpr/EisenbergerNKLN21,
  author       = {Marvin Eisenberger and
                  David Novotn{\'{y}} and
                  Gael Kerchenbaum and
                  Patrick Labatut and
                  Natalia Neverova and
                  Daniel Cremers and
                  Andrea Vedaldi},
  title        = {NeuroMorph: Unsupervised Shape Interpolation and Correspondence in
                  One Go},
  booktitle    = {{IEEE} Conference on Computer Vision and Pattern Recognition, {CVPR}
                  2021, virtual, June 19-25, 2021},
  pages        = {7473--7483},
  publisher    = {Computer Vision Foundation / {IEEE}},
  year         = {2021},
  url          = {https://openaccess.thecvf.com/content/CVPR2021/html/Eisenberger\_NeuroMorph\_Unsupervised\_Shape\_Interpolation\_and\_Correspondence\_in\_One\_Go\_CVPR\_2021\_paper.html},
  doi          = {10.1109/CVPR46437.2021.00739},
  bibsource    = {dblp computer science bibliography, https://dblp.org}
}

@article{10.1145/1276377.1276457,
author = {Kilian, Martin and Mitra, Niloy J. and Pottmann, Helmut},
title = {Geometric modeling in shape space},
year = {2007},
issue_date = {July 2007},
publisher = {Association for Computing Machinery},
address = {New York, NY, USA},
volume = {26},
number = {3},
issn = {0730-0301},
url = {https://doi.org/10.1145/1276377.1276457},
doi = {10.1145/1276377.1276457},
journal = {ACM Trans. Graph.},
month = jul,
pages = {64–es},
numpages = {8}
}

@inproceedings{Zuffi:CVPR:2017,
        title = {{3D} Menagerie: Modeling the {3D} Shape and Pose of Animals},
        author = {Zuffi, Silvia and Kanazawa, Angjoo and Jacobs, David and Black, Michael J.},
        booktitle = {IEEE Conf. on Computer Vision and Pattern Recognition (CVPR)},
        month = jul,
        year = {2017},
        month_numeric = {7}
      }

@article{arun1987least,
  title={Least-Squares Fitting of Two 3-D Point Sets},
  author={Arun, K. S. and Huang, T. S. and Blostein, S. D.},
  journal={IEEE Transactions on Pattern Analysis and Machine Intelligence},
  volume={9},
  number={1},
  pages={67--71},
  year={1987},
  publisher={IEEE}
}

@inproceedings{DBLP:conf/nips/KarrasAAL22,
  author       = {Tero Karras and
                  Miika Aittala and
                  Timo Aila and
                  Samuli Laine},
  editor       = {Sanmi Koyejo and
                  S. Mohamed and
                  A. Agarwal and
                  Danielle Belgrave and
                  K. Cho and
                  A. Oh},
  title        = {Elucidating the Design Space of Diffusion-Based Generative Models},
  booktitle    = {Advances in Neural Information Processing Systems 35: Annual Conference
                  on Neural Information Processing Systems 2022, NeurIPS 2022, New Orleans,
                  LA, USA, November 28 - December 9, 2022},
  year         = {2022},
  url          = {http://papers.nips.cc/paper\_files/paper/2022/hash/a98846e9d9cc01cfb87eb694d946ce6b-Abstract-Conference.html},
  bibsource    = {dblp computer science bibliography, https://dblp.org}
}

@article{10.1145/3618371,
author = {Yang, Haitao and Sun, Bo and Chen, Liyan and Pavel, Amy and Huang, Qixing},
title = {GeoLatent: A Geometric Approach to Latent Space Design for Deformable Shape Generators},
year = {2023},
issue_date = {December 2023},
publisher = {Association for Computing Machinery},
address = {New York, NY, USA},
volume = {42},
number = {6},
issn = {0730-0301},
url = {https://doi.org/10.1145/3618371},
doi = {10.1145/3618371},
journal = {ACM Trans. Graph.},
month = dec,
articleno = {242},
numpages = {20}
}

@InProceedings{Huang_2021_ICCV,
    author    = {Huang, Qixing and Huang, Xiangru and Sun, Bo and Zhang, Zaiwei and Jiang, Junfeng and Bajaj, Chandrajit},
    title     = {ARAPReg: An As-Rigid-As Possible Regularization Loss for Learning Deformable Shape Generators},
    booktitle = {Proceedings of the IEEE/CVF International Conference on Computer Vision (ICCV)},
    month     = {October},
    year      = {2021},
    pages     = {5815-5825}
}

@inproceedings{conf/eccv/ZhouBP20,
  author = {Zhou, Keyang and Bhatnagar, Bharat Lal and Pons-Moll, Gerard},
  booktitle = {ECCV (22)},
  editor = {Vedaldi, Andrea and Bischof, Horst and Brox, Thomas and Frahm, Jan-Michael},
  ee = {https://doi.org/10.1007/978-3-030-58542-6_21},
  isbn = {978-3-030-58542-6},
  pages = {341-357},
  publisher = {Springer},
  series = {Lecture Notes in Computer Science},
  title = {Unsupervised Shape and Pose Disentanglement for 3D Meshes.},
  url = {http://dblp.uni-trier.de/db/conf/eccv/eccv2020-22.html#ZhouBP20},
  volume = 12367,
  year = 2020
}

@inproceedings{DBLP:conf/nips/VahdatKK21,
  author       = {Arash Vahdat and
                  Karsten Kreis and
                  Jan Kautz},
  editor       = {Marc'Aurelio Ranzato and
                  Alina Beygelzimer and
                  Yann N. Dauphin and
                  Percy Liang and
                  Jennifer Wortman Vaughan},
  title        = {Score-based Generative Modeling in Latent Space},
  booktitle    = {Advances in Neural Information Processing Systems 34: Annual Conference
                  on Neural Information Processing Systems 2021, NeurIPS 2021, December
                  6-14, 2021, virtual},
  pages        = {11287--11302},
  year         = {2021},
  url          = {https://proceedings.neurips.cc/paper/2021/hash/5dca4c6b9e244d24a30b4c45601d9720-Abstract.html},
  bibsource    = {dblp computer science bibliography, https://dblp.org}
}

@inproceedings{DBLP:conf/iclr/0005HSBH24,
  author       = {Haitao Yang and
                  Xiangru Huang and
                  Bo Sun and
                  Chandrajit L. Bajaj and
                  Qixing Huang},
  title        = {GenCorres: Consistent Shape Matching via Coupled Implicit-Explicit
                  Shape Generative Models},
  booktitle    = {The Twelfth International Conference on Learning Representations,
                  {ICLR} 2024, Vienna, Austria, May 7-11, 2024},
  publisher    = {OpenReview.net},
  year         = {2024},
  url          = {https://openreview.net/forum?id=dGH4kHFKFj},
  bibsource    = {dblp computer science bibliography, https://dblp.org}
}

@inproceedings{lu2025dposer,
  title={DPoser-X: Diffusion Model as Robust 3D Whole-body Human Pose Prior},
  author={Lu, Junzhe and Lin, Jing and Dou, Hongkun and Zeng, Ailing and Deng, Yue and Liu, Xian and Cai, Zhongang and Yang, Lei and Zhang, Yulun and Wang, Haoqian and others},
  booktitle={Proceedings of the IEEE/CVF International Conference on Computer Vision},
  pages={9988--9997},
  year={2025}
}

@inproceedings{ho2025phd,
  title={PHD: Personalized 3D Human Body Fitting with Point Diffusion},
  author={Ho, Hsuan-I and Guo, Chen and Wu, Po-Chen and Shugurov, Ivan and Tang, Chengcheng and Mittal, Abhay and An, Sizhe and Kaufmann, Manuel and Zhang, Linguang},
  booktitle={Proceedings of the IEEE/CVF International Conference on Computer Vision},
  pages={7526--7537},
  year={2025}
}

\newpage
\appendix
\section{Normalized ARAP Reg}

Our goal is to compute the derivative of 
$$
r^{\theta}(\bs{z}) = \textup{Tr}\big(E^{\theta}(z)^{-\frac{1}{2}}H^{\theta}(\bs{z})E^{\theta}(z)^{-\frac{1}{2}}\big) = \textup{Tr}\big(H^{\theta}(\bs{z})E^{\theta}(z)^{-1}\big)
$$
where
$$
E^{\theta}(\bs{z}) = {J^{\theta}(\bs{z})}^TJ^{\theta}(\bs{z}), \quad H^{\theta}(\bs{z}) = {J^{\theta}(\bs{z})}^T\overline{H}^{\theta}(\bs{z}) J^{\theta}(\bs{z})
$$
where
$$
J^{\theta}(\bs{z}) = \frac{\partial \bs{g}^{\theta}(\bs{z})}{\partial \bs{z}}.
$$
Using the chain rule, we have
\begin{align}
\frac{\partial r^{\theta}(\bs{z})}{\partial \theta} & = \textup{Tr}\big(E^{\theta}(z)^{-\frac{1}{2}}\frac{\partial H^{\theta}(\bs{z})}{\partial \theta}E^{\theta}(z)^{-\frac{1}{2}}\big) \nonumber \\
& - \textup{Tr}\big({H^{\theta}(\bs{z})}^{\frac{1}{2}}E^{\theta}(z)^{-1}\frac{\partial E^{\theta}(\bs{z})}{\partial \theta}E^{\theta}(z)^{-1}{H^{\theta}(\bs{z})}^{\frac{1}{2}}\big)
\end{align}
Let eigenvalue decomposition of $E^{\theta}$ and singular value decomposition of ${H^{\theta}(\bs{z})}^{\frac{1}{2}}E^{\theta}(z)^{-1}$ be
$$
E^{\theta}(\bs{z}) = {U^{\theta}(\bs{z})}\Lambda^{\theta}(\bs{z}){U^{\theta}(\bs{z})}^T
$$
and
$$
{H^{\theta}(\bs{z})}^{\frac{1}{2}}E^{\theta}(z)^{-1} = {V^{\theta}(\bs{z})}\Sigma^{\theta}(\bs{z}){W^{\theta}(\bs{z})}^T.
$$
We have
\begin{align*}
& \textup{Tr}\big(E^{\theta}(z)^{-\frac{1}{2}}\frac{\partial H^{\theta}(\bs{z})}{\partial \theta}E^{\theta}(z)^{-\frac{1}{2}}\big) \\
= & \frac{2}{\epsilon^2}\sum\limits_{i=1}^d \frac{1}{\lambda_i^{\theta}(\bs{z})}\big(\frac{\partial \bs{g}^{\theta}(\bs{z}+\epsilon\bs{u}_i^{\theta}(\bs{z}))}{\partial \theta}-\frac{\partial \bs{g}^{\theta}(\bs{z})}{\partial \theta}\big)^T \overline{H}^{\theta}(\bs{z})\\
& \qquad \qquad \big(\bs{g}^{\theta}(\bs{z}+\epsilon\bs{u}_i^{\theta}(\bs{z}))-\bs{g}^{\theta}(\bs{z})\big) \\
+ & \frac{1}{\epsilon^2}\sum\limits_{i=1}^d \frac{1}{\lambda_i^{\theta}(\bs{z})}\big(\bs{g}^{\theta}(\bs{z}+\epsilon\bs{u}_i^{\theta}(\bs{z}))-\bs{g}^{\theta}(\bs{z})\big)^T \frac{\partial \overline{H}^{\theta}(\bs{z})}{\partial \theta}\\
& \qquad \qquad \big(\bs{g}^{\theta}(\bs{z}+\epsilon\bs{u}_i^{\theta}(\bs{z}))-\bs{g}^{\theta}(\bs{z})\big)
\end{align*}
Similarly, we have
\begin{align*}
& \textup{Tr}\big({H^{\theta}(\bs{z})}^{\frac{1}{2}}E^{\theta}(z)^{-1}\frac{\partial E^{\theta}(\bs{z})}{\partial \theta}E^{\theta}(z)^{-1}{H^{\theta}(\bs{z})}^{\frac{1}{2}}\big) \\
= & \frac{2}{\epsilon^2}\sum\limits_{i=1}^d {\sigma_i^{\theta}(\bs{z})}^2\big(\frac{\partial \bs{g}^{\theta}(\bs{z}+\epsilon\bs{w}_i^{\theta}(\bs{z}))}{\partial \theta}-\frac{\partial \bs{g}^{\theta}(\bs{z})}{\partial \theta }\big)^T \\
& \qquad \qquad \big(\bs{g}^{\theta}(\bs{z}+\epsilon\bs{w}_i^{\theta}(\bs{z}))-\bs{g}^{\theta}(\bs{z})\big).
\end{align*}

\section{Validation of the Deformation Score}
\label{App:DeformationScoreValidation}

We further validate the deformation score $r_{\deform}$ used for geometry-to-distribution feedback in the main paper. For mesh generators, we examine its correlation with SMPL/SMAL fitting error and its sensitivity to the percentile parameter $p_c$. For implicit generators, we perform the corresponding analysis using the deformation score defined through GenCorres-based correspondence.

\begin{figure}[H]
\centering
\setlength{\tabcolsep}{3pt}
\begin{tabular}{cccc}
\includegraphics[height=0.245\textwidth]{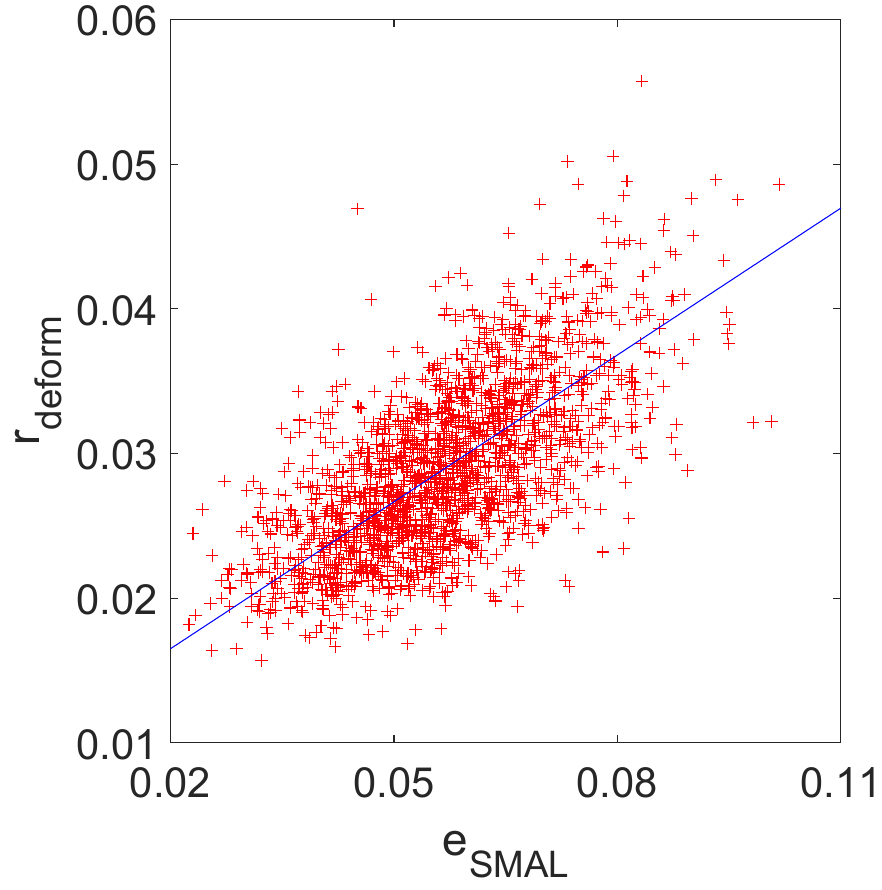} &
\includegraphics[height=0.24\textwidth]{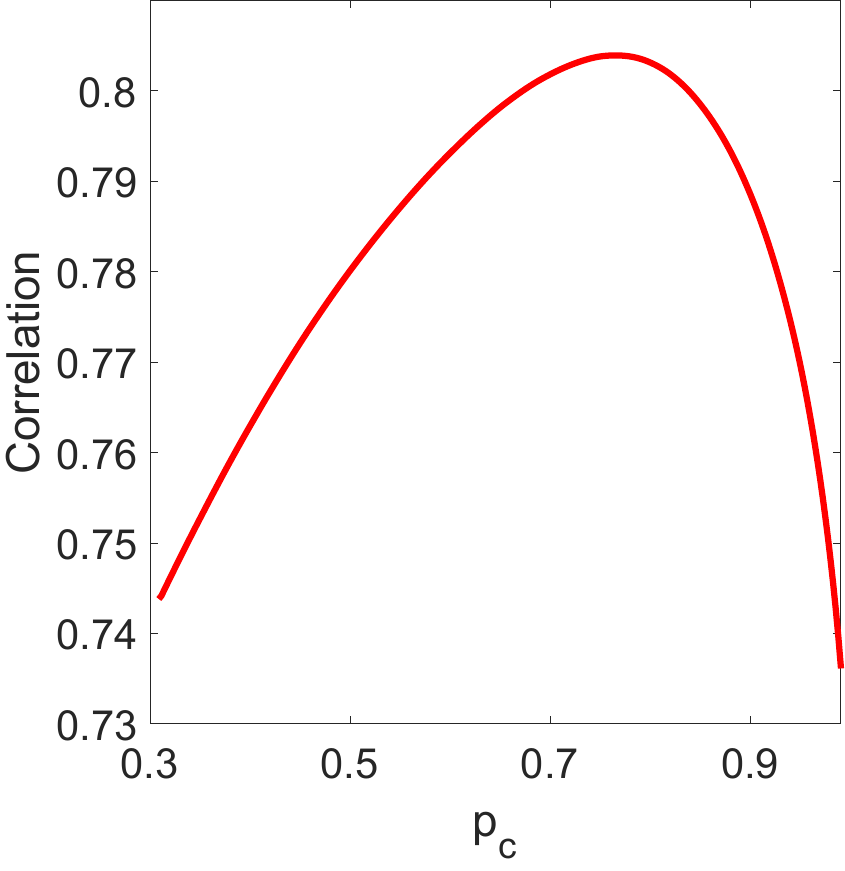} &
\includegraphics[height=0.245\textwidth]{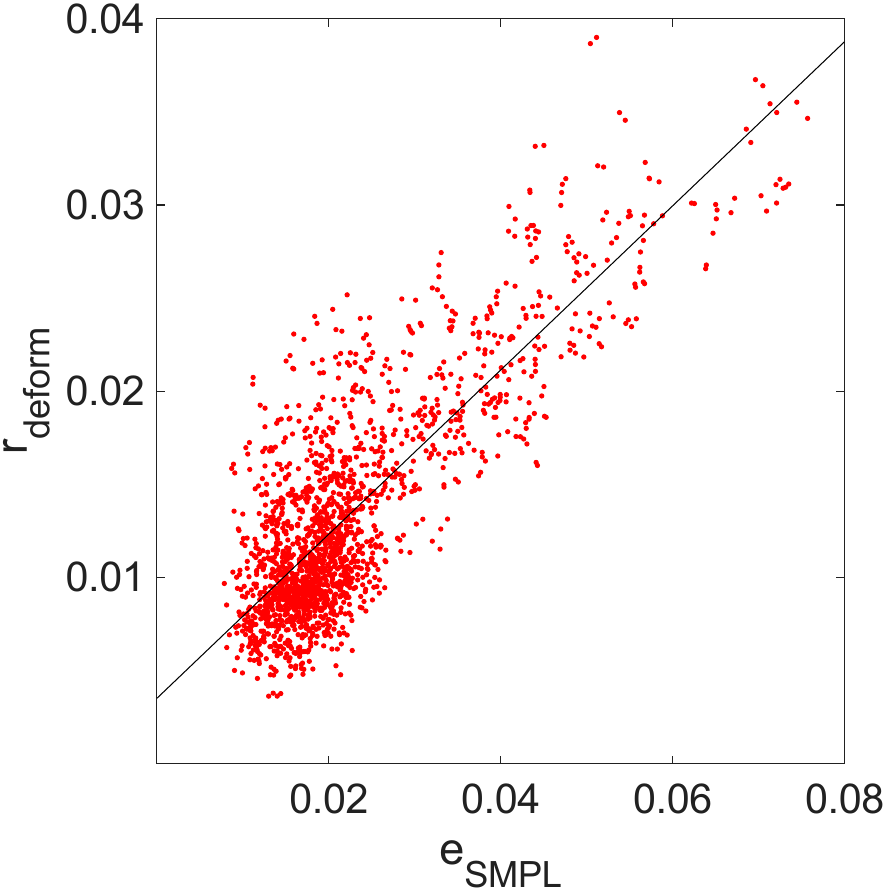} &
\includegraphics[height=0.24\textwidth]{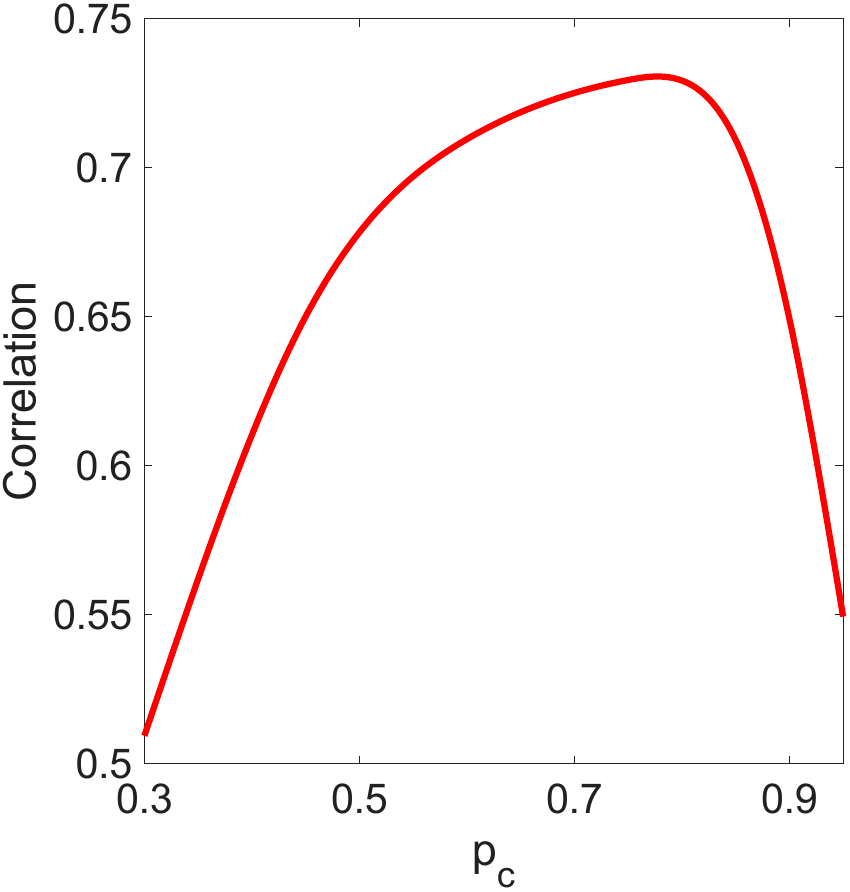} \\
(a) & (b) & (c) & (d)
\end{tabular}
\vspace{-0.1in}
\caption{Validation of the deformation score $r_{\deform}$. Panels (a,b) show the relationship between $r_{\deform}$ and SMPL/SMAL fitting error and its sensitivity to $p_c$ for mesh generators; panels (c,d) show the corresponding behavior for implicit generators. The correlation remains stable over a broad range of $p_c$.}
\label{Figure:Correlation}
\end{figure}

Across both mesh and implicit generators, $r_{\deform}$ remains positively associated with the corresponding SMPL/SMAL fitting error over a broad range of $p_c$, supporting its use as a family-relative geometric quality signal rather than relying on a narrowly tuned percentile choice.

\section{Additional Deformation-Score Variants}

For completeness, we describe additional deformation-score variants considered in our study. All results reported in the main paper use the deformation score defined in Sec.~\ref{Subsec:StageIII}; the variants below are not used to produce the main quantitative results. Deformation energy between two meshes of the same topology. Suppose the vertex vectors are $\bs{p}\in \R^{3n}$ and $\bs{q}\in \R^{3n}$ where the $i$-th vertices are given by $\bs{p}_i$ and $\bs{q}_i$. Let $\set{N}(i)$ collect 1-ring neighbors of the $i$-th vertex.
We measure the generalized deformation energy between them as 
\begin{align}
f(\bs{p},\bs{q}) = & \sum\limits_{i=1}^n \Big(\lambda^{\textup{rigid}}f_i^{\textup{rigid}}(\bs{p},\bs{q}) +\lambda^{\textup{sim}}f_i^{\textup{sim}}(\bs{p},\bs{q})\nonumber \\
& +\lambda^{\textup{aff}}f_i^{\textup{aff}}(\bs{p},\bs{q})\Big)
\end{align}
Here
$$
f_i^{\textup{rigid}}(\bs{p},\bs{q}) =\min\limits_{R\in \textup{SO}(3)} \sum\limits_{j\in \set{N}(i)}\|R(\bs{p}_i-\bs{p}_j)-(\bs{q}_i-\bs{q}_j)\|^2.
$$
Denote $P_i = (\bs{p}_i-\bs{p}_j) \in \R^{3\times |\set{N}(i)|}$ and $Q_i = (\bs{q}_i-\bs{q}_j)\in \R^{3\times |\set{N}(i)|}$. Then
\begin{align}
f_i^{\textup{rigid}}(\bs{p},\bs{q}) & =\min\limits_{R\in\textup{SO}(3)}\|RP_i-Q_i\|_{\set{F}}^2 \nonumber \\
& = \|P_i\|_{\set{F}}^2 + \|Q_i\|^2_{\set{F}}-2\max\limits_{R\in\textup{SO}(3)}\langle RP_i, Q_i \rangle \nonumber \\
& = \|P_i\|_{\set{F}}^2 + \|Q_i\|^2_{\set{F}}-2\max\limits_{R\in\textup{SO}(3)}\textup{Trace}(RP_iQ_i^T).
\end{align}
Denote $S = P_iQ_i^T$. Let its singular-value decomposition (SVD) be $S = U\Sigma V^T$ where the singular values in the diagonal of $\Sigma$ are ordered in the decreasing manner. The optimal $R$ that maximizes $\textup{Trace}(RS)$ is given by $R = V\diag(1,1,\textup{sign}(\textup{det}(S)))U^T$. Note that $\textup{sign}(\textup{det}(S))$ enforces that $R\in \textup{SO}(3)$.

Similarly, 
$$
f_i^{\textup{sim}}(\bs{p},\bs{q}) =\min\limits_{s>0, R\in \textup{SO}(3)} \|sRP_i-Q_i\|_{\set{F}}^2.
$$
It is easy to see that the optimal rotation can be derived from SVD of $S = P_iQ_i^T$ in the same way as above. Given the optimal $R^{\star}$, the optimal scale is 
$$
s^{\star} = \frac{\langle R^{\star}P_i, Q_i\rangle}{\|P_i\|_{\set{F}}^2}.
$$

Finally, 
$$
f_i^{\textup{aff}}(\bs{p},\bs{q}) =\min\limits_{A} \|AP_i-Q_i\|_{\set{F}}^2
$$
and the optimal 
$$
A^{\star} = (Q_iP_i^T)(P_iP_i^T)^{-1}.
$$

\begin{itemize}
\item For each synthetic shape, obtain top-$K$ closest training shapes using $L^2$-distance between vertex positions. 
\item For each closest shape, calculate the deformation energy using $f(\bs{p},\bs{q})$ with hyper-parameters $\lambda^{\textup{sim}}$ and $\lambda^{\textup{aff}}$.
\item Sort the distances and use the $k$-th smallest value as the score.
\item There are three critical hyper-parameters above, i.e., $\lambda^{\textup{sim}}$, $\lambda^{\textup{aff}}$, and $k$. $K$ is less sensitive. 
\end{itemize}

\section{Details in Preliminaries}

\subsection{Details of ARAPReg}
\label{App:ARAPReg:Details}

ARAPReg~\cite{Huang_2021_ICCV} introduces as-rigid-as-possible (ARAP) deformation geometry into a mesh generator $\bs{m}^{\phi}:\set{Z}\rightarrow\R^{3n}$ with fixed connectivity. Its objective combines data fitting with geometric regularization at latent codes sampled from a fixed Gaussian:
\begin{equation}
\min_{\phi} l_{\data}(\set{S},\bs{m}^{\phi})+
\lambda
\underset{\bs{z}\sim\mathcal{N}_d(\bs{0},I_d)}{\mathbb{E}}
l_{\reg}(\bs{m}^{\phi}(\bs{z})).
\end{equation}
The corresponding latent-space deformation metric is obtained by pulling back the mesh-space ARAP metric,
\[
H(\bs{m}^{\phi}(\bs{z})) =
{\frac{\partial \bs{m}^{\phi}(\bs{z})}{\partial \bs{z}}}^{T}
\overline{H}(\bs{m}^{\phi}(\bs{z}))
\frac{\partial \bs{m}^{\phi}(\bs{z})}{\partial \bs{z}}.
\]
We next describe the construction of $\overline{H}$.

The sparse matrix $\overline{H}(\bs{m}^{\phi}(\bs{z}))$ is based on an ARAP deformation model that measures the deformation between a mesh whose vertices are given by $\bs{p}\in \R^{3n}$ and the same mesh topology whose vertices are given by $\bs{p}+\bs{d}\in \R^{3n}$. Let $\set{N}(i)$ collect the adjacent vertices of $i$. The deformation model is given by
$$
e(\bs{p},\bs{d}) = \sum\limits_{i=1}^{n}\min\limits_{\bs{c}_i}\sum\limits_{j\in \set{N}(i)}\|\bs{c}_i\times (\bs{p}_i-\bs{p}_j)-(\bs{d}_i-\bs{d}_j)\|^2.
$$
Introduce $L\in \R^{n\times n}$, $B(\bs{p})\in \R^{3n\times 3n}$, and $C(\bs{p})\in \R^{3n\times 3n}$:
\begin{align*}
L_{ij} = \left\{
\begin{array}{cc}
|\set{N}(i)| & i = j \\
-1 & j\in \set{N}(i) \\
0 & \textup{otherwise}
\end{array}
\right.\, \quad B_{ij} = \left\{
\begin{array}{cc}
\sum\limits_{j\in \set{N}(i)}(\bs{p}_i-\bs{p}_j)\times & i = j \\
-(\bs{p}_i-\bs{p}_j)\times& j\in \set{N}(i)\\
0 & \textup{otherwise}
\end{array}
\right.\, \\ C_{ij} = \left\{
\begin{array}{cc}
\sum\limits_{j\in \set{N}(i)}((\bs{p}_i-\bs{p}_j)\times )((\bs{p}_i-\bs{p}_j)\times )^T& i = j \\
0& \textup{otherwise}
\end{array}
\right.\
\end{align*}

Let $\bs{c}\in \R^{3n}$ collect $\bs{c}_i$ in its corresponding blocks. Since the neighborhood relation is symmetric, the directed sum over $i$ and $j\in\set{N}(i)$ counts each mesh edge twice. With the definitions above, the deformation energy can therefore be written as
\begin{align*}
e(\bs{p},\bs{d})
=
\min_{\bs{c}}
\Big(
2\bs{d}^T(L\otimes I_3)\bs{d}
-
2\bs{d}^T B(\bs{p})^T\bs{c}
+
\bs{c}^T C(\bs{p})\bs{c}
\Big)\\
&=
\bs{d}^T
\Big(
2L\otimes I_3
-
B(\bs{p})^T C(\bs{p})^{-1}B(\bs{p})
\Big)
\bs{d}.
\end{align*}
Hence, the mesh-space ARAP metric is
\[
\overline{H}(\bs{m}^{\phi}(\bs{z}))
=
2L\otimes I_3
-
B(\bs{m}^{\phi}(\bs{z}))^T
C(\bs{m}^{\phi}(\bs{z}))^{-1}
B(\bs{m}^{\phi}(\bs{z})).
\]

By construction, $\overline{H}$ is positive semidefinite and vanishes on infinitesimal rigid motions, consistent with the original squared-residual ARAP energy.

\subsection{Details of GenCorres}
\label{App:GenCorres:Details}

Consider a point $\bs{m}_i^{\phi}$ on the implicit surface $g^{\phi}(\bs{x},\bs{z}) = 0$. Let $\epsilon \bs{d}$ be its perturbation so that $\bs{m}_i^{\phi}+\epsilon \bs{d}$ lies on $g^{\phi}(\bs{x},\bs{z}+\epsilon\bs{v}) = 0$. Applying the first-order Taylor expansion of $g^{\phi}(\bs{x},\bs{z})$, we have
\begin{equation}
\frac{\partial g^{\phi}}{\partial \bs{x}}(\bs{m}_i^{\phi}(\bs{z}),\bs{z})\cdot \bs{d} + \frac{\partial g^{\phi}}{\partial \bs{z}}(\bs{m}_i^{\phi}(\bs{z}),\bs{z})\bs{v} = 0.
\label{Eq:Imp:Linear:Cons}
\end{equation}
Stacking Eq.~(\ref{Eq:Imp:Linear:Cons}) over all mesh vertices yields the linear system used by GenCorres to compute the local deformation induced by the latent displacement $\bs{v}$.

\section{Details of Experiments}
\label{Sec:Experimental:Details}

\subsection{Hyper Parameters of ARAPDiffusion}

The initialization step follows the standard protocol of ARAPReg and GenCorres to train the AE, which go through 200 epochs, where each epoch goes through a training shape once. The latent diffusion model is trained with 200 epochs as well.

We perform 10 alternations between generation-aware geometric refinement of the encoder--decoder and geometry-aware diffusion refinement. See Fig.~\ref{fig:overall_convergence} for illustration of convergence. In each alternation, we train the encoder--decoder for 50 epochs and fine-tune the diffusion model for 100 epochs.

When training the AE with regularization, each epoch picks a subset of 12800 samples for training. When training the latent diffusion model with regularization,  we use $4N$ samples $p_{\sync}$ to train $\D^{\theta}$ in each epoch.

\subsection{Parametric Fitting of SMPL and SMAL Models}

For simplicity, we only describe the procedure of fitting the SMPL model described as $\bs{p}(\theta, \beta)$ to a 3D shape, where $\theta$ and $\beta$ are the pose and shape parameters, respectively. The same fitting procedure, including initialization and parameter regularization, is used for all compared methods; we therefore use $e_t$ as a common fitting-based comparison metric rather than an exact projection distance to the parametric shape manifold.

\subsubsection{Mesh Fitting}

Suppose $S$ has the same mesh as the SMPL model where mesh vertices are given by $\bs{p}^{0}$. For reconstruction, our goal is to solve for $\theta$ and $\beta$ to optimize
\begin{equation}
\min_{\theta,\beta}
\|p^0-p(\theta,\beta)\|_2^2
+
\lambda
\left(
\|\theta-\theta_m\|_2^2
+
\|\beta-\beta_m\|_2^2
\right)
\end{equation}
where $\theta_m$ and $\beta_m$ are the codes of the closest training shape to $p^0$. The value of $\lambda$ is set so that the fitting term and the regularization term have the same values. Optimization employs the Gauss-Newton method, starting from $\theta_m$ and $\beta_m$. 

\subsubsection{Implicit Surface Fitting}

For an implicit surface defined by $g(\bs{p})$, we solve for 
\begin{equation}
\min_{\theta,\beta}
\sum_{i=1}^{n} g^2(p_i(\theta,\beta))
+
\lambda
\left(
\|\theta-\theta_m\|_2^2
+
\|\beta-\beta_m\|_2^2
\right)
\end{equation}
where $\theta_{m}$ and $\beta_m$ are the codes of the closest training shape to the mesh discretization of $g(\bs{p})$ with respect to the Chamfer distance. We again set value of $\lambda$ so that the fitting term and the regularization term have the same values. Optimization employs the Gauss-Newton method, starting from $\theta_m$ and $\beta_m$. 

\subsection{Running Time}

All experiments use a single NVIDIA H200 GPU (140GB). For explicit SMAL/SMPL models, latent-diffusion initialization takes approximately 5/10 hours and alternating feedback takes 3/6 hours, respectively. For implicit models, the corresponding times are 20/50 hours and 10/20 hours.

\section{Convergence of Alternating Feedback Optimization}

We examine convergence using the Wasserstein distance $d_W$ and SMPL/SMAL fitting error $e_t$.

\begin{figure*}[t]
    \centering
    % 子图 (a)
    \begin{subfigure}{0.48\textwidth}
        \centering
        \includegraphics[width=0.49\linewidth]{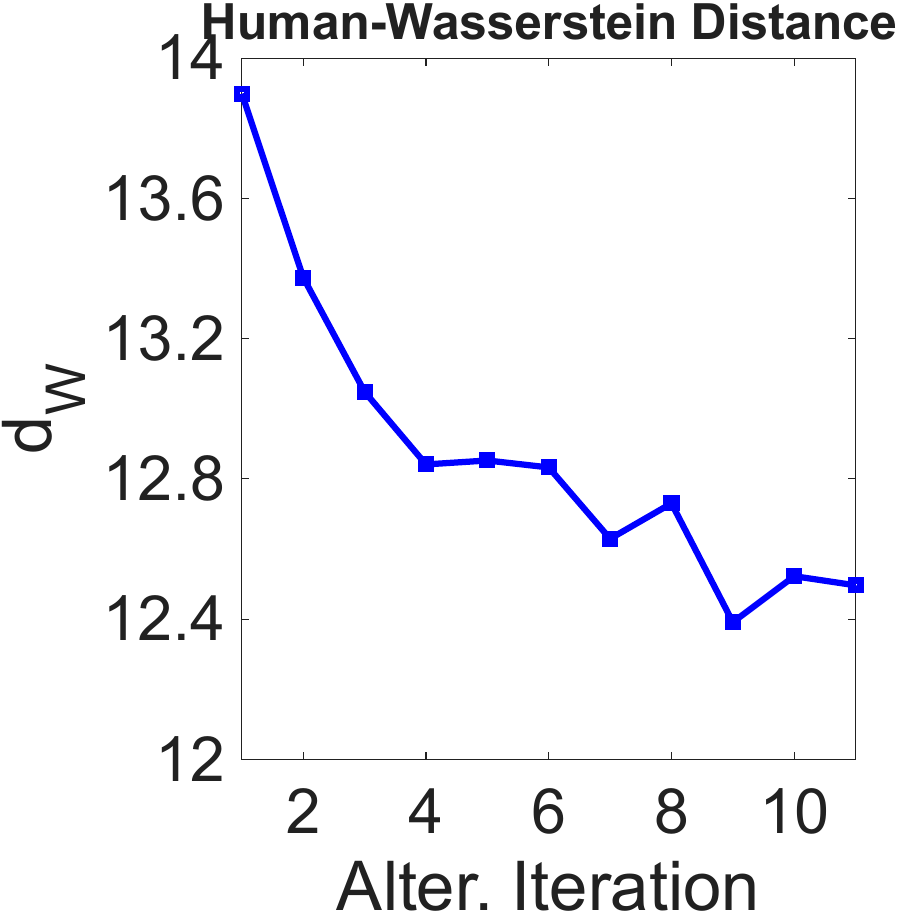}
        \includegraphics[width=0.49\linewidth]{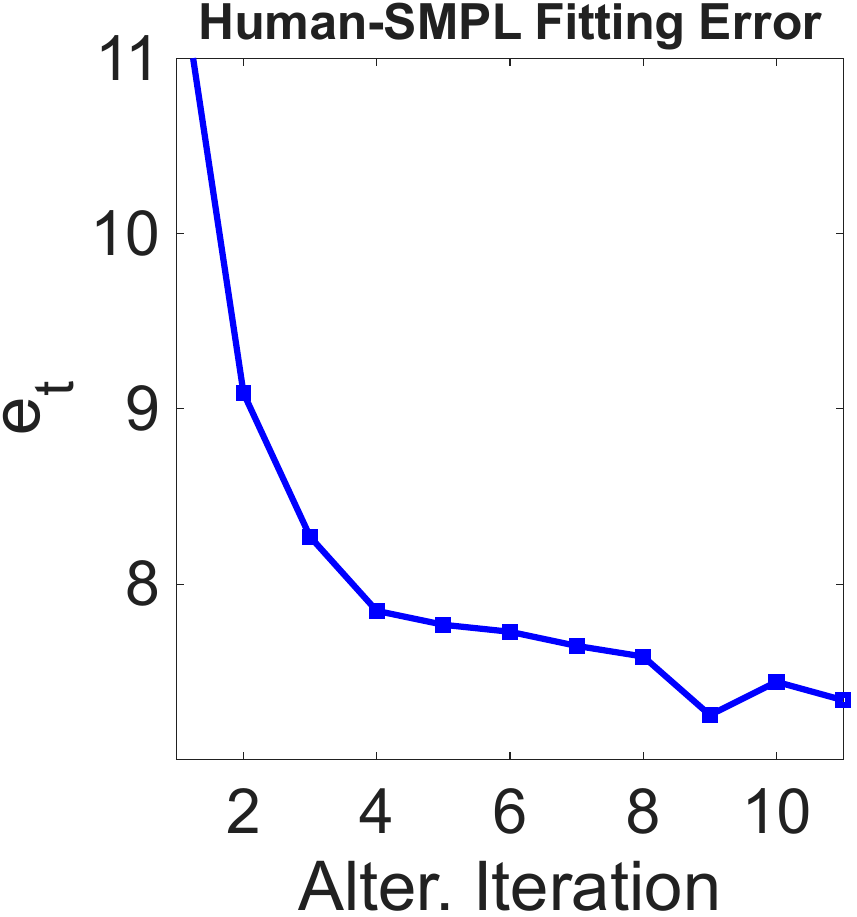}
        \caption{Human (Explicit Mesh)}
        \label{Figure:Human:Mesh}
    \end{subfigure}
    \hfill
    % 子图 (b)
    \begin{subfigure}{0.48\textwidth}
        \centering
        \includegraphics[width=0.49\linewidth]{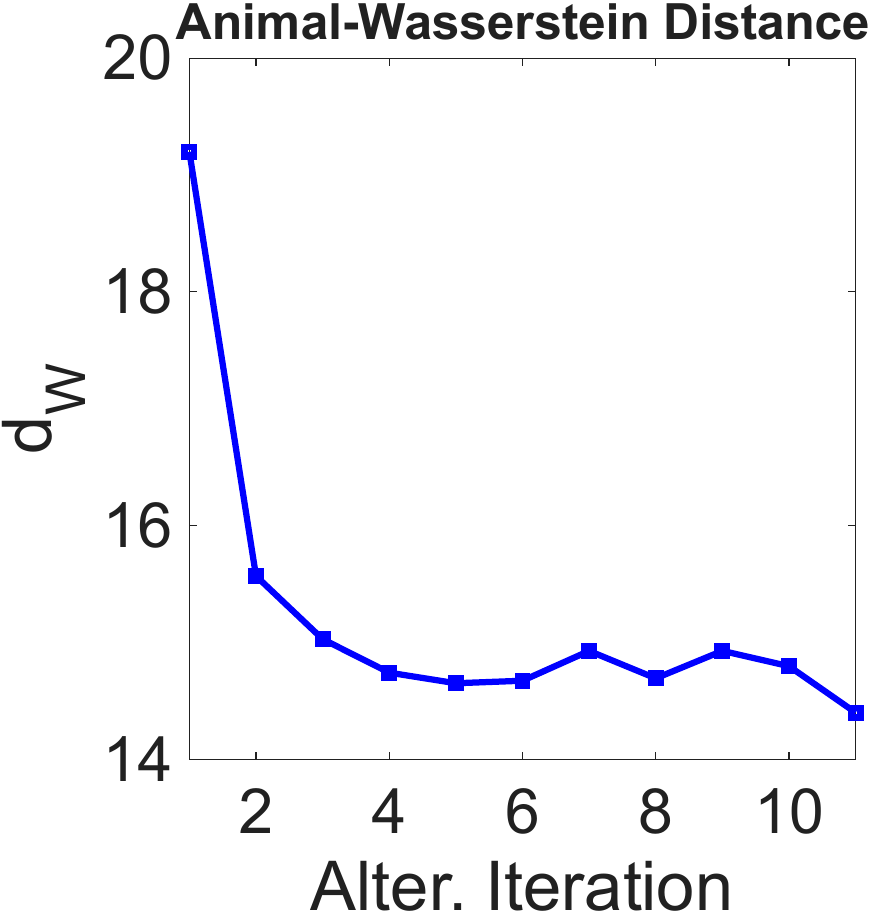}
        \includegraphics[width=0.49\linewidth]{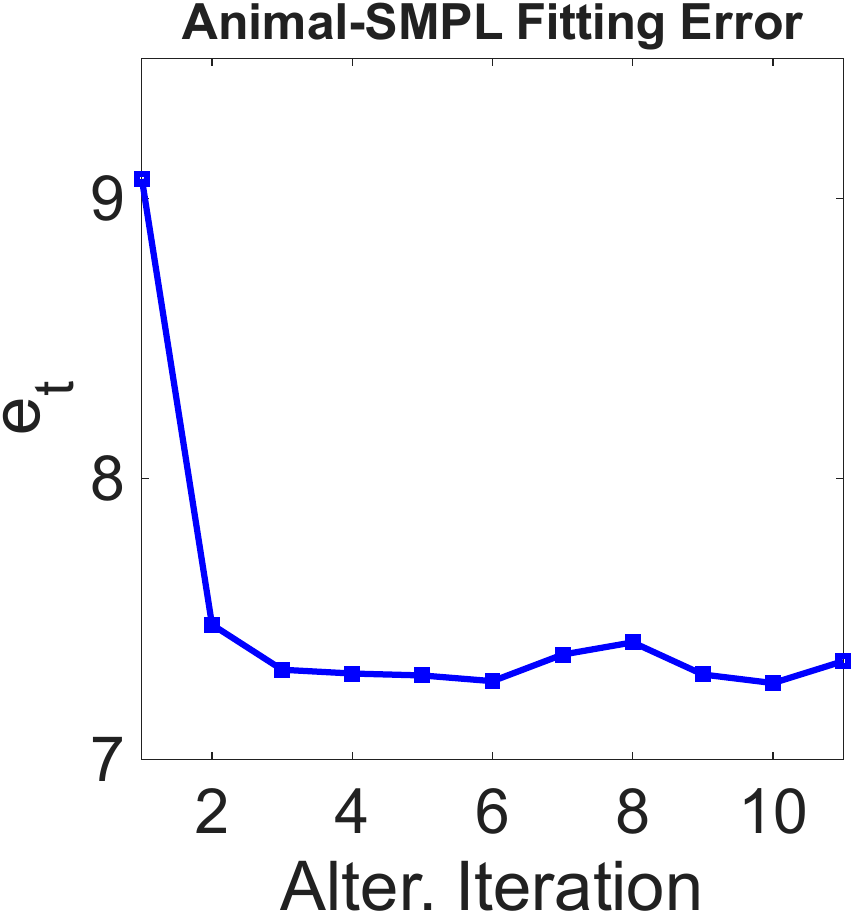}
        \caption{Animal (Explicit Mesh)}
        \label{Figure:Animal:Mesh2}
    \end{subfigure}

    \vspace{1em} % 两行之间的间距

    % 子图 (c)
    \begin{subfigure}{0.48\textwidth}
        \centering
        \includegraphics[width=0.49\linewidth]{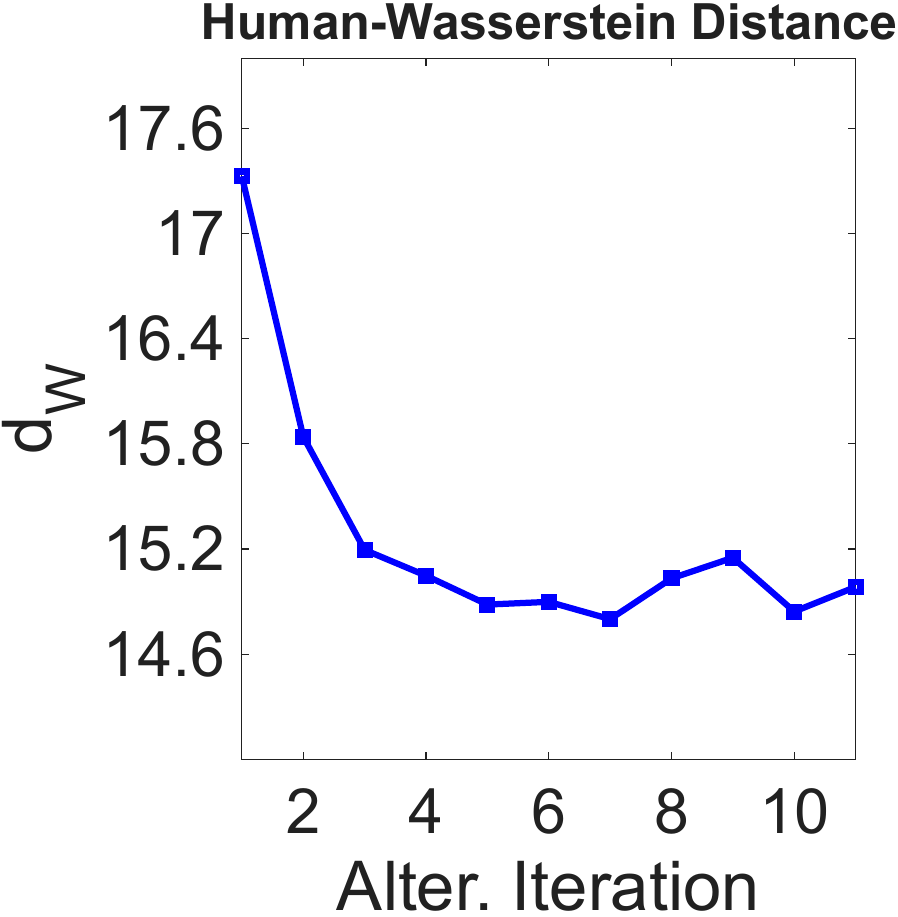}
        \includegraphics[width=0.49\linewidth]{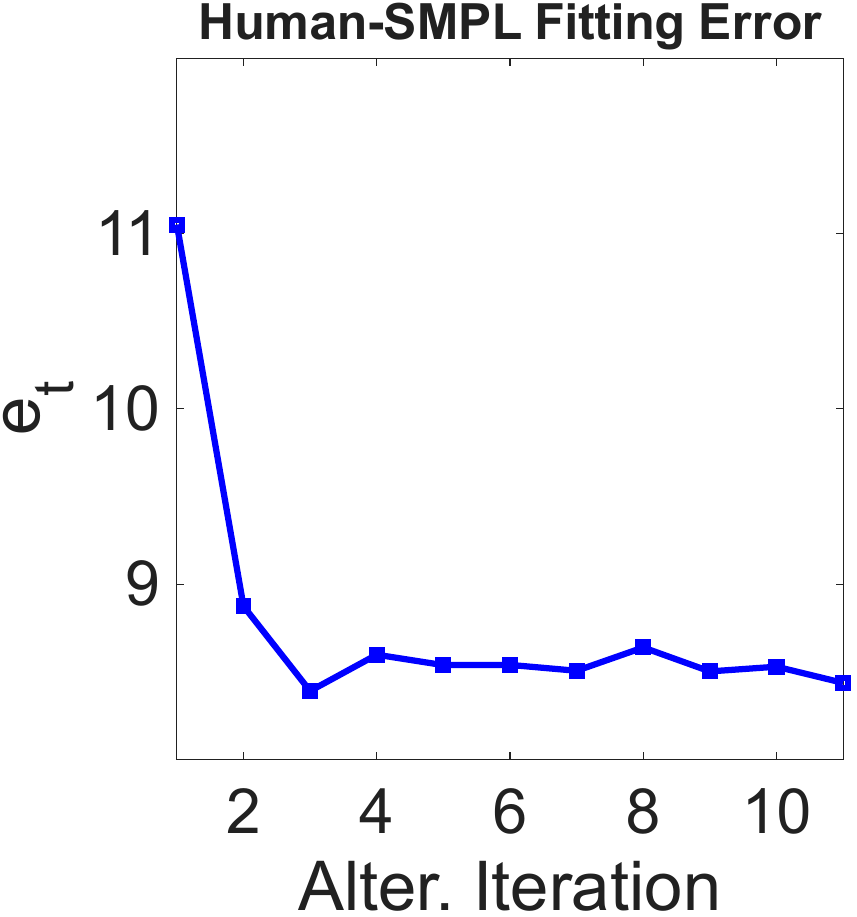}
        \caption{Human (Implicit Surface)}
        \label{Figure:Animal:Mesh3}
    \end{subfigure}
    \hfill
    % 子图 (d)
    \begin{subfigure}{0.48\textwidth}
        \centering
        \includegraphics[width=0.49\linewidth]{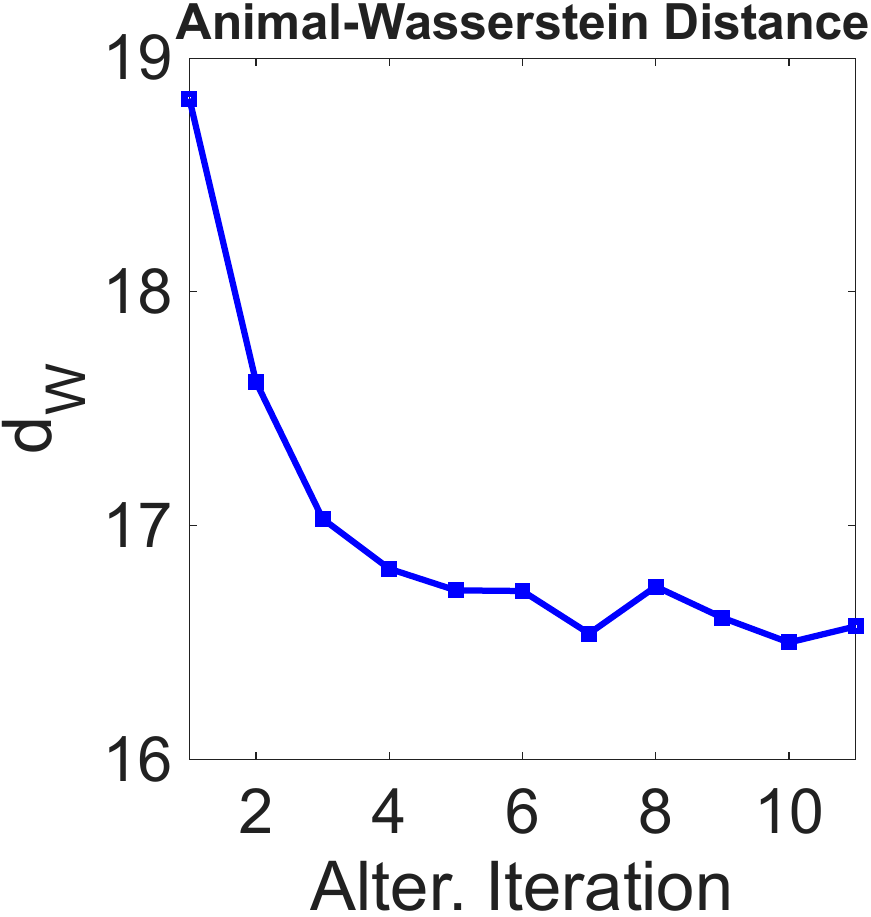}
        \includegraphics[width=0.49\linewidth]{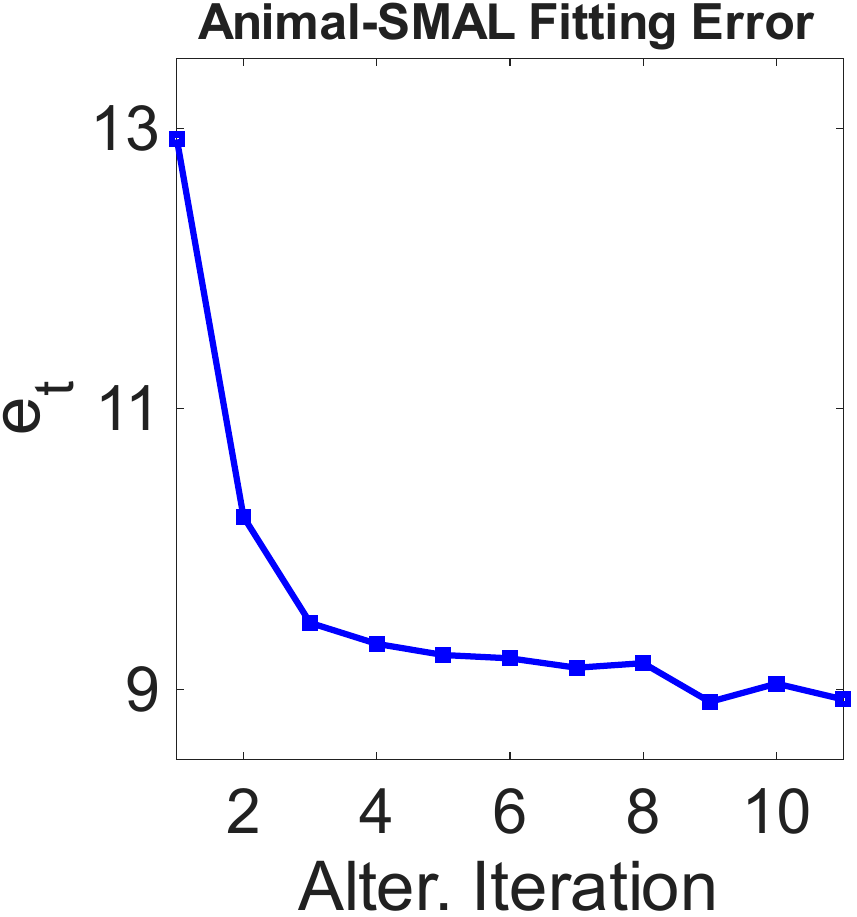}
        \caption{Animal (Implicit Surface)}
        \label{Figure:Animal:Mesh4}
    \end{subfigure}
    
    \caption{Convergence of alternating geometry--distribution feedback under different settings: (a) and (b) show the explicit mesh generator for Human and Animal; (c) and (d) show the implicit surface generator for Human and Animal.}
    \label{fig:overall_convergence}
\end{figure*}

Across all four settings in Fig.~\ref{fig:overall_convergence}, $d_W$ and $e_t$ improve rapidly in the first two alternations and then stabilize within a narrow range.

\section{Additional Experiments}

\subsection{Performance on Bone Datasets}

\begin{wraptable}{r}{0.54\textwidth}
\vspace{-0.8in}
\setlength\tabcolsep{1.5pt}
\centering
\begin{tabular}{c|c|c|c|c}
Bone & $e_r$ & $e_t$ & $d_W$ & $d_n$ \\ \hline
SP-disent.& 5.34 & 8.01 & 12.9 & 6.86 \\ \hline
COMA& 4.14 & 7.09& 12.1 & 6.78\\ \hline
3DMM& 4.03 & 7.12& 11.9 & 6.72\\ \hline
MeshConv& 4.47 & 7.34& 12.1 & 6.69\\ \hline
ARAPReg& 3.76 & 6.77& 11.6 & 6.76\\ \hline
FrameAVE& 3.92 & 6.81& 11.9 & 6.91\\ \hline
GeoLatent& 3.67 & 6.37& 11.5 & 6.83 \\ \hline
BRESA& 3.61 & 6.45& 11.7& 6.91\\ \hline
ARAPDiffusion & \textbf{3.25} & \textbf{5.92} & \textbf{10.4} & 6.79 \\ \hline
\end{tabular}
\end{wraptable} 

We additionally evaluate on the Bone dataset introduced in ARAPReg, which contains 40 shapes per category. ARAPDiffusion achieves the lowest $e_r$, $e_t$, and $d_W$ among the compared methods, while $d_n$ remains comparable to the baselines.

\subsection{Performance When Scaling the Data}

The quantitative results on Human are shown below, comparing ARAPDiffusion with the standard latent-diffusion baseline (NoReg) as the amount of training data varies:

\begin{wraptable}{r}{0.54\textwidth}
\vspace{-0in}
\setlength\tabcolsep{1.5pt}
\centering
\begin{tabular}{c|c|c|c}
Training data (\%) & Method & $e_t$ & $d_W$ \\\hline
50  & Standard LD (No Reg) & 14.20 & 19.50\\\hline
50 & ARAPDiffusion &  10.85 & 15.20\\\hline
100 & Standard LD (No Reg)& 10.60& 15.10\\\hline 
100 & ARAPDiffusion & 7.21 & 12.40 \\\hline
200 & Standard LD (No Reg)& 6.95& 12.10\\\hline 
200 & ARAPDiffusion & 5.8 & 10.20 \\\hline
\end{tabular}
\end{wraptable} 

Note: 100\% data corresponds to the original setting in Table 1 of our main paper. These results suggest that the geometric refinement improves data efficiency: ARAPDiffusion with reduced training data approaches the performance of the unregularized model trained with substantially more data. Furthermore, the refinement remains beneficial as the training set grows, with ARAPDiffusion continuing to outperform the unregularized baseline at 200\% data.

\section{Samples of Random Generated Shapes}

\subsection{Human}

Fig.~\ref{Figure:SMPL:Mesh:ARAPReg}, Fig.~\ref{Figure:SMPL:Mesh:FrameAVE}, Fig.~\ref{Figure:SMPL:Mesh:Geolatent}, and Fig.~\ref{Figure:SMPL:Mesh:BRESA} show randomly generated shapes of the baseline approaches ARAPReg, FrameAVE, Geolatent, and BRESA, respectively. All of them share the same issue that a subset of the synthetic shapes have low-quality overall shapes. Moreover, their synthetic shapes show a variety of noticeable distortions. 

In contrast, the random samples in Fig.~\ref{Figure:SMPL:Mesh:ARAPDiff} show substantially fewer visible distortions and more consistently plausible geometry. The distortion of each generate shape is also visually smaller. 

Fig.~\ref{Figure:SMPL:Mesh:ARAPDiff:Closest:Training} shows for each generated shape (on the left) its closest training shape (on the right). We can see that the generated shapes do not replicate the training shapes.

\begin{figure*}
\setlength\tabcolsep{2.5pt}
\begin{tabular}{cccccccc}
\includegraphics[height=0.19\textwidth]{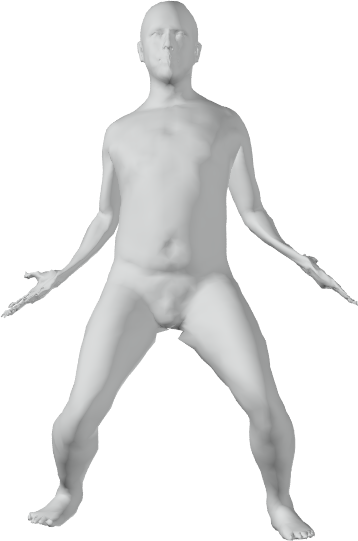}&
\includegraphics[height=0.19\textwidth]{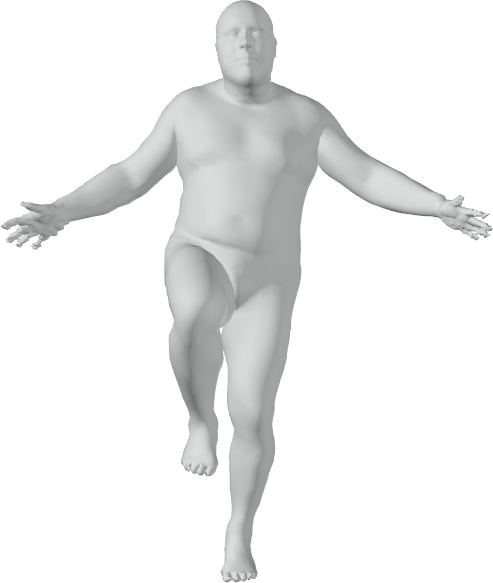}&
\includegraphics[height=0.19\textwidth]{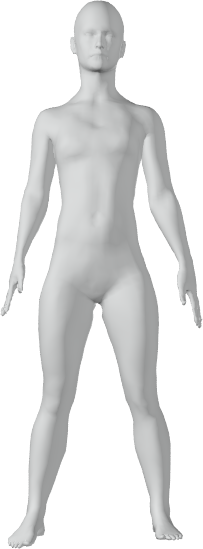}&
\bluepart{\includegraphics[height=0.19\textwidth]{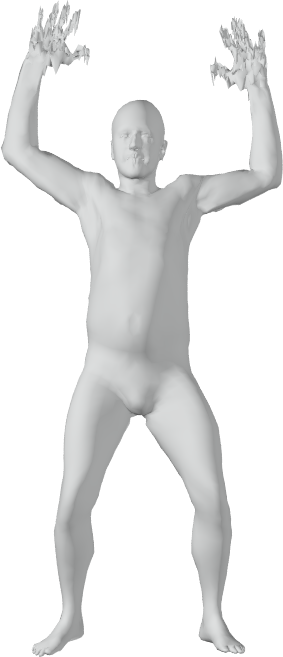}}{0.0}{1}{1}{0.85}&
\includegraphics[height=0.19\textwidth]{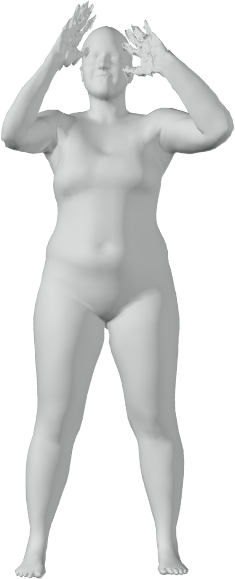}&
\includegraphics[height=0.19\textwidth]{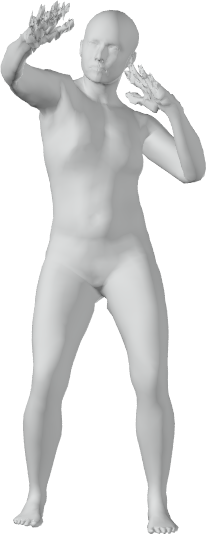}&
\includegraphics[height=0.19\textwidth]{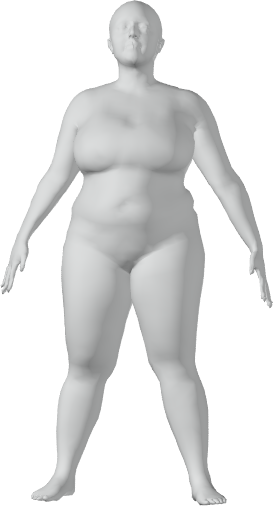}&
\includegraphics[height=0.19\textwidth]{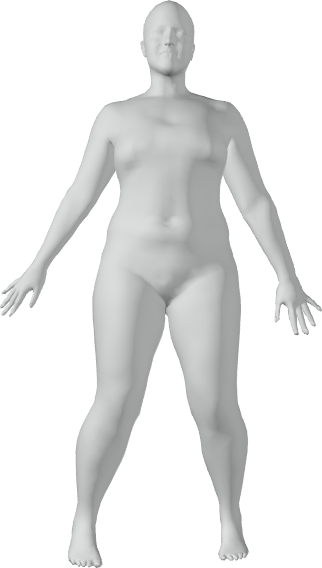}\\
% Row 2
% R2C1 Correction: Moved box RIGHT to catch the hand (0.15 - 0.35)
\includegraphics[height=0.19\textwidth]{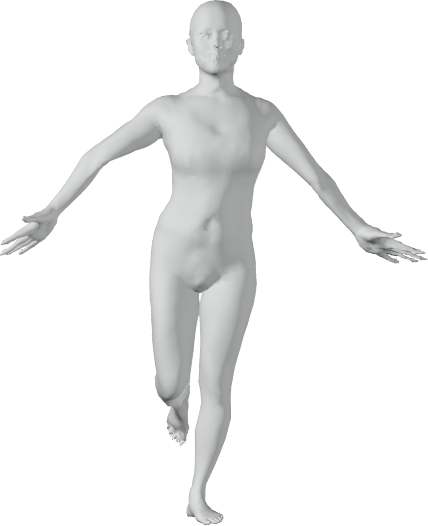}& 
\bluepart{\includegraphics[height=0.19\textwidth]{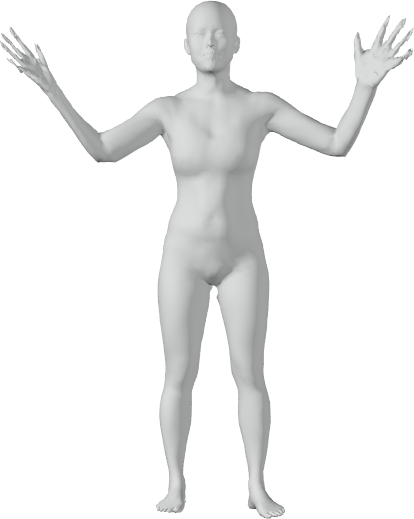}}{0.0}{0.97}{0.2}{0.75}&
\includegraphics[height=0.19\textwidth]{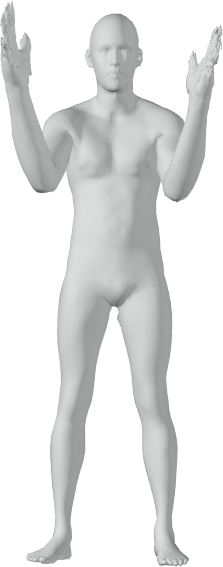}& 
\includegraphics[height=0.19\textwidth]{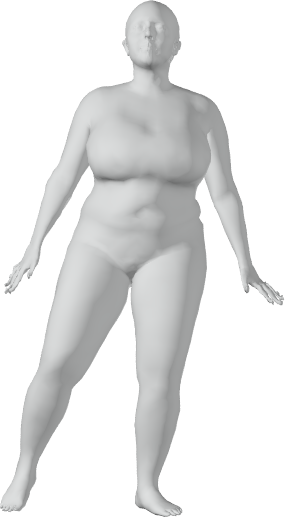}&
\includegraphics[height=0.19\textwidth]{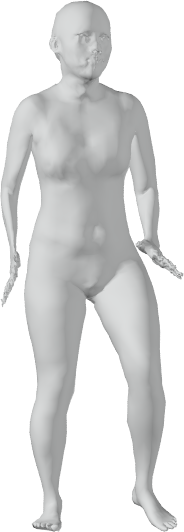}&
\includegraphics[height=0.19\textwidth]{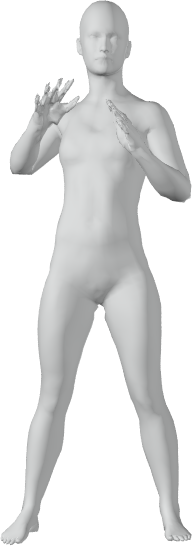}&
\redbox{\includegraphics[height=0.19\textwidth]{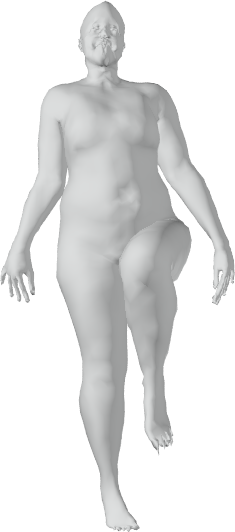}}&
\includegraphics[height=0.19\textwidth]{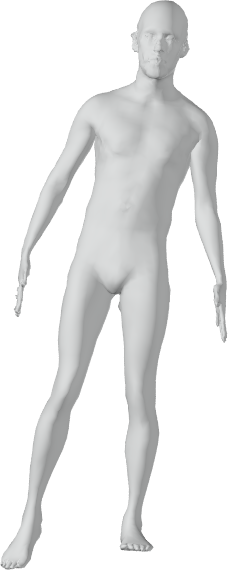}\\
% Row 3
\includegraphics[height=0.19\textwidth]{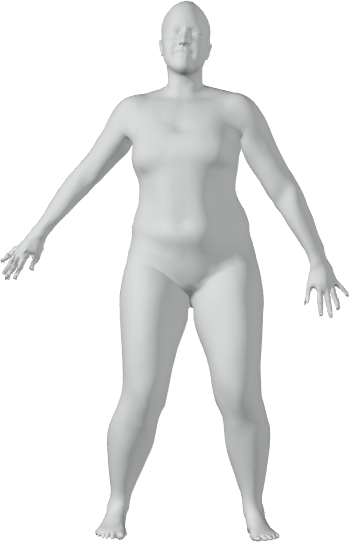}&
\includegraphics[height=0.19\textwidth]{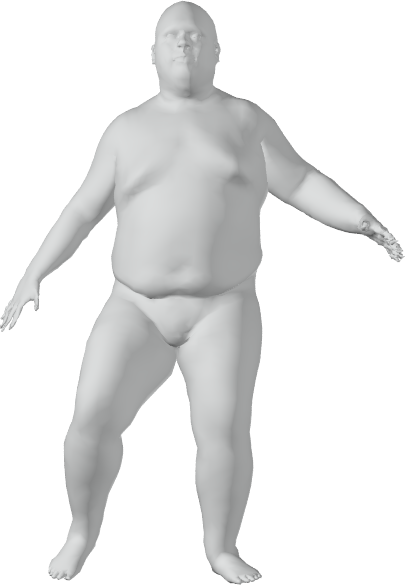}&
\includegraphics[height=0.19\textwidth]{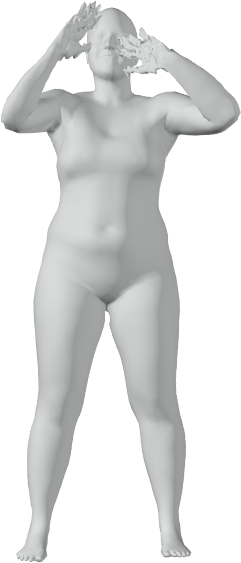}&
\includegraphics[height=0.19\textwidth]{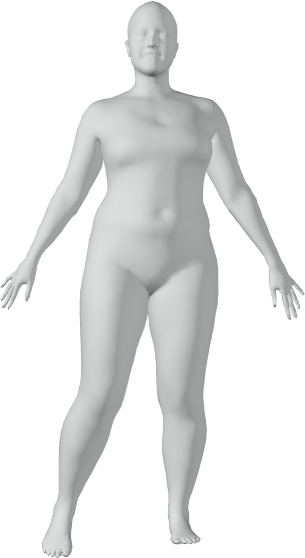}& 
\bluepart{\includegraphics[height=0.19\textwidth]{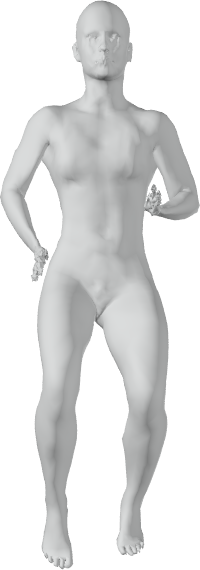}}{0.55}{0.55}{0.95}{0.75}&
% R3C6 Correction: Moved box UP and RIGHT to catch hip intersection (0.60 - 0.78)
\includegraphics[height=0.19\textwidth]{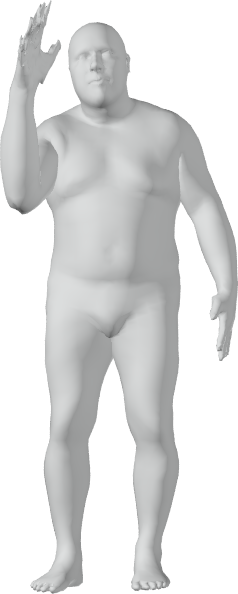}& 
\includegraphics[height=0.19\textwidth]{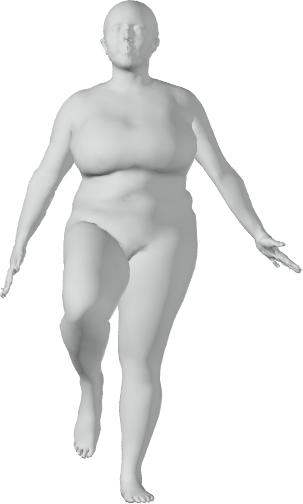}&
\includegraphics[height=0.19\textwidth]{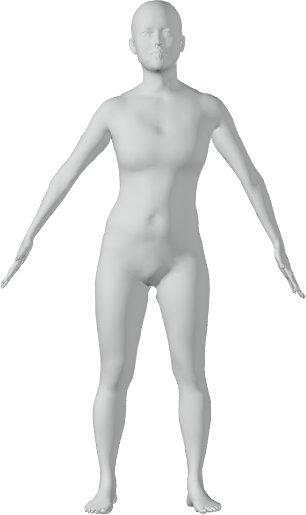}
\end{tabular}
\caption{ARAPReg results. Red boxes indicate global failures; blue boxes indicate local artifacts.}
\label{Figure:SMPL:Mesh:ARAPReg}    
\end{figure*}

\begin{figure*}
\setlength\tabcolsep{1.5pt}
\begin{tabular}{cccccccc}
\redbox{\includegraphics[height=0.19\textwidth]{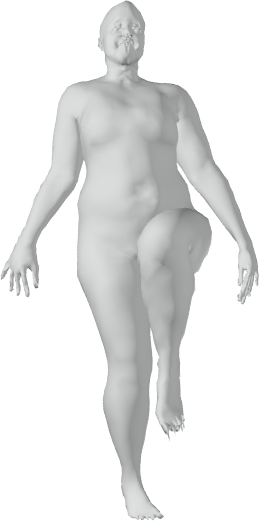}}&
\includegraphics[height=0.19\textwidth]{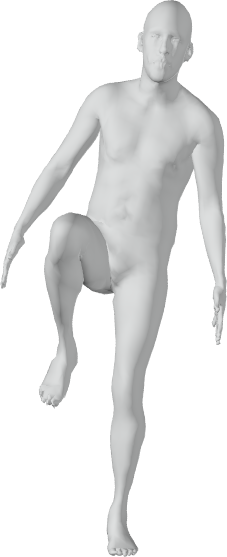}&
\bluepart{\includegraphics[height=0.19\textwidth]{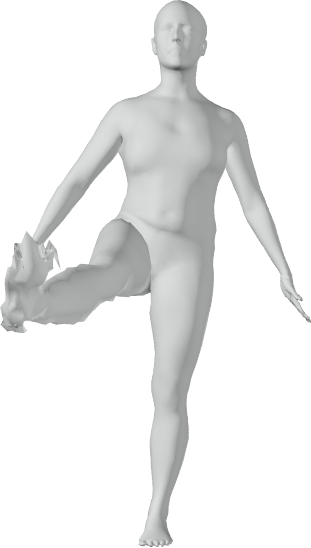}}{-0.05}{0.35}{0.2}{0.6}&
\redbox{\includegraphics[height=0.19\textwidth]{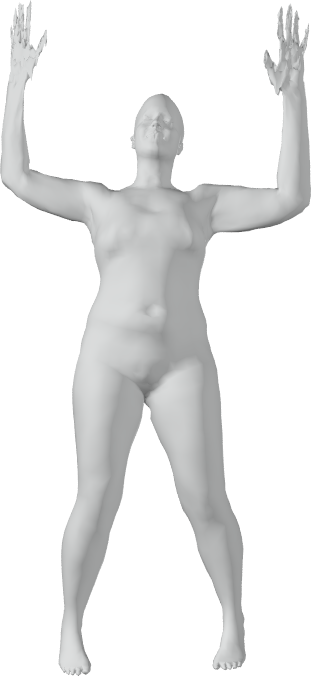}}&
\includegraphics[height=0.19\textwidth]{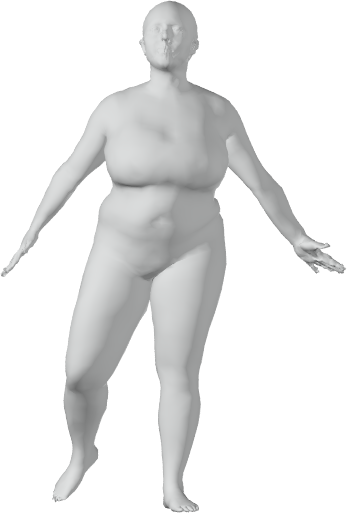}&
\bluepart{\includegraphics[height=0.19\textwidth]{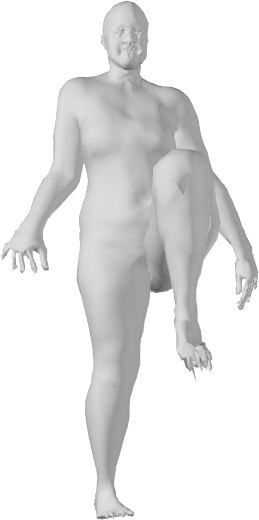}}{0.6}{0.25}{0.85}{0.4}&
\includegraphics[height=0.19\textwidth]{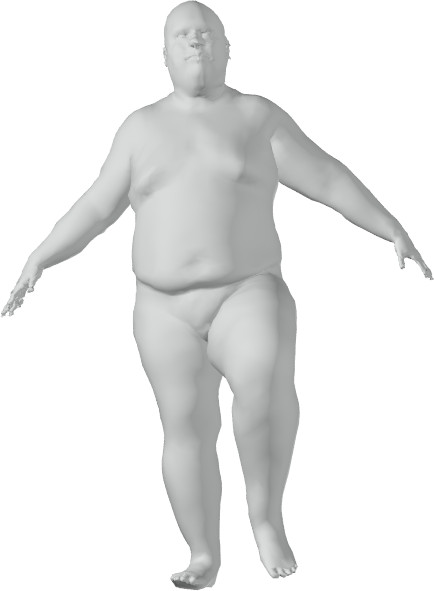}&
\includegraphics[height=0.19\textwidth]{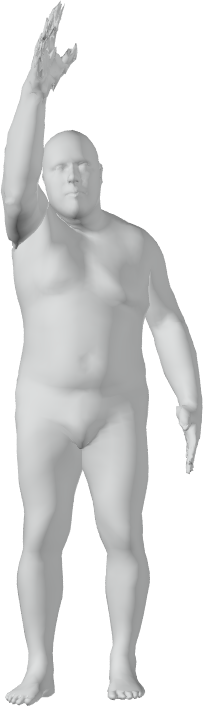}\\
\includegraphics[height=0.19\textwidth]{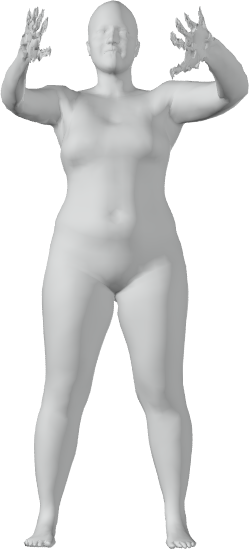}&
\bluepart{\includegraphics[height=0.19\textwidth]{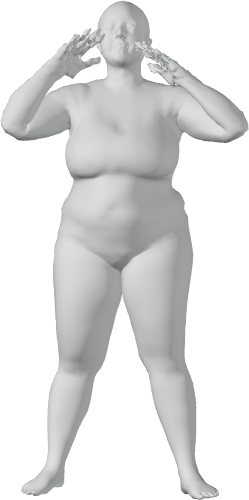}}{0.2}{0.8}{0.8}{0.95}&
\includegraphics[height=0.19\textwidth]{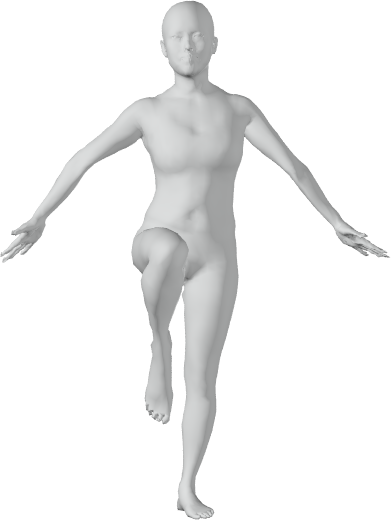}&
\includegraphics[height=0.19\textwidth]{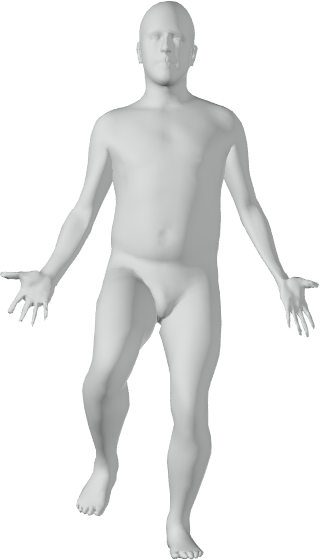}&
\includegraphics[height=0.19\textwidth]{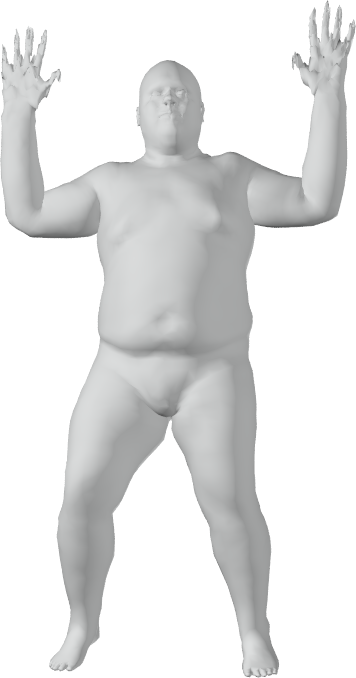}&
\includegraphics[height=0.19\textwidth]{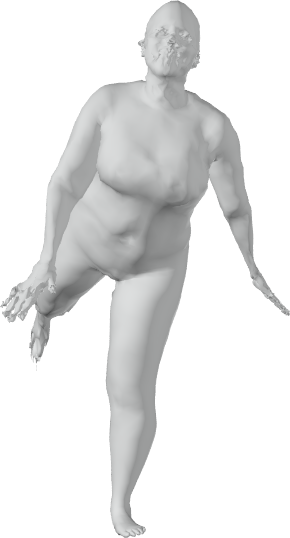}&
\includegraphics[height=0.19\textwidth]{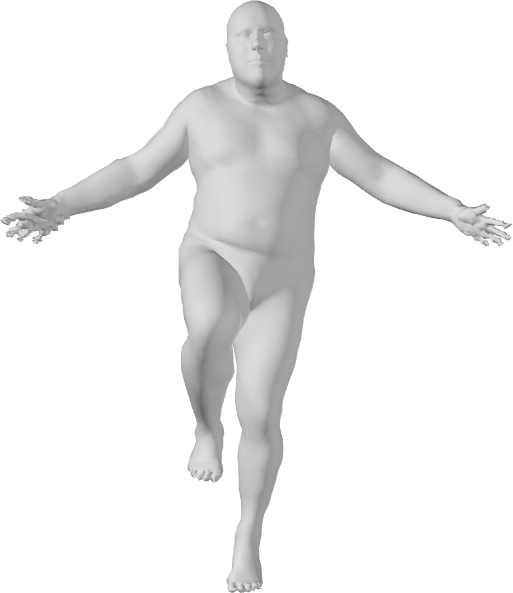}&
\includegraphics[height=0.19\textwidth]{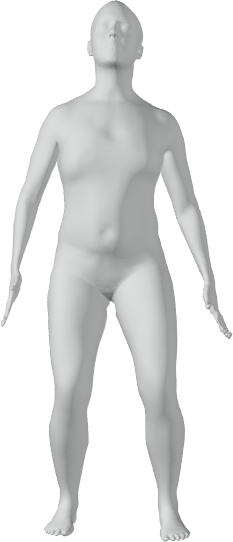}\\
\includegraphics[height=0.19\textwidth]{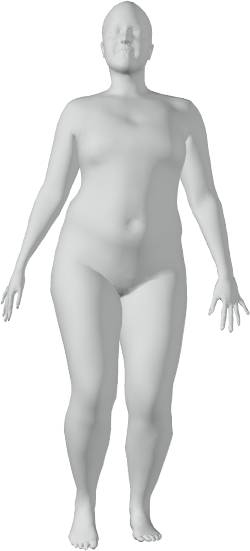}&
\includegraphics[height=0.19\textwidth]{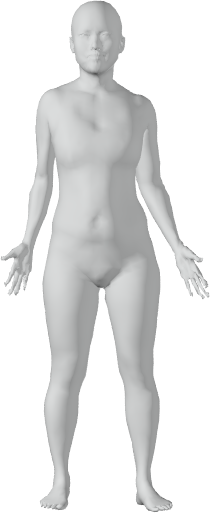}&
\includegraphics[height=0.19\textwidth]{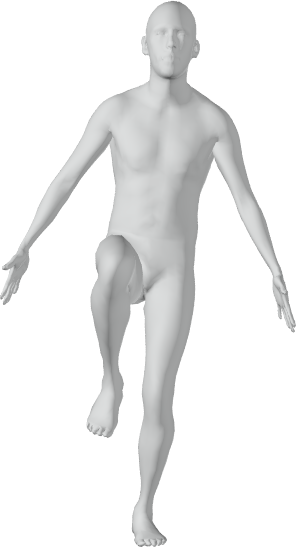}&
\includegraphics[height=0.19\textwidth]{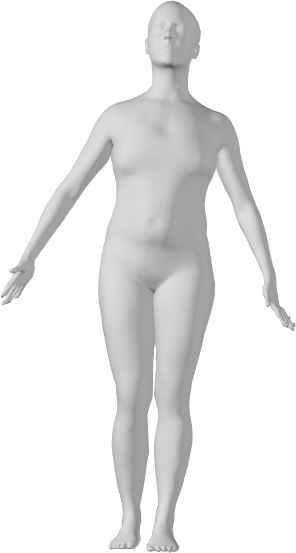}&
\includegraphics[height=0.19\textwidth]{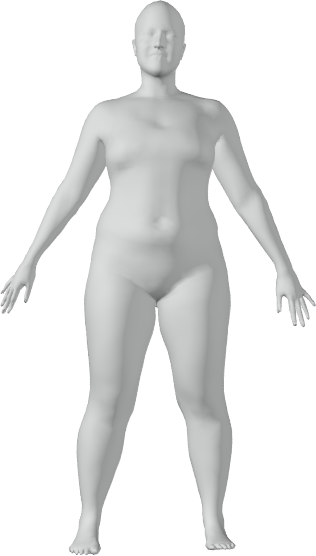}&
\includegraphics[height=0.19\textwidth]{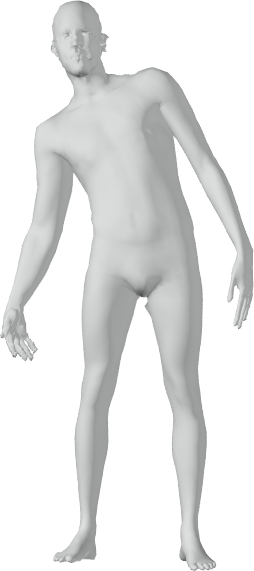}&
\includegraphics[height=0.19\textwidth]{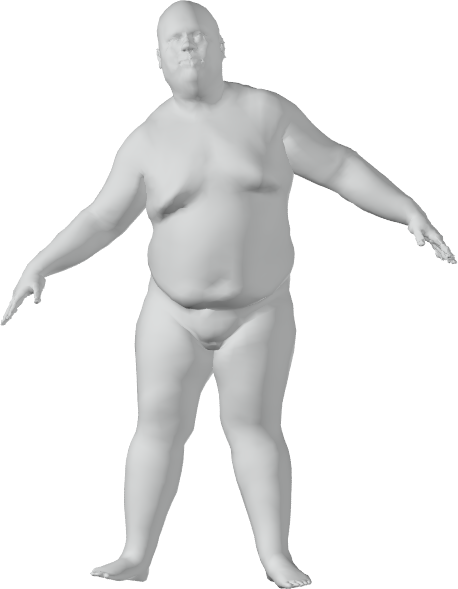}&
\includegraphics[height=0.19\textwidth]{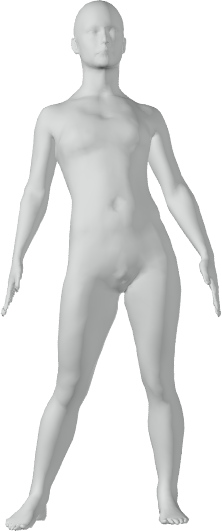}
\end{tabular}
\caption{FrameAVE results. They show improved individual shape quality and higher success rate. }
\label{Figure:SMPL:Mesh:FrameAVE}    
\end{figure*}

\begin{figure*}
\setlength\tabcolsep{2pt}
\begin{tabular}{cccccccc}
\redbox{\includegraphics[height=0.19\textwidth]{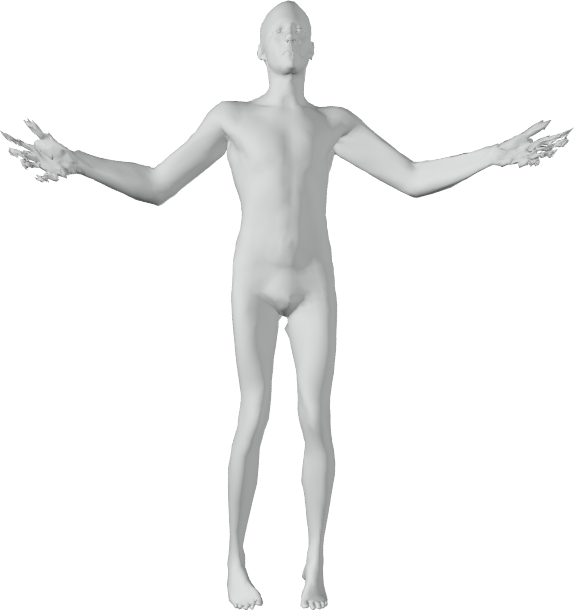}}&
\includegraphics[height=0.19\textwidth]{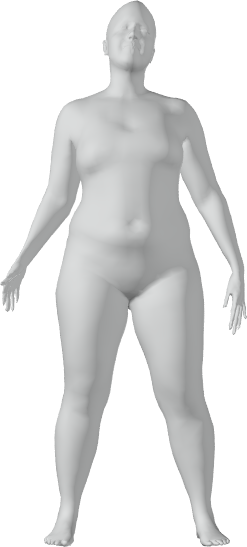}&
\includegraphics[height=0.19\textwidth]{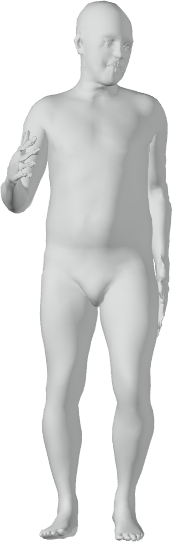}&
\includegraphics[height=0.19\textwidth]{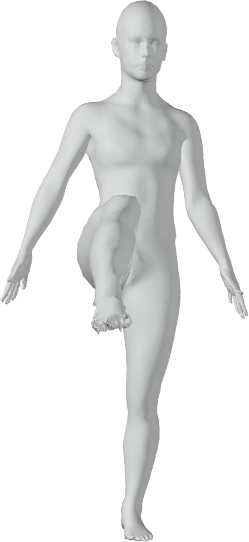}&
\includegraphics[height=0.19\textwidth]{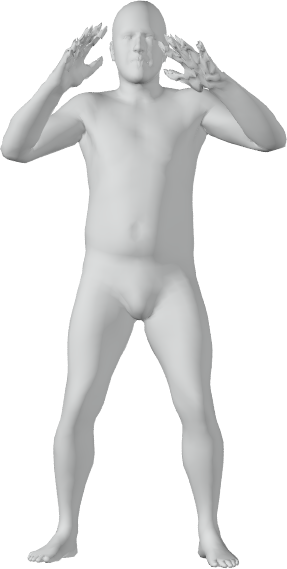}&
\includegraphics[height=0.19\textwidth]{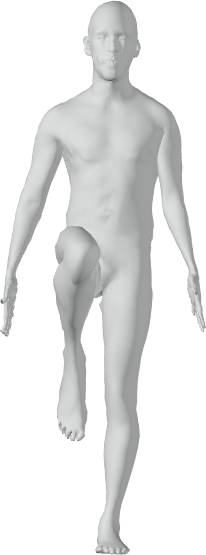}&
\includegraphics[height=0.19\textwidth]{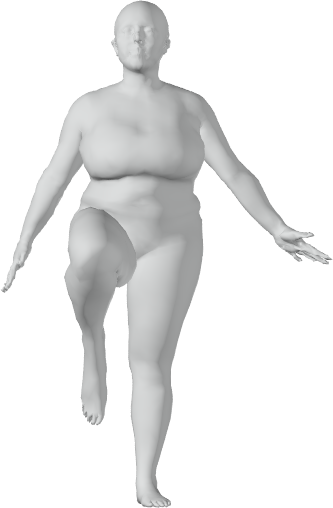}&
\includegraphics[height=0.19\textwidth]{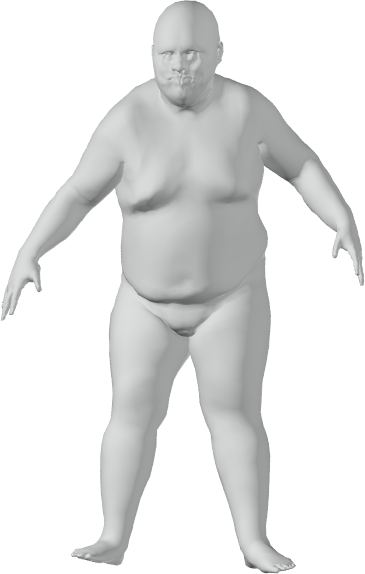}\\
\includegraphics[height=0.19\textwidth]{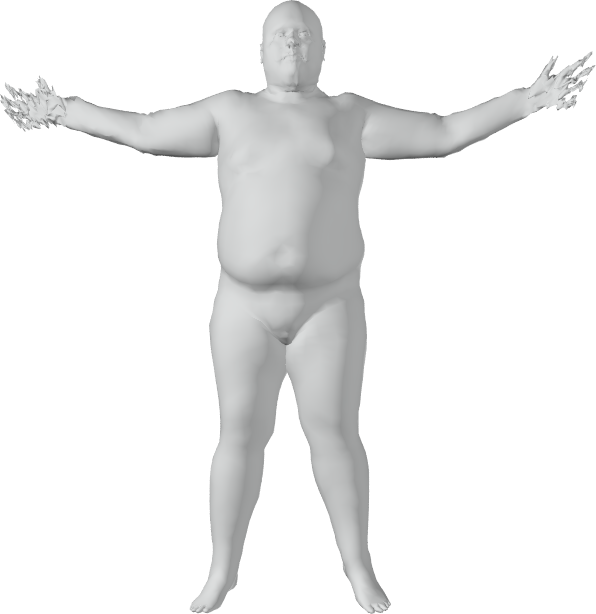}&
\includegraphics[height=0.19\textwidth]{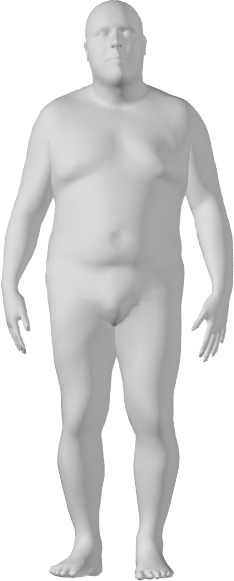}&
\includegraphics[height=0.19\textwidth]{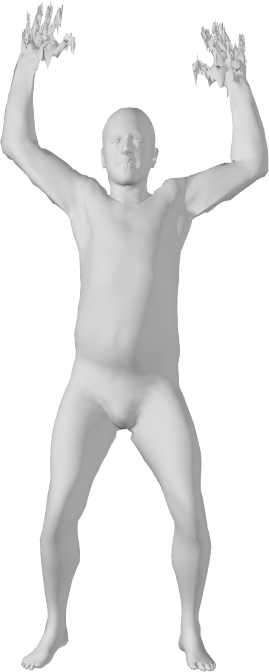}&
\bluepart{\includegraphics[height=0.19\textwidth]{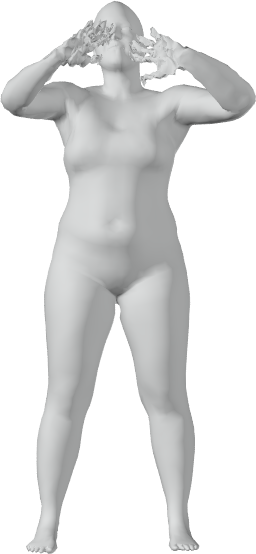}}{0.1}{1}{0.9}{0.8}&
\includegraphics[height=0.19\textwidth]{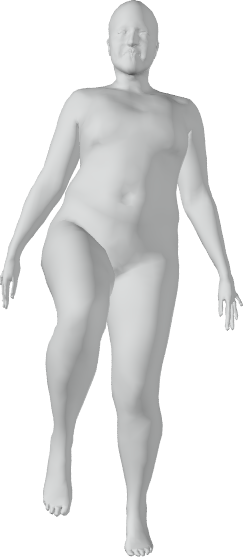}&
\includegraphics[height=0.19\textwidth]{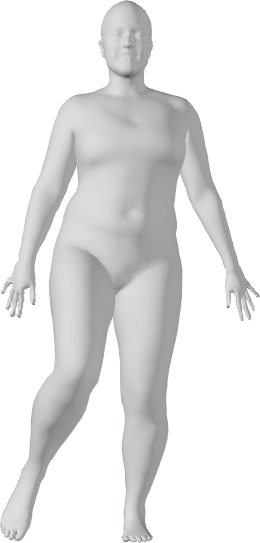}&
\includegraphics[height=0.19\textwidth]{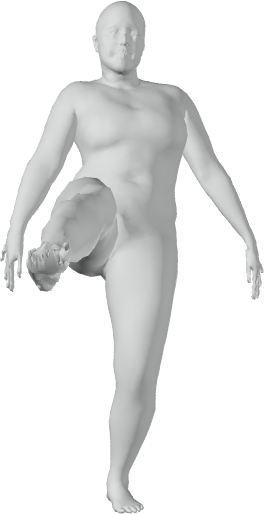}&
\includegraphics[height=0.19\textwidth]{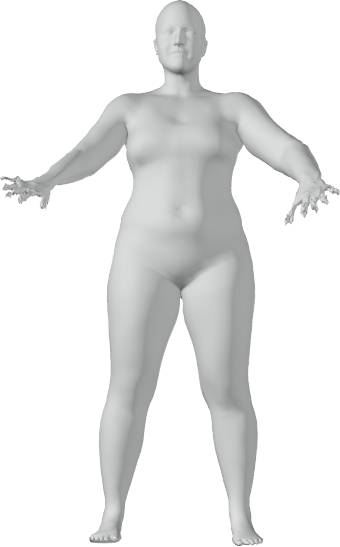}\\
\includegraphics[height=0.19\textwidth]{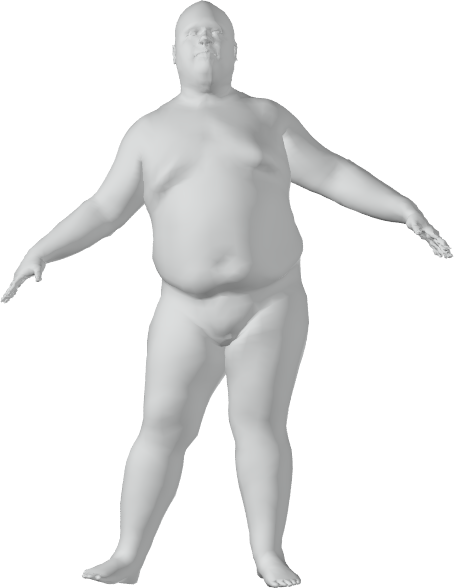}&
\includegraphics[height=0.19\textwidth]{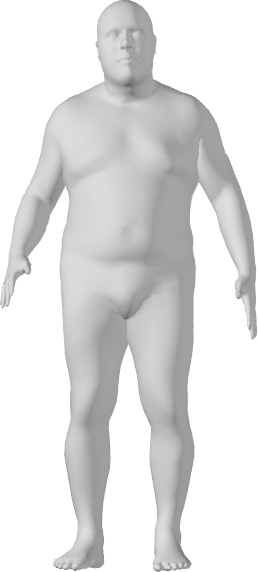}&
\includegraphics[height=0.19\textwidth]{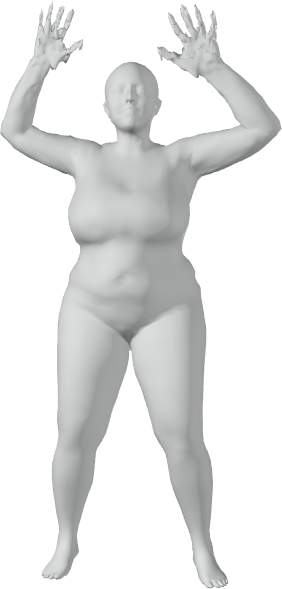}&
\includegraphics[height=0.19\textwidth]{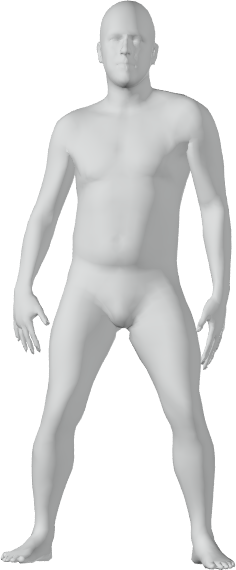}&
\includegraphics[height=0.19\textwidth]{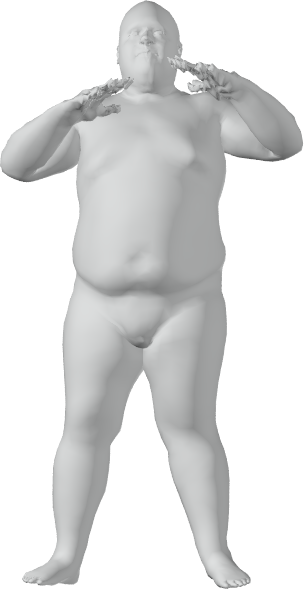}&
\includegraphics[height=0.19\textwidth]{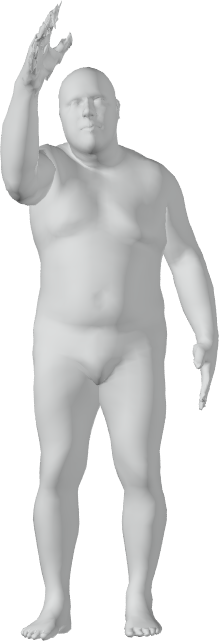}&
\redbox{\includegraphics[height=0.19\textwidth]{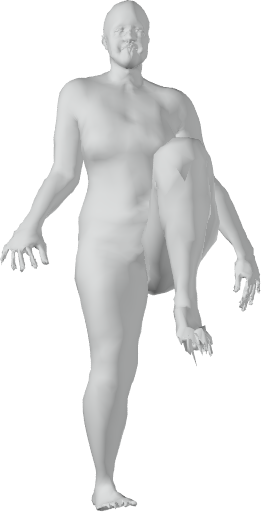}}&
\includegraphics[height=0.19\textwidth]{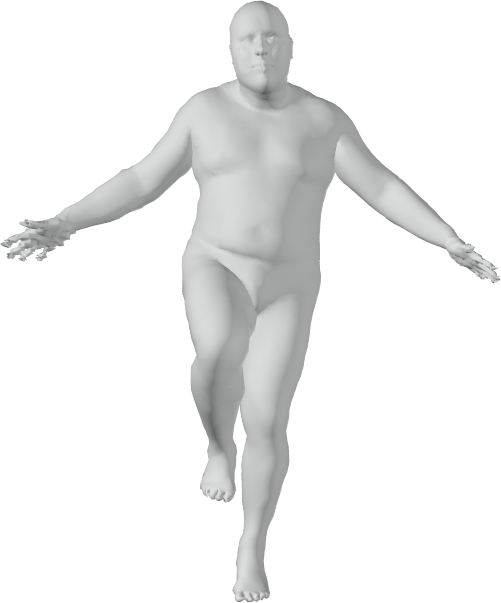}
\end{tabular}
\caption{GeoLatent results. They show improved individual shape quality and higher success rate. }
\label{Figure:SMPL:Mesh:Geolatent}    
\end{figure*}

\begin{figure*}
\setlength\tabcolsep{2.5pt}
\begin{tabular}{cccccccc}
\includegraphics[height=0.2\textwidth]{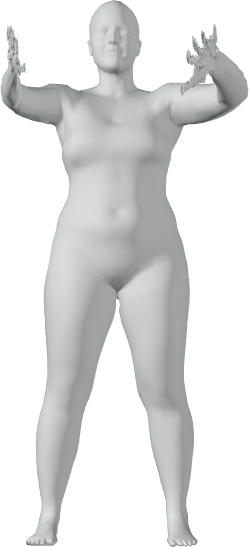}&
\includegraphics[height=0.2\textwidth]{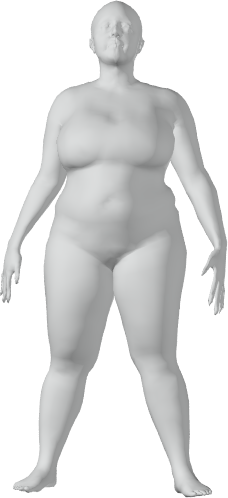}&
\includegraphics[height=0.2\textwidth]{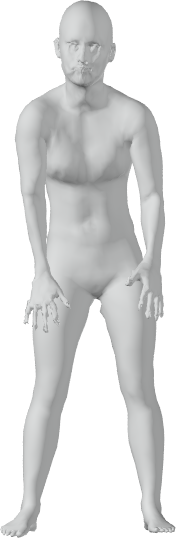}&
\includegraphics[height=0.2\textwidth]{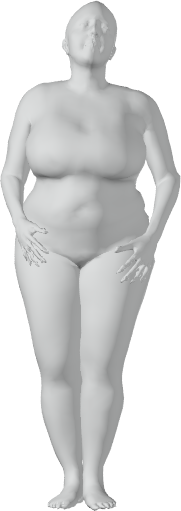}&
\redbox{\includegraphics[height=0.2\textwidth]{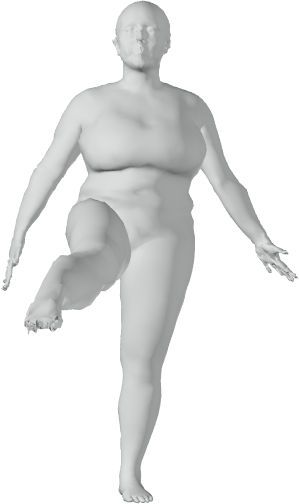}}&
\includegraphics[height=0.2\textwidth]{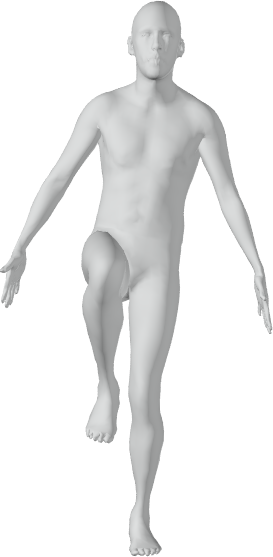}&
\includegraphics[height=0.2\textwidth]{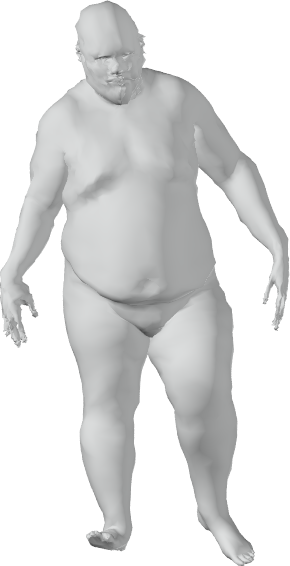}&
\includegraphics[height=0.2\textwidth]{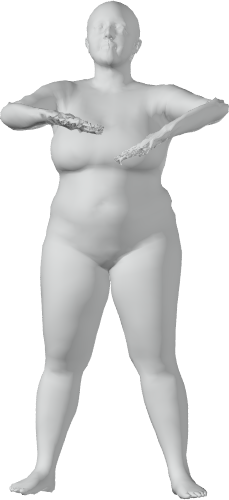}\\
\bluepart{\includegraphics[height=0.2\textwidth]{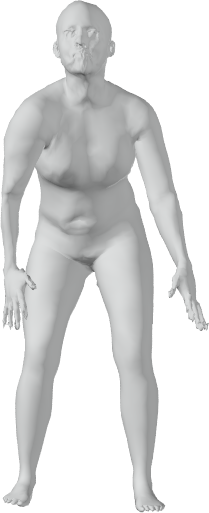}}{0.2}{1}{0.6}{0.8}&
\includegraphics[height=0.2\textwidth]{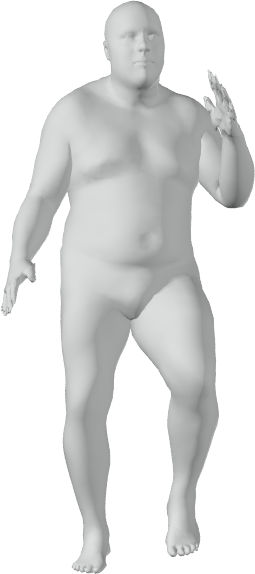}&
\includegraphics[height=0.2\textwidth]{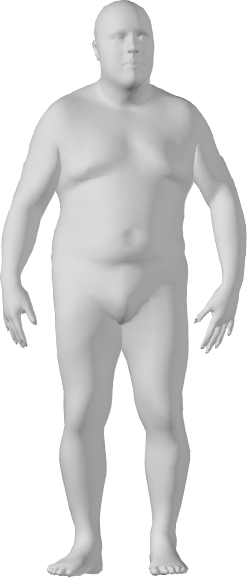}&
\includegraphics[height=0.2\textwidth]{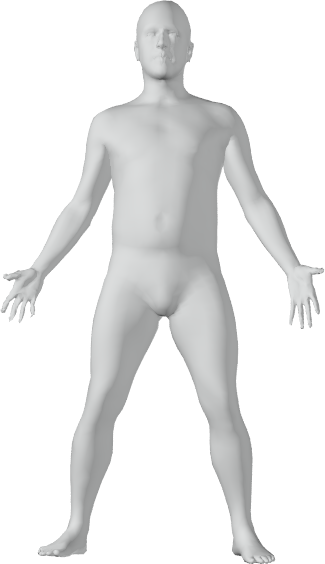}&
\includegraphics[height=0.2\textwidth]{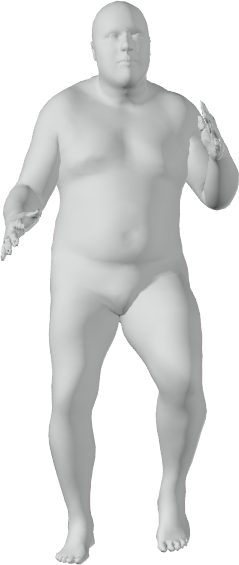}&
\includegraphics[height=0.2\textwidth]{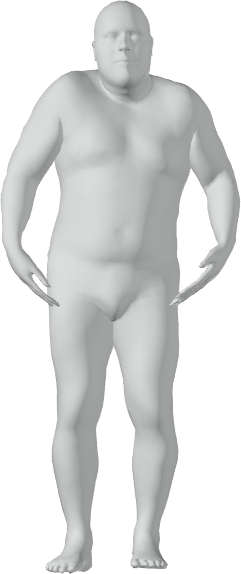}&
\includegraphics[height=0.2\textwidth]{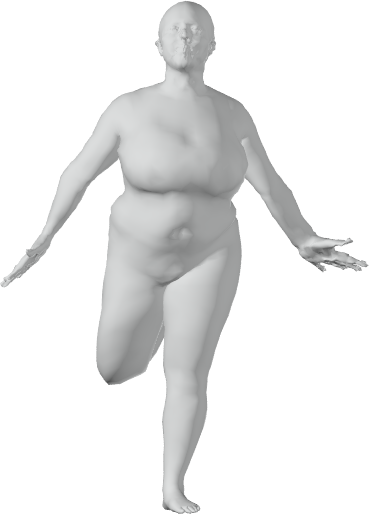}&
\bluepart{\includegraphics[height=0.2\textwidth]{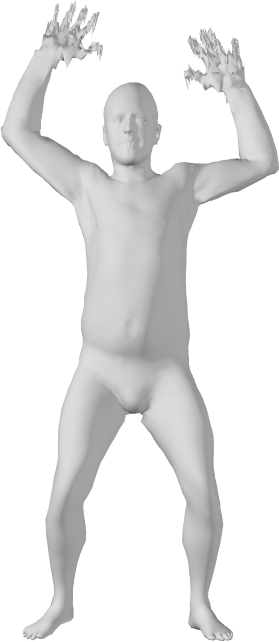}}{0.1}{1}{0.9}{0.8}\\
\includegraphics[height=0.2\textwidth]{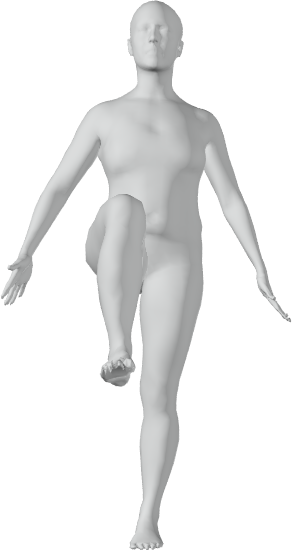}&
\includegraphics[height=0.2\textwidth]{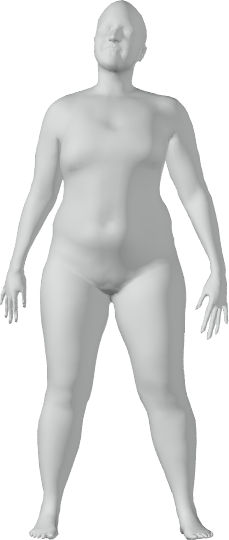}&
\includegraphics[height=0.2\textwidth]{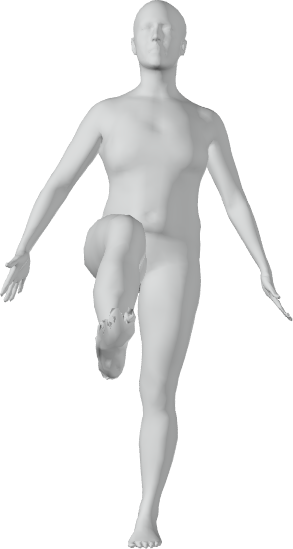}&
\includegraphics[height=0.2\textwidth]{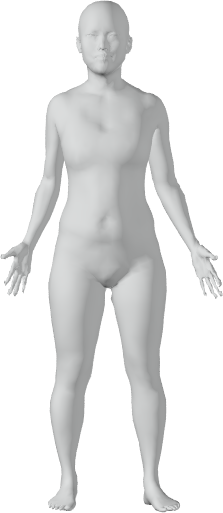}&
\includegraphics[height=0.2\textwidth]{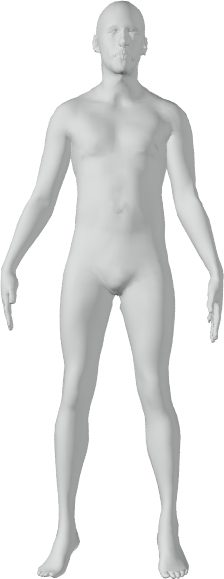}&
\includegraphics[height=0.2\textwidth]{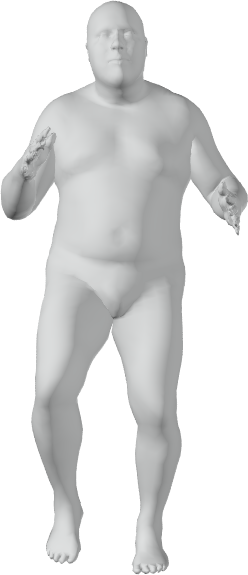}&
\redbox{\includegraphics[height=0.2\textwidth]{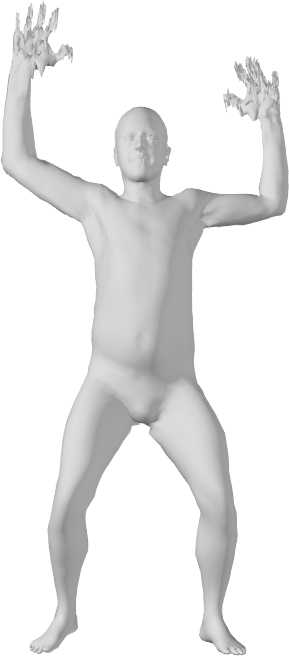}}&
\includegraphics[height=0.2\textwidth]{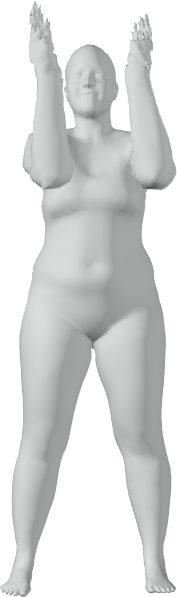}
\end{tabular}
\caption{BRESA results. They show improved individual shape quality and higher success rate. }
\label{Figure:SMPL:Mesh:BRESA}    
\end{figure*}

\begin{figure*}
\setlength\tabcolsep{2pt}
\begin{tabular}{cccccccc}
\includegraphics[height=0.19\textwidth]{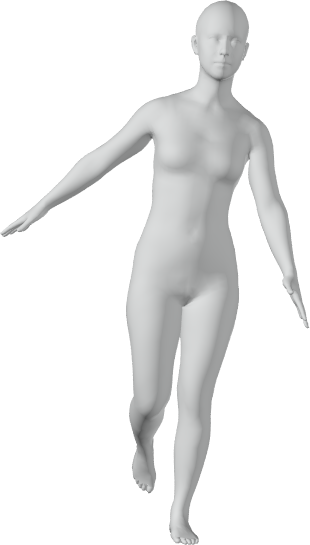}&
\includegraphics[height=0.19\textwidth]{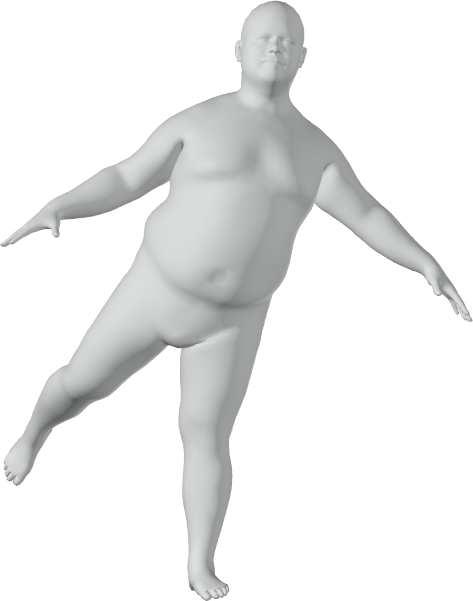}&
\includegraphics[height=0.19\textwidth]{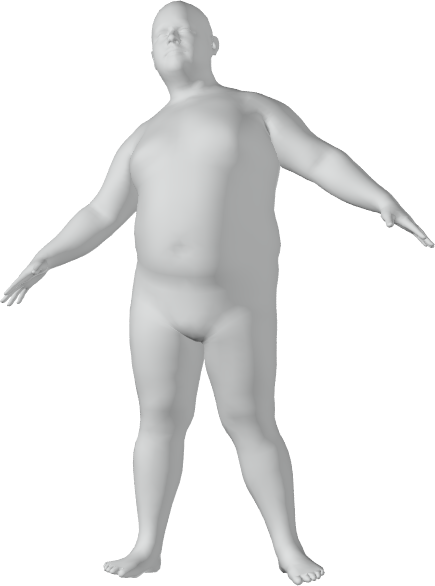}&
\includegraphics[height=0.19\textwidth]{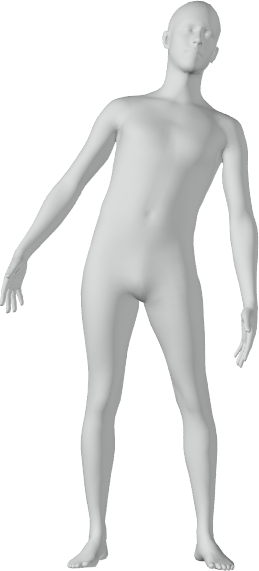}&
\includegraphics[height=0.19\textwidth]{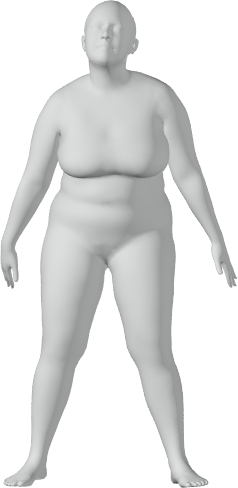}&
\includegraphics[height=0.19\textwidth]{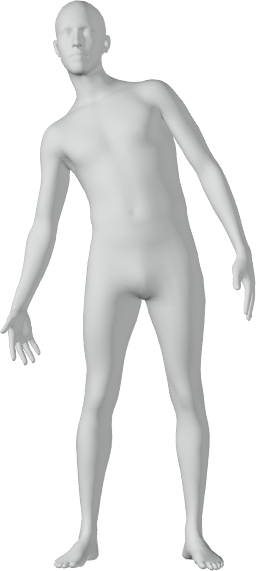}&
\includegraphics[height=0.19\textwidth]{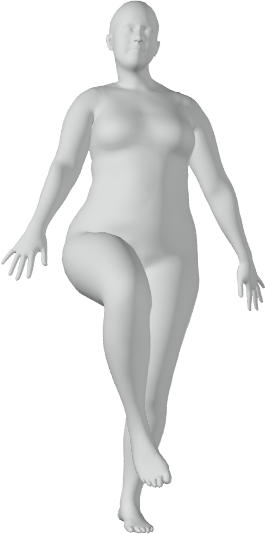}&
\includegraphics[height=0.19\textwidth]{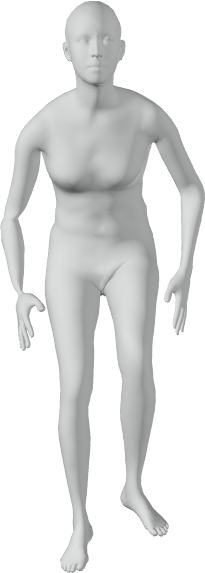}\\
\includegraphics[height=0.19\textwidth]{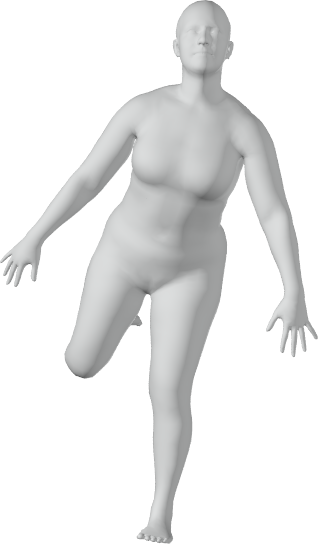}&
\includegraphics[height=0.19\textwidth]{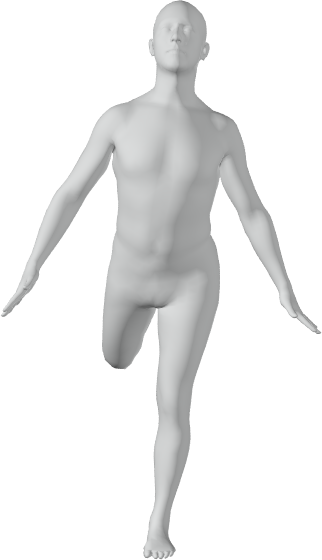}&
\includegraphics[height=0.19\textwidth]{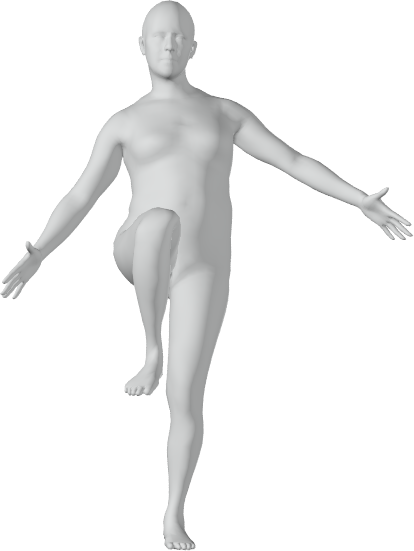}&
\includegraphics[height=0.19\textwidth]{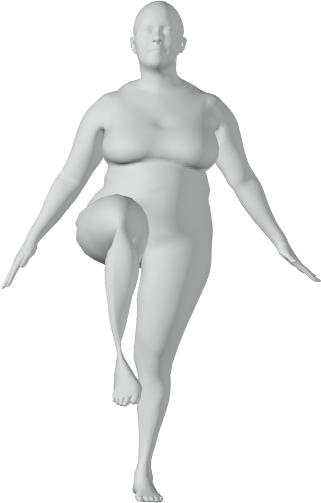}&
\includegraphics[height=0.19\textwidth]{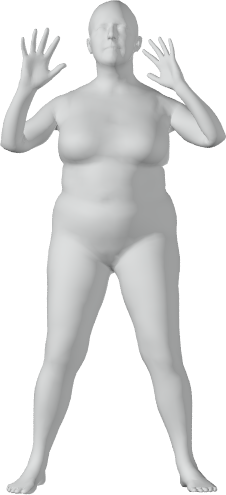}&
\includegraphics[height=0.19\textwidth]{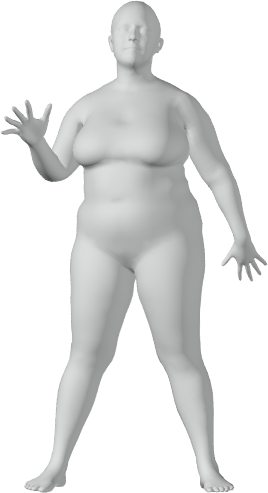}&
\includegraphics[height=0.19\textwidth]{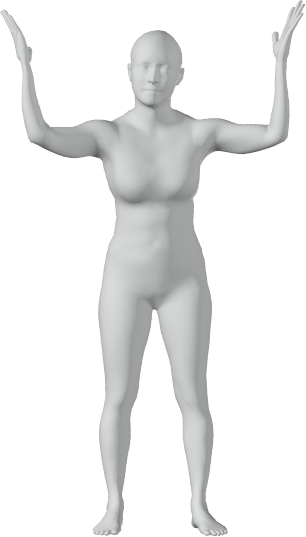}&
\includegraphics[height=0.19\textwidth]{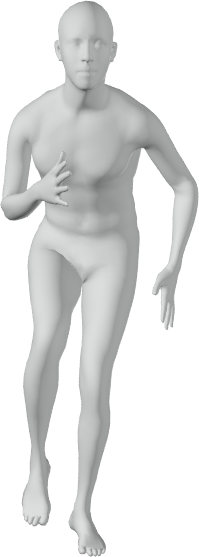}\\
\includegraphics[height=0.19\textwidth]{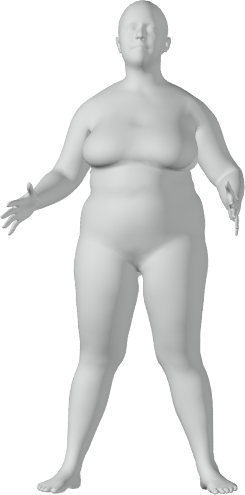}&
\includegraphics[height=0.19\textwidth]{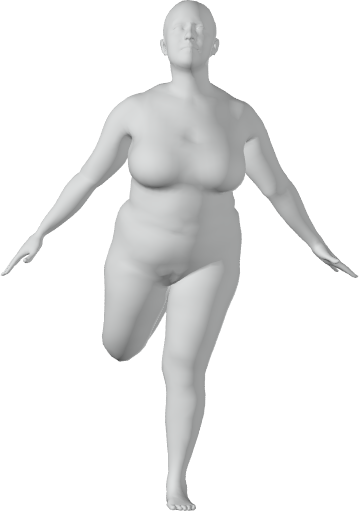}&
\includegraphics[height=0.19\textwidth]{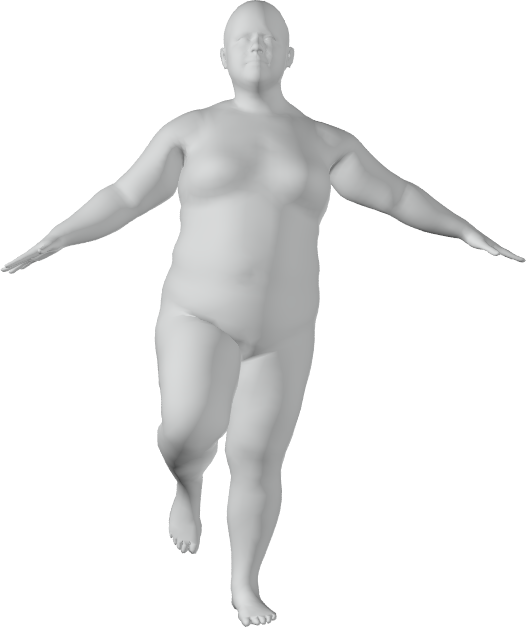}&
\includegraphics[height=0.19\textwidth]{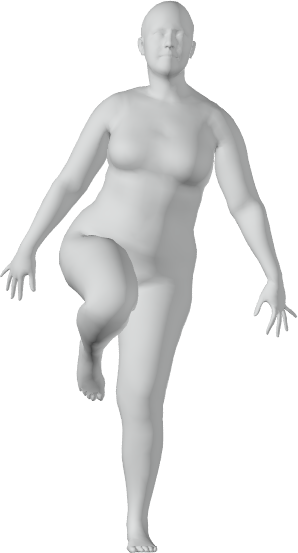}&
\includegraphics[height=0.19\textwidth]{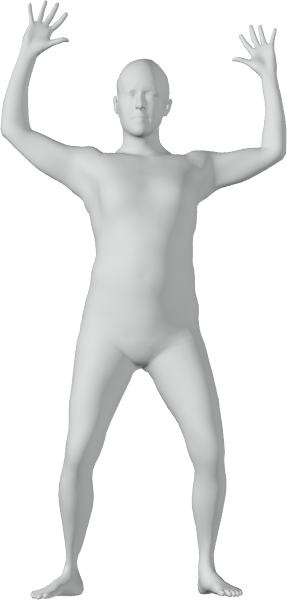}&
\includegraphics[height=0.19\textwidth]{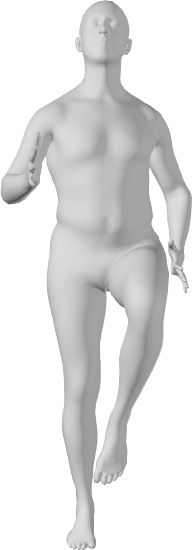}&
\includegraphics[height=0.19\textwidth]{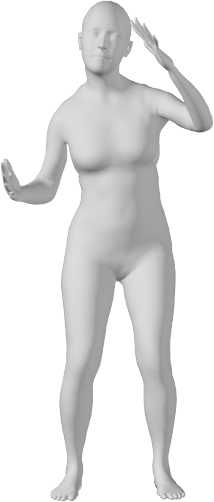}&
\includegraphics[height=0.19\textwidth]{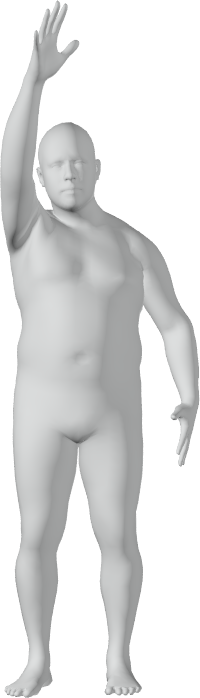}
\end{tabular}
\caption{ARAPDiffusion results. They show improved individual shape quality and higher success rate. }
\label{Figure:SMPL:Mesh:ARAPDiff}    
\end{figure*}

\begin{figure*}
\setlength\tabcolsep{1.5pt}
\begin{tabular}{cc|cc|cc|cc}
\includegraphics[height=0.20\textwidth]{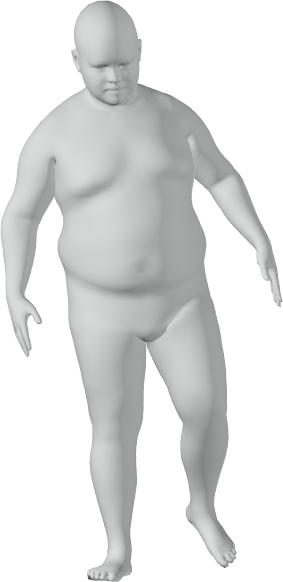}&
\includegraphics[height=0.20\textwidth]{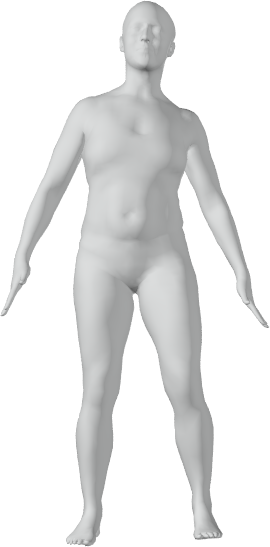}&
\includegraphics[height=0.20\textwidth]{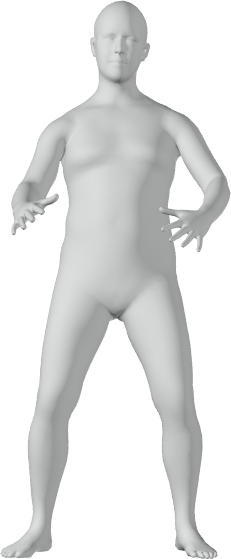}&
\includegraphics[height=0.20\textwidth]{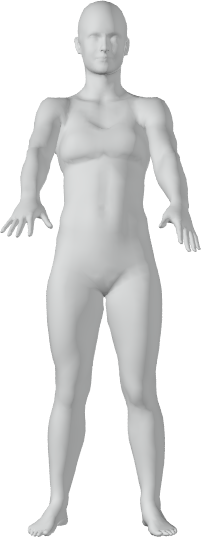}&
\includegraphics[height=0.20\textwidth]{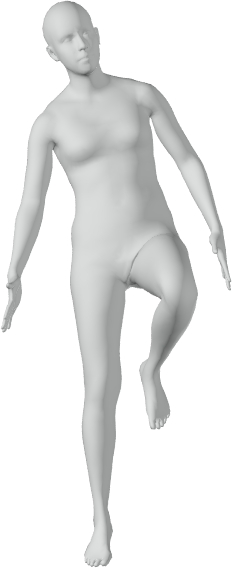}&
\includegraphics[height=0.20\textwidth]{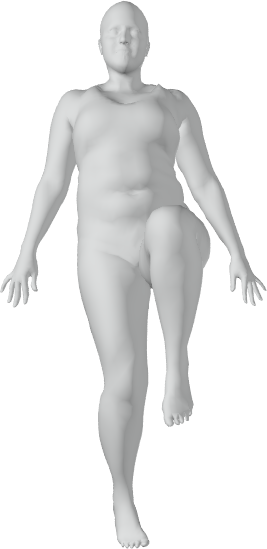}&
\includegraphics[height=0.20\textwidth]{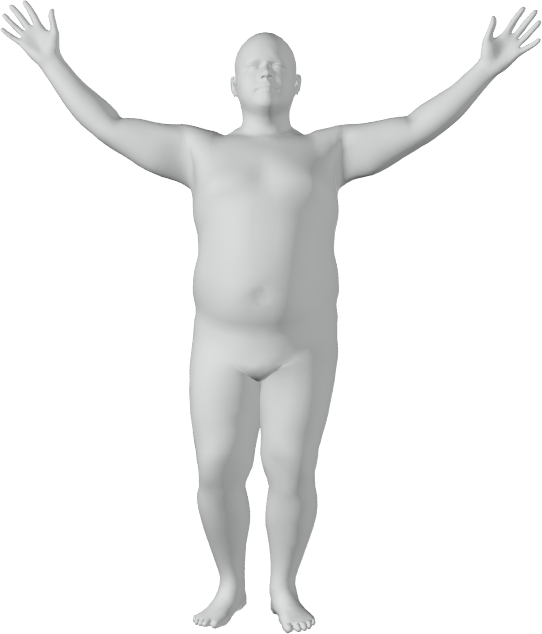}&
\includegraphics[height=0.20\textwidth]{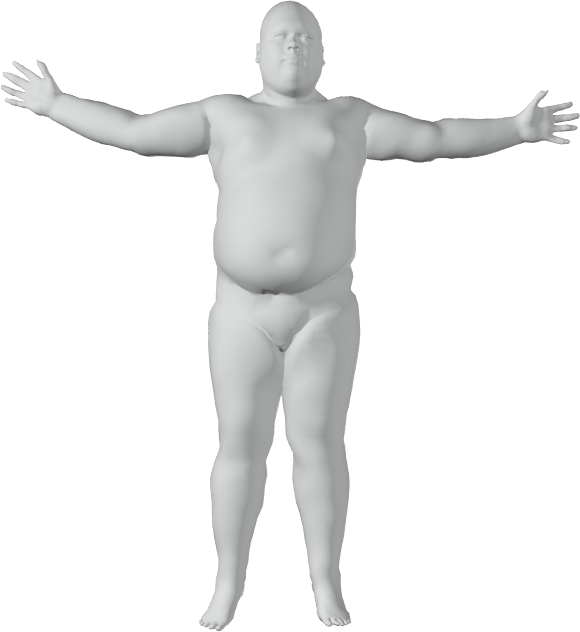}\\\hline
\includegraphics[height=0.20\textwidth]{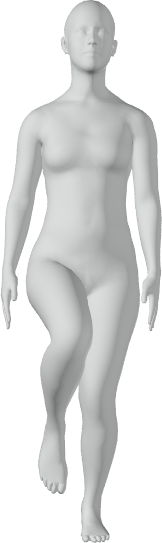}&
\includegraphics[height=0.20\textwidth]{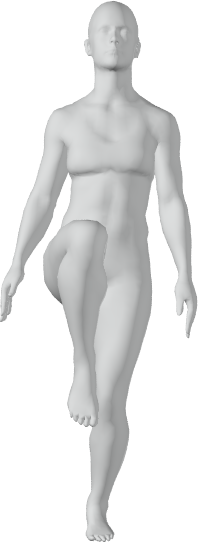}&
\includegraphics[height=0.20\textwidth]{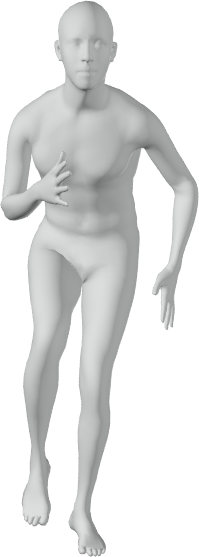}&
\includegraphics[height=0.20\textwidth]{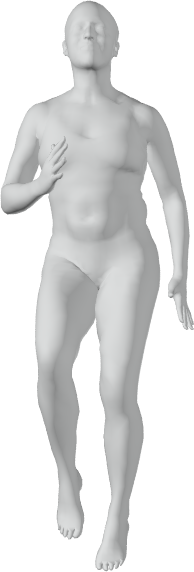}&
\includegraphics[height=0.20\textwidth]{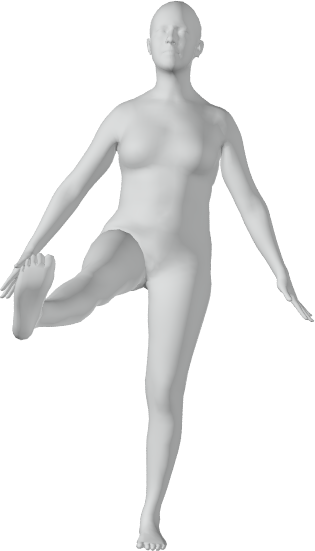}&
\includegraphics[height=0.20\textwidth]{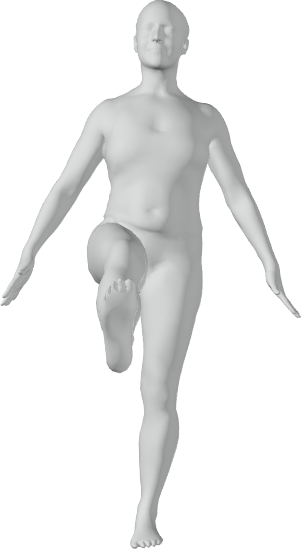}&
\includegraphics[height=0.20\textwidth]{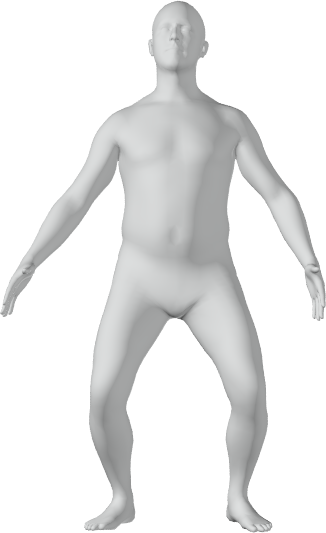}&
\includegraphics[height=0.20\textwidth]{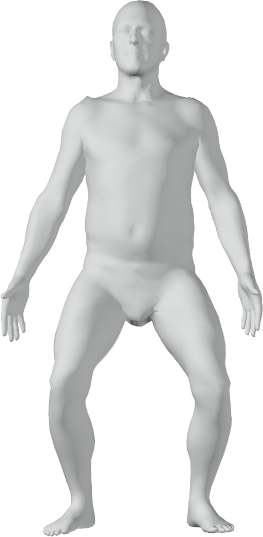}\\\hline
\includegraphics[height=0.20\textwidth]{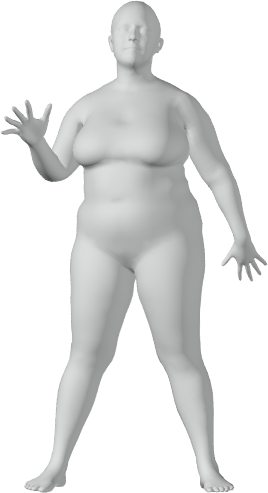}&
\includegraphics[height=0.20\textwidth]{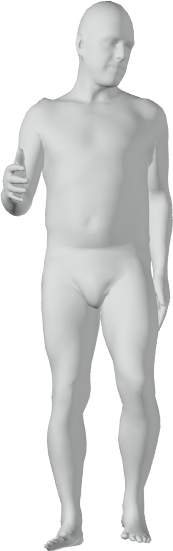}&
\includegraphics[height=0.20\textwidth]{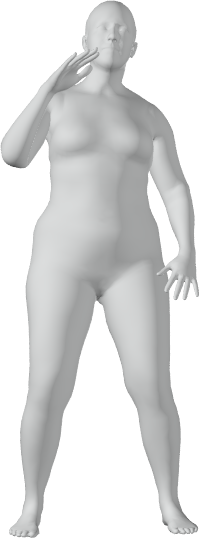}&
\includegraphics[height=0.20\textwidth]{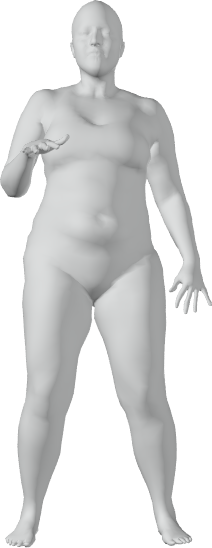}&
\includegraphics[height=0.20\textwidth]{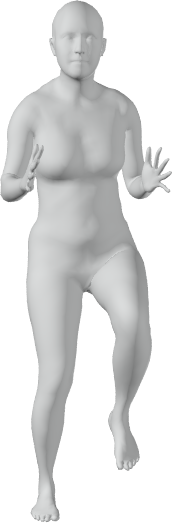}&
\includegraphics[height=0.20\textwidth]{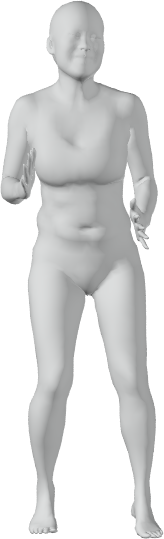}&
\includegraphics[height=0.20\textwidth]{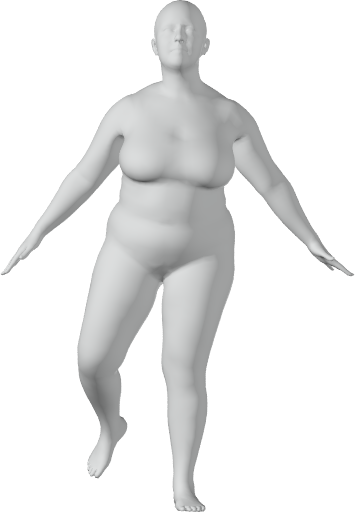}&
\includegraphics[height=0.20\textwidth]{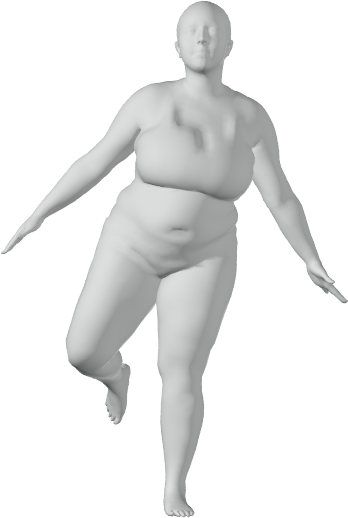}
\end{tabular}
\caption{The closest training shape of each generated shape. In each cell, the generated shape is shown on the left, while the closest training shape is shown on the right. }
\label{Figure:SMPL:Mesh:ARAPDiff:Closest:Training}    
\end{figure*}

\subsection{Animal}

Fig.~\ref{Figure:SMAL:Mesh:ARAPReg}, Fig.~\ref{Figure:SMAL:Mesh:FrameAVE}, Fig.~\ref{Figure:SMAL:Mesh:GeoLatent}, and Fig.~\ref{Figure:SMAL:Mesh:BRESA} show randomly generated shapes of the baseline approaches ARAPReg, FrameAVE, Geolatent, and BRESA, respectively. All of them share the same issue that a subset of the synthetic shapes have low-quality overall shapes. Moreover, their synthetic shapes show a variety of noticeable distortions. 

In contrast, as shown in Fig.~\ref{Figure:SMAL:Mesh:ARAPDiff}, ARAPDiffusion does not have low-quality generated shapes. The distortion of each generate shape is also visually smaller. 

Fig.~\ref{Figure:SMAL:Mesh:ARAPDiff:Closest:Training} shows for each generated shape (on the left) its closest training shape (on the right). We can see that the generated shapes do not replicate the training shapes. 

\begin{figure*}
\setlength\tabcolsep{4pt}
\begin{tabular}{cccccc}
\includegraphics[width=0.15\textwidth]{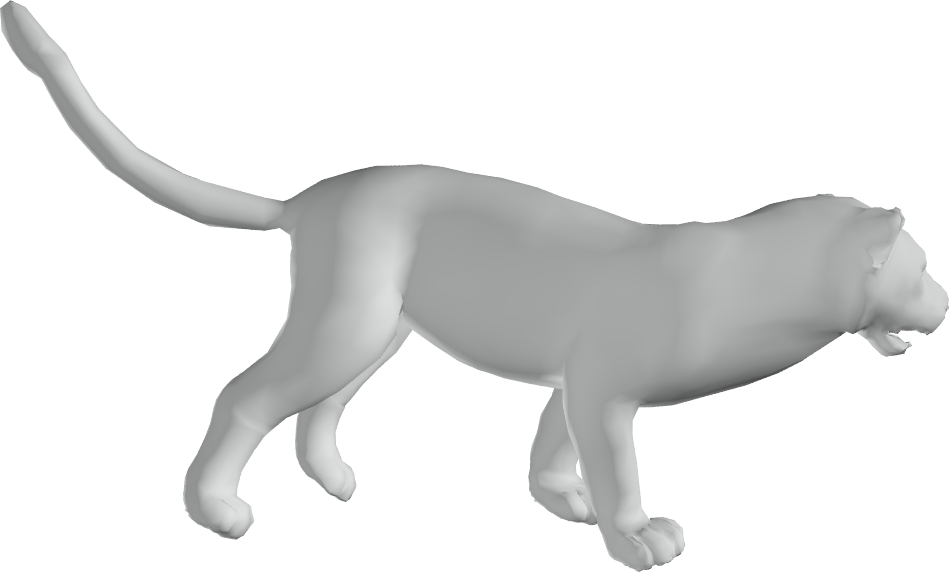}&
\redbox{\includegraphics[width=0.15\textwidth]{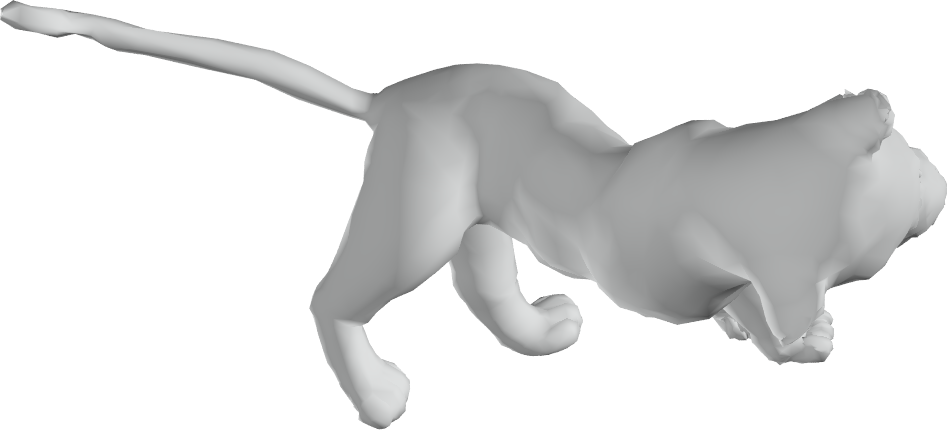}}&
\includegraphics[width=0.15\textwidth]{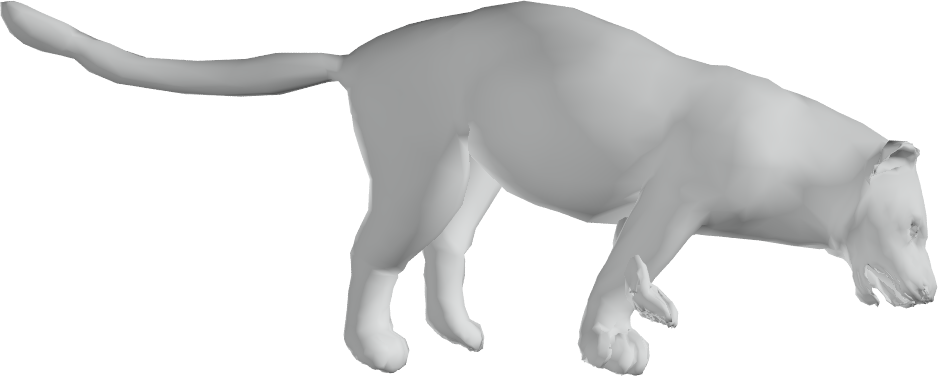}&
\includegraphics[width=0.15\textwidth]{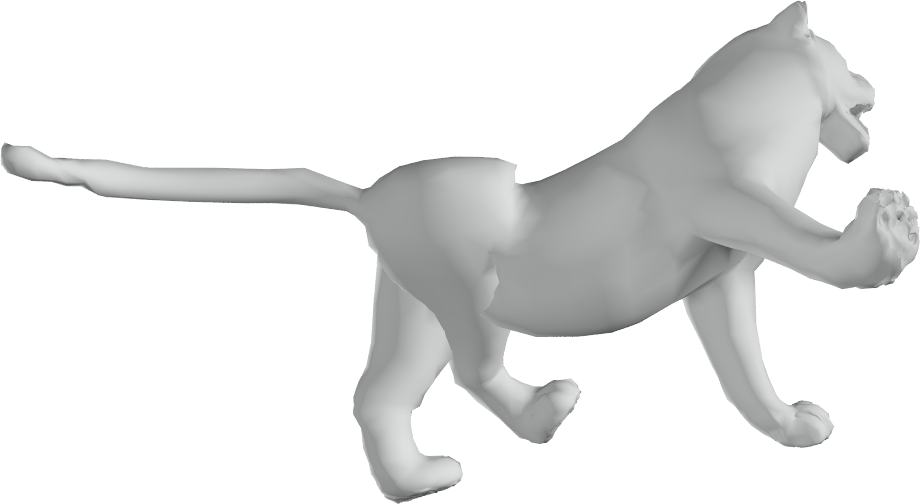}&
\includegraphics[width=0.15\textwidth]{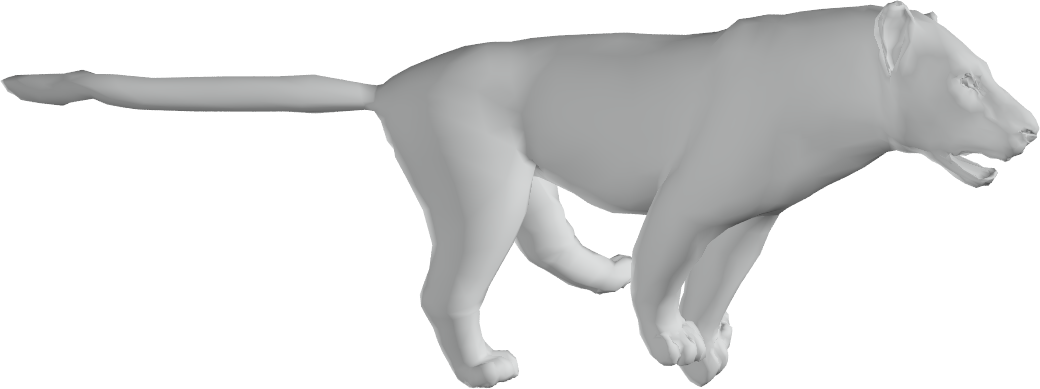}&
\includegraphics[width=0.15\textwidth]{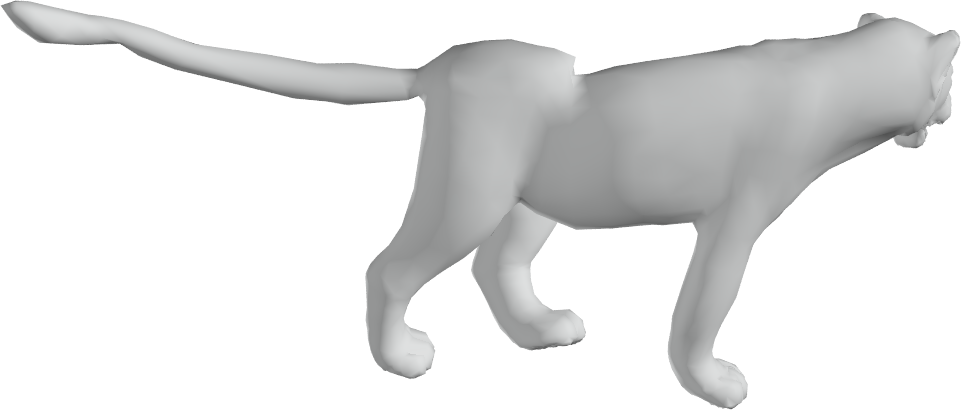}\\
\includegraphics[width=0.15\textwidth]{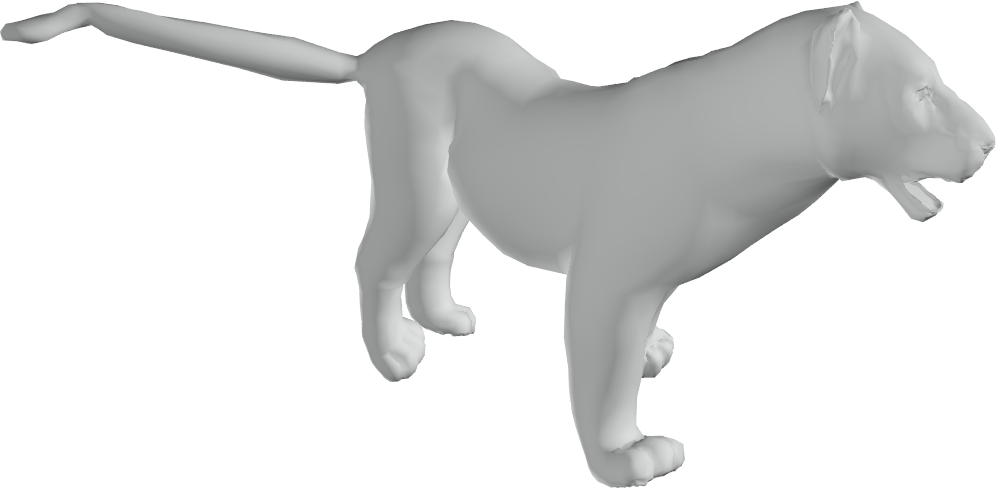}&
\includegraphics[width=0.15\textwidth]{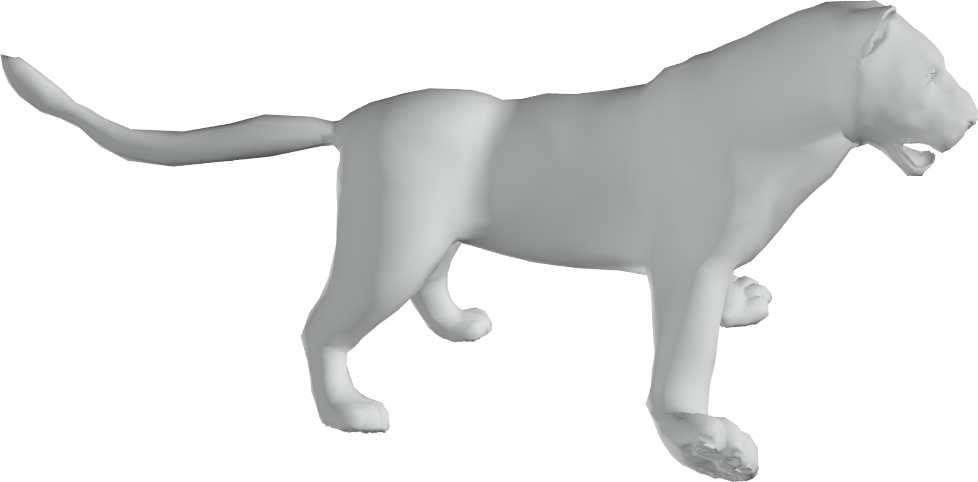}&
\includegraphics[width=0.15\textwidth]{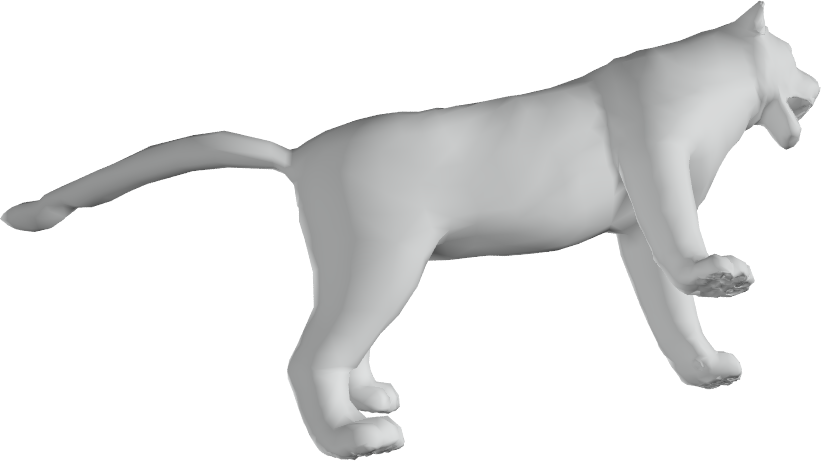}&
\includegraphics[width=0.15\textwidth]{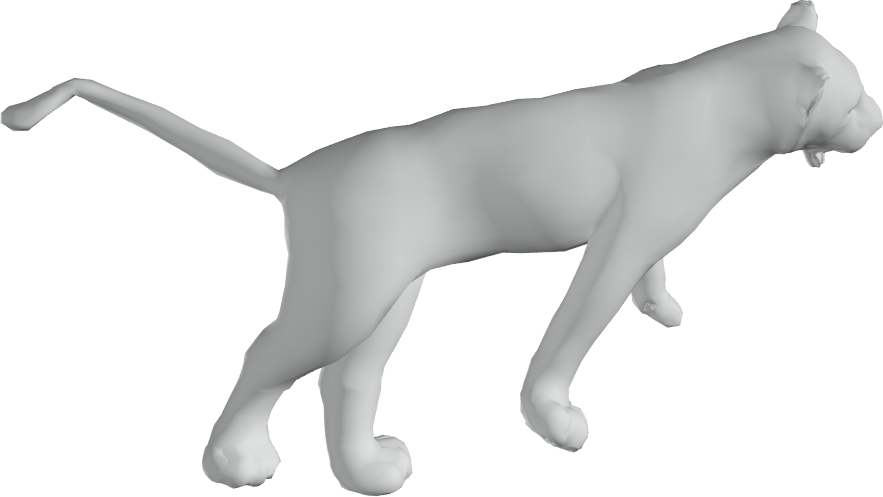}&
\includegraphics[width=0.15\textwidth]{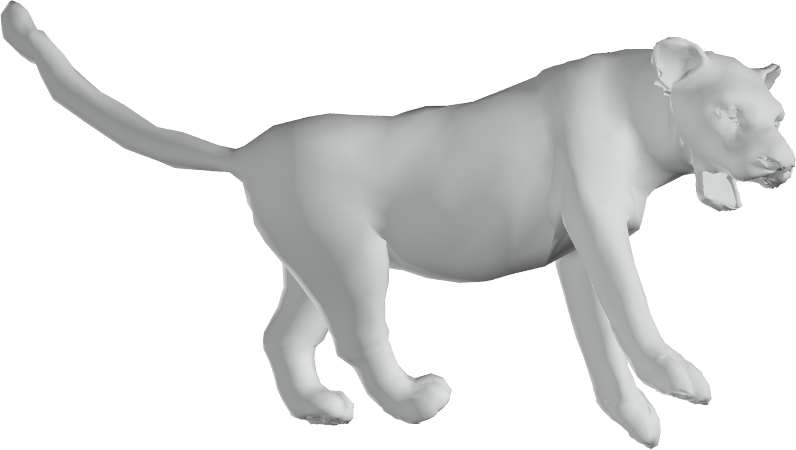}&
\includegraphics[width=0.15\textwidth]{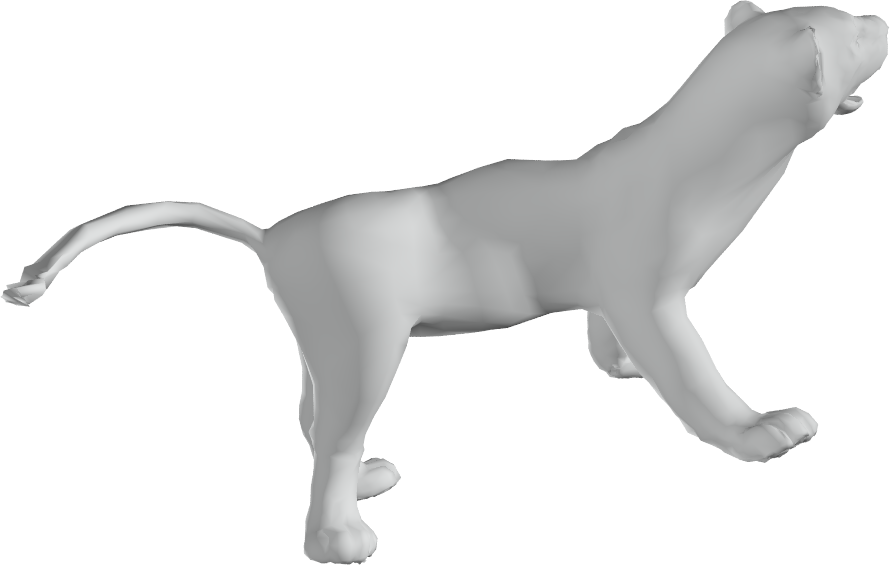}\\
\includegraphics[width=0.15\textwidth]{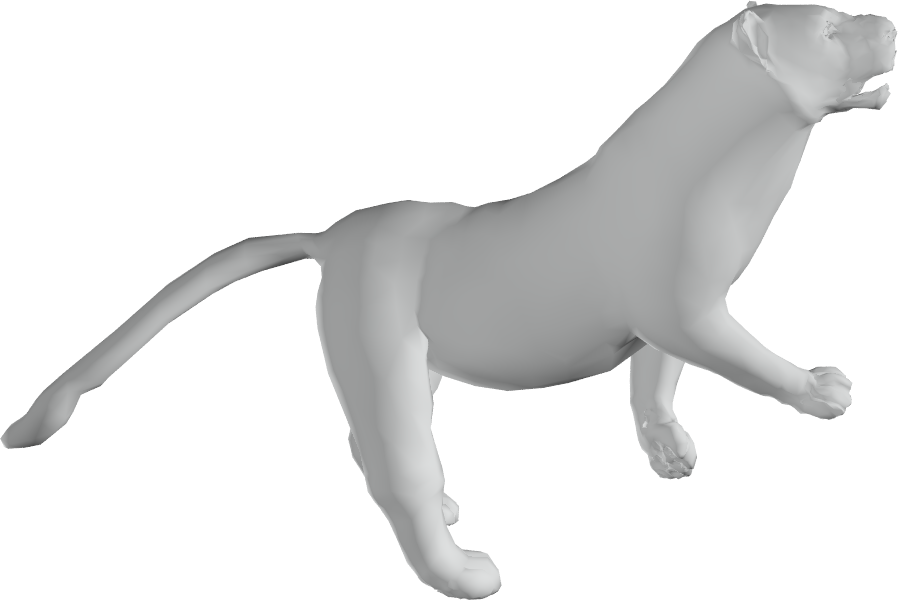}&
\includegraphics[width=0.15\textwidth]{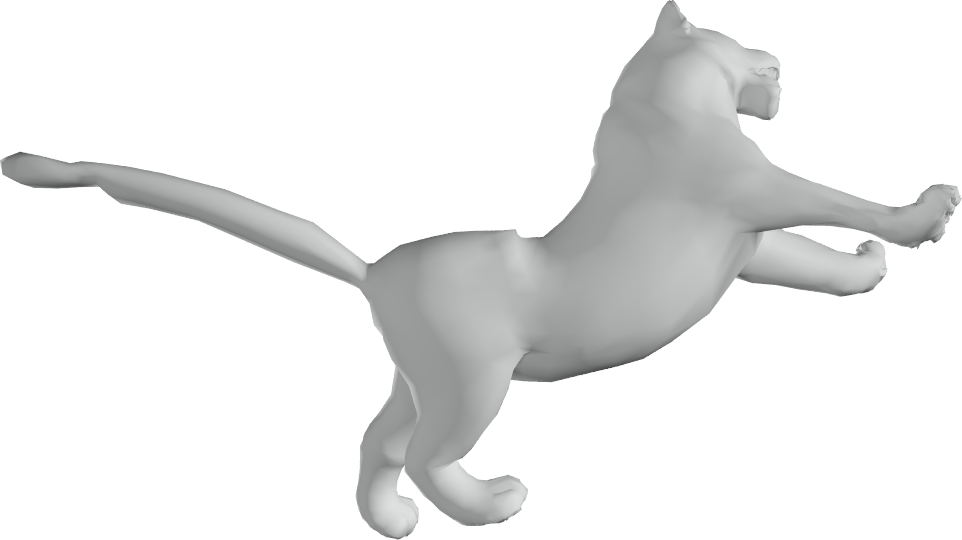}&
\includegraphics[width=0.15\textwidth]{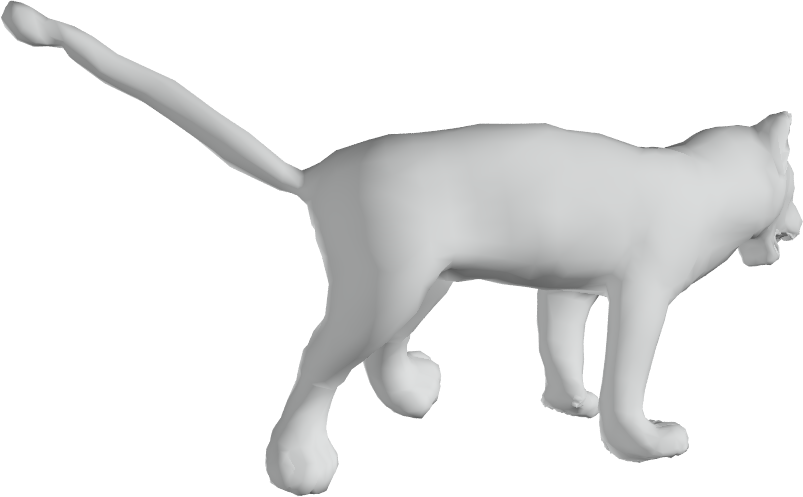}&
\includegraphics[width=0.15\textwidth]{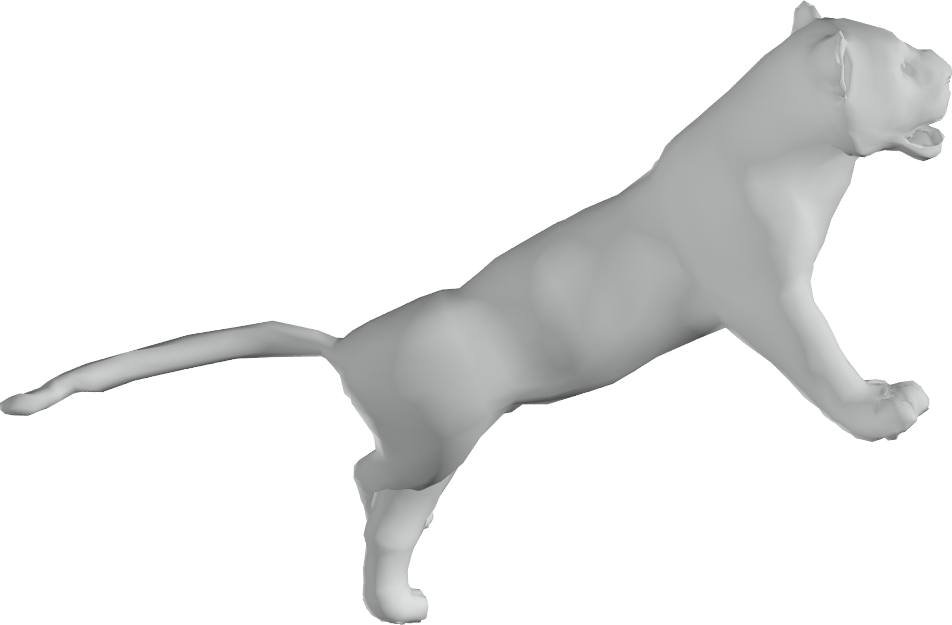}&
\includegraphics[width=0.15\textwidth]{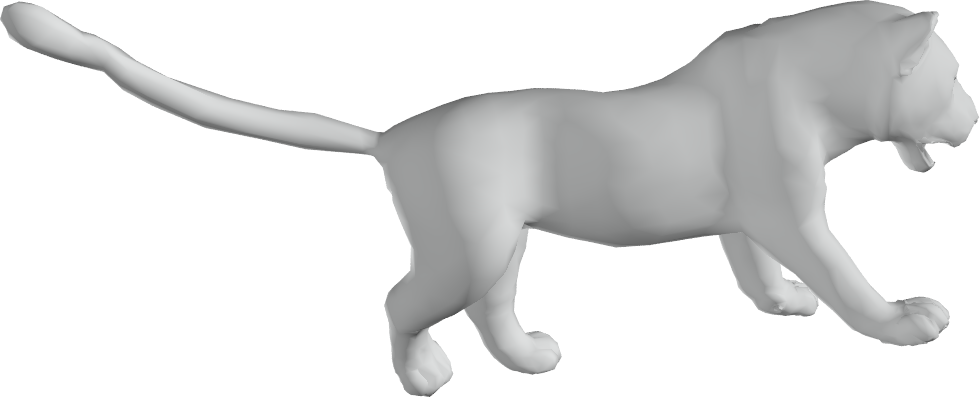}&
\includegraphics[width=0.15\textwidth]{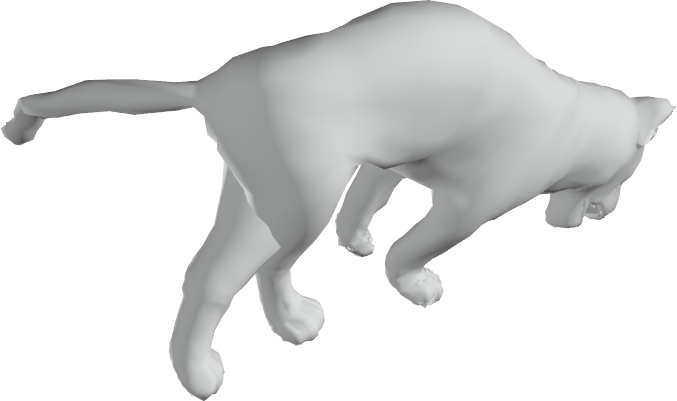}\\
\includegraphics[width=0.15\textwidth]{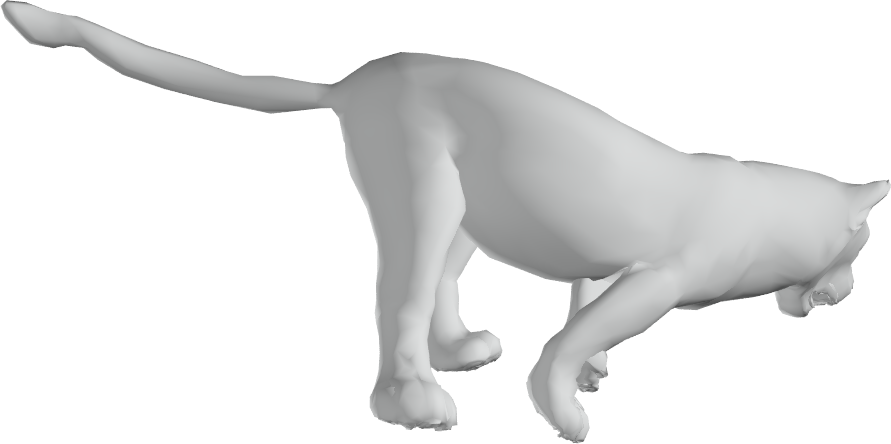}&
\includegraphics[width=0.15\textwidth]{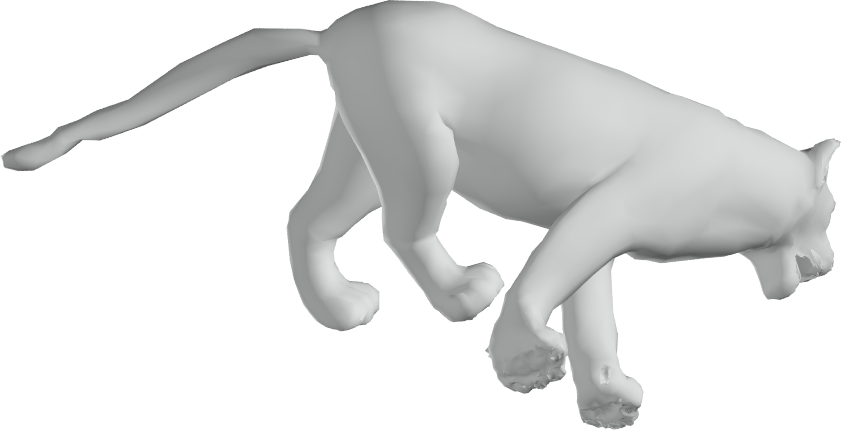}&
\includegraphics[width=0.15\textwidth]{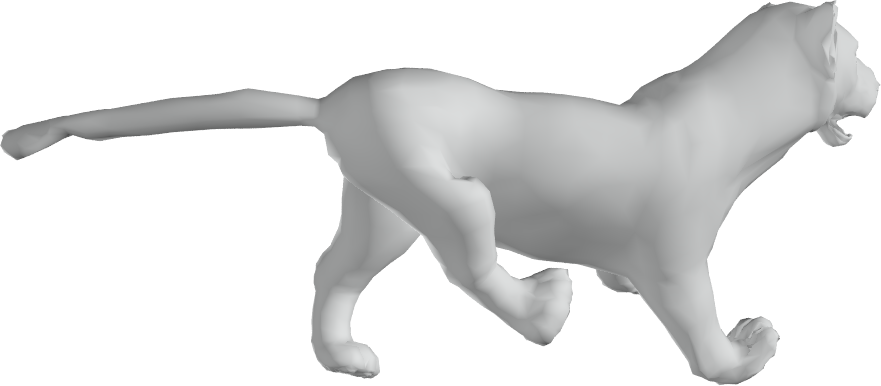}&
\includegraphics[width=0.15\textwidth]{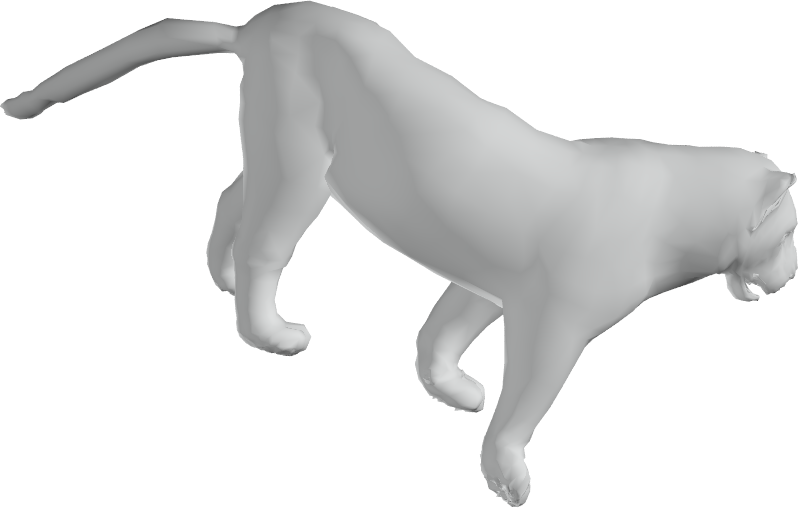}&
\redbox{\includegraphics[width=0.15\textwidth]{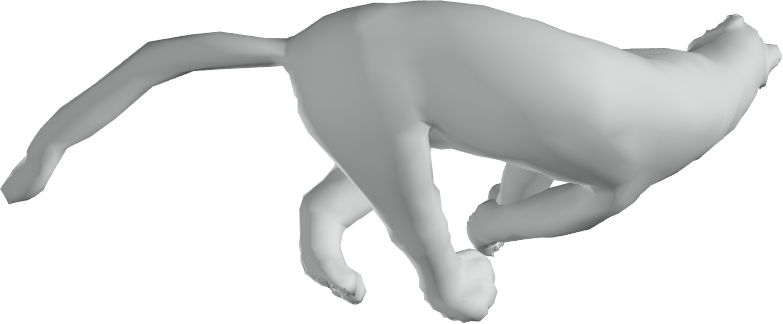}}&
\includegraphics[width=0.15\textwidth]{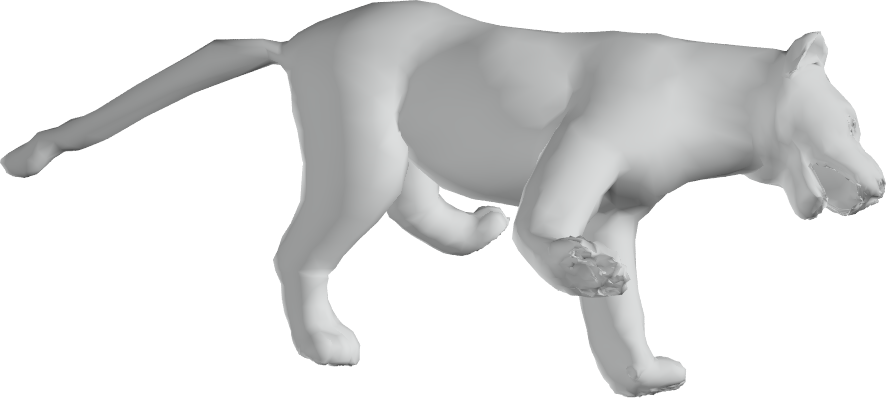}
\end{tabular}
\caption{ARAPReg results. Similar to other baseline approaches, some of the generated shapes have significantly distorted geometries. Most of shapes have noticeable local distortions. }
\label{Figure:SMAL:Mesh:ARAPReg}    
\end{figure*}

\begin{figure*}
\setlength\tabcolsep{4pt}
\begin{tabular}{cccccc}
\bluepart{\includegraphics[width=0.15\textwidth]{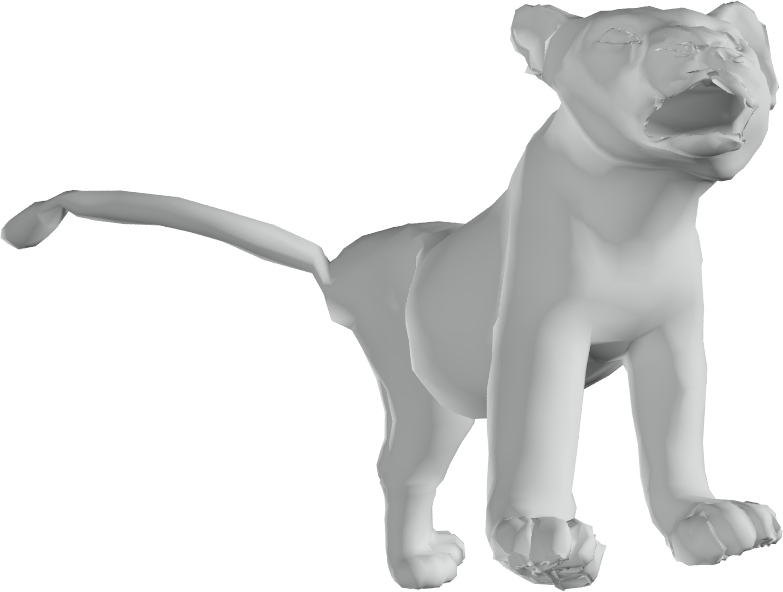}}{0.65}{1}{1}{0.65}&
\includegraphics[width=0.15\textwidth]{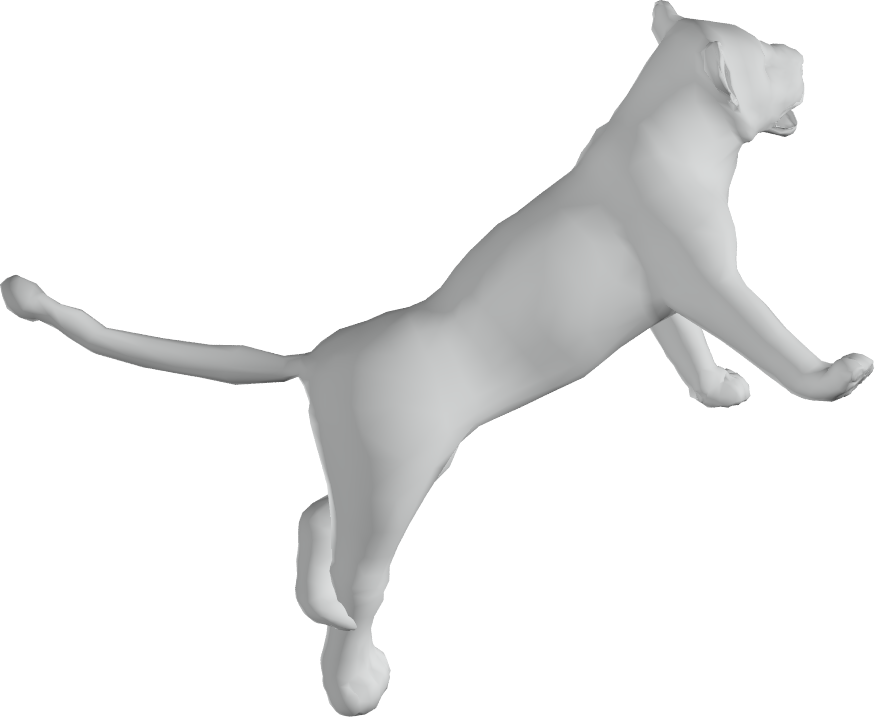}&
\includegraphics[width=0.15\textwidth]{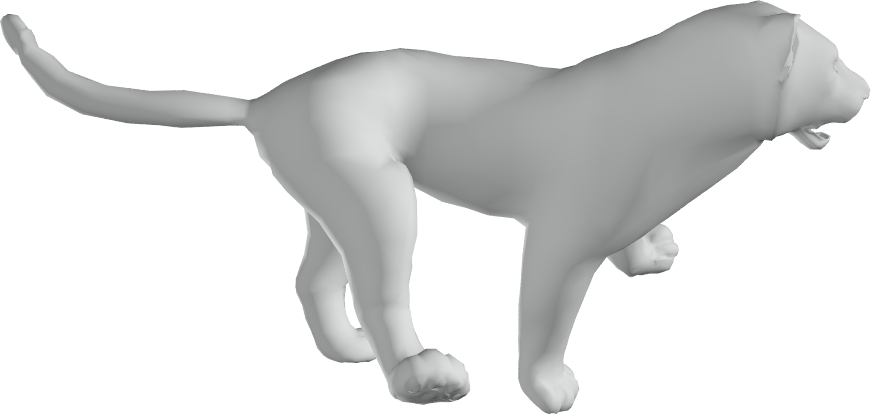}&
\includegraphics[width=0.15\textwidth]{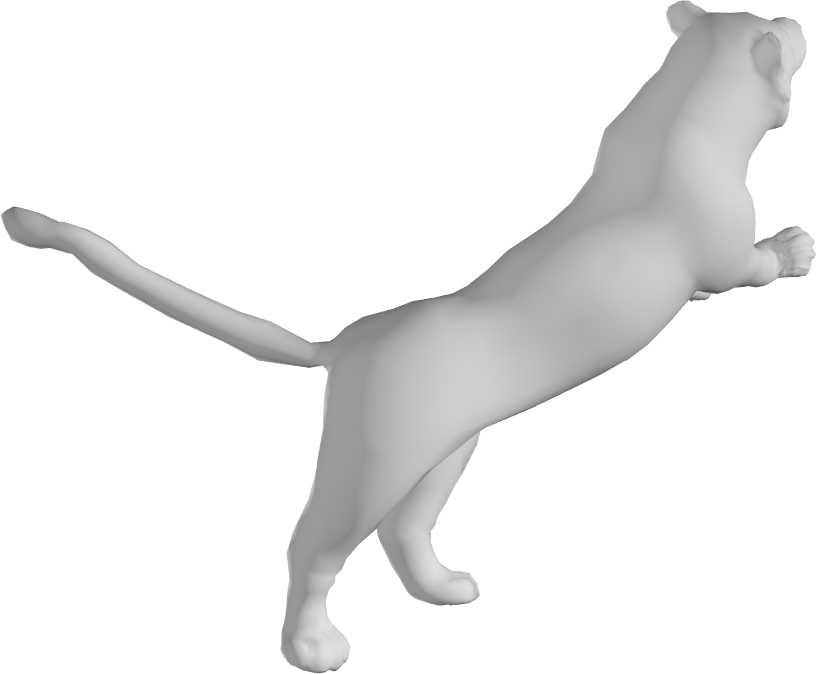}&
\includegraphics[width=0.15\textwidth]{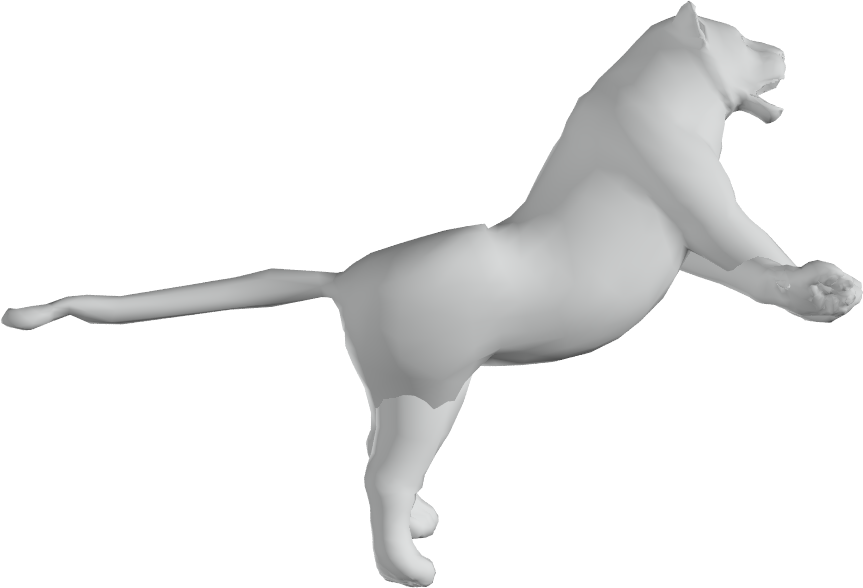}&
\includegraphics[width=0.15\textwidth]{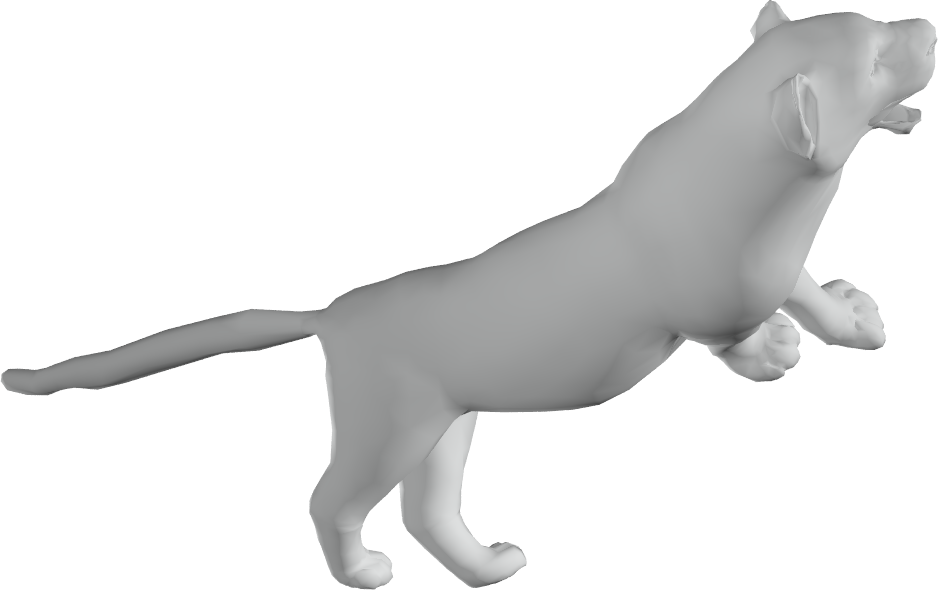}\\
\includegraphics[width=0.15\textwidth]{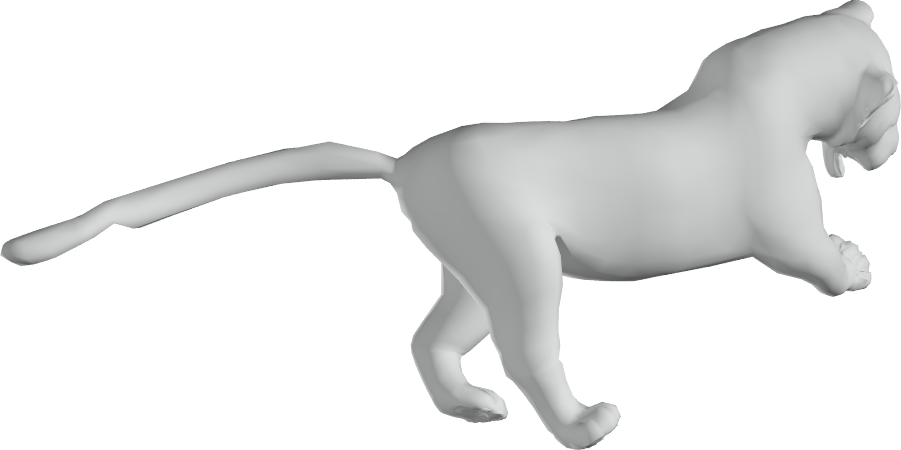}&
\includegraphics[width=0.15\textwidth]{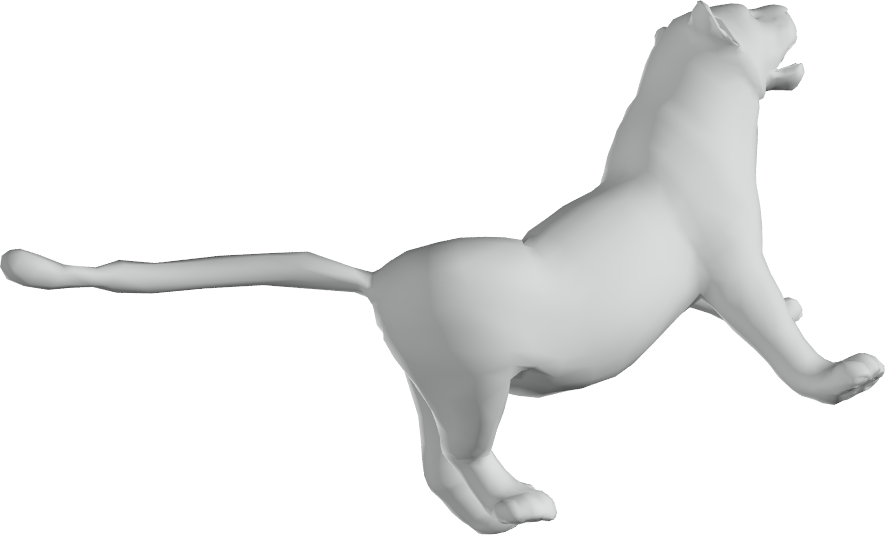}&
\includegraphics[width=0.15\textwidth]{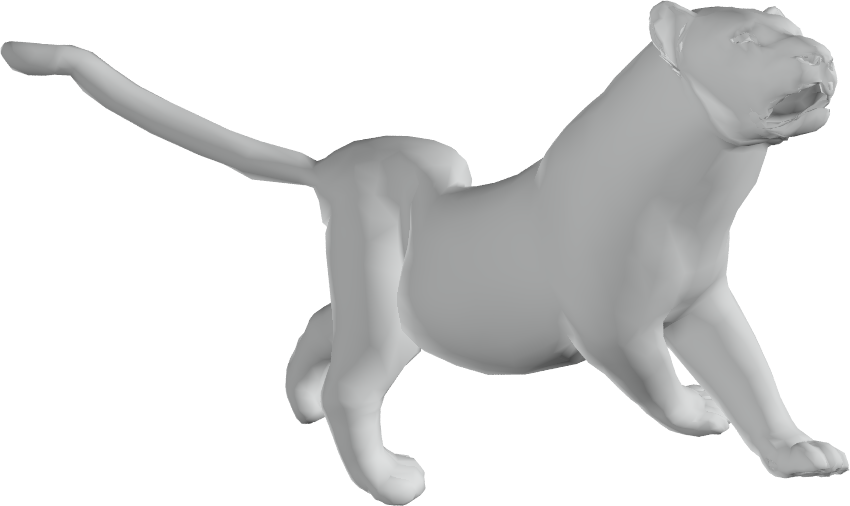}&
\includegraphics[width=0.15\textwidth]{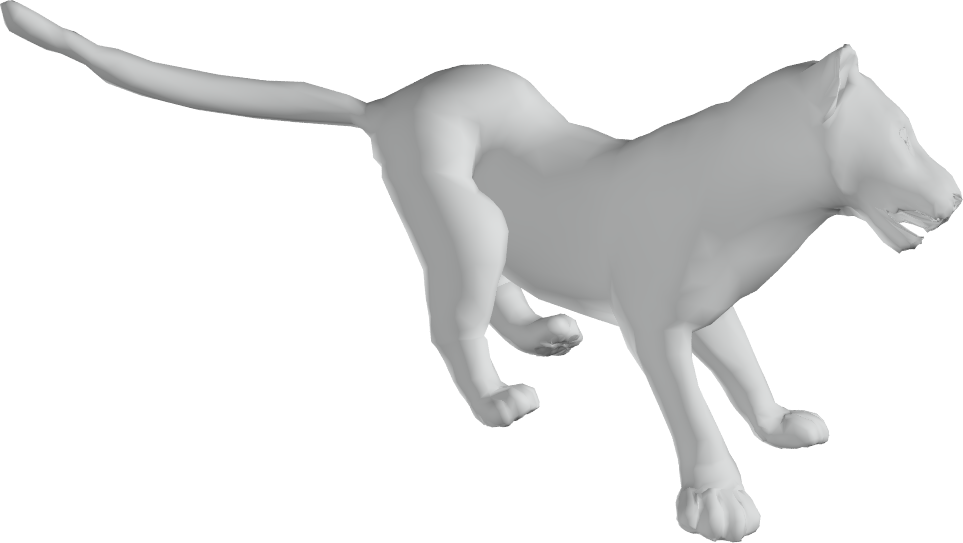}&
\includegraphics[width=0.15\textwidth]{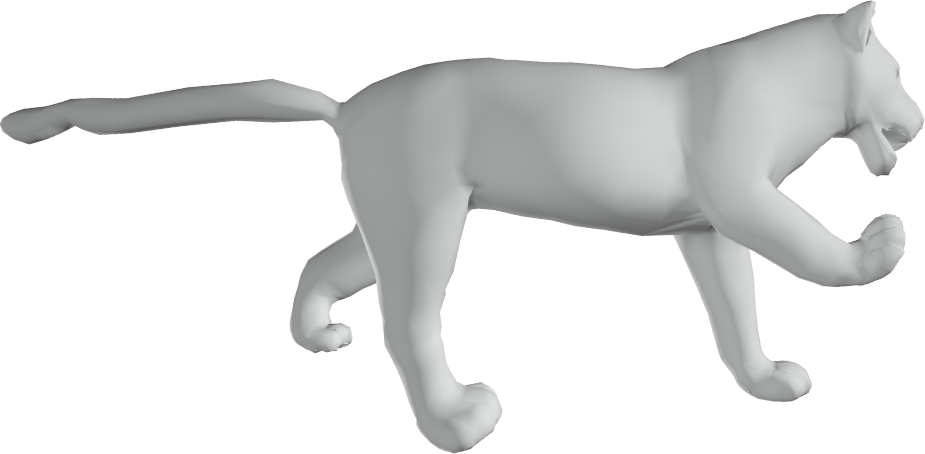}&
\includegraphics[width=0.15\textwidth]{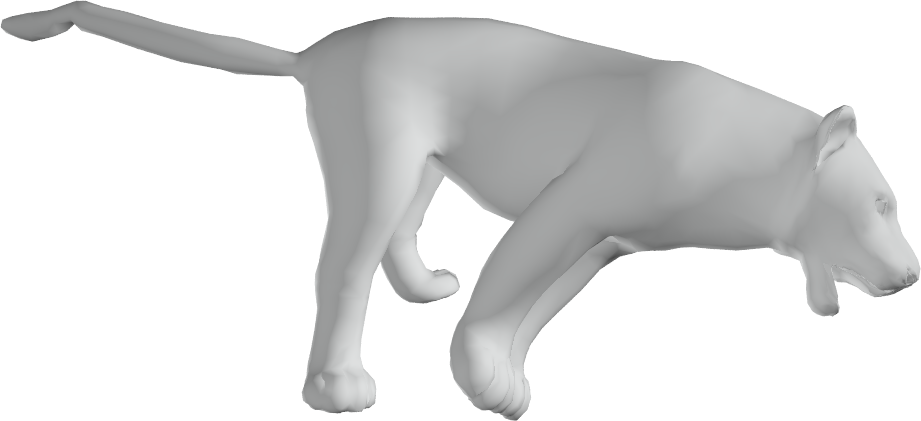}\\
\includegraphics[width=0.15\textwidth]{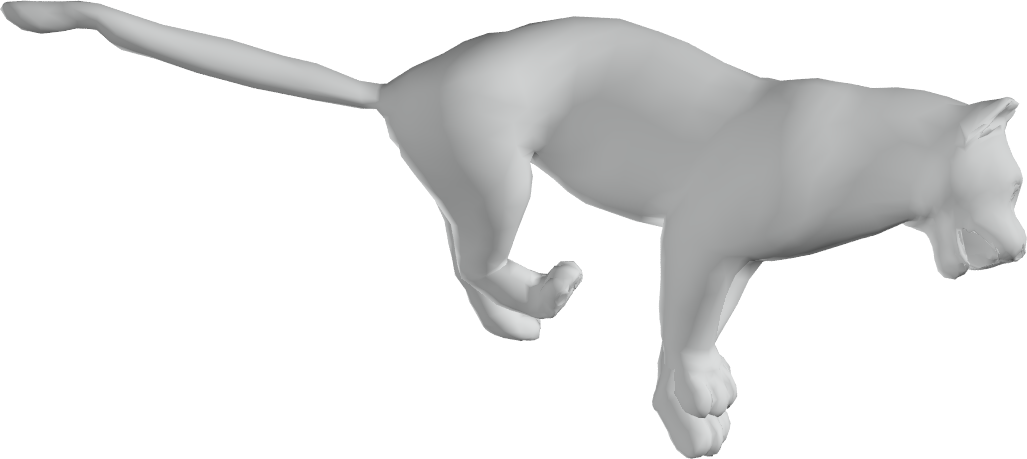}&
\includegraphics[width=0.15\textwidth]{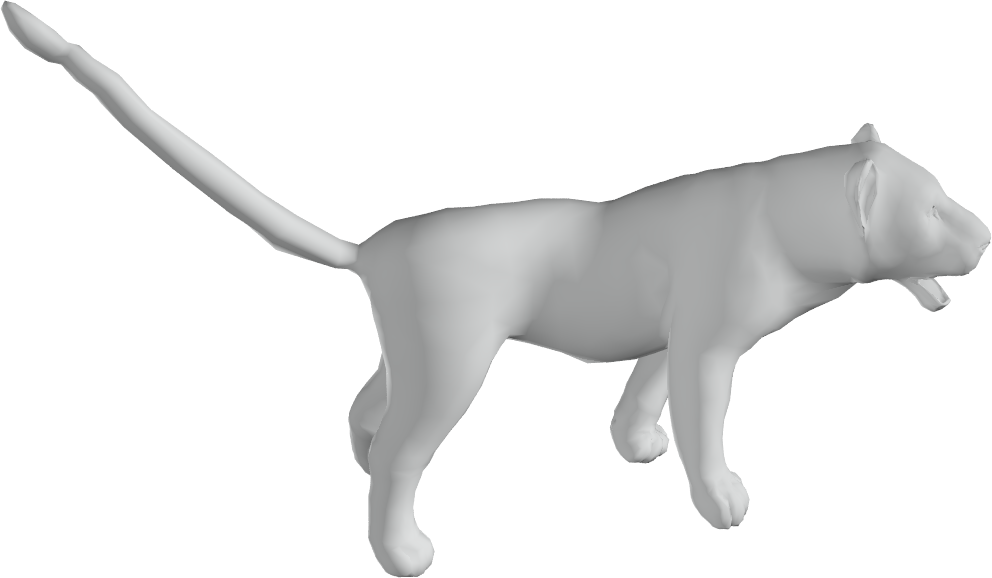}&
\includegraphics[width=0.15\textwidth]{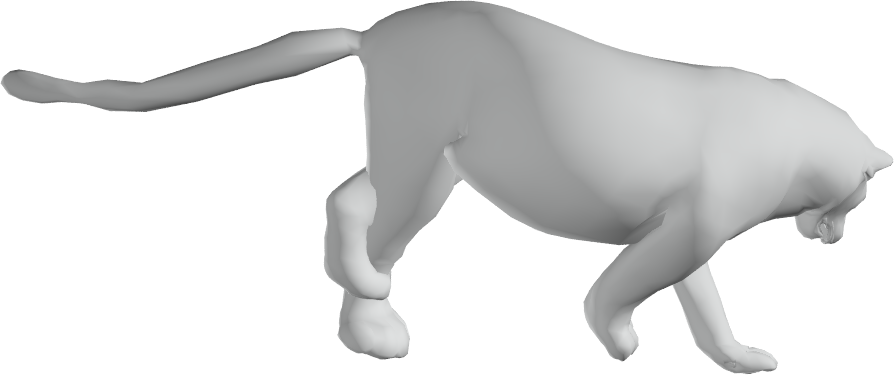}&
\includegraphics[width=0.15\textwidth]{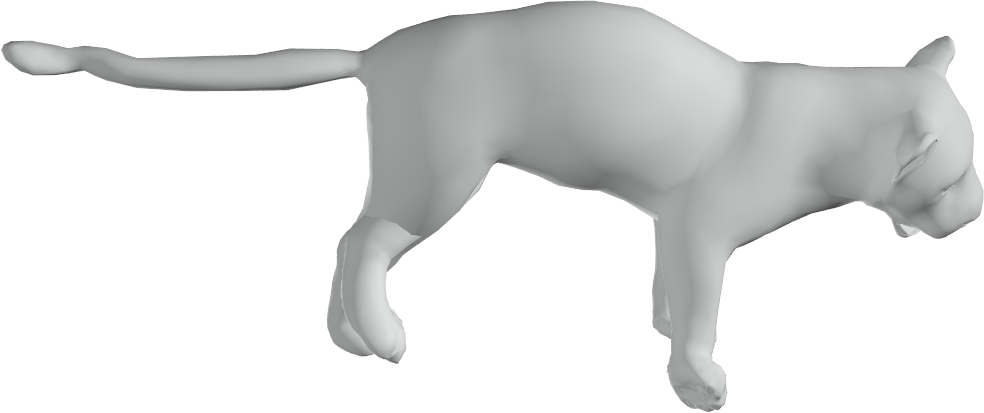}&
\includegraphics[width=0.15\textwidth]{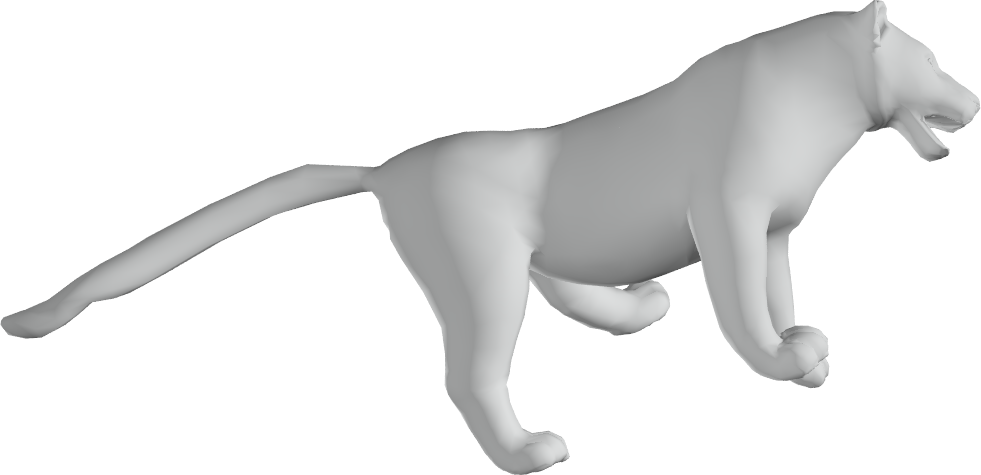}&
\redbox{\includegraphics[width=0.15\textwidth]{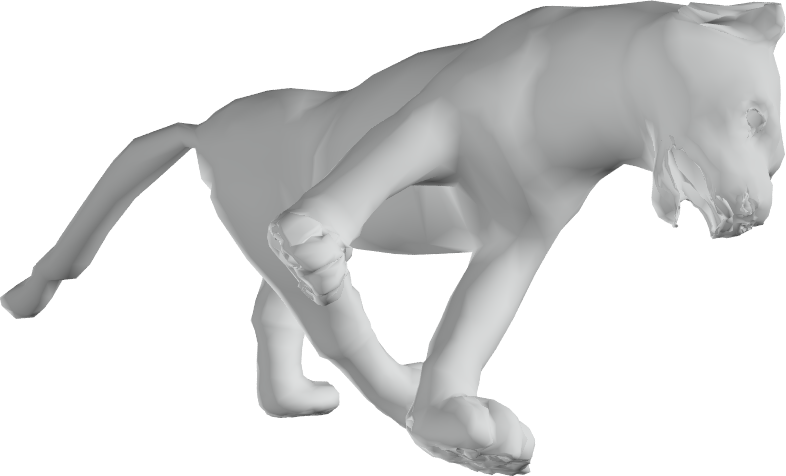}}\\
\includegraphics[width=0.15\textwidth]{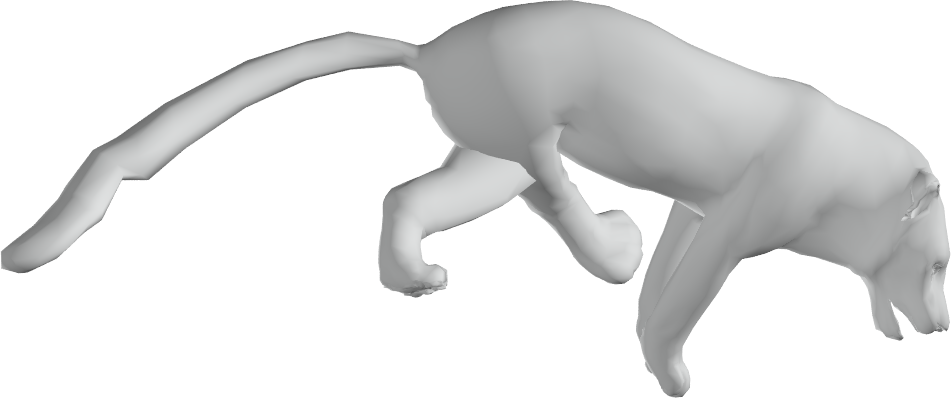}&
\includegraphics[width=0.15\textwidth]{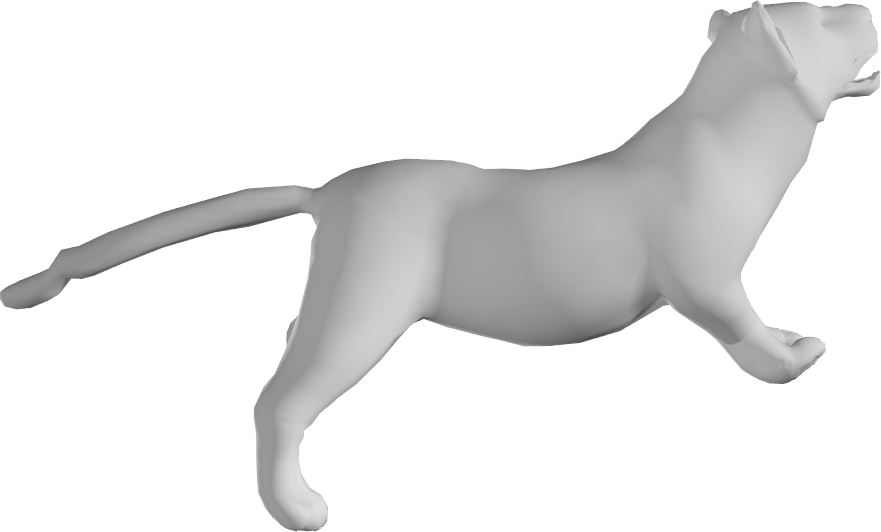}&
\includegraphics[width=0.15\textwidth]{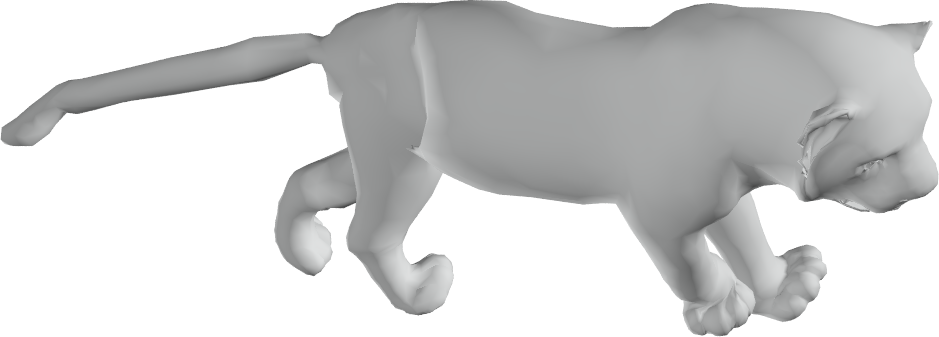}&
\includegraphics[width=0.15\textwidth]{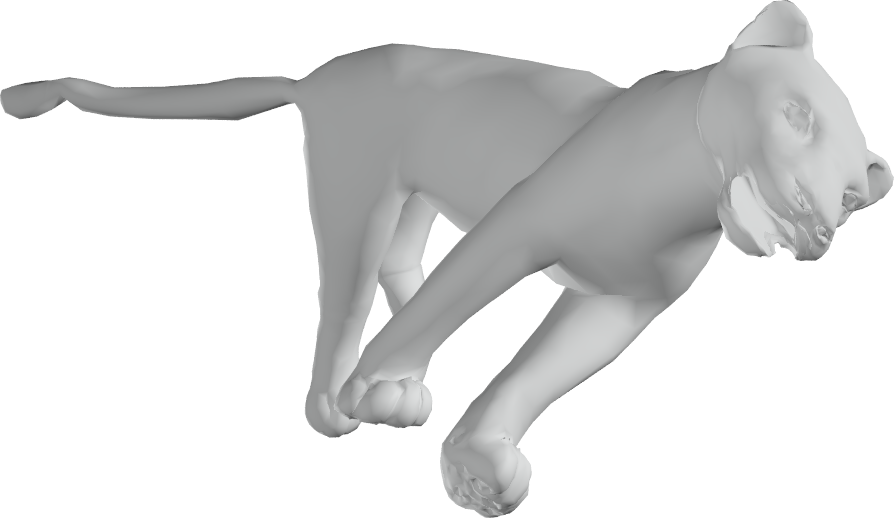}&
\includegraphics[width=0.15\textwidth]{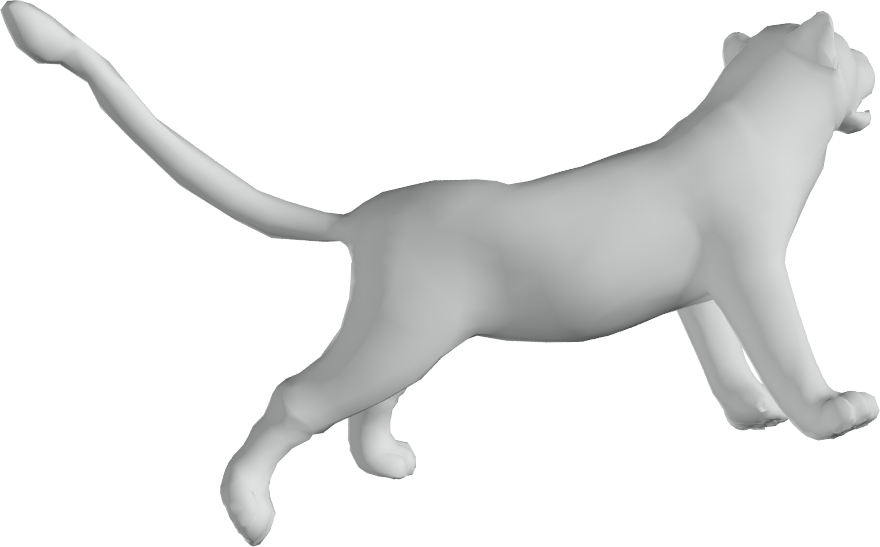}&
\includegraphics[width=0.15\textwidth]{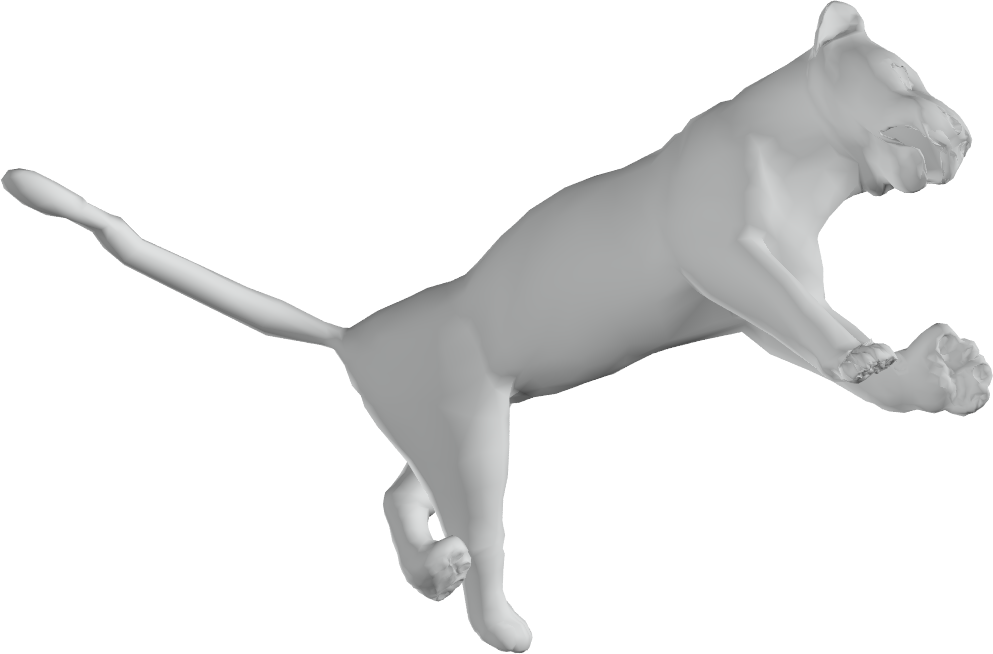}
\end{tabular}
\caption{FrameAVE results. Similar to other baseline approaches, some of the generated shapes have significantly distorted geometries. Most of shapes have noticeable local distortions. }
\label{Figure:SMAL:Mesh:FrameAVE}    
\end{figure*}

\begin{figure*}
\setlength\tabcolsep{4pt}
\begin{tabular}{cccccc}
\includegraphics[width=0.15\textwidth]{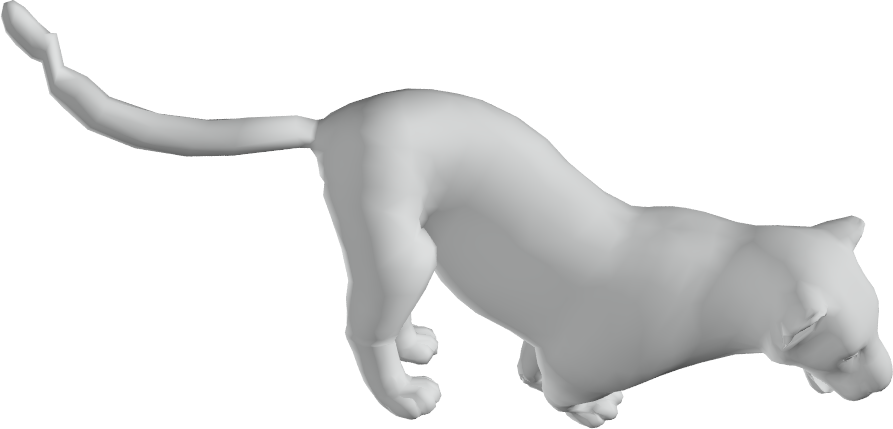}&
\includegraphics[width=0.15\textwidth]{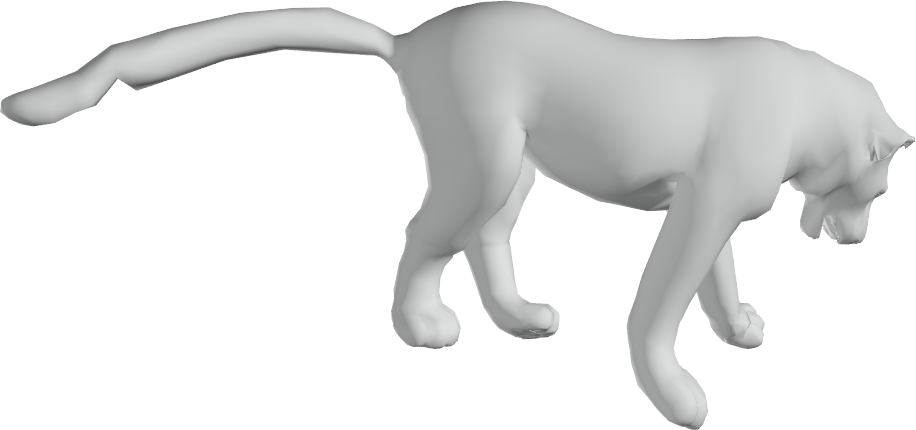}&
\includegraphics[width=0.15\textwidth]{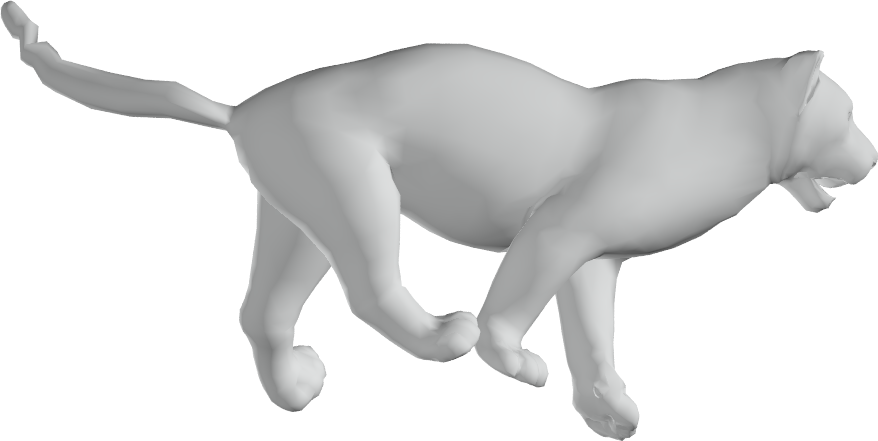}&
\includegraphics[width=0.15\textwidth]{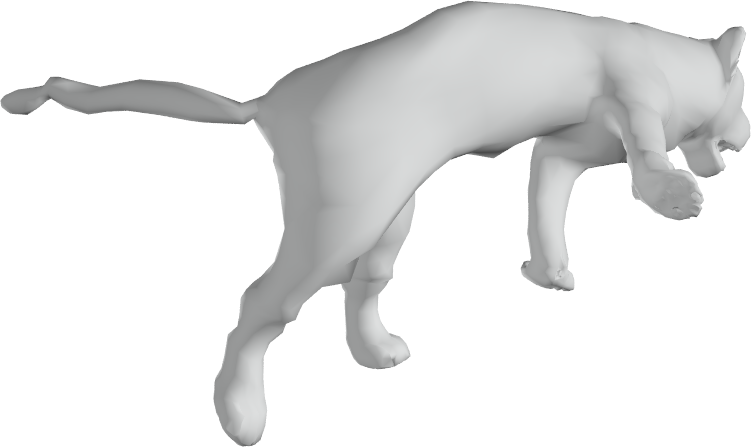}&
\includegraphics[width=0.15\textwidth]{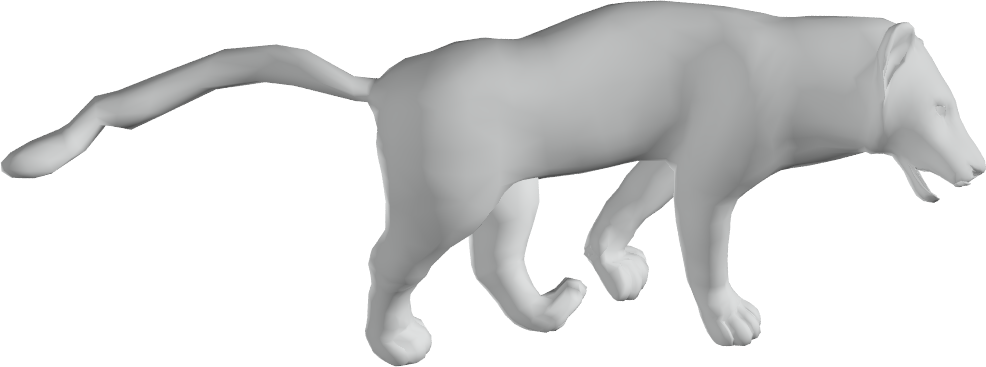}&
\bluepart{\includegraphics[width=0.15\textwidth]{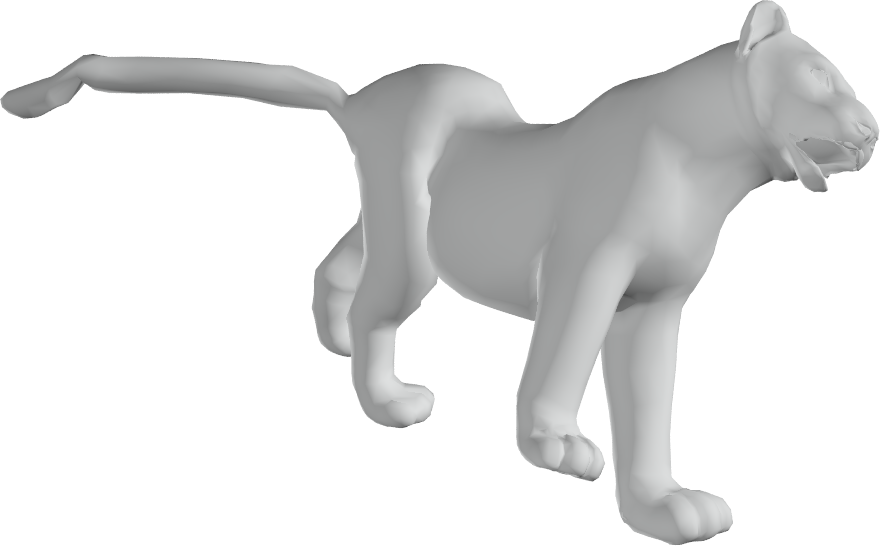}}{0.4}{0.85}{0.6}{0.4}\\
\includegraphics[width=0.15\textwidth]{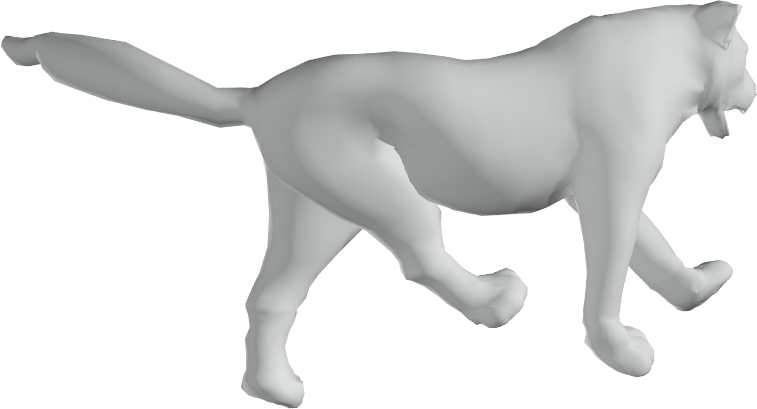}&
\includegraphics[width=0.15\textwidth]{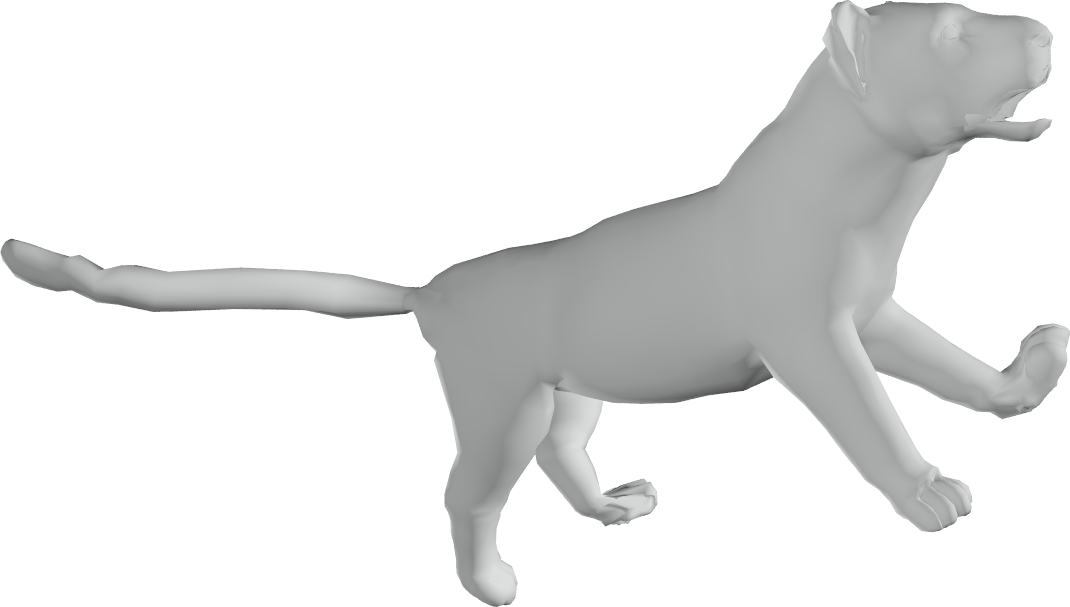}&
\includegraphics[width=0.15\textwidth]{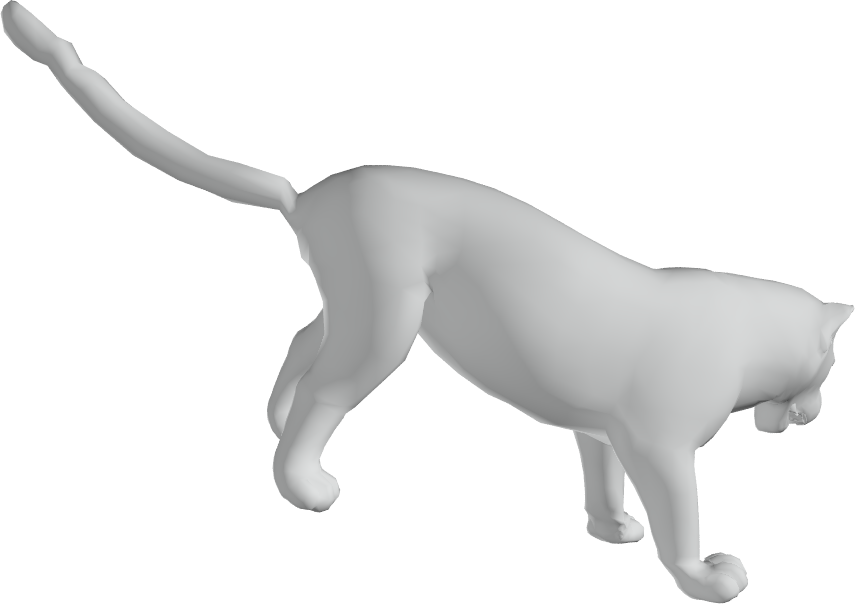}&
\includegraphics[width=0.15\textwidth]{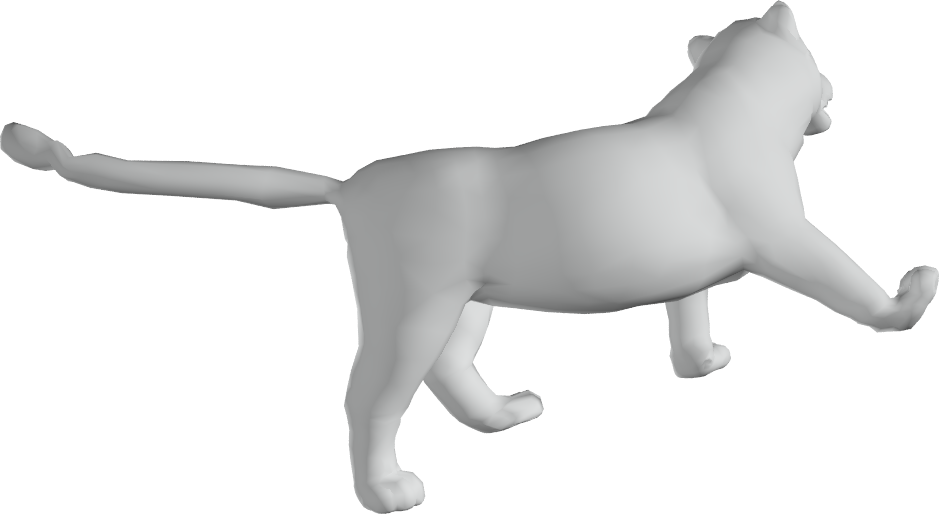}&
\includegraphics[width=0.15\textwidth]{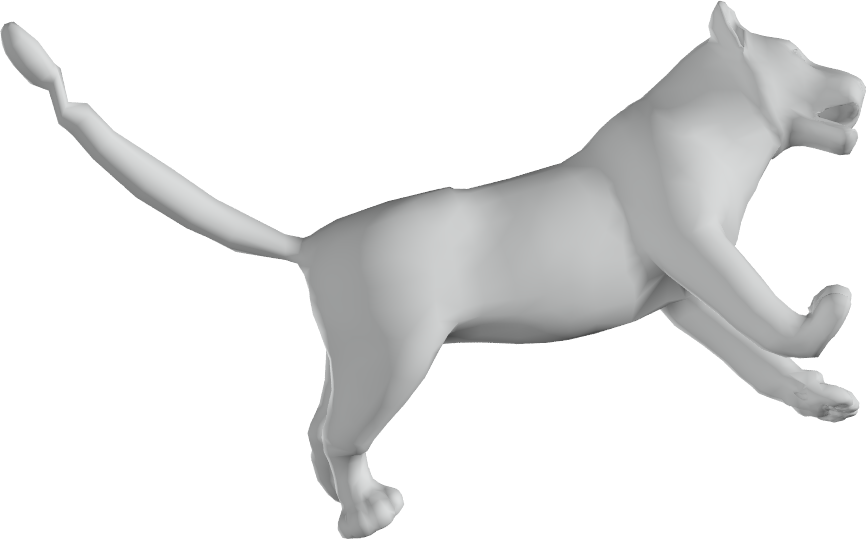}&
\includegraphics[width=0.15\textwidth]{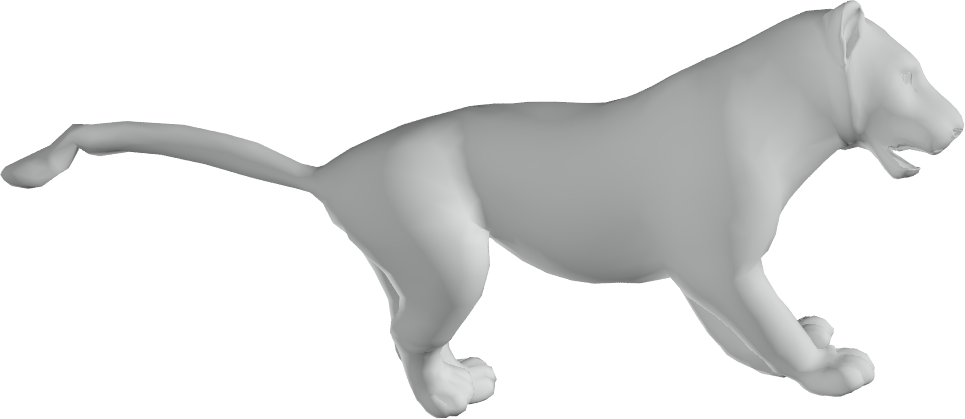}\\
\includegraphics[width=0.15\textwidth]{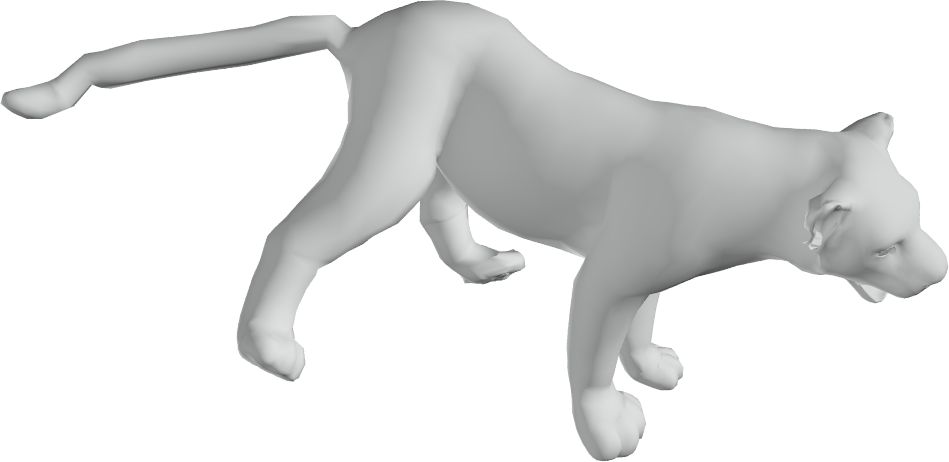}&
\includegraphics[width=0.15\textwidth]{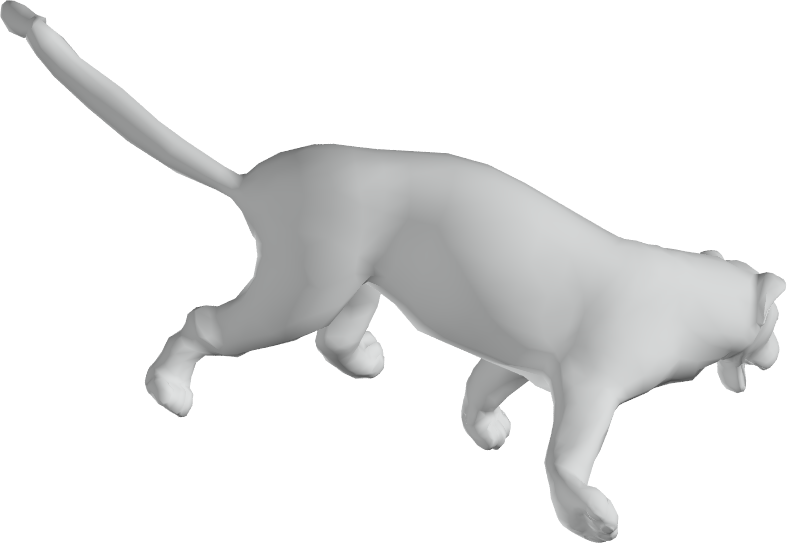}&
\bluepart{\includegraphics[width=0.15\textwidth]{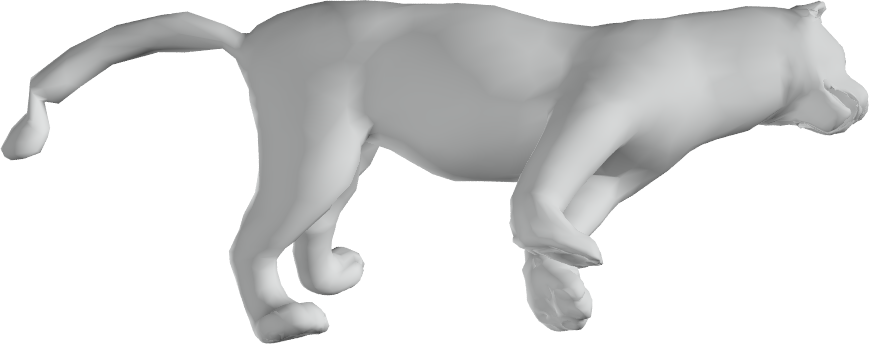}}{-0.02}{0.45}{0.1}{0.9}&
\includegraphics[width=0.15\textwidth]{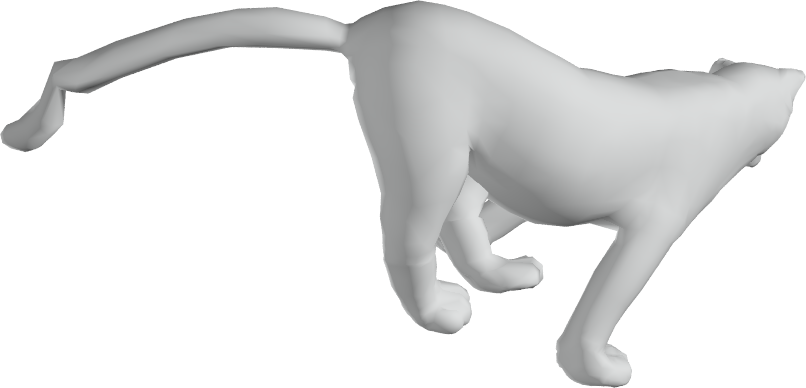}&
\includegraphics[width=0.15\textwidth]{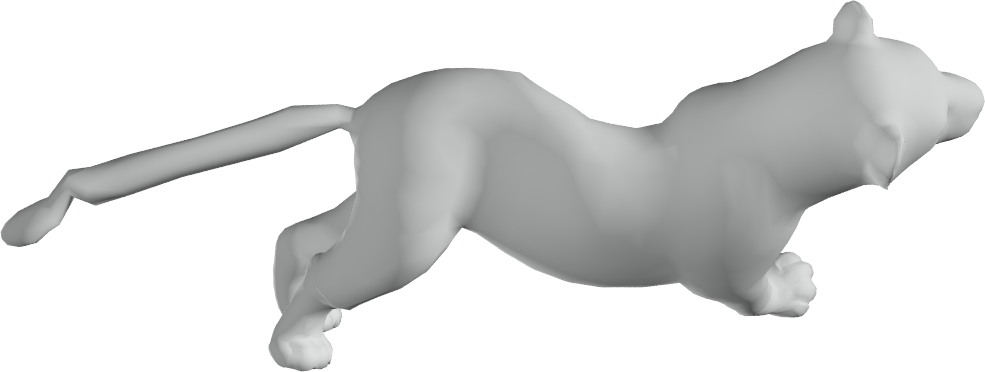}&
\includegraphics[width=0.15\textwidth]{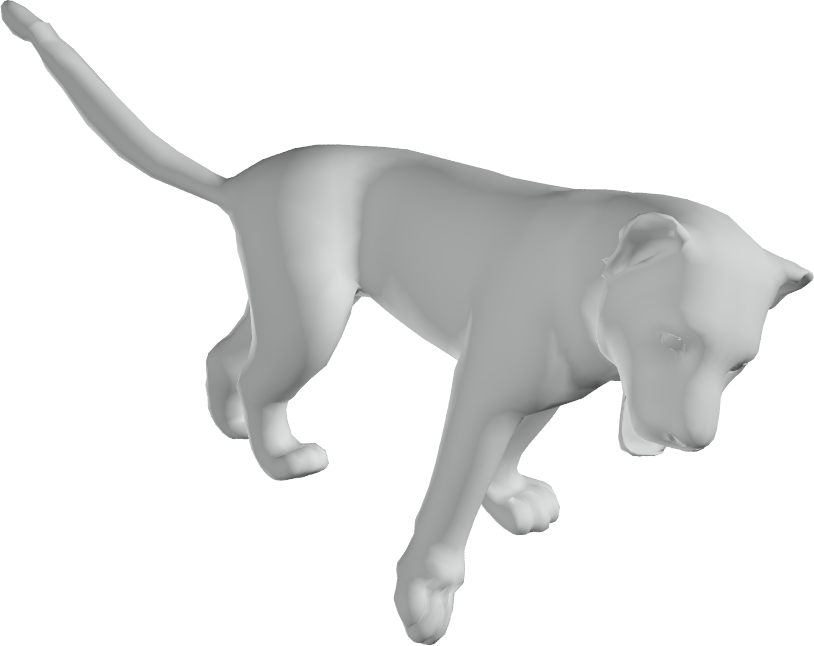}\\
\includegraphics[width=0.15\textwidth]{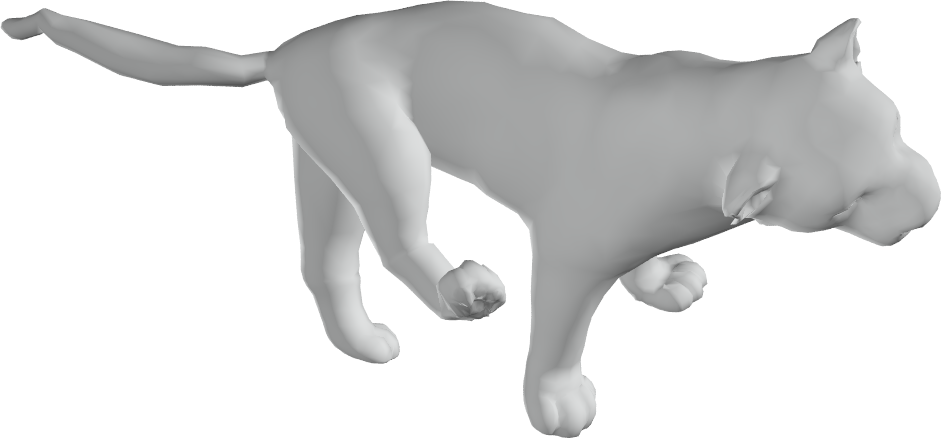}&
\includegraphics[width=0.15\textwidth]{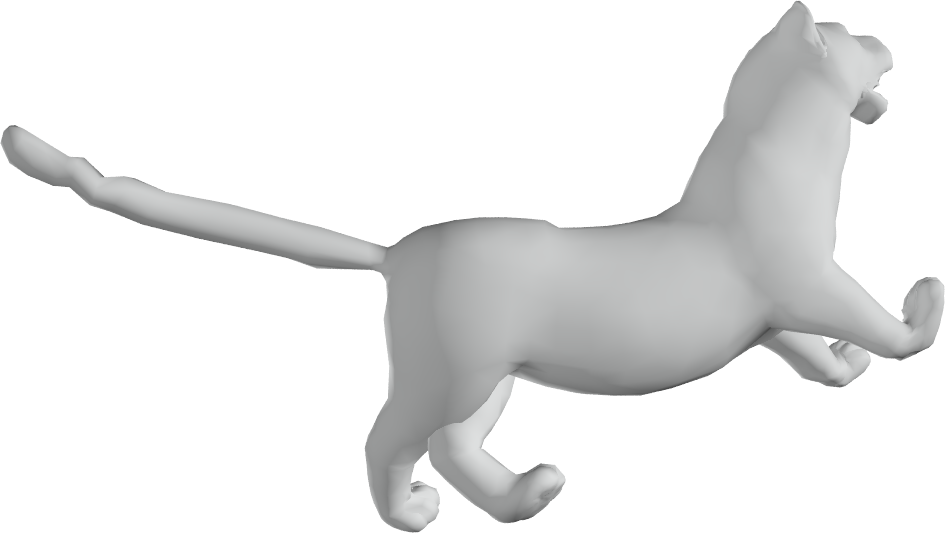}&
\redbox{\includegraphics[width=0.15\textwidth]{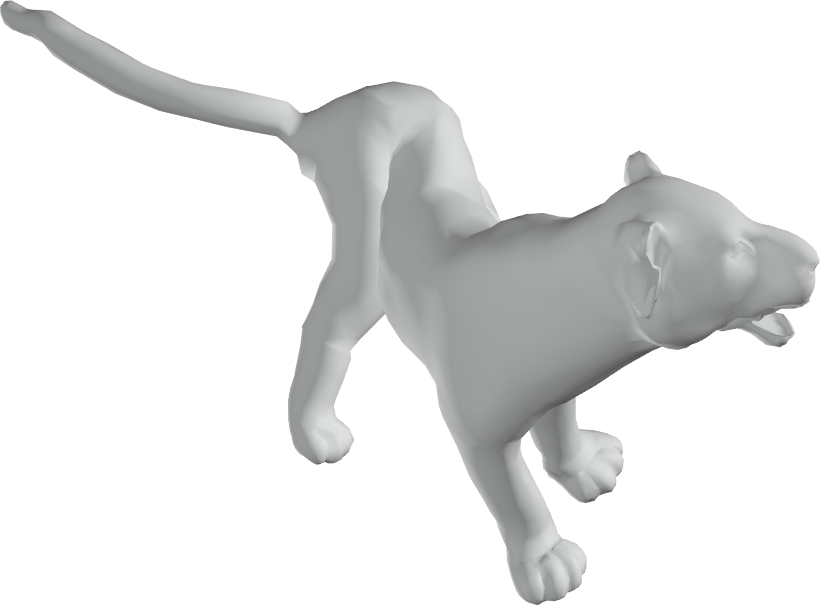}}&
\includegraphics[width=0.15\textwidth]{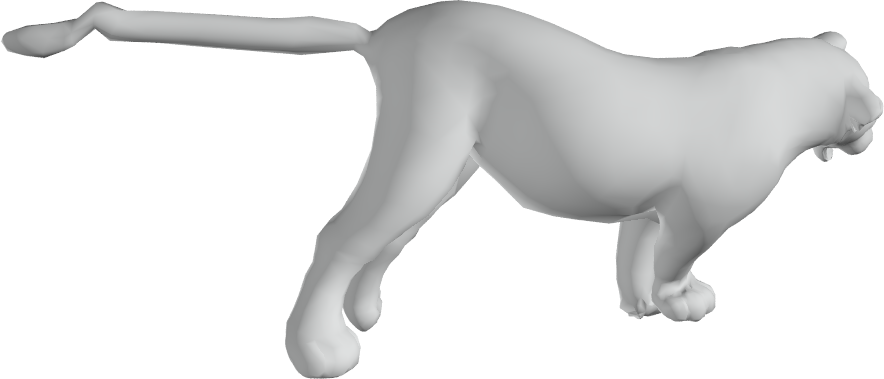}&
\includegraphics[width=0.15\textwidth]{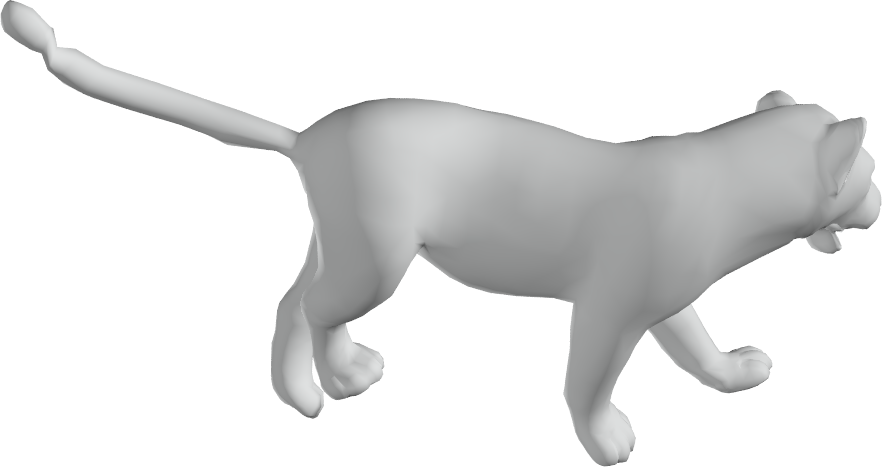}&
\redbox{\includegraphics[width=0.15\textwidth]{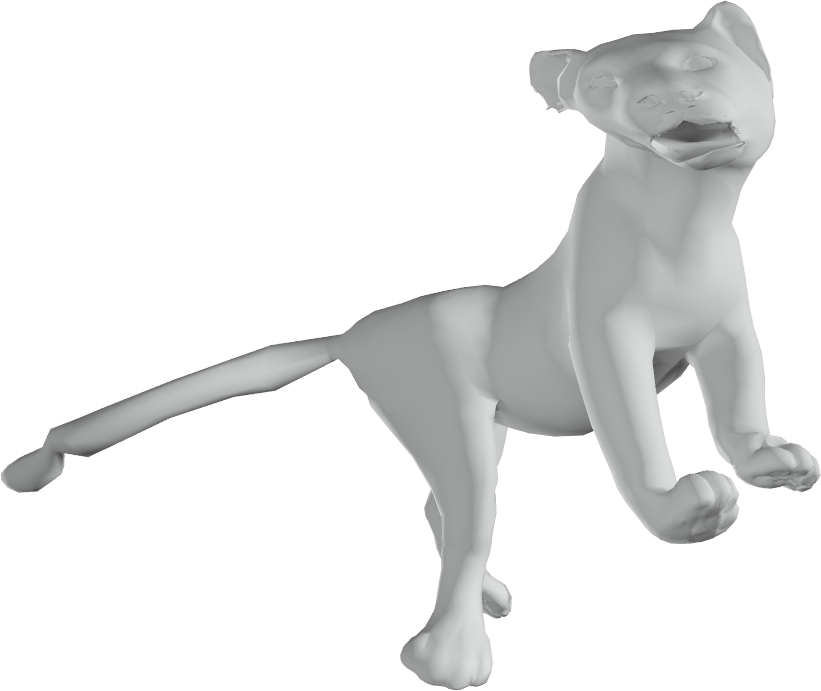}}
\end{tabular}
\caption{GeoLatent results. Similar to other baseline approaches, some of the generated shapes have significantly distorted geometries. Most of shapes have noticeable local distortions. }
\label{Figure:SMAL:Mesh:GeoLatent}    
\end{figure*}

\begin{figure*}
\setlength\tabcolsep{4pt}
\begin{tabular}{cccccc}
\includegraphics[width=0.15\textwidth]{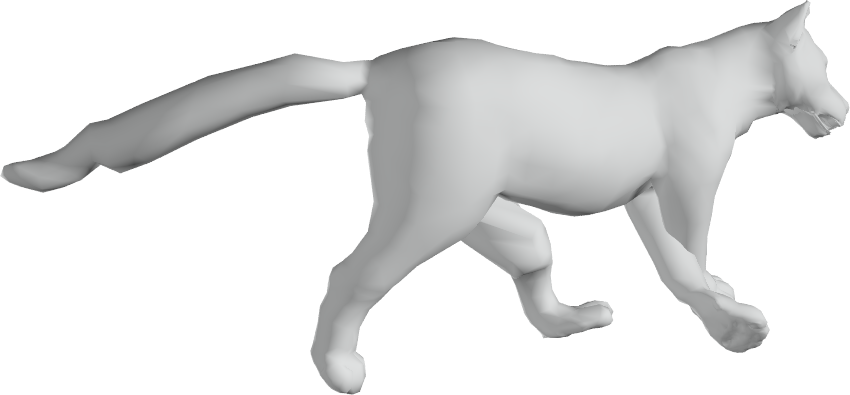}&
\includegraphics[width=0.15\textwidth]{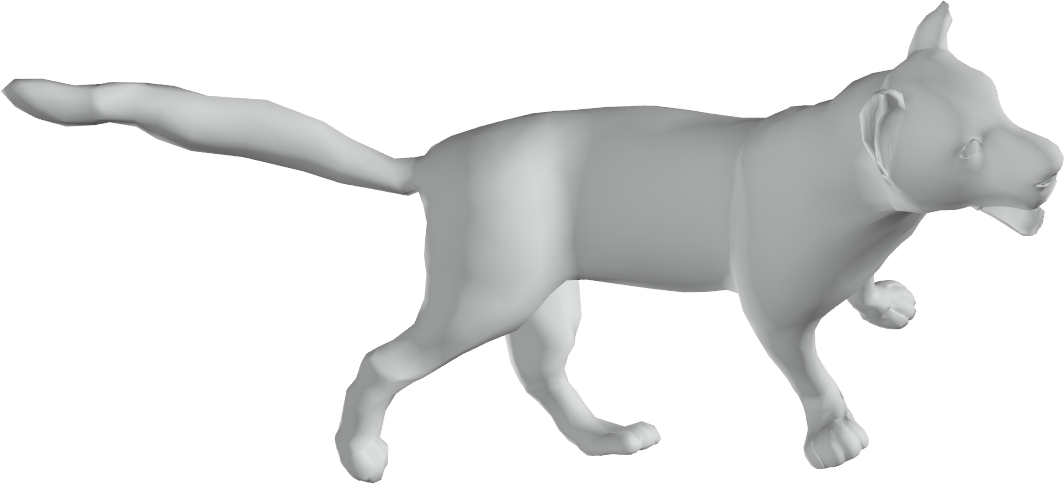}&
\bluepart{\includegraphics[width=0.15\textwidth]{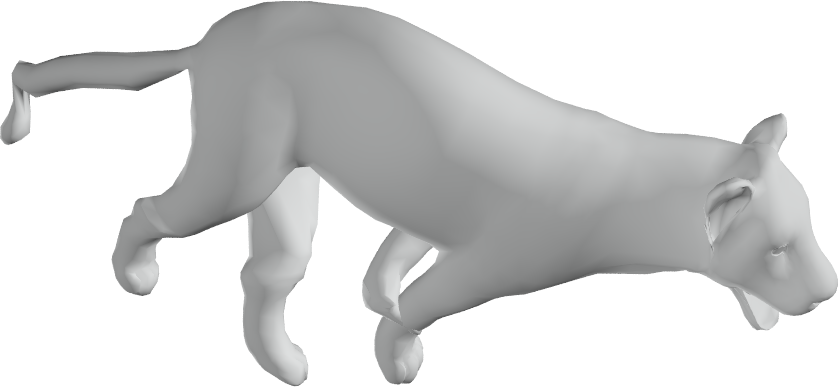}}{-0.02}{0.5}{0.1}{0.9}&
\includegraphics[width=0.15\textwidth]{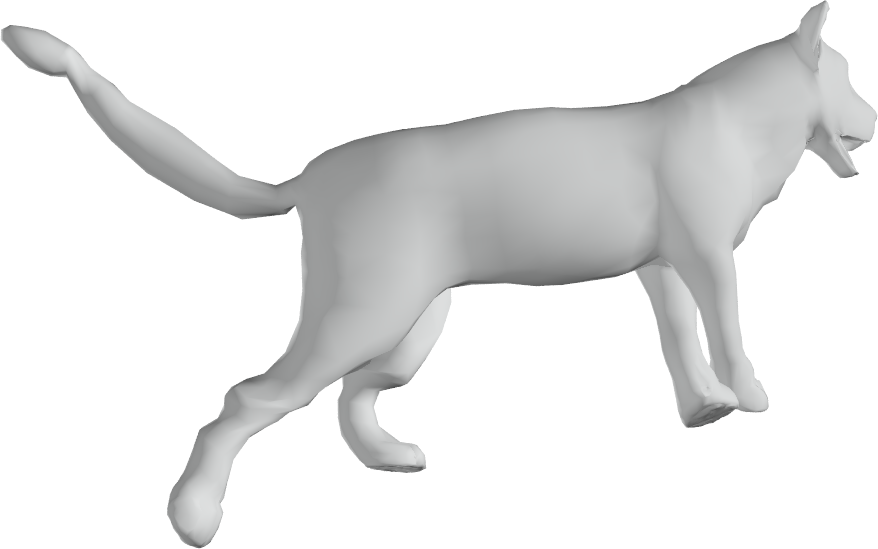}&
\includegraphics[width=0.15\textwidth]{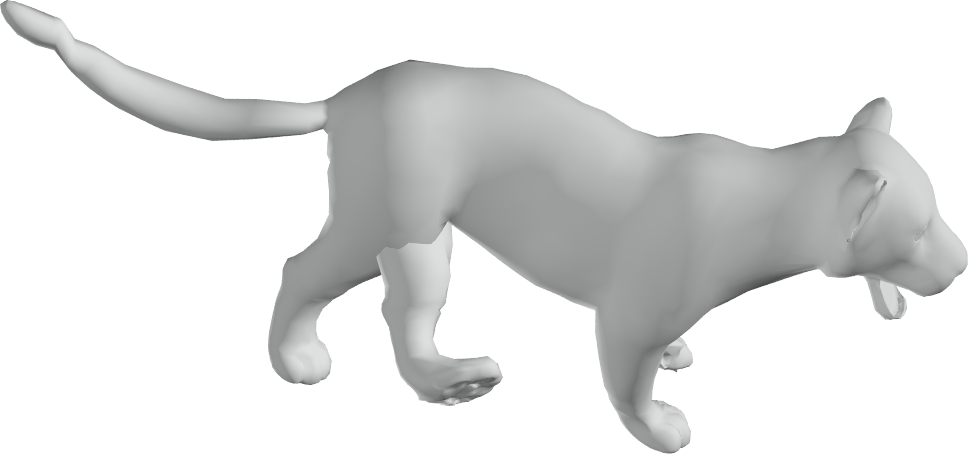}&
\includegraphics[width=0.15\textwidth]{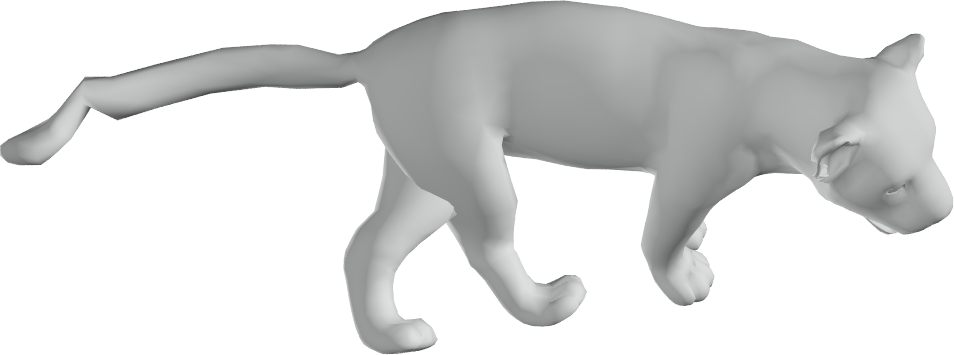}\\
\includegraphics[width=0.15\textwidth]{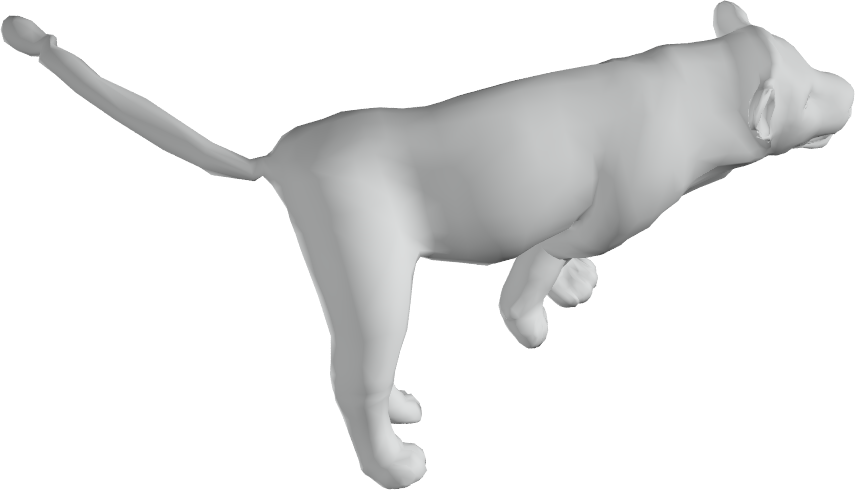}&
\includegraphics[width=0.15\textwidth]{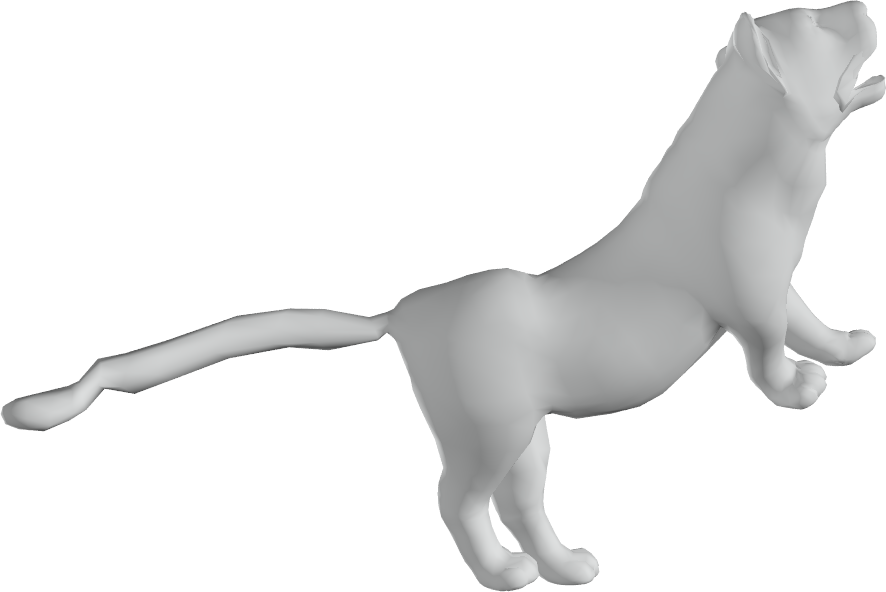}&
\includegraphics[width=0.15\textwidth]{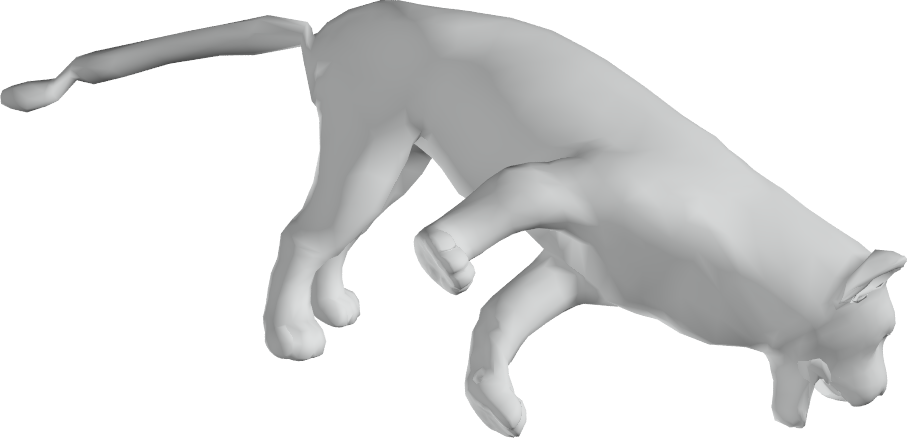}&
\includegraphics[width=0.15\textwidth]{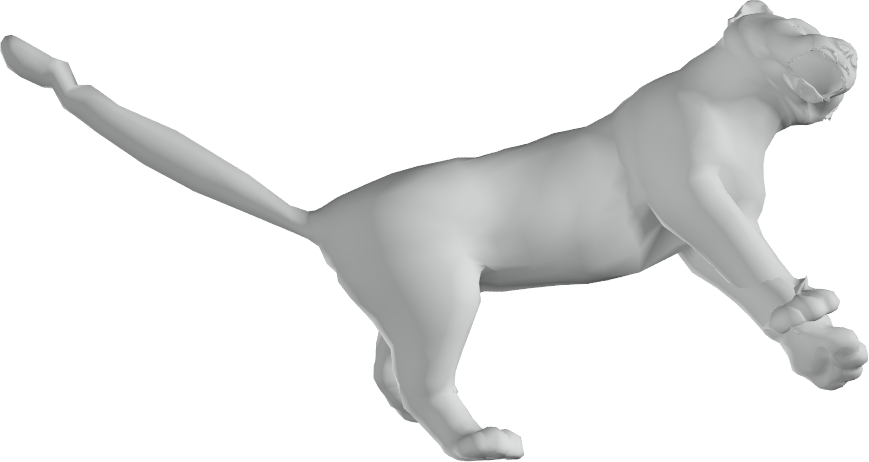}&
\includegraphics[width=0.15\textwidth]{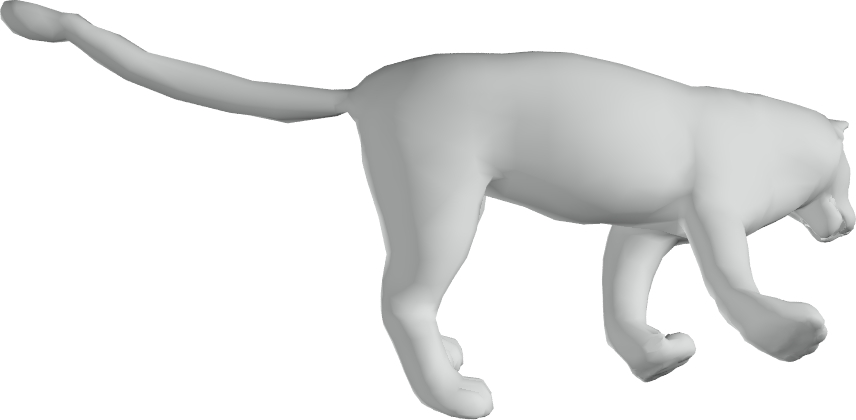}&
\includegraphics[width=0.15\textwidth]{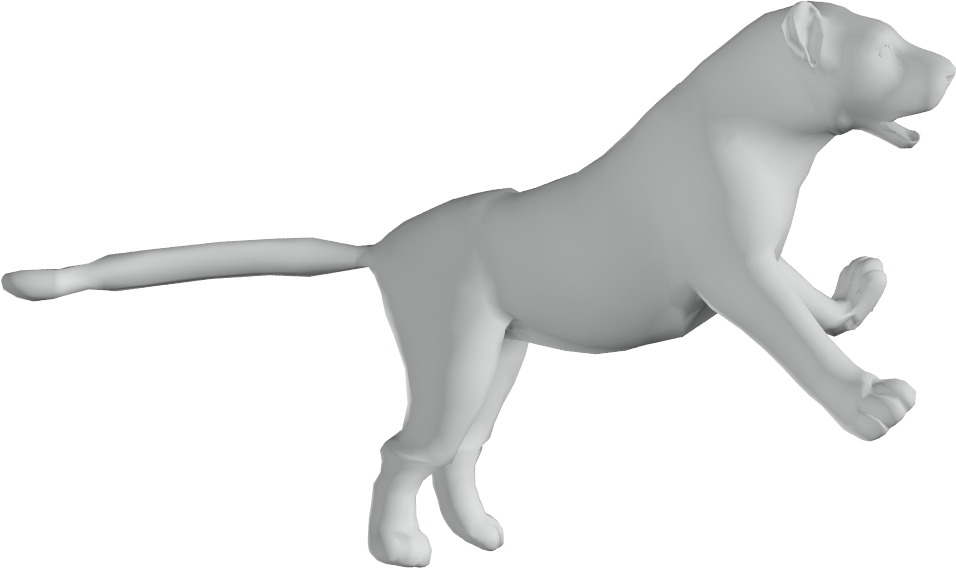}\\
\includegraphics[width=0.15\textwidth]{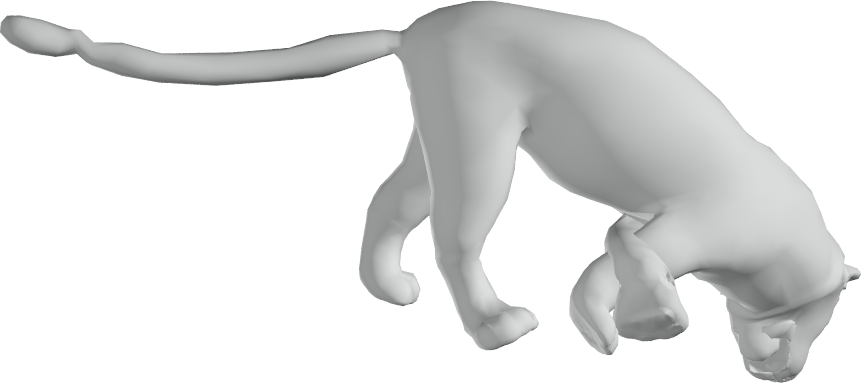}&
\includegraphics[width=0.15\textwidth]{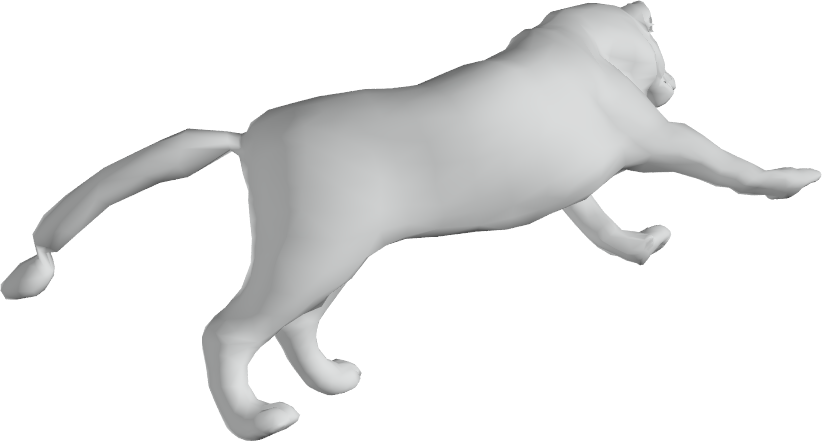}&
\includegraphics[width=0.15\textwidth]{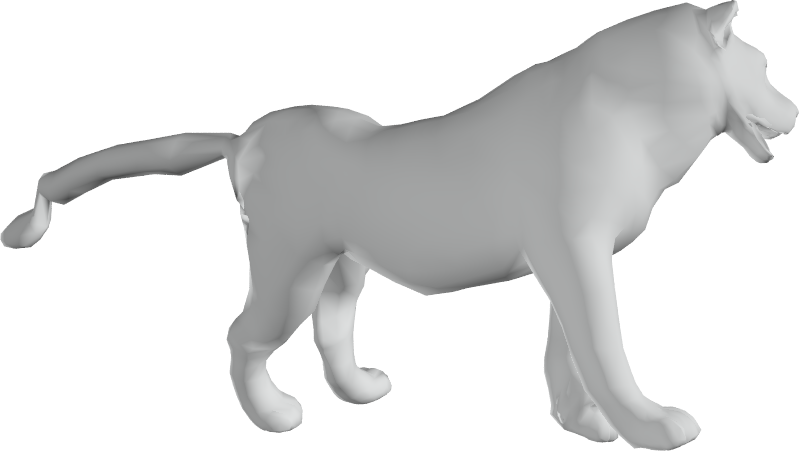}&
\includegraphics[width=0.15\textwidth]{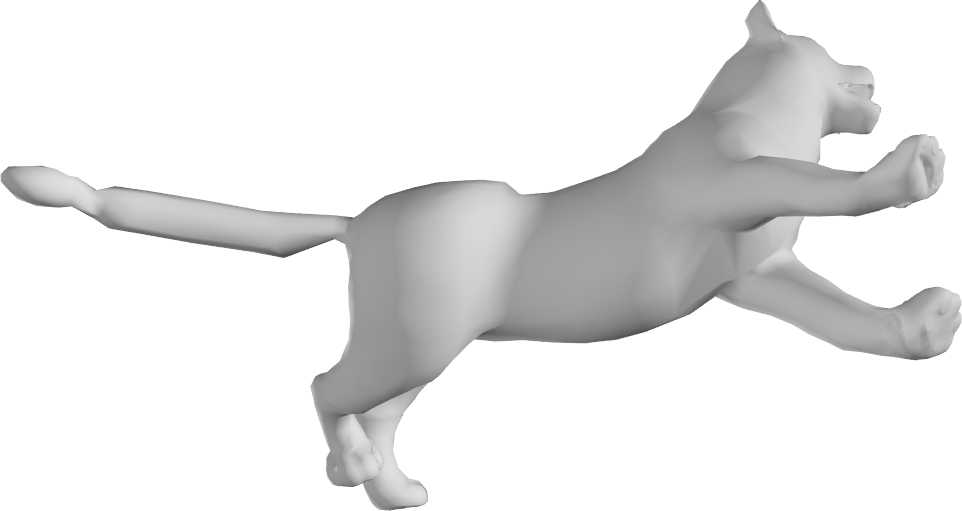}&
\includegraphics[width=0.15\textwidth]{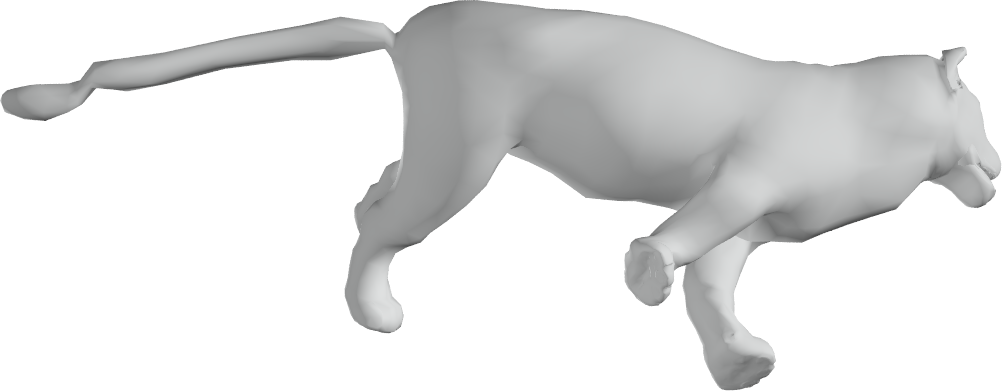}&
\bluepart{\includegraphics[width=0.15\textwidth]{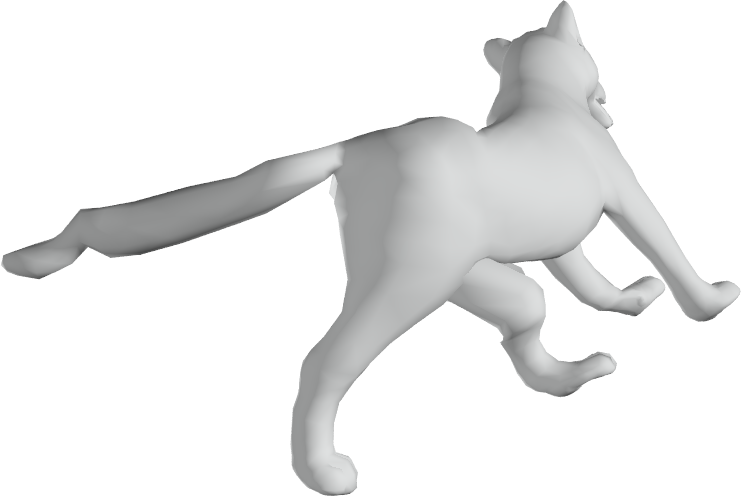}}{0.6}{0.15}{0.9}{0.4}\\
\includegraphics[width=0.15\textwidth]{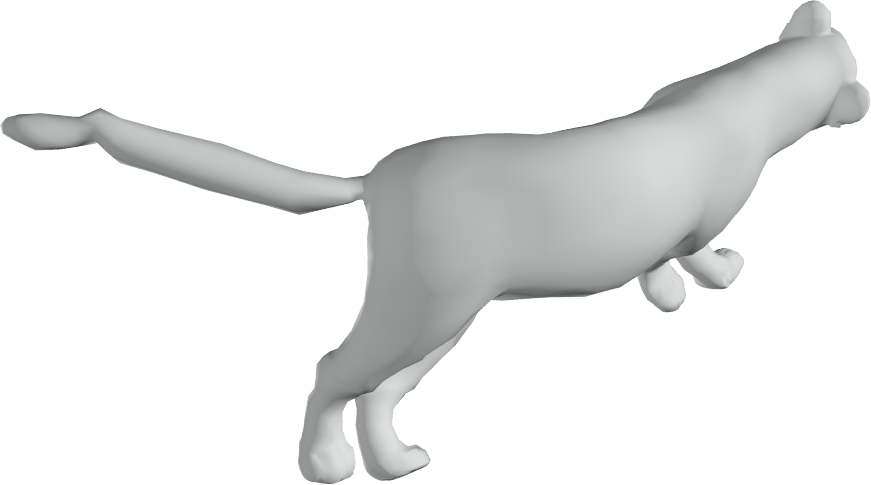}&
\includegraphics[width=0.15\textwidth]{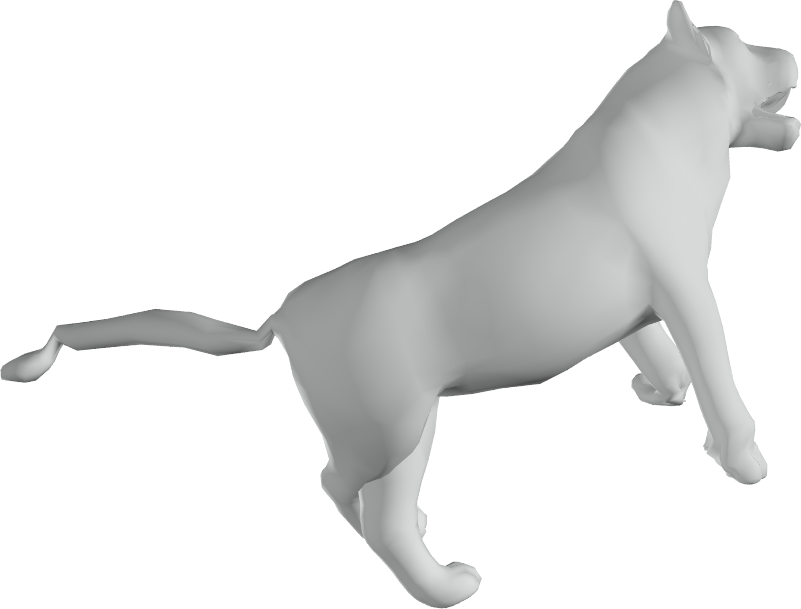}&
\redbox{\includegraphics[width=0.15\textwidth]{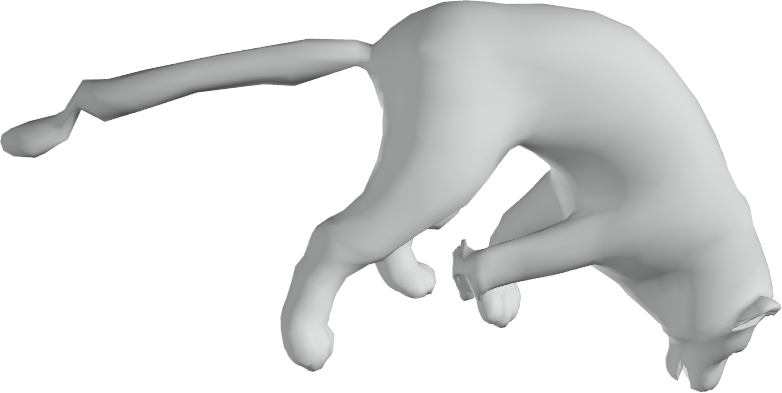}}&
\includegraphics[width=0.15\textwidth]{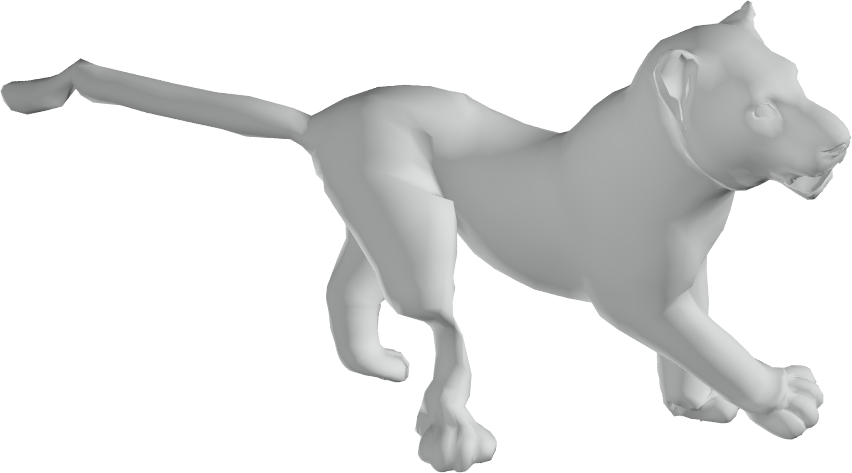}&
\includegraphics[width=0.15\textwidth]{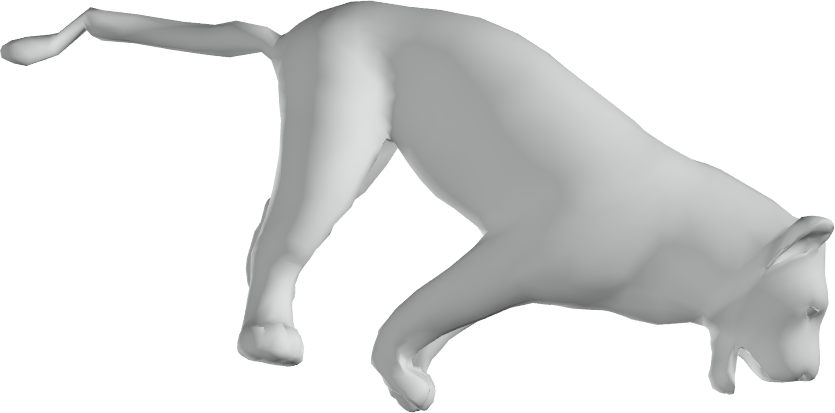}&
\redbox{\includegraphics[width=0.15\textwidth]{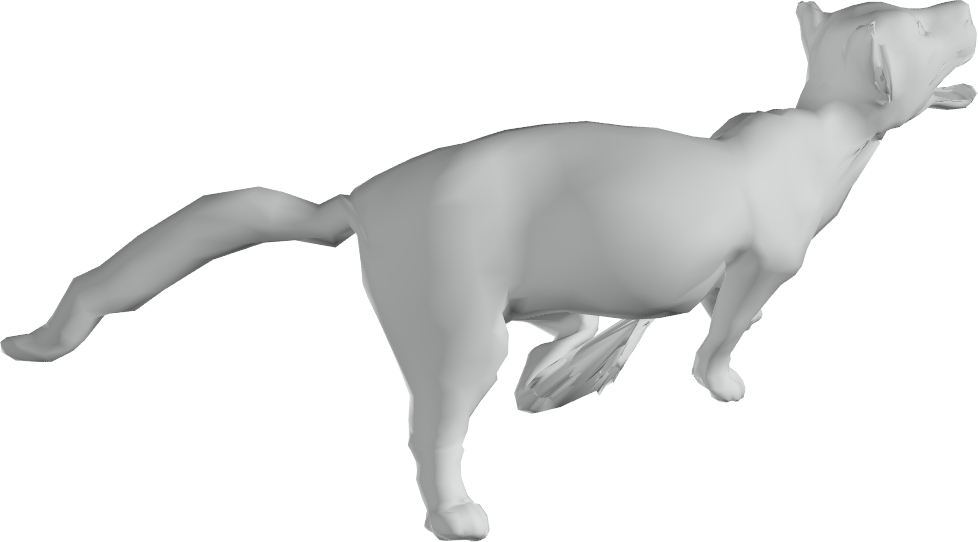}}
\end{tabular}
\caption{BREAS results. Similar to other baseline approaches, some of the generated shapes have significantly distorted geometries. Most of shapes have noticeable local distortions. }
\label{Figure:SMAL:Mesh:BRESA}    
\end{figure*}

\begin{figure*}
\setlength\tabcolsep{4pt}
\begin{tabular}{cccccc}
\includegraphics[width=0.15\textwidth]{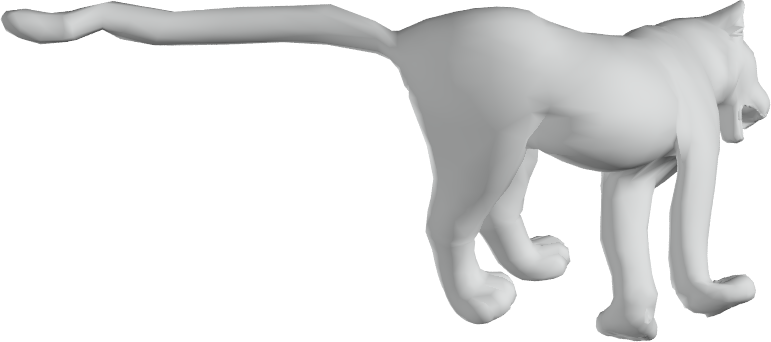}&
\includegraphics[width=0.15\textwidth]{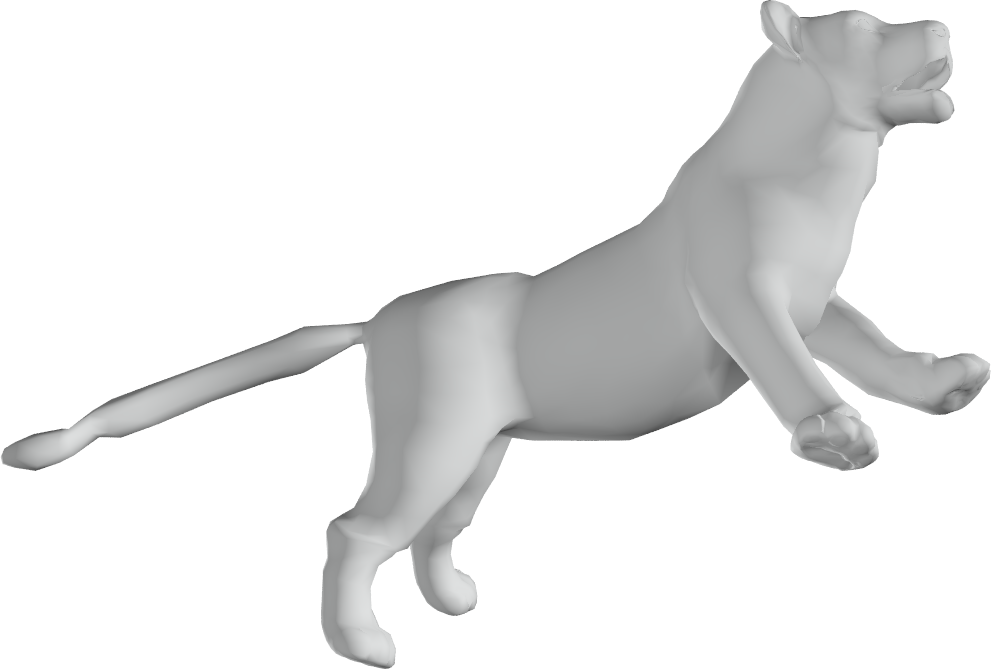}&
\includegraphics[width=0.15\textwidth]{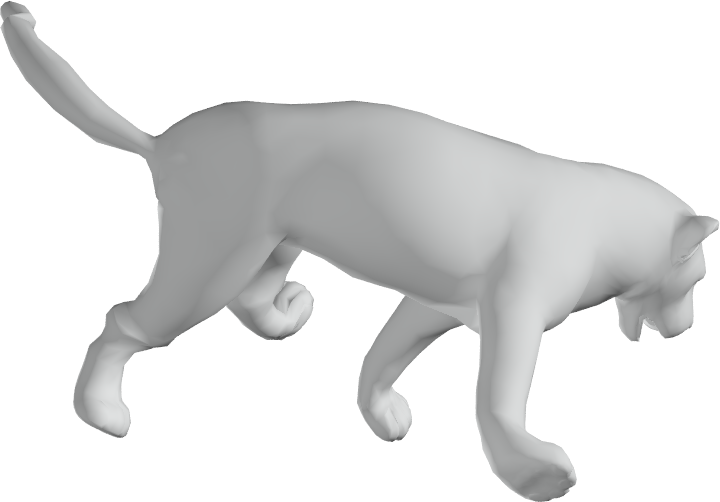}&
\includegraphics[width=0.15\textwidth]{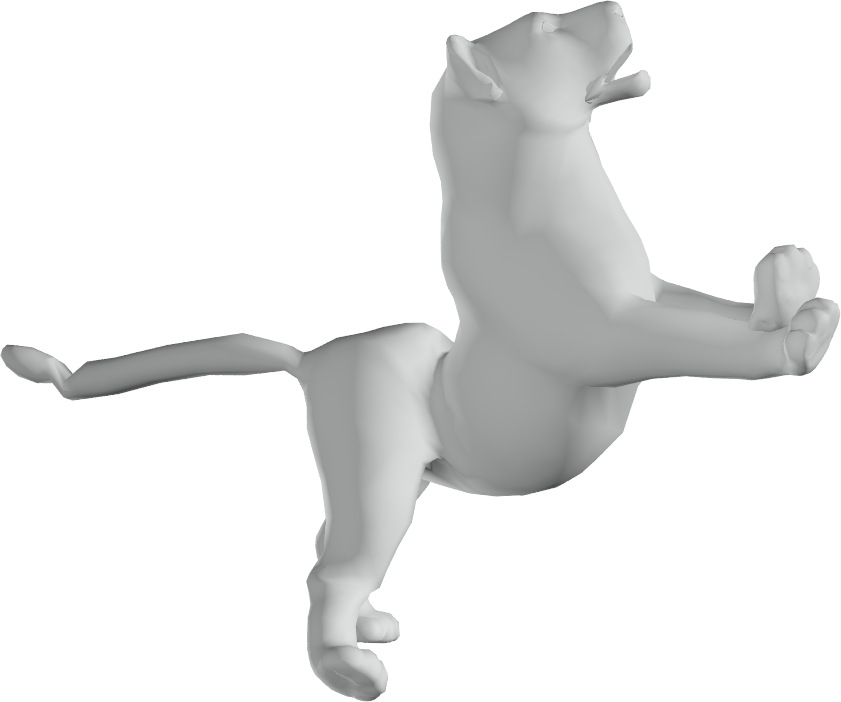}&
\includegraphics[width=0.15\textwidth]{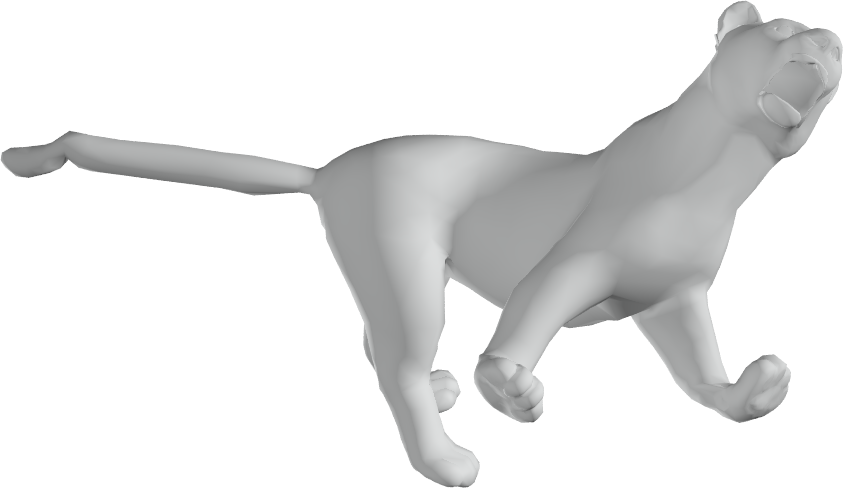}&
\includegraphics[width=0.15\textwidth]{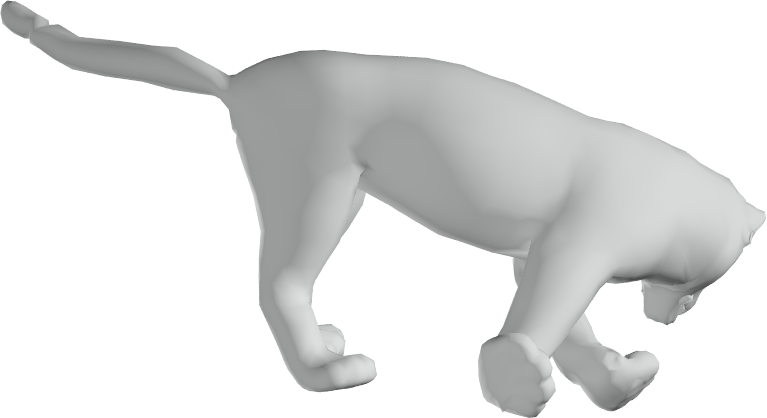}\\
\includegraphics[width=0.15\textwidth]{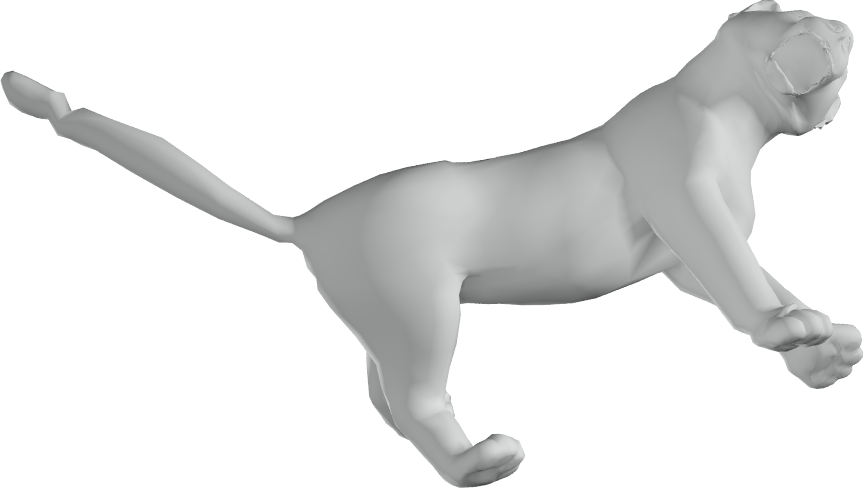}&
\includegraphics[width=0.15\textwidth]{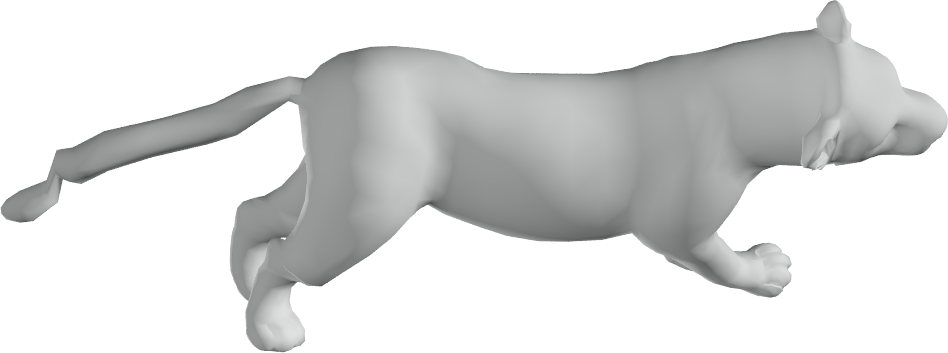}&
\includegraphics[width=0.15\textwidth]{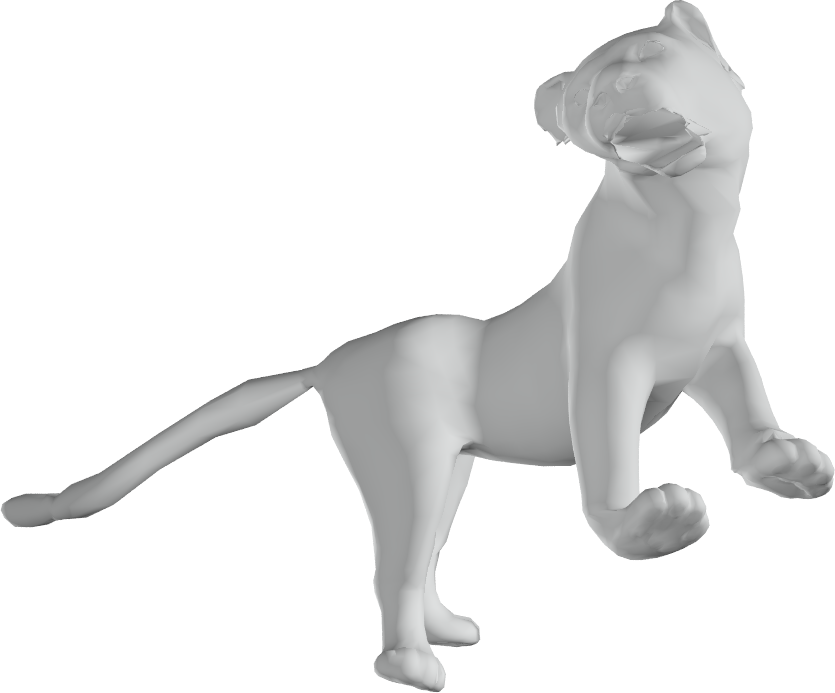}&
\includegraphics[width=0.15\textwidth]{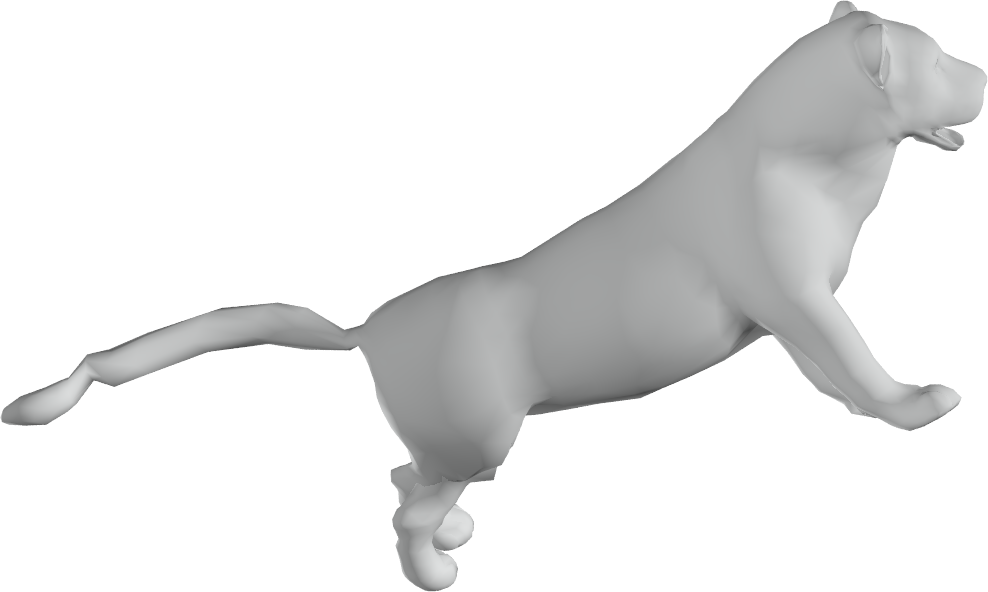}&
\includegraphics[width=0.15\textwidth]{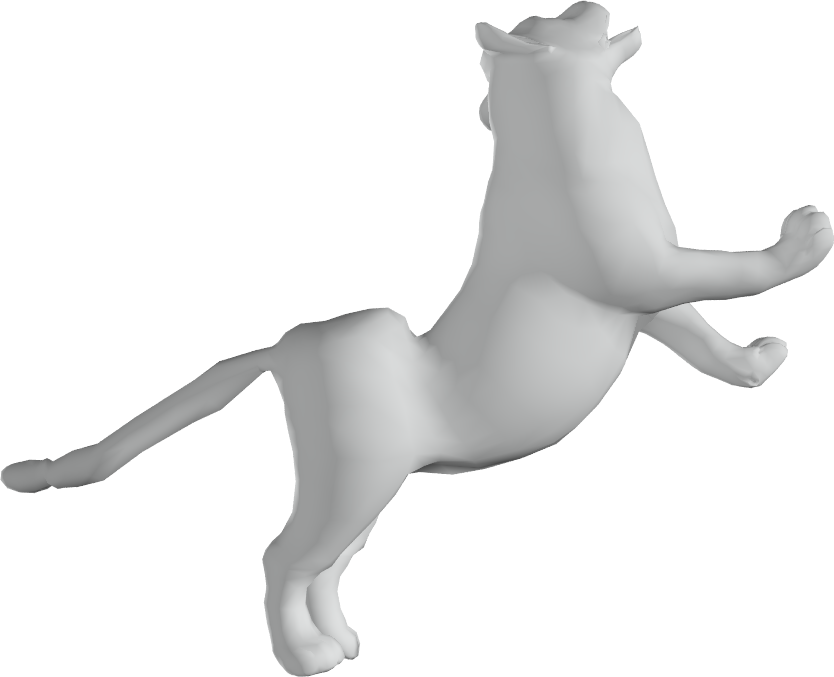}&
\includegraphics[width=0.15\textwidth]{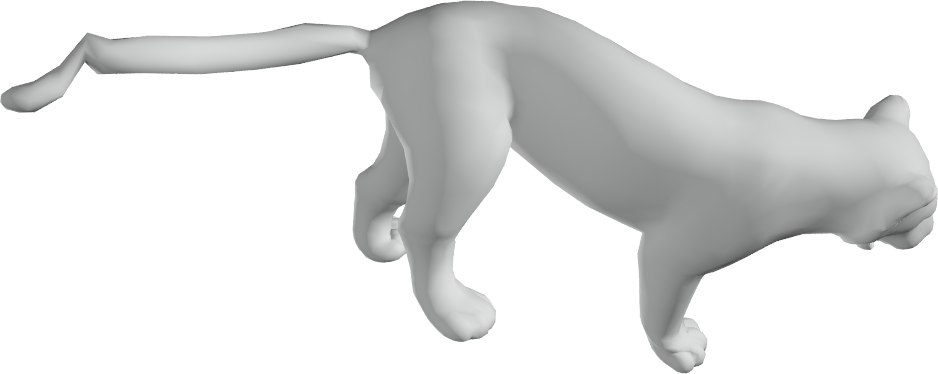}\\
\includegraphics[width=0.15\textwidth]{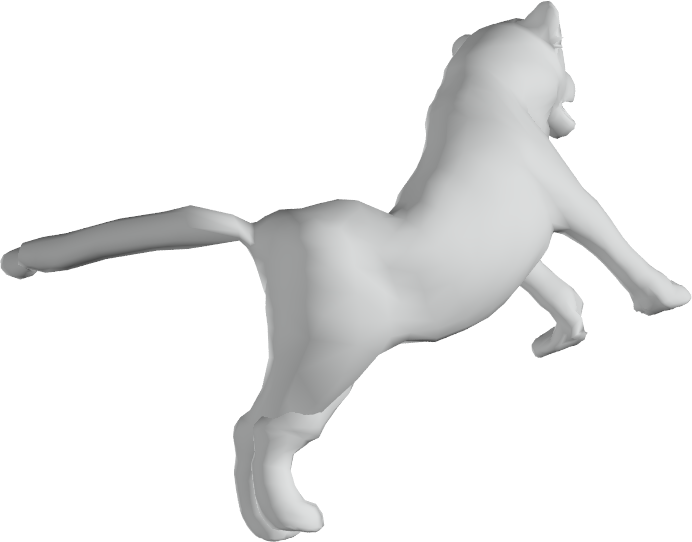}&
\includegraphics[width=0.15\textwidth]{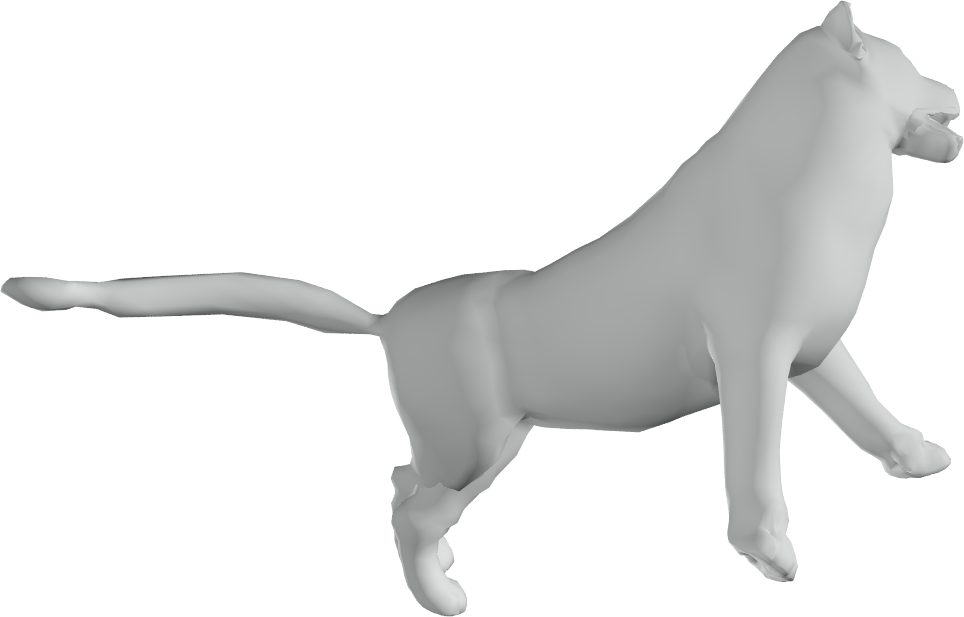}&
\includegraphics[width=0.15\textwidth]{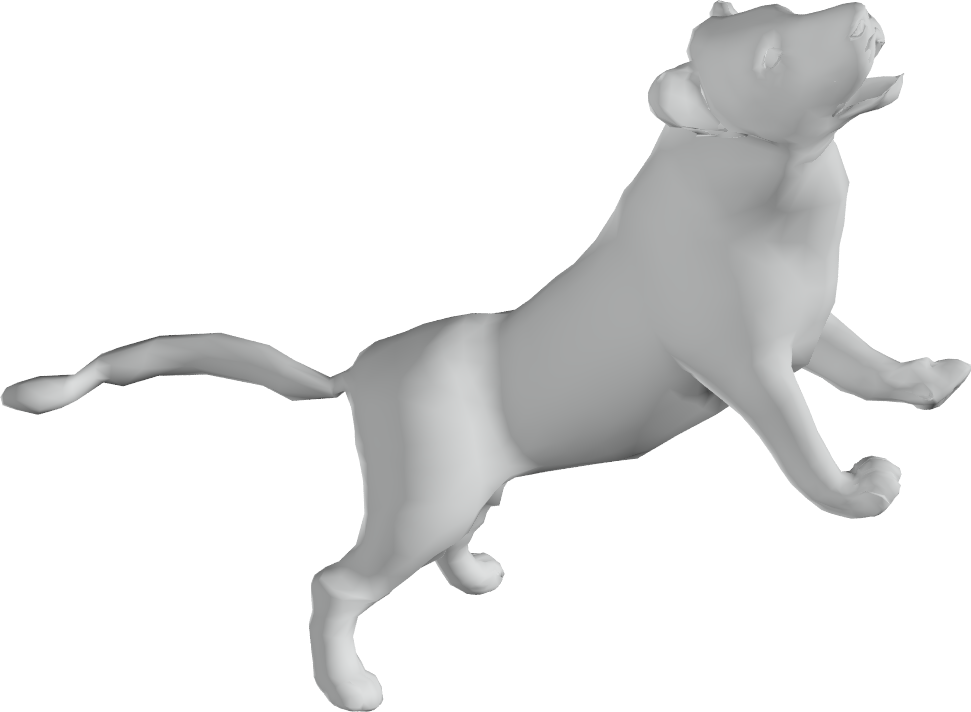}&
\includegraphics[width=0.15\textwidth]{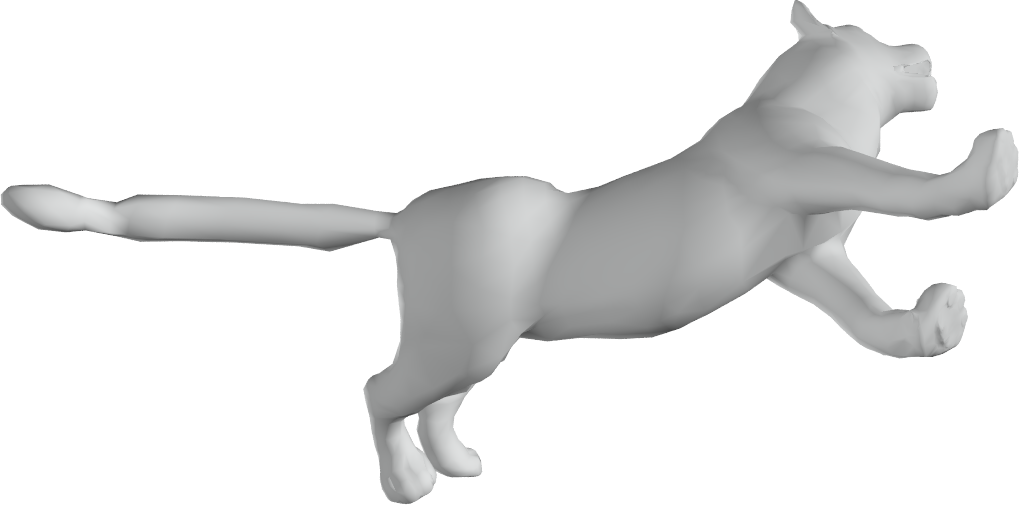}&
\includegraphics[width=0.15\textwidth]{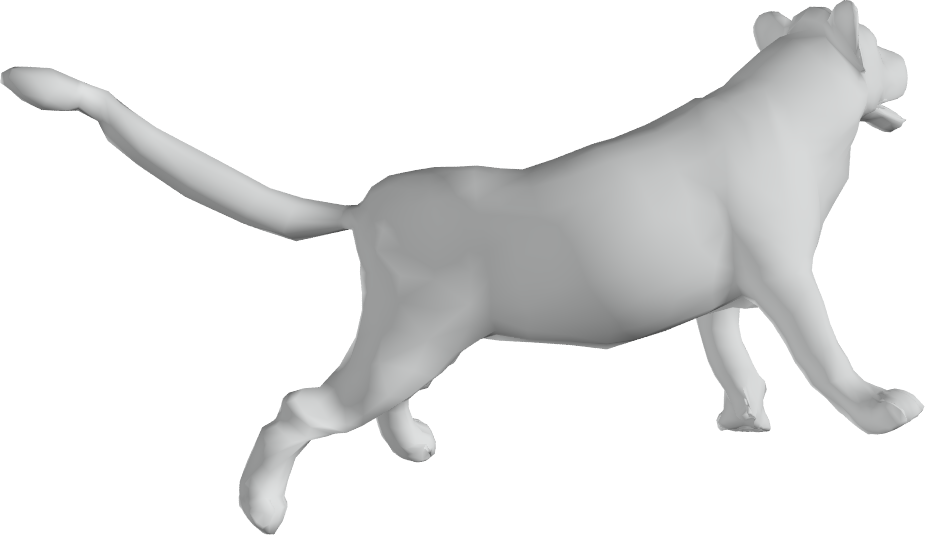}&
\includegraphics[width=0.15\textwidth]{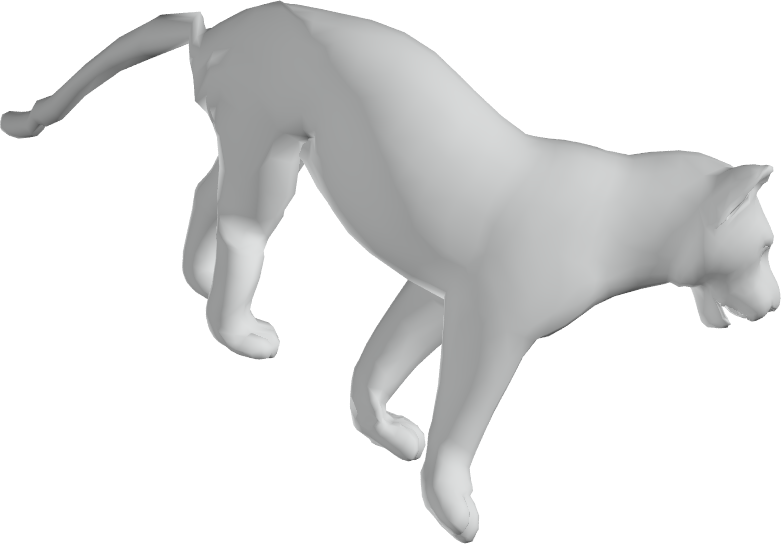}\\
\includegraphics[width=0.15\textwidth]{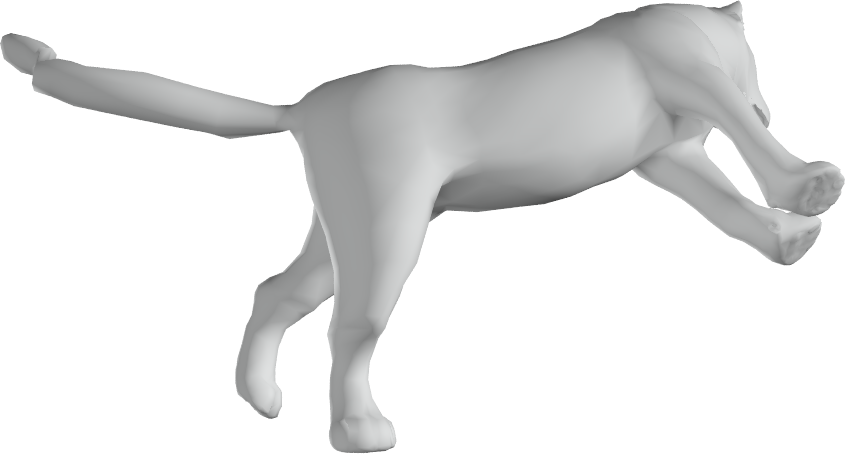}&
\includegraphics[width=0.15\textwidth]{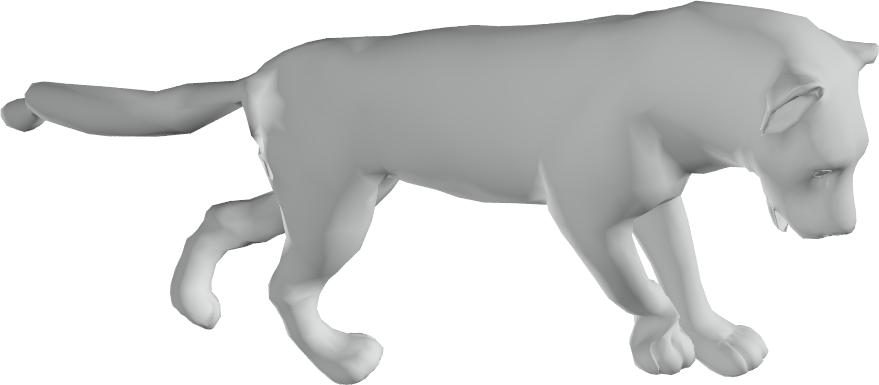}&
\includegraphics[width=0.15\textwidth]{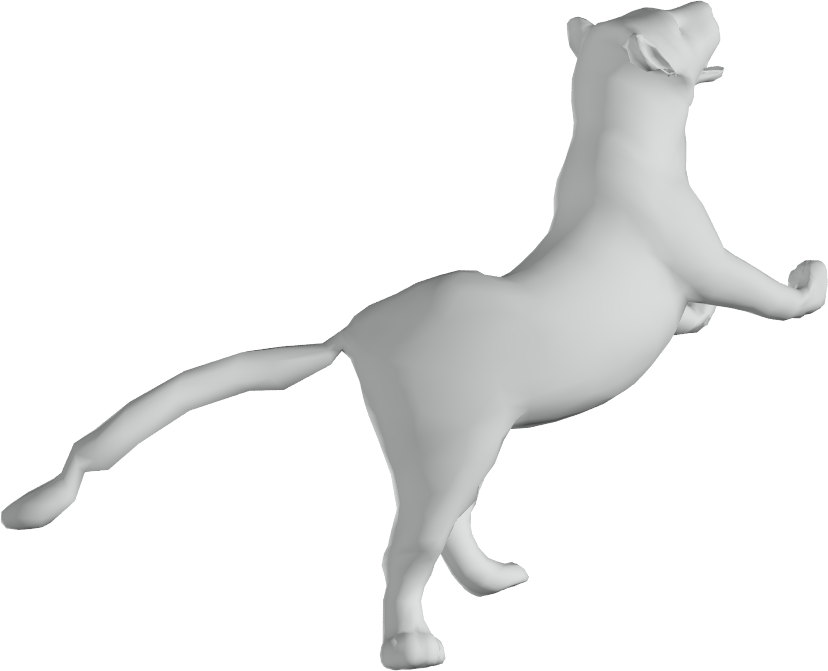}&
\includegraphics[width=0.15\textwidth]{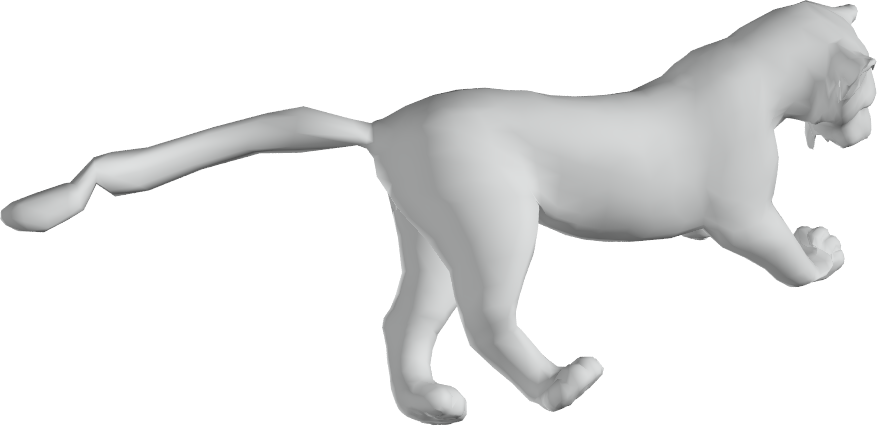}&
\includegraphics[width=0.15\textwidth]{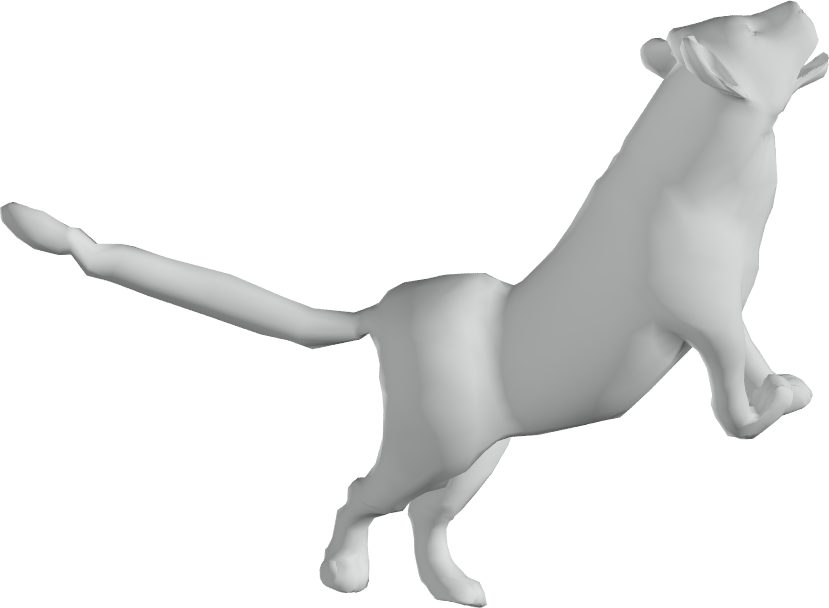}&
\includegraphics[width=0.15\textwidth]{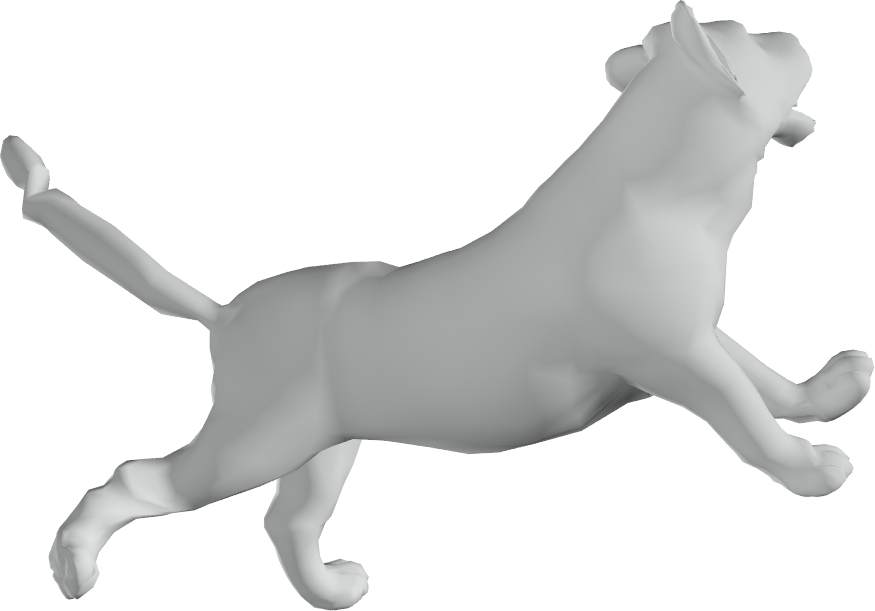}
\end{tabular}
\caption{ARAPDiffusion results. They show improved individual shape quality and higher success rate. }
\label{Figure:SMAL:Mesh:ARAPDiff}    
\end{figure*}

\begin{figure*}
\setlength\tabcolsep{4pt}
\begin{tabular}{cc|cc|cc}
\includegraphics[width=0.15\textwidth]{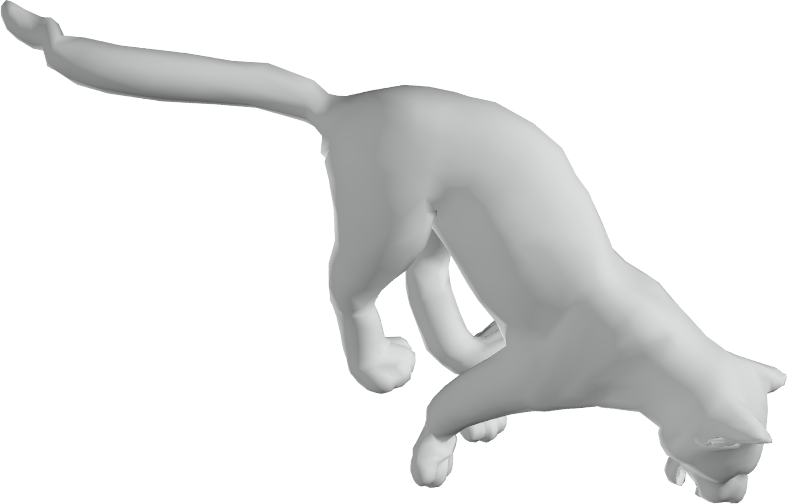}&
\includegraphics[width=0.15\textwidth]{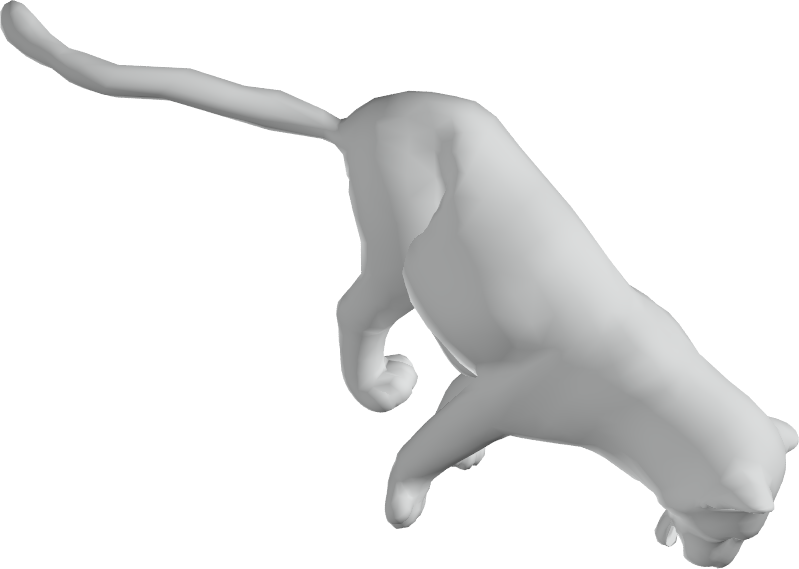}&
\includegraphics[width=0.15\textwidth]{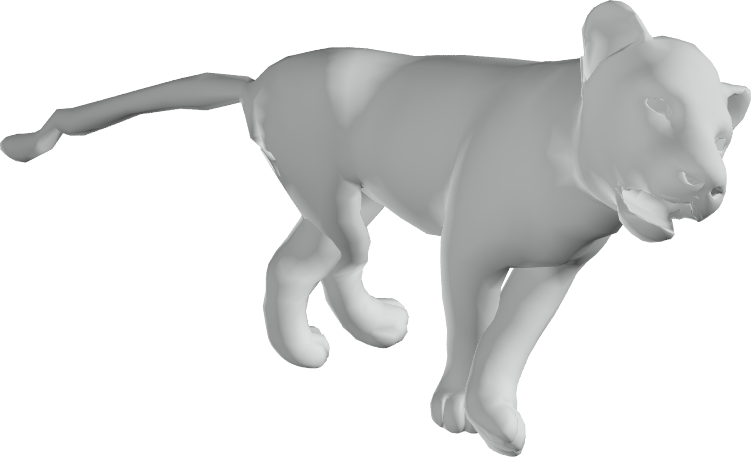}&
\includegraphics[width=0.15\textwidth]{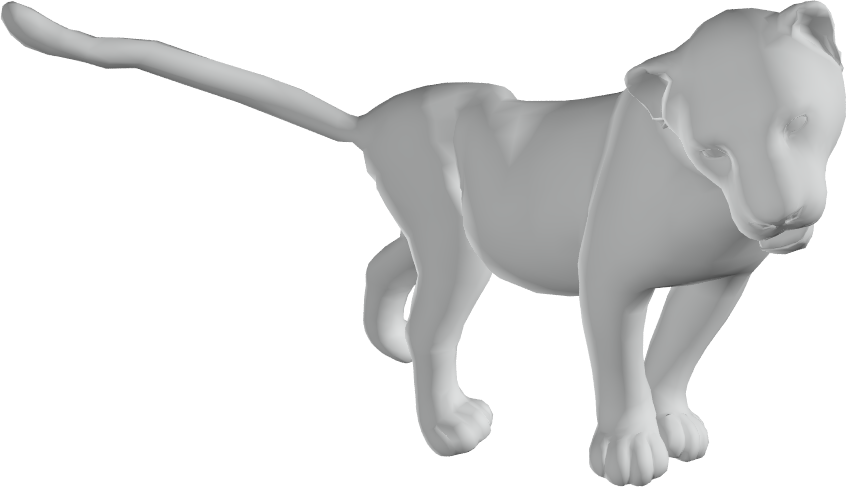}&
\includegraphics[width=0.15\textwidth]{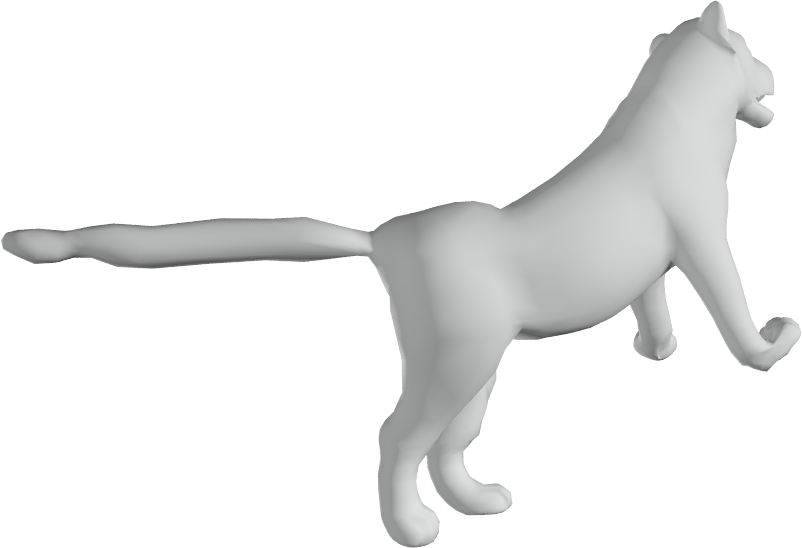}&
\includegraphics[width=0.15\textwidth]{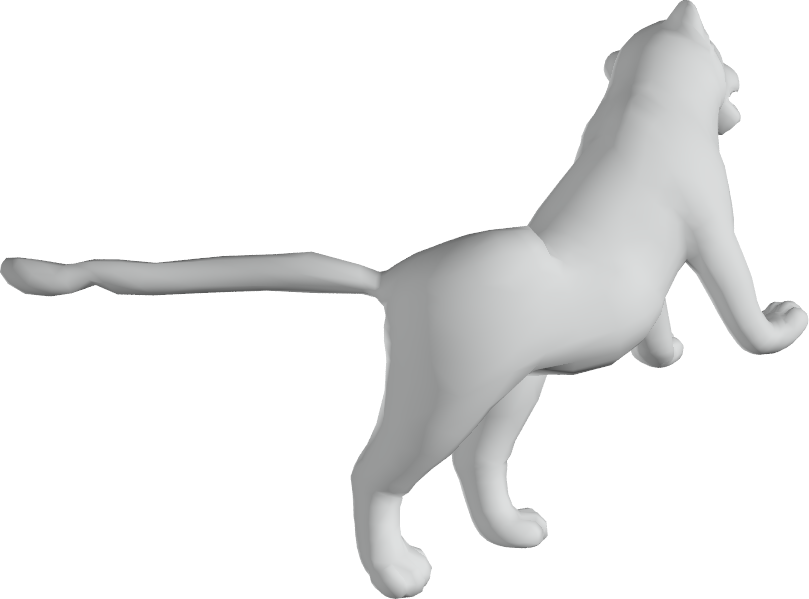}\\\hline
\includegraphics[width=0.15\textwidth]{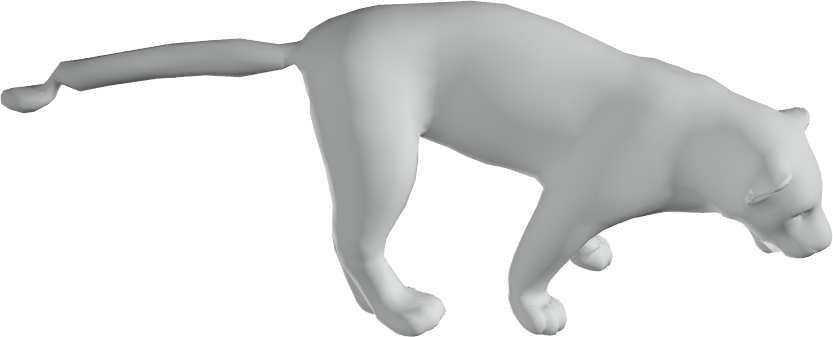}&
\includegraphics[width=0.15\textwidth]{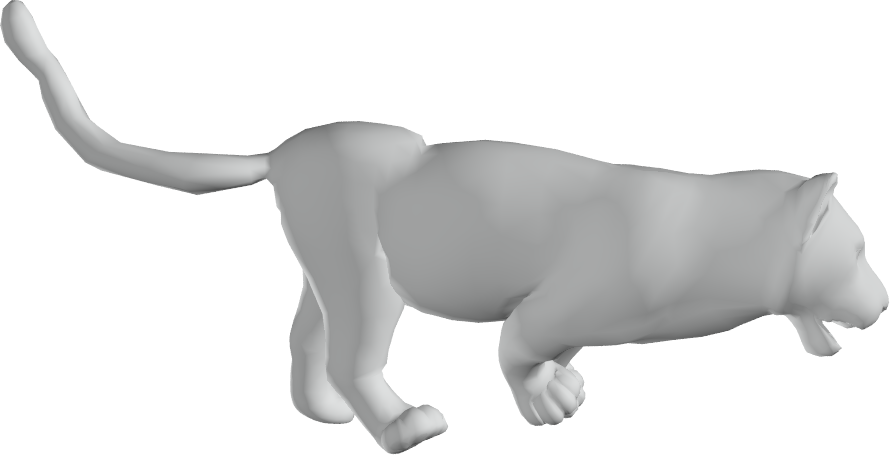}&
\includegraphics[width=0.15\textwidth]{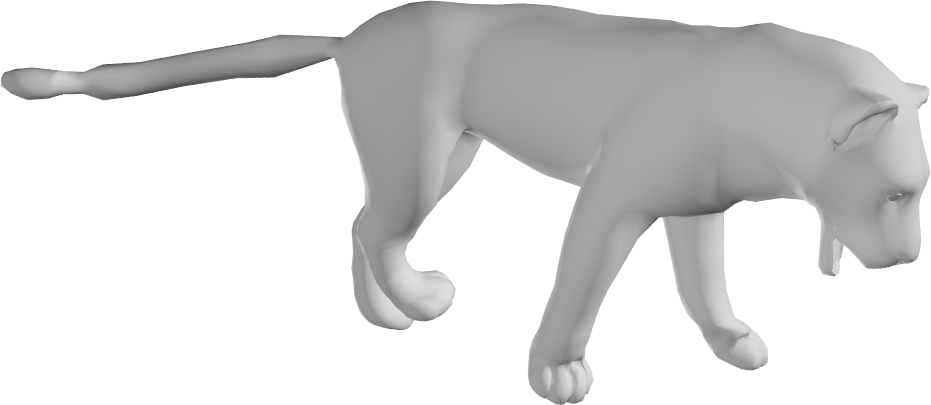}&
\includegraphics[width=0.15\textwidth]{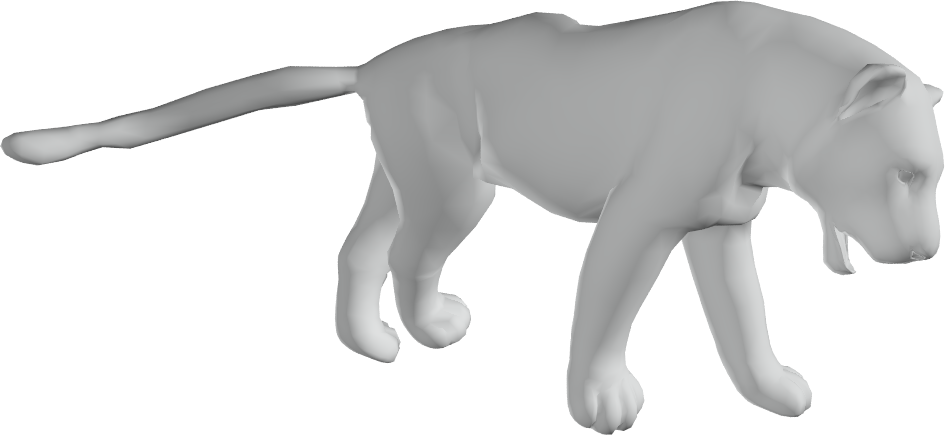}&
\includegraphics[width=0.15\textwidth]{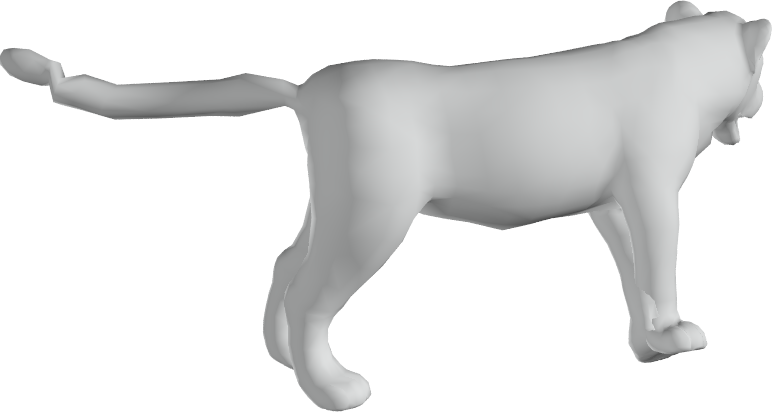}&
\includegraphics[width=0.15\textwidth]{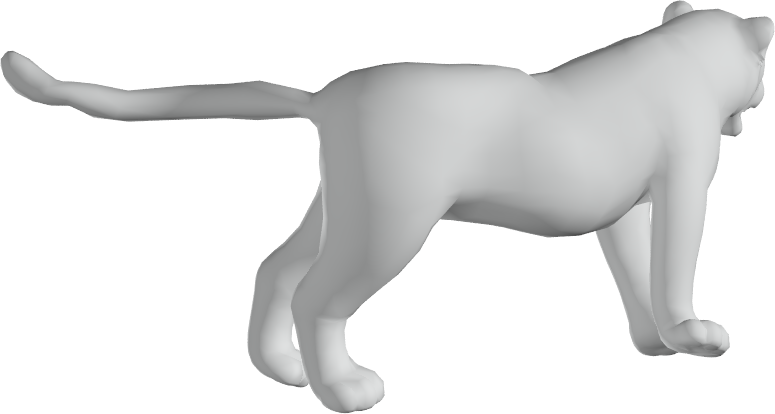}\\\hline
\includegraphics[width=0.15\textwidth]{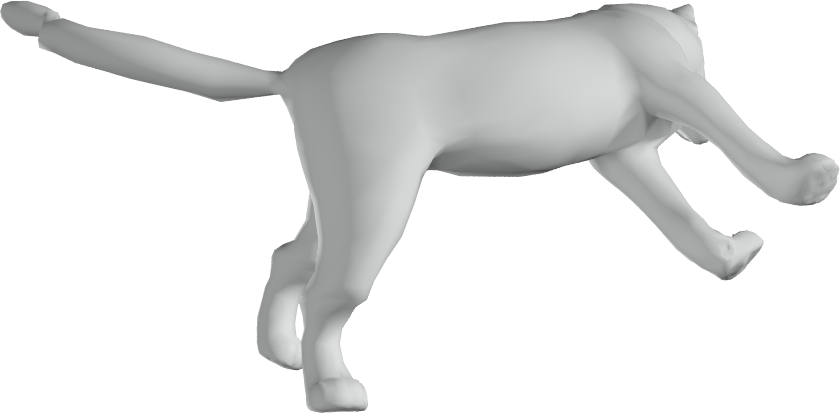}&
\includegraphics[width=0.15\textwidth]{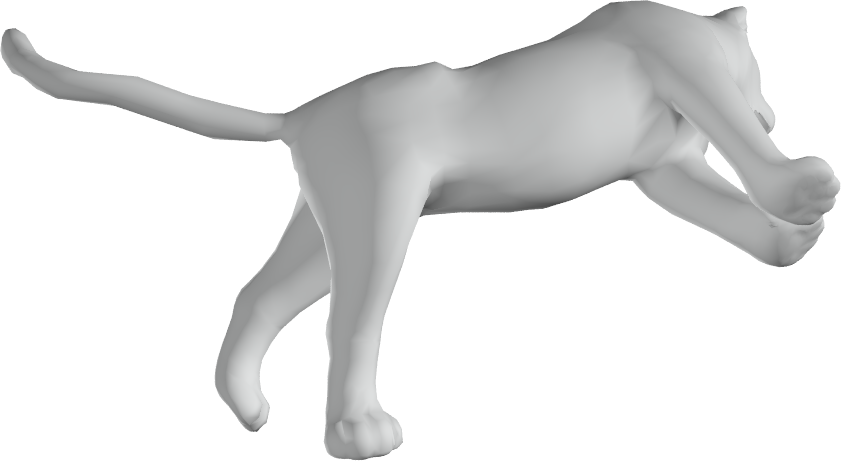}&
\includegraphics[width=0.15\textwidth]{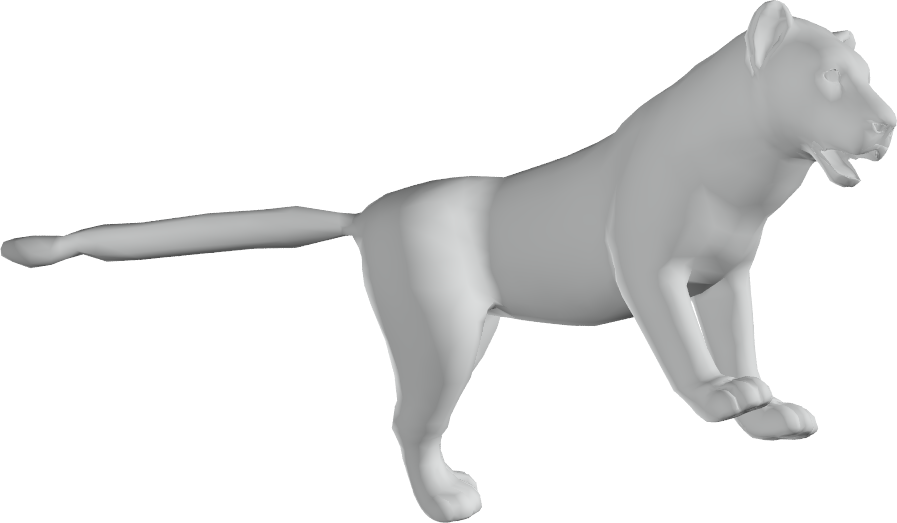}&
\includegraphics[width=0.15\textwidth]{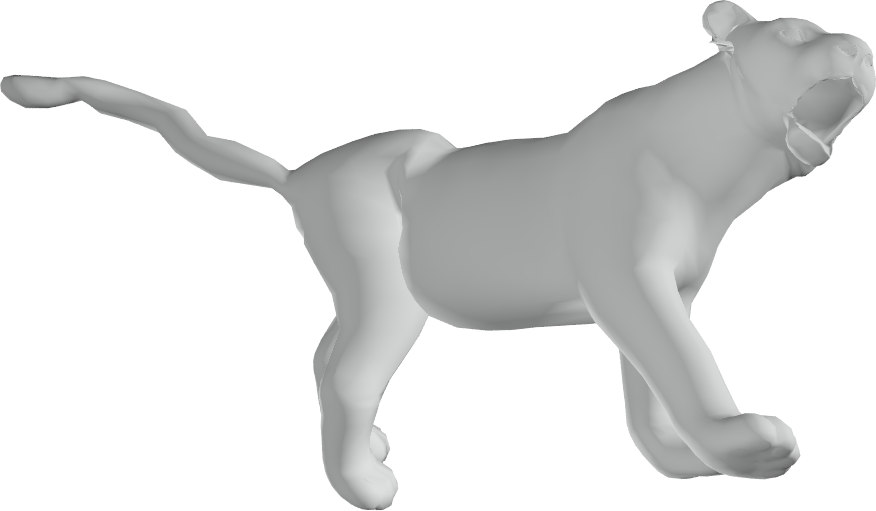}&
\includegraphics[width=0.15\textwidth]{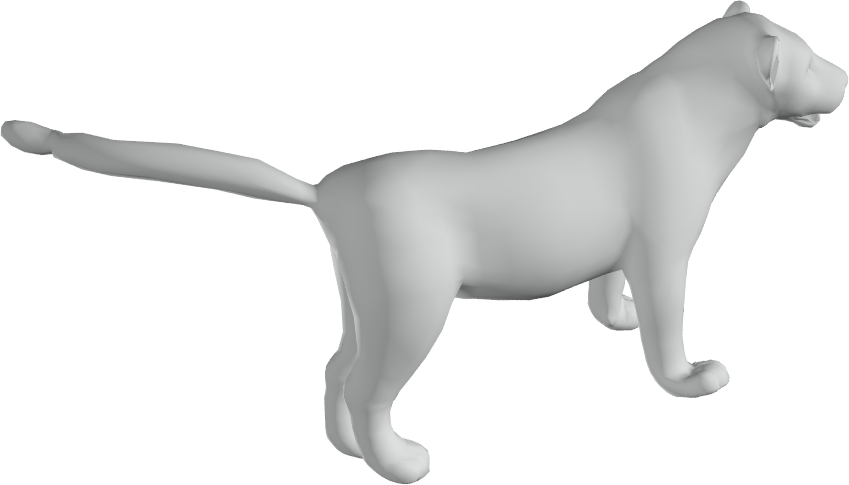}&
\includegraphics[width=0.15\textwidth]{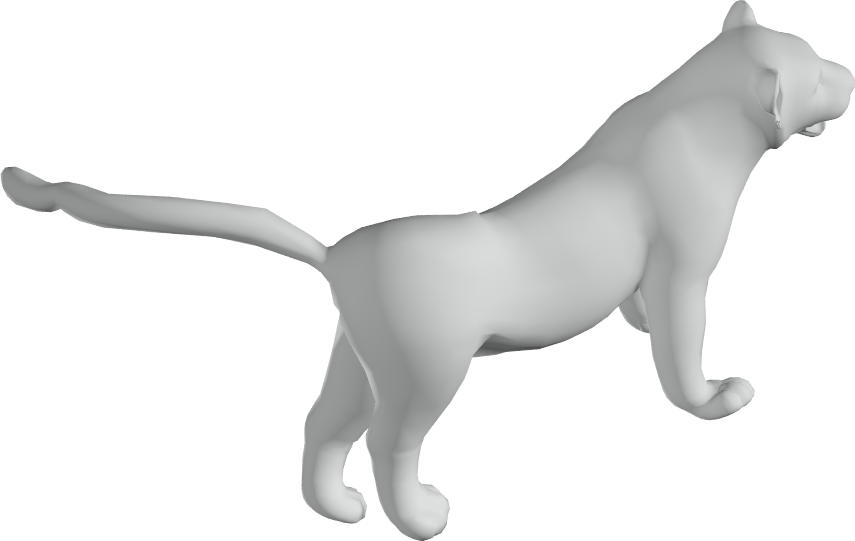}\\\hline
\includegraphics[width=0.15\textwidth]{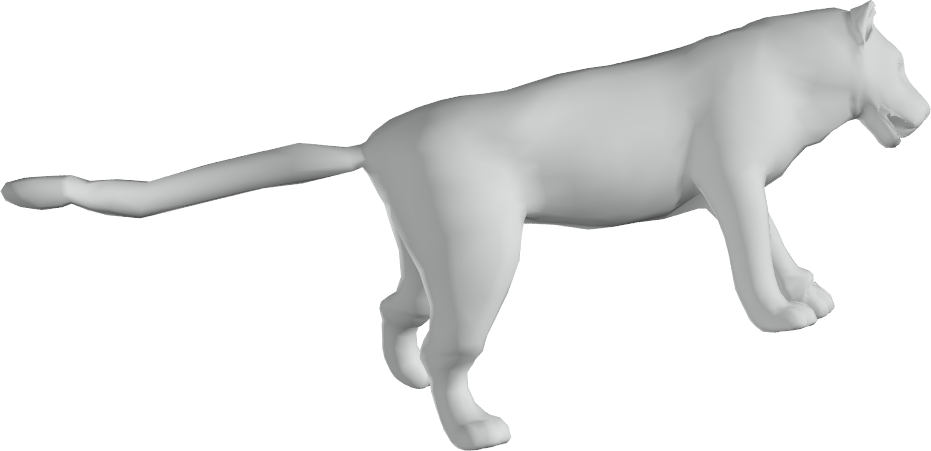}&
\includegraphics[width=0.15\textwidth]{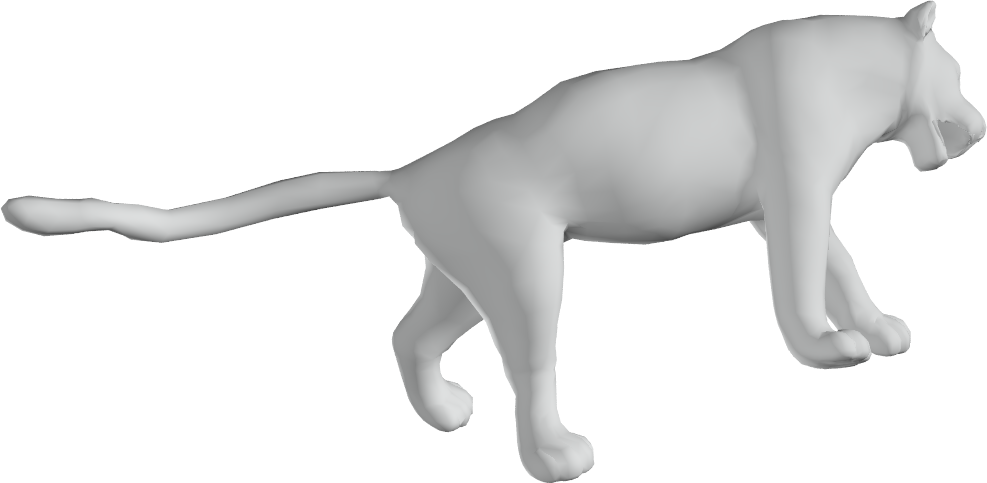}&
\includegraphics[width=0.15\textwidth]{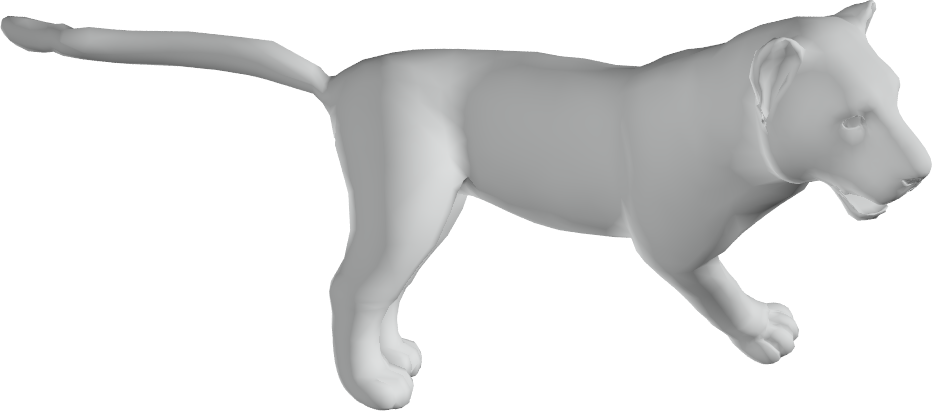}&
\includegraphics[width=0.15\textwidth]{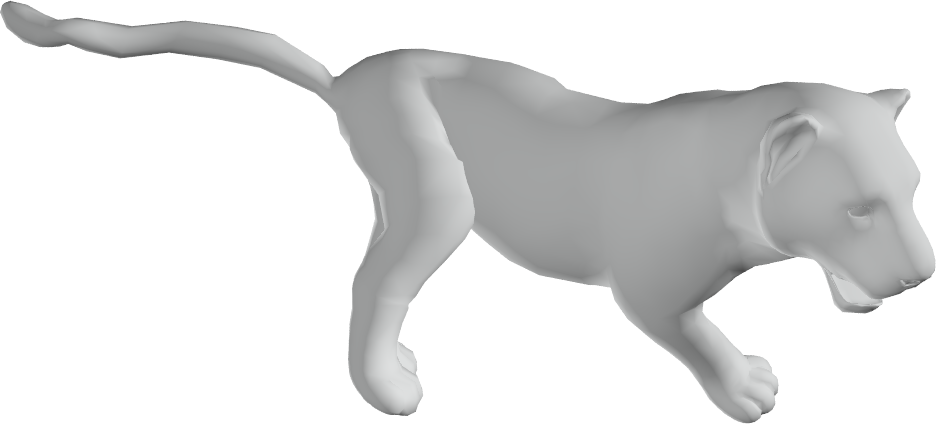}&
\includegraphics[width=0.15\textwidth]{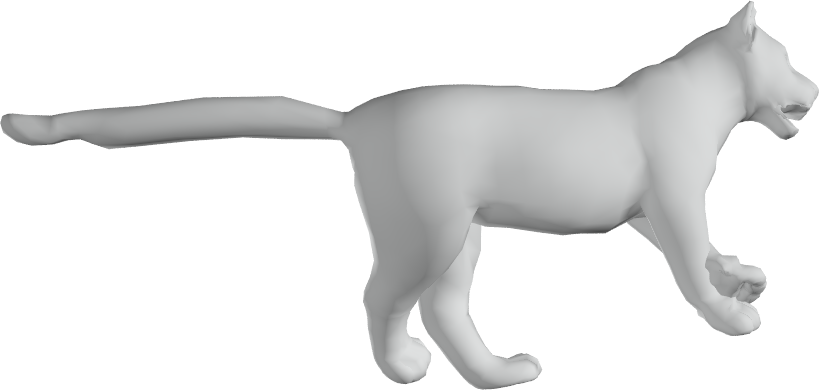}&
\includegraphics[width=0.15\textwidth]{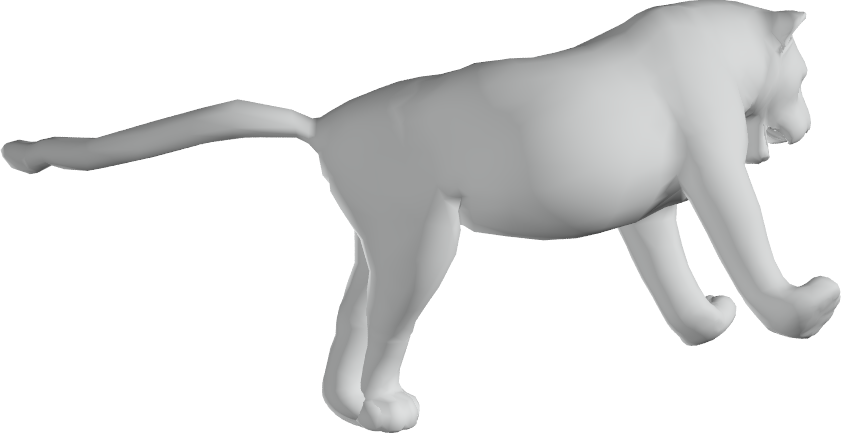}
\end{tabular}
\caption{The closest training shape of each generated shape. In each cell, the generated shape is shown on the left, while the closest training shape is shown on the right. }
\label{Figure:SMAL:Mesh:ARAPDiff:Closest:Training}    
\end{figure*}

\section{Additional Conditional Generation Results}

Fig.~\ref{Fig:Human:CG} and Fig.~\ref{Fig:Animal:CG} show qualitative conditional generation results. Compared to baseline approaches, our approach leads to results that with significantly improved alignments with the ground-truth. Even if our results do not align with the ground-truth due to limited observations, our results are still plausible. 

\begin{figure*}
\centering
\setlength\tabcolsep{1pt}
\begin{tabular}{ccccc|ccccc}
Input & GeoLatent& BRESA & ARAPDiff. & GT& Input & GeoLatent& BRESA & ARAPDiff. & GT\\
\includegraphics[height=0.18\textwidth]{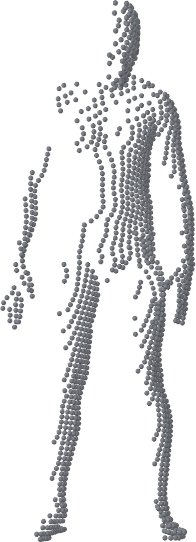}
&  
\includegraphics[height=0.18\textwidth]{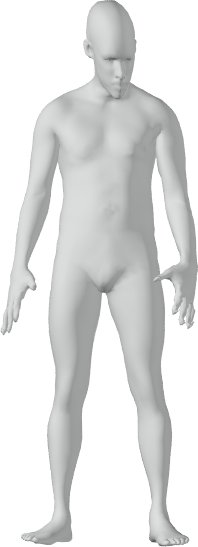}
&
\includegraphics[height=0.18\textwidth]{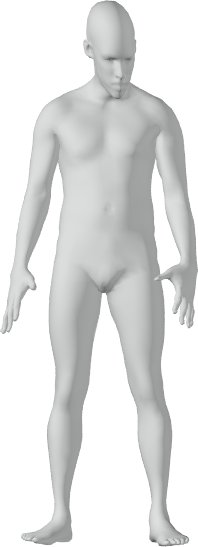}
& 
\includegraphics[height=0.18\textwidth]{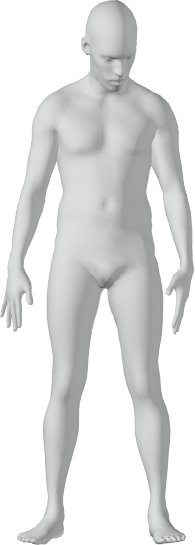}
& 
\includegraphics[height=0.18\textwidth]{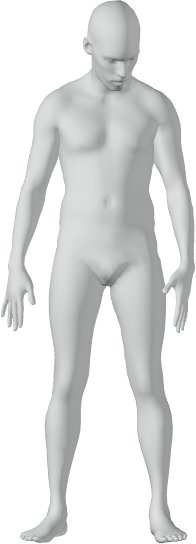} & \includegraphics[height=0.18\textwidth]{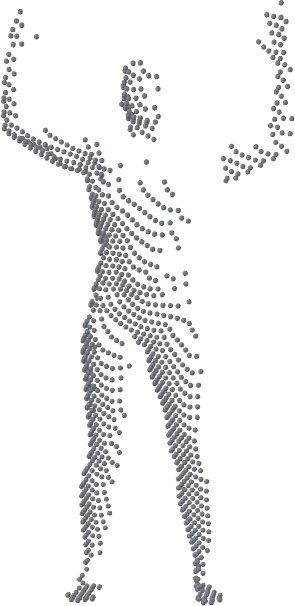}
&  
\includegraphics[height=0.18\textwidth]{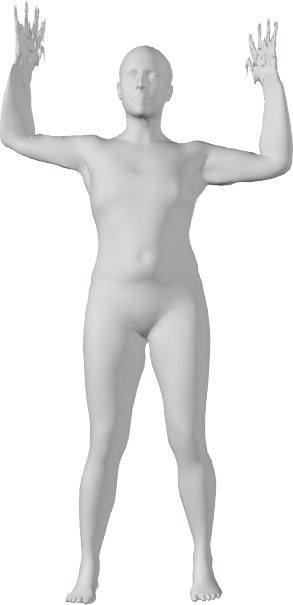}
&
\includegraphics[height=0.18\textwidth]{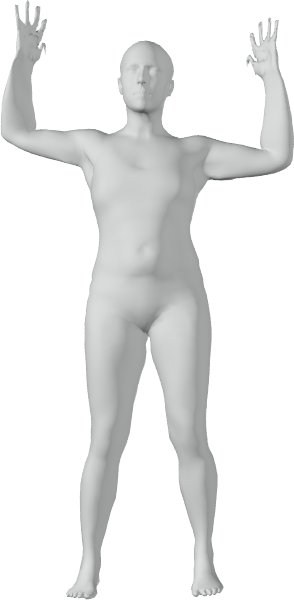}
& 
\includegraphics[height=0.18\textwidth]{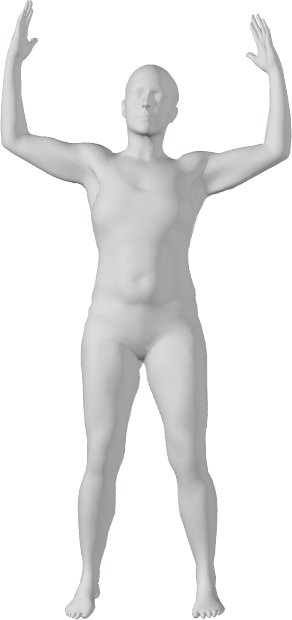}
& 
\includegraphics[height=0.18\textwidth]{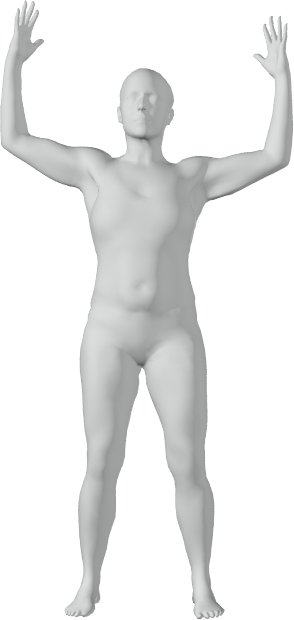}\\ 
\includegraphics[height=0.18\textwidth]{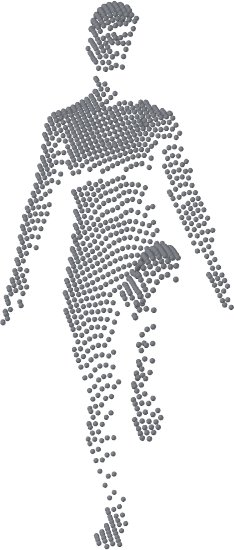}
&  
\includegraphics[height=0.18\textwidth]{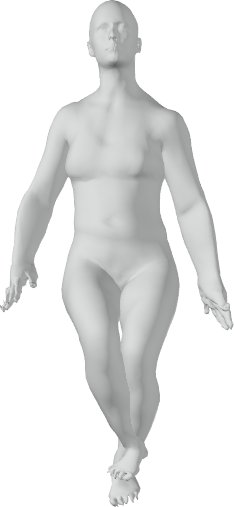}
&
\includegraphics[height=0.18\textwidth]{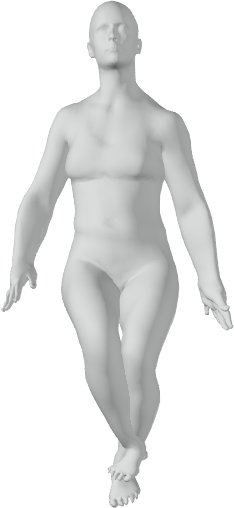}
& 
\includegraphics[height=0.18\textwidth]{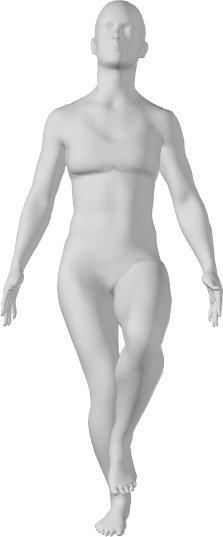}
& 
\includegraphics[height=0.18\textwidth]{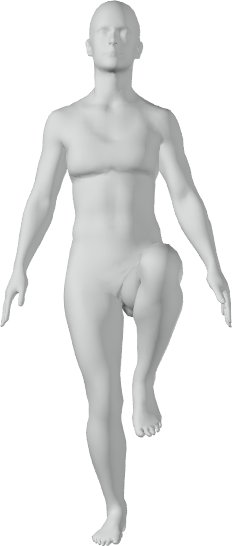} & \includegraphics[height=0.18\textwidth]{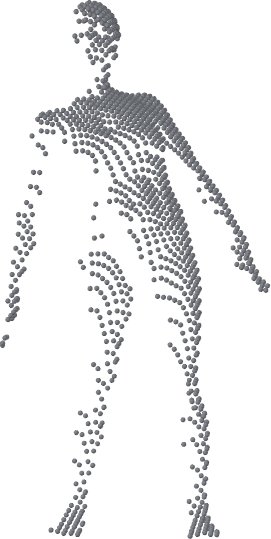}
&  
\includegraphics[height=0.18\textwidth]{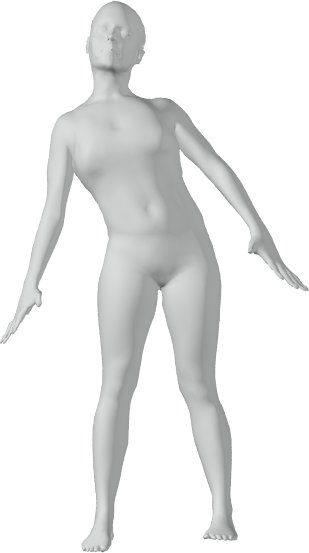}
&
\includegraphics[height=0.18\textwidth]{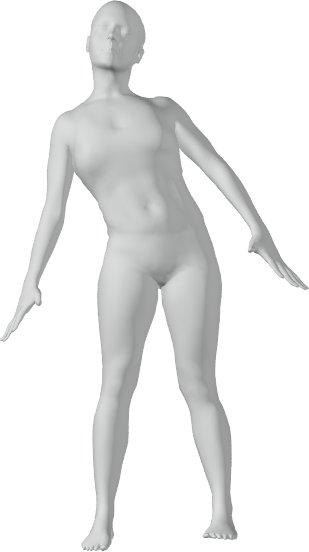}
& 
\includegraphics[height=0.18\textwidth]{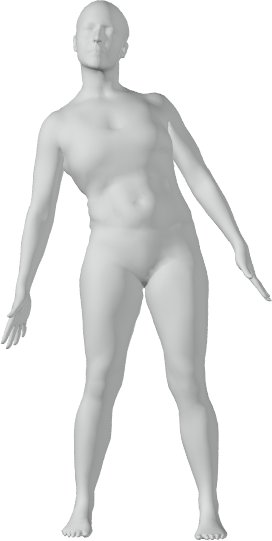}
& 
\includegraphics[height=0.18\textwidth]{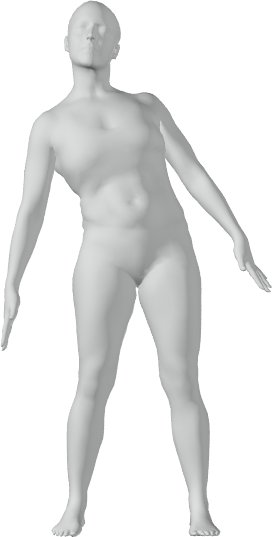}\\ 
\includegraphics[height=0.18\textwidth]{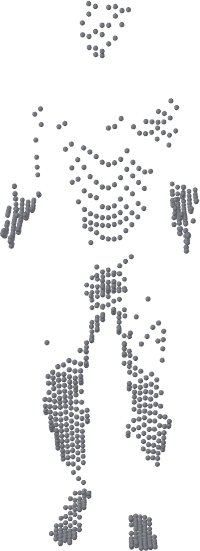}
&  
\includegraphics[height=0.18\textwidth]{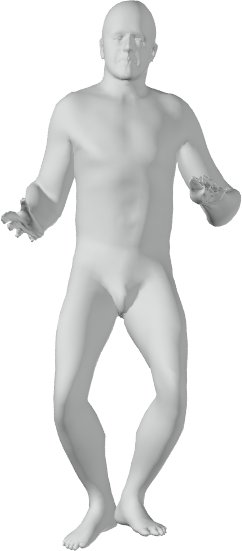}
&
\includegraphics[height=0.18\textwidth]{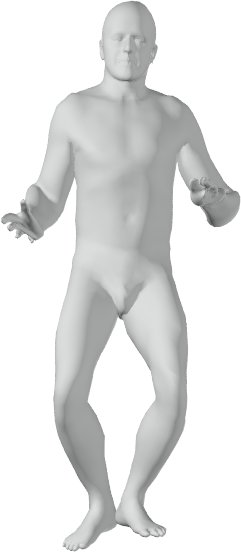}
& 
\includegraphics[height=0.18\textwidth]{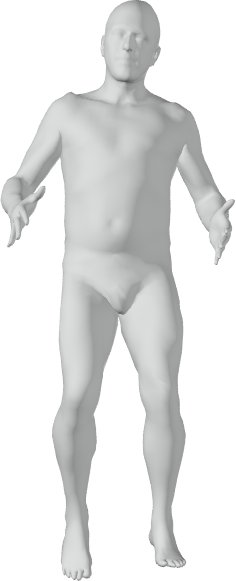}
& 
\includegraphics[height=0.18\textwidth]{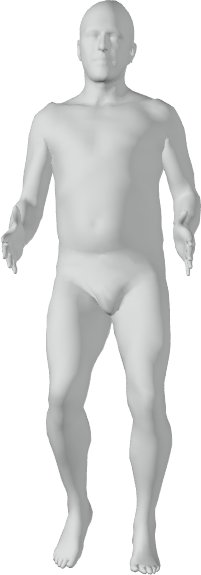} & \includegraphics[height=0.18\textwidth]{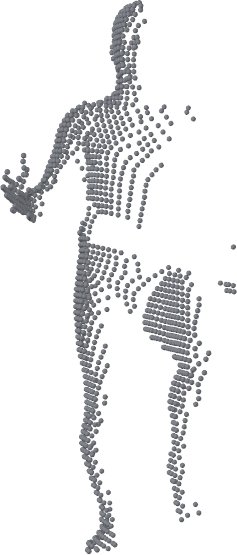}
&  
\includegraphics[height=0.18\textwidth]{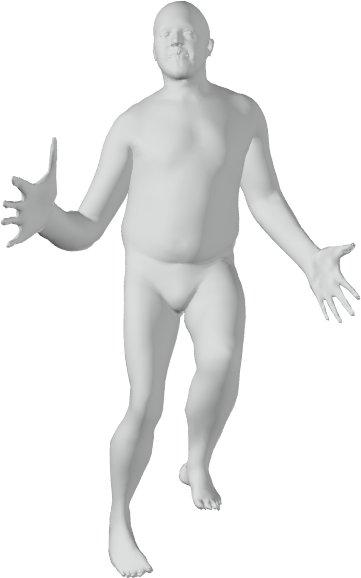}
&
\includegraphics[height=0.18\textwidth]{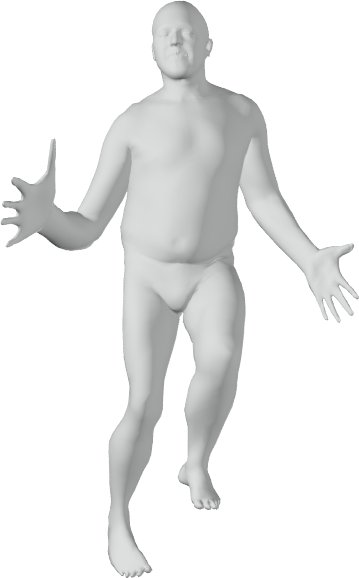}
& 
\includegraphics[height=0.18\textwidth]{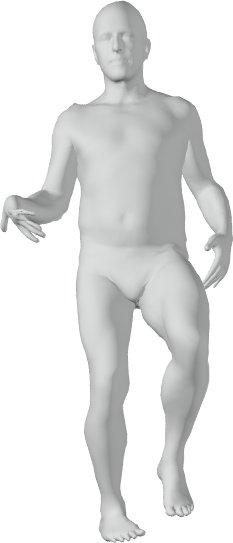}
& 
\includegraphics[height=0.18\textwidth]{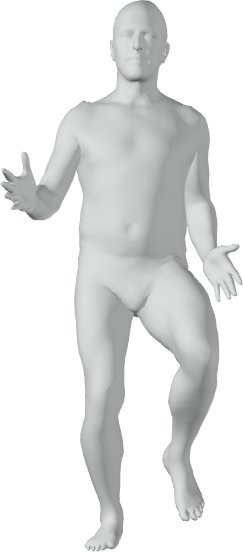}\\
\includegraphics[height=0.18\textwidth]{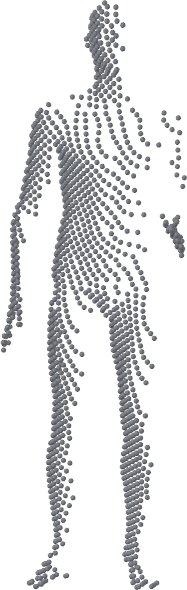}
&  
\includegraphics[height=0.18\textwidth]{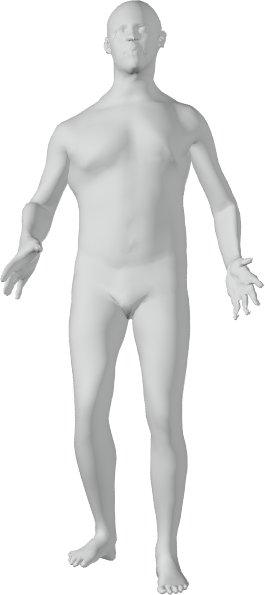}
&
\includegraphics[height=0.18\textwidth]{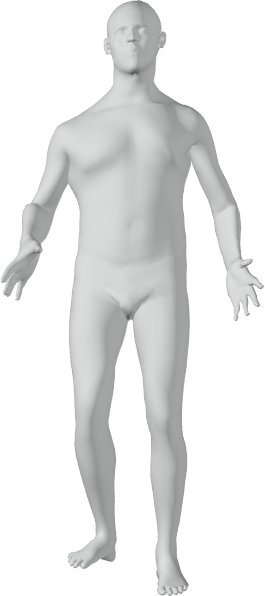}
& 
\includegraphics[height=0.18\textwidth]{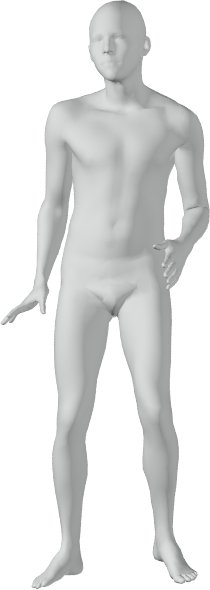}
& 
\includegraphics[height=0.18\textwidth]{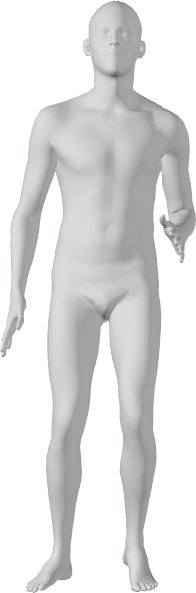} & \includegraphics[height=0.18\textwidth]{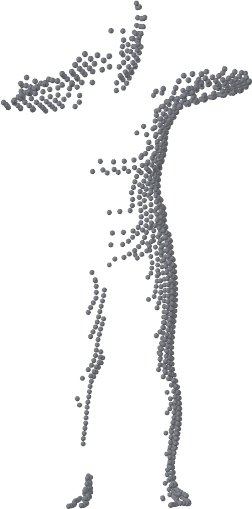}
&  
\includegraphics[height=0.18\textwidth]{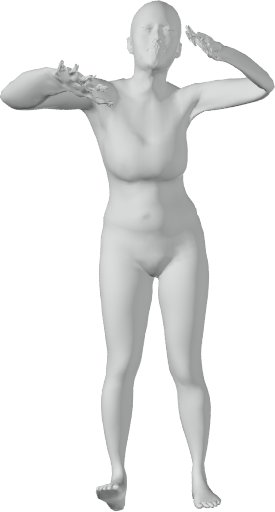}
&
\includegraphics[height=0.18\textwidth]{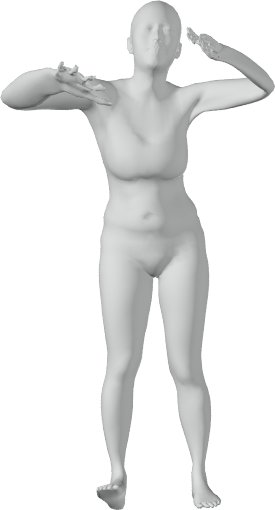}
& 
\includegraphics[height=0.18\textwidth]{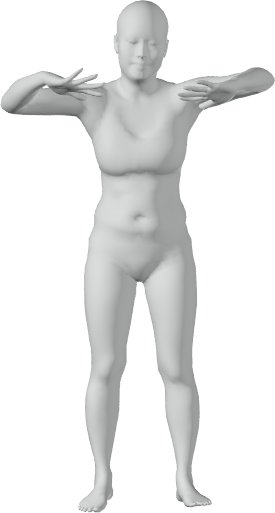}
& 
\includegraphics[height=0.18\textwidth]{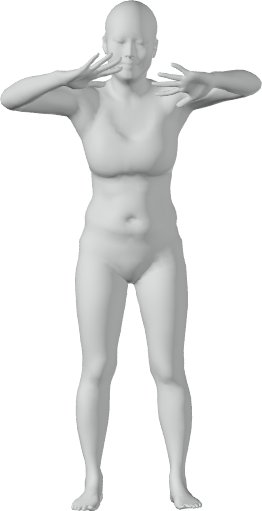} \\
\includegraphics[height=0.18\textwidth]{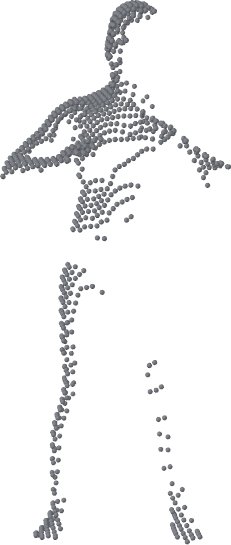}
&  
\includegraphics[height=0.18\textwidth]{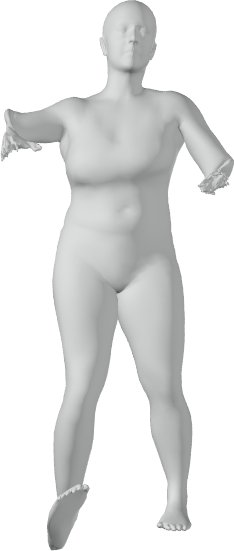}
&
\includegraphics[height=0.18\textwidth]{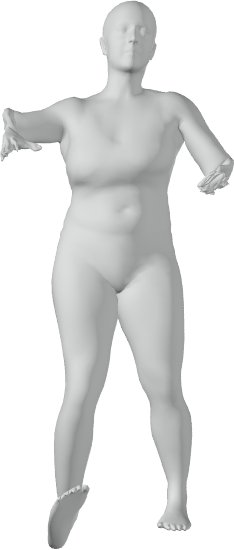}
& 
\includegraphics[height=0.18\textwidth]{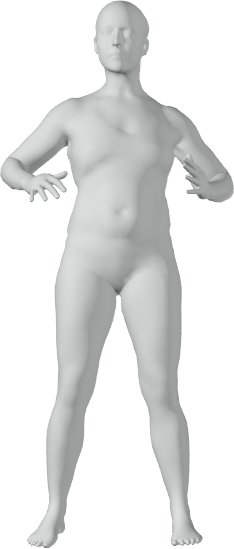}
& 
\includegraphics[height=0.18\textwidth]{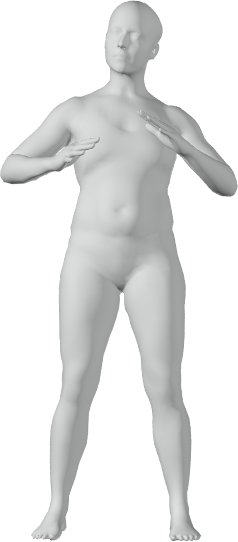} & \includegraphics[height=0.18\textwidth]{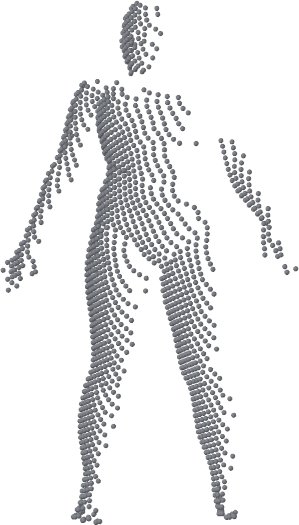}
&  
\includegraphics[height=0.18\textwidth]{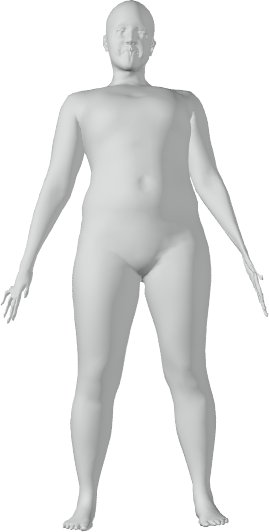}
&
\includegraphics[height=0.18\textwidth]{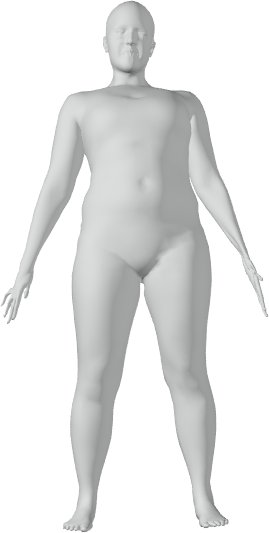}
& 
\includegraphics[height=0.18\textwidth]{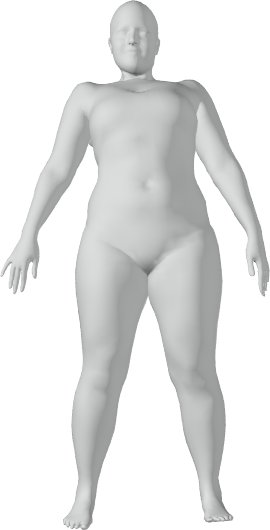}
& 
\includegraphics[height=0.18\textwidth]{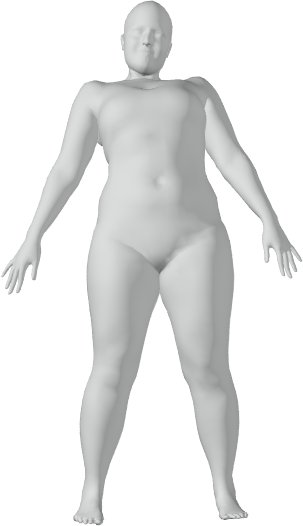} \\
\includegraphics[height=0.18\textwidth]{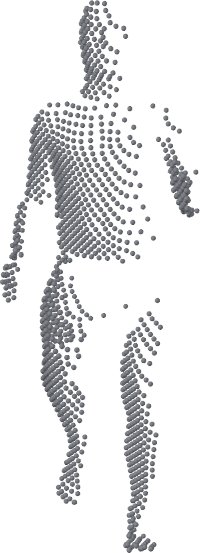}
&  
\includegraphics[height=0.18\textwidth]{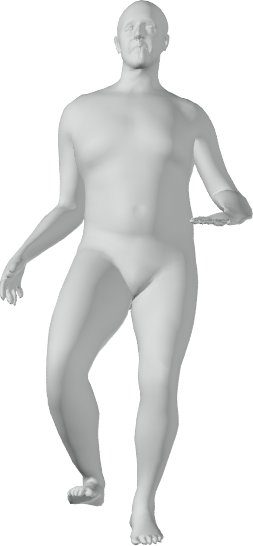}
&
\includegraphics[height=0.18\textwidth]{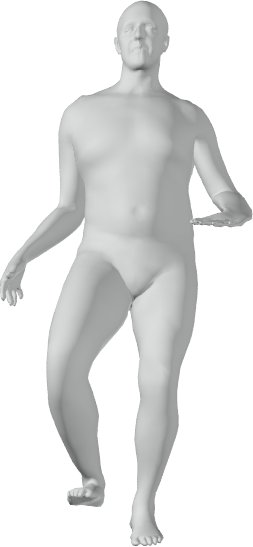}
& 
\includegraphics[height=0.18\textwidth]{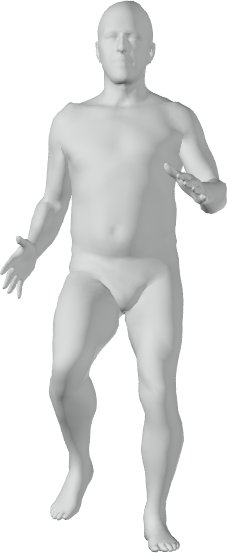}
& 
\includegraphics[height=0.18\textwidth]{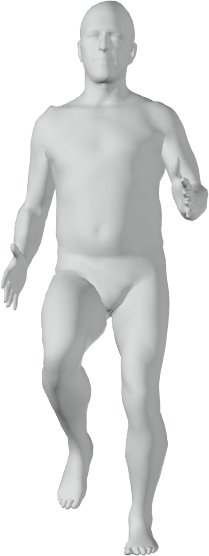} & \includegraphics[height=0.18\textwidth]{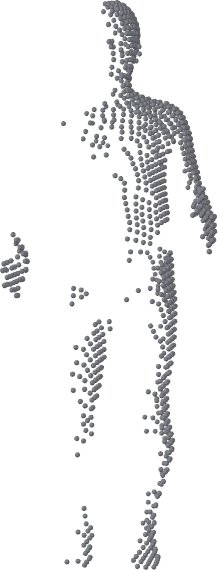}
&  
\includegraphics[height=0.18\textwidth]{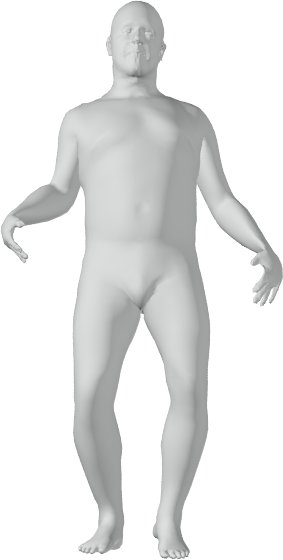}
&
\includegraphics[height=0.18\textwidth]{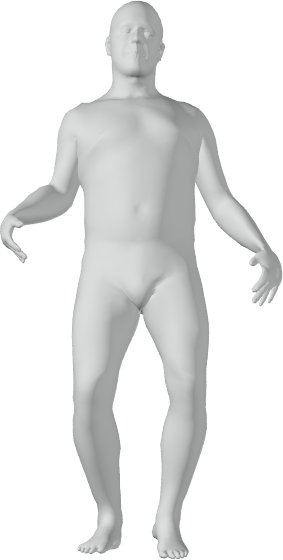}
& 
\includegraphics[height=0.18\textwidth]{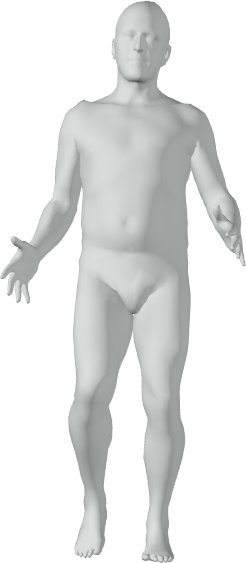}
& 
\includegraphics[height=0.18\textwidth]{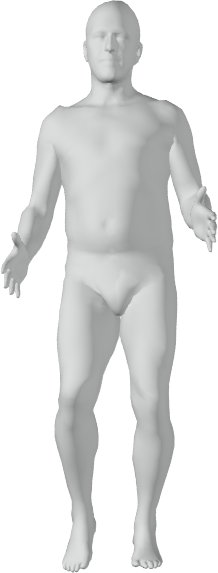} \\
\includegraphics[height=0.18\textwidth]{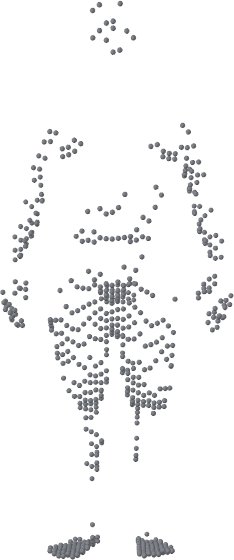}
&  
\includegraphics[height=0.18\textwidth]{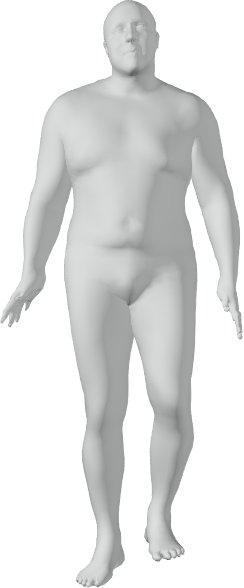}
&
\includegraphics[height=0.18\textwidth]{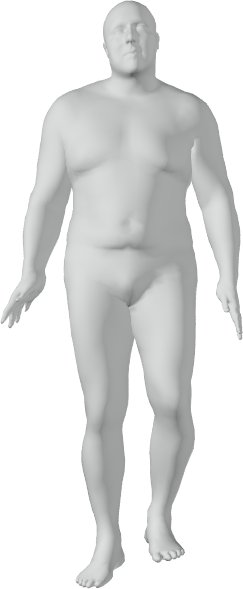}
& 
\includegraphics[height=0.18\textwidth]{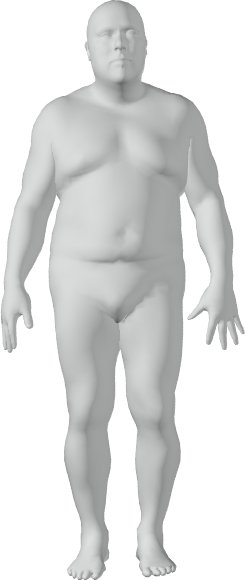}
& 
\includegraphics[height=0.18\textwidth]{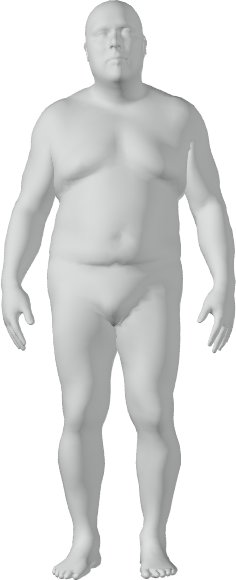} & \includegraphics[height=0.18\textwidth]{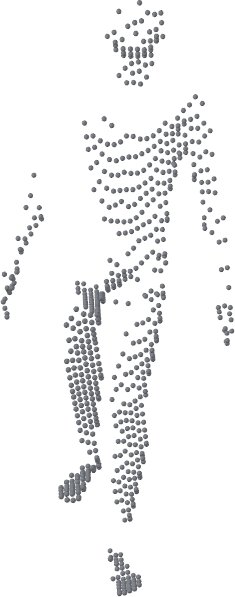}
&  
\includegraphics[height=0.18\textwidth]{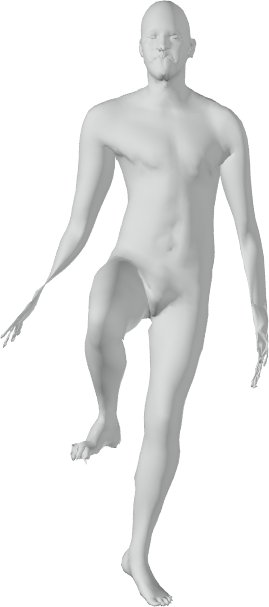}
&
\includegraphics[height=0.18\textwidth]{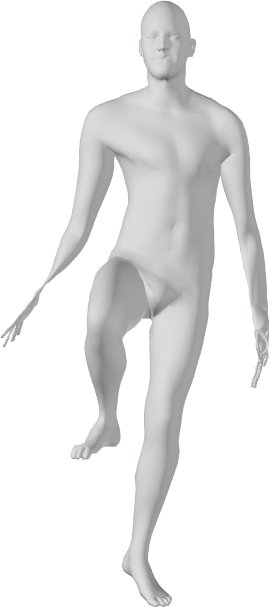}
& 
\includegraphics[height=0.18\textwidth]{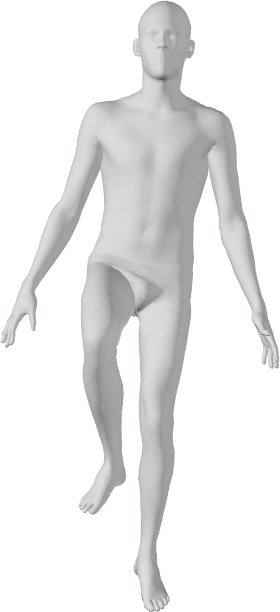}
& 
\includegraphics[height=0.18\textwidth]{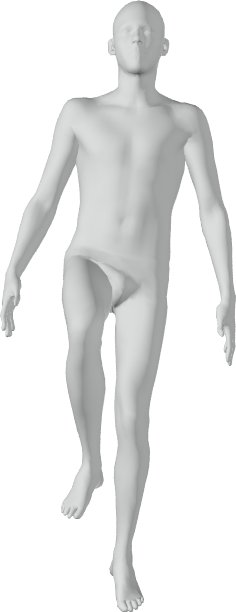} 
\end{tabular}
\caption{Human conditional generation results. Our results align with the ground-truth better than baseline approaches. }
\label{Fig:Human:CG}
\end{figure*}

\begin{figure*}
\setlength\tabcolsep{2pt}
\begin{tabular}{ccccc}
Input & GeoLatent& BRESA & ARAPDiff. & GT\\
\includegraphics[width=0.19\textwidth]{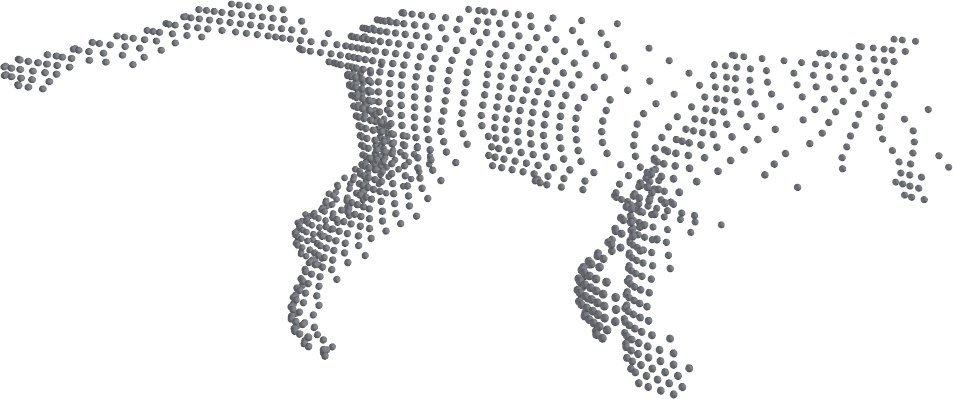}
&  
\includegraphics[width=0.19\textwidth]{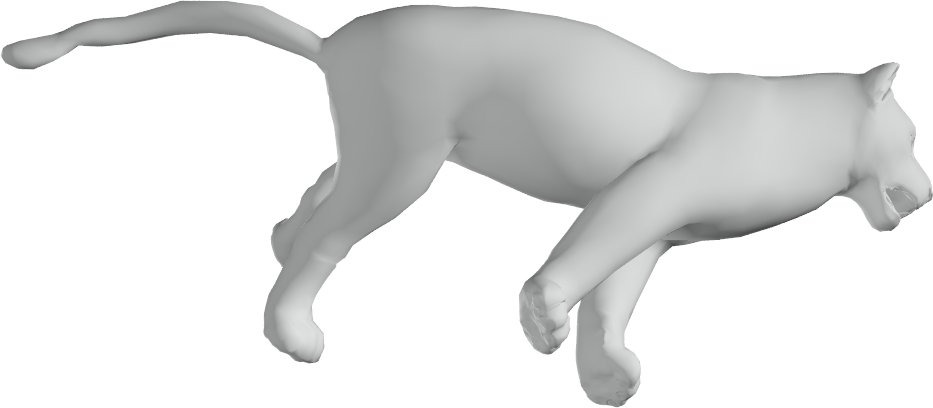}
&
\includegraphics[width=0.19\textwidth]{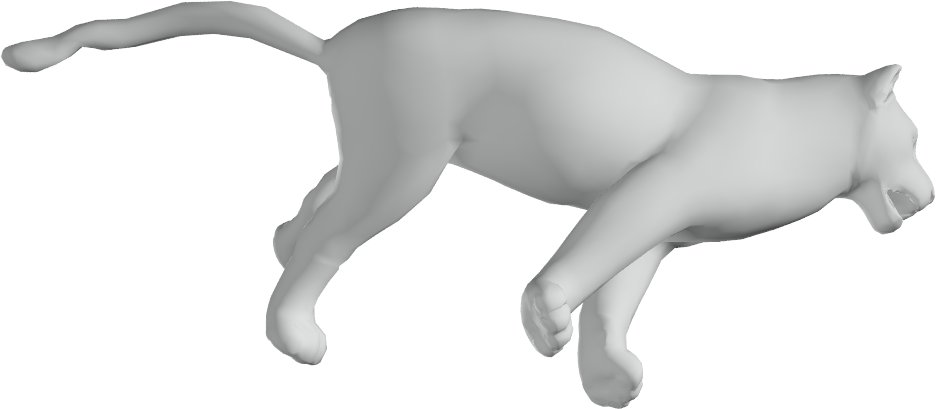}
& 
\includegraphics[width=0.19\textwidth]{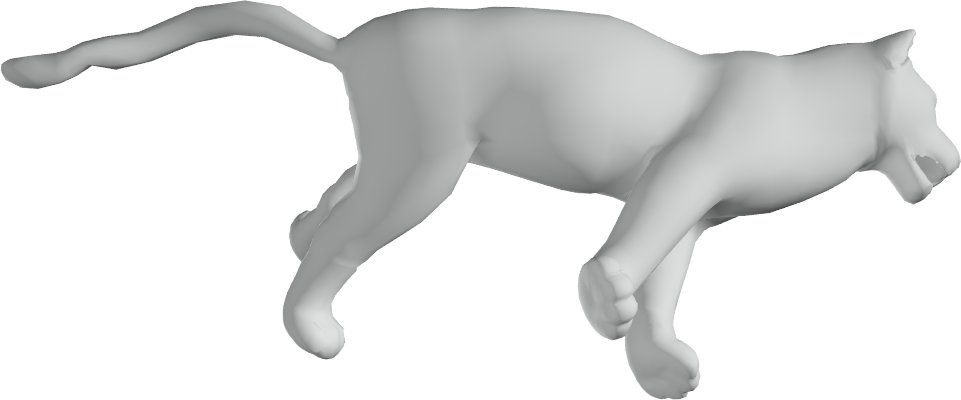}
& 
\includegraphics[width=0.19\textwidth]{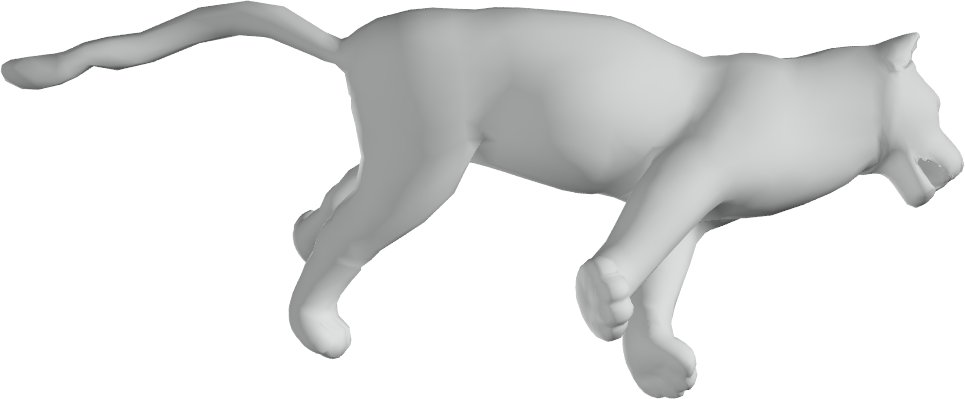}
\\
\includegraphics[width=0.19\textwidth]{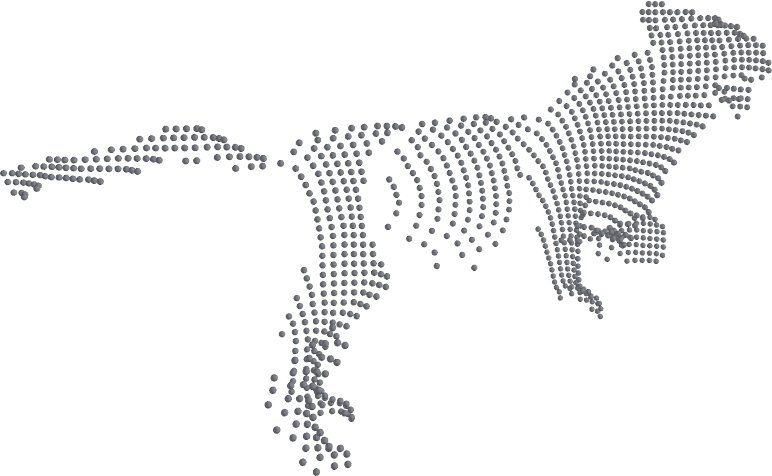}
&  
\includegraphics[width=0.19\textwidth]{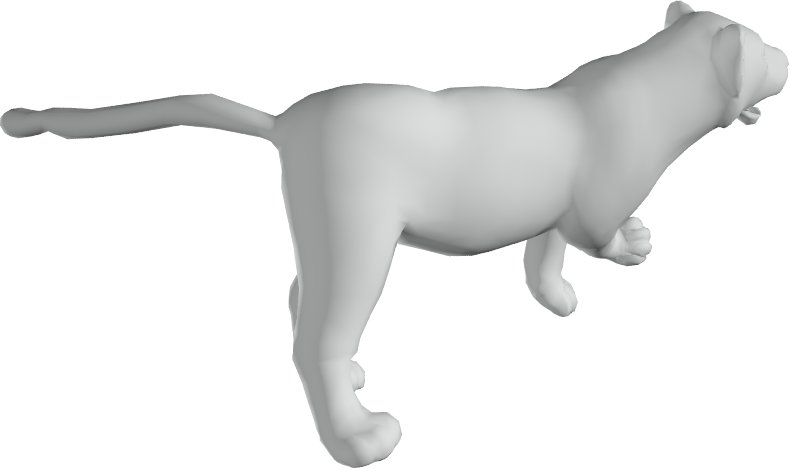}
&
\includegraphics[width=0.19\textwidth]{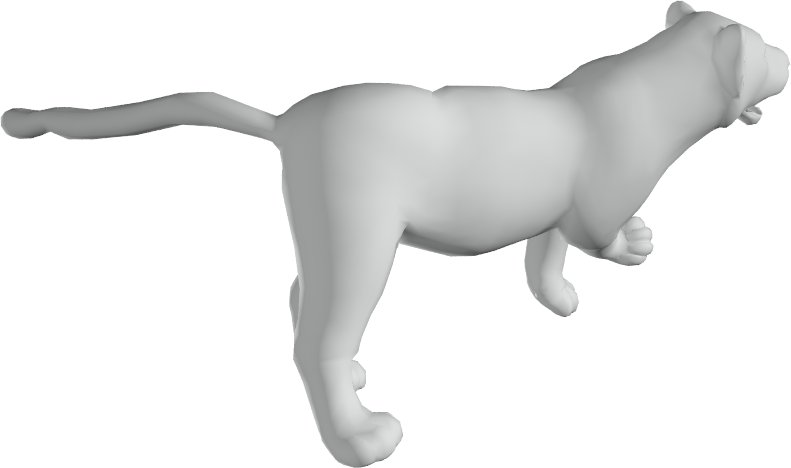}
& 
\includegraphics[width=0.19\textwidth]{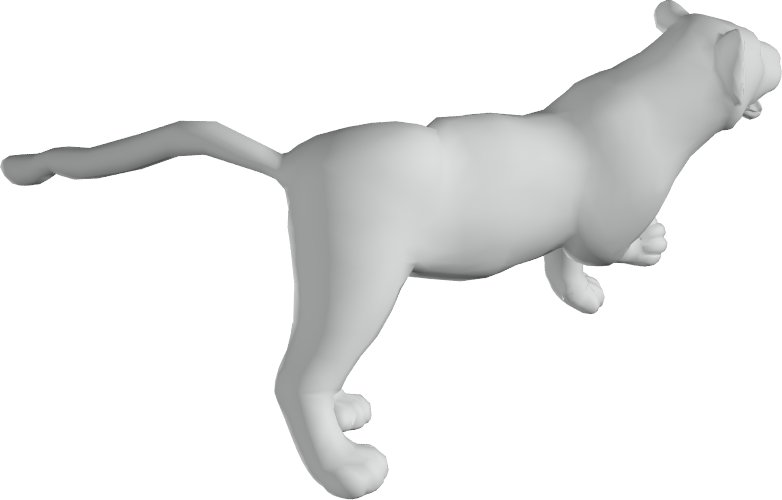}
& 
\includegraphics[width=0.19\textwidth]{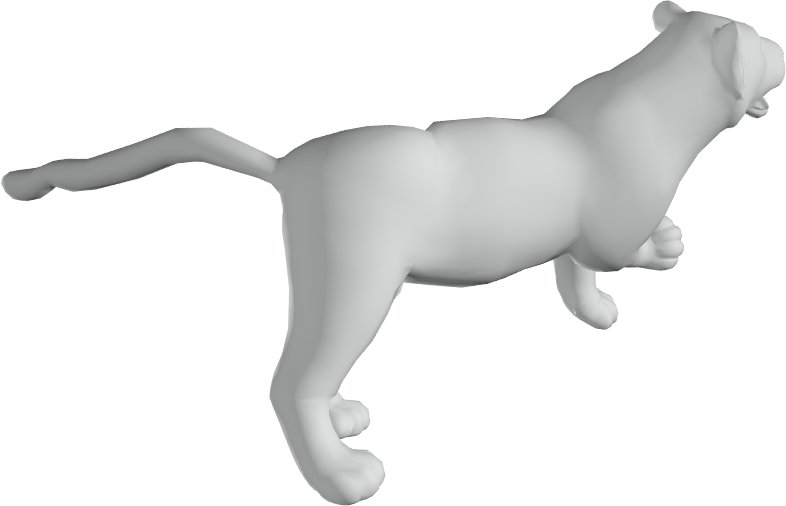}
\\
\includegraphics[width=0.19\textwidth]{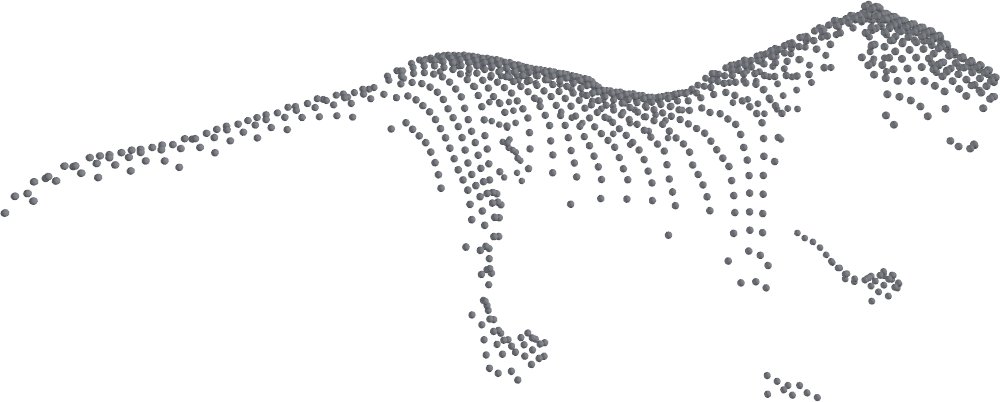}
&  
\includegraphics[width=0.19\textwidth]{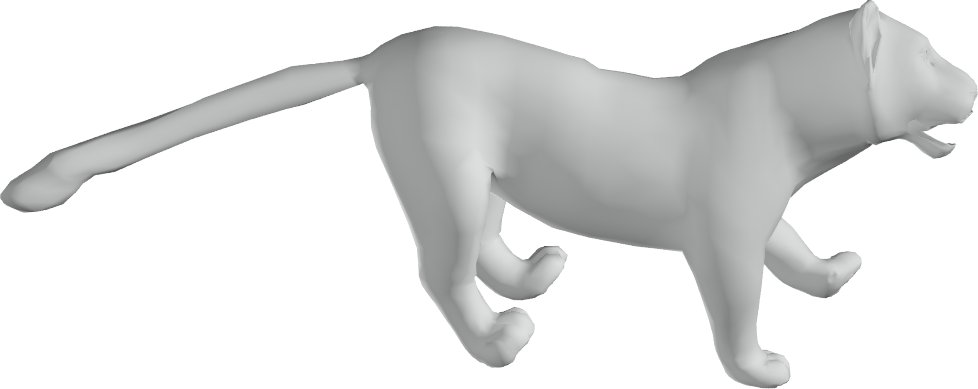}
&
\includegraphics[width=0.19\textwidth]{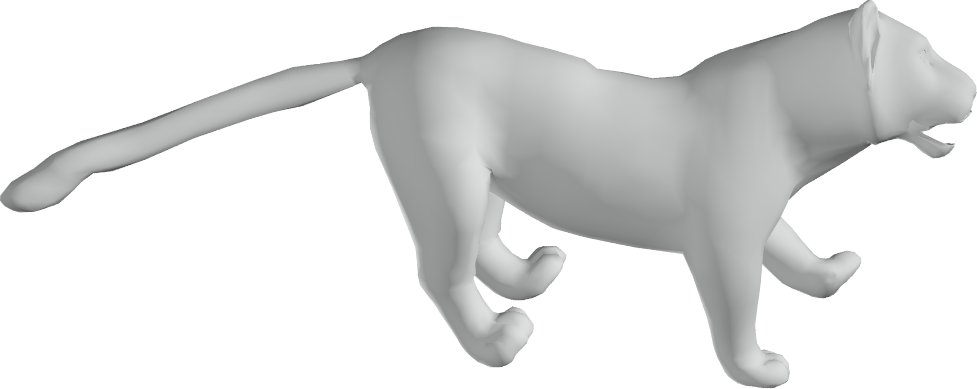}
& 
\includegraphics[width=0.19\textwidth]{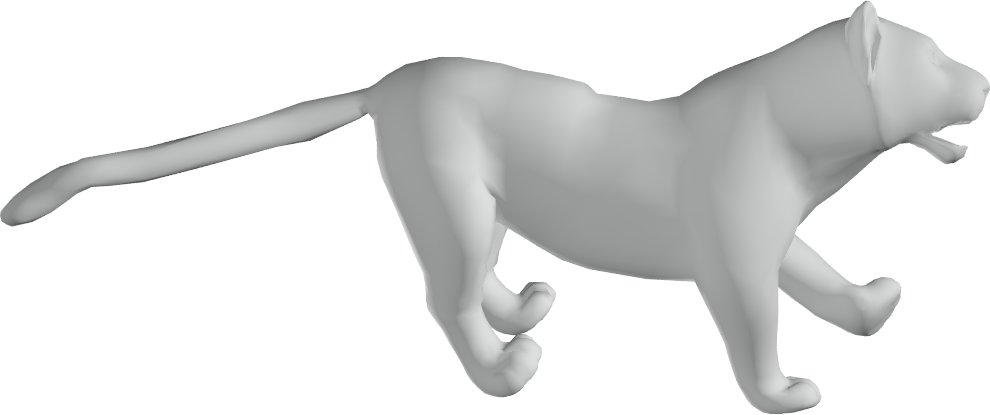}
& 
\includegraphics[width=0.19\textwidth]{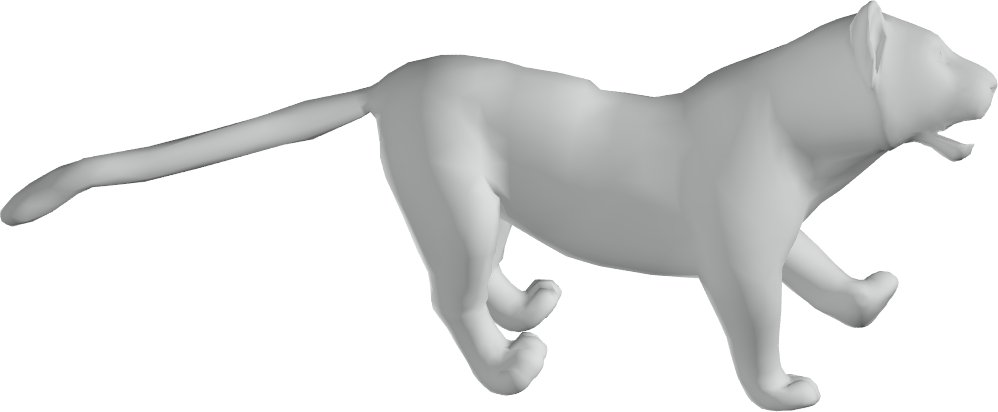}\\
\includegraphics[width=0.19\textwidth]{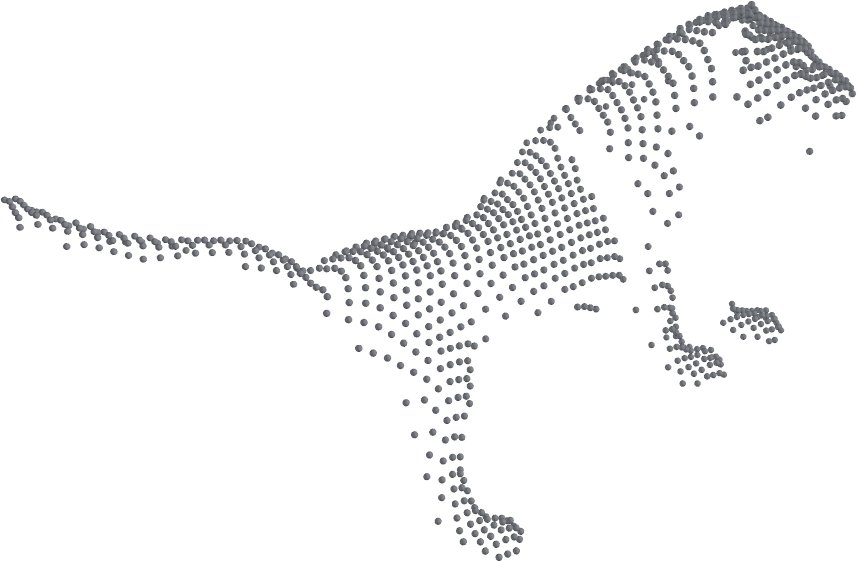}
&  
\includegraphics[width=0.19\textwidth]{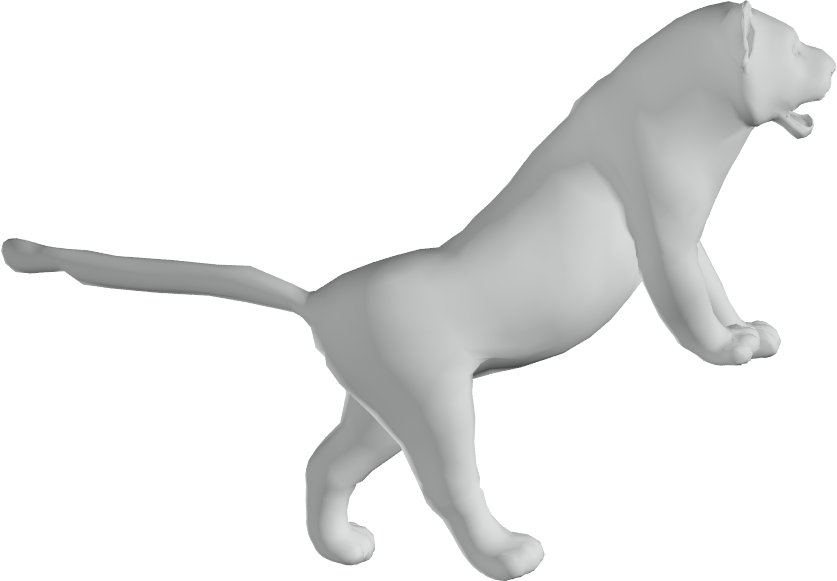}
&
\includegraphics[width=0.19\textwidth]{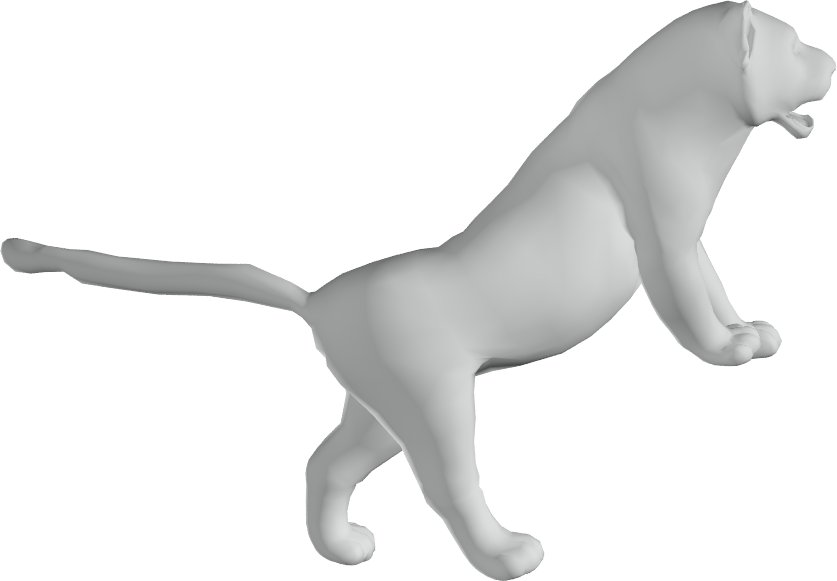}
& 
\includegraphics[width=0.19\textwidth]{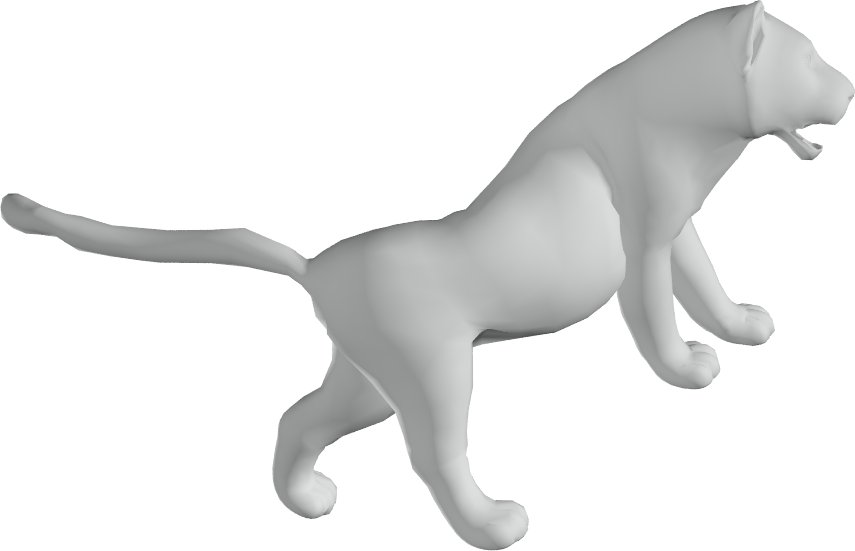}
& 
\includegraphics[width=0.19\textwidth]{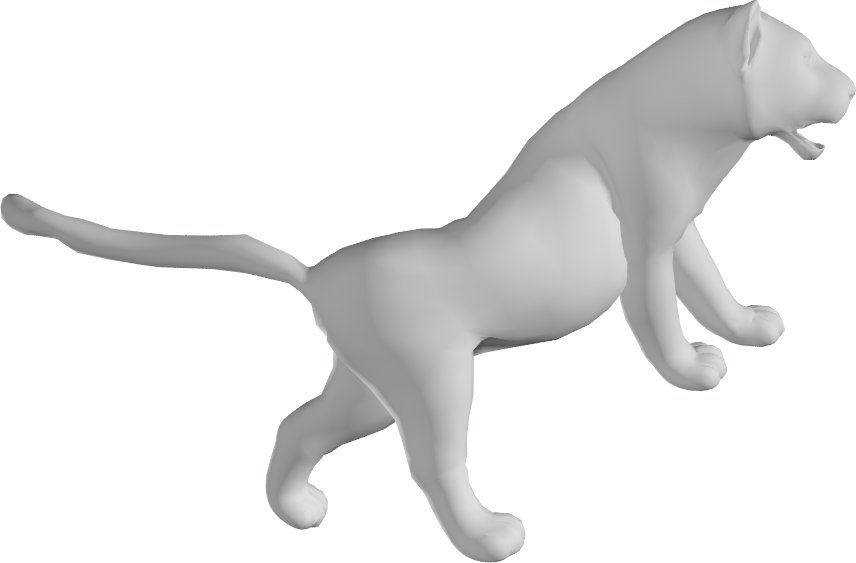}
\\
\includegraphics[width=0.19\textwidth]{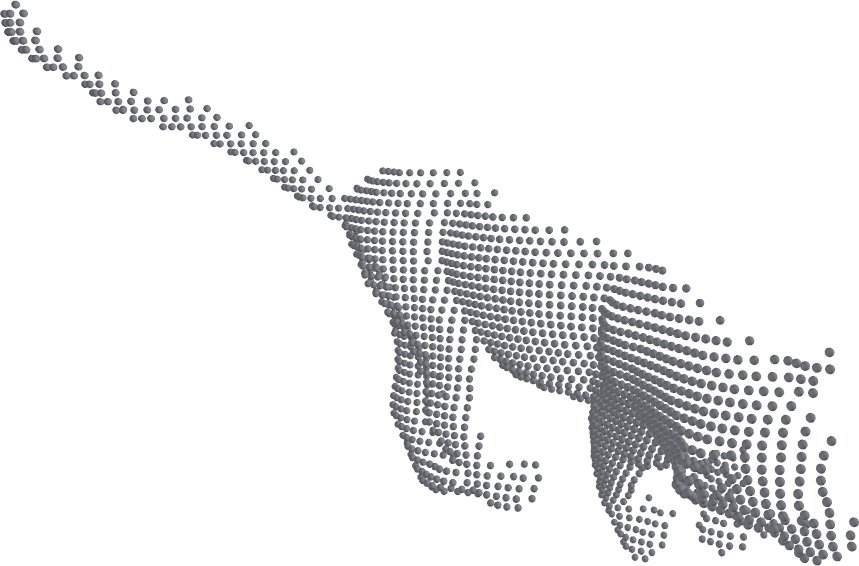}
&  
\includegraphics[width=0.19\textwidth]{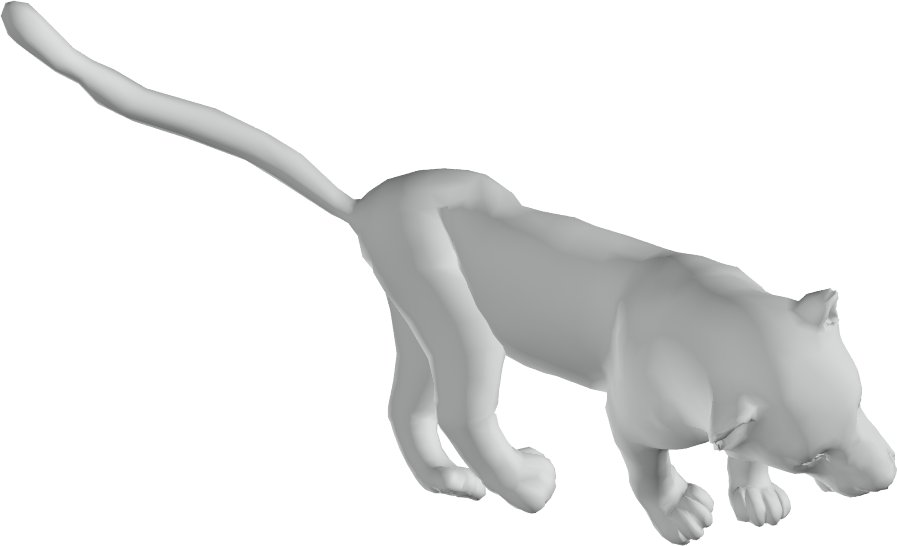}
&
\includegraphics[width=0.19\textwidth]{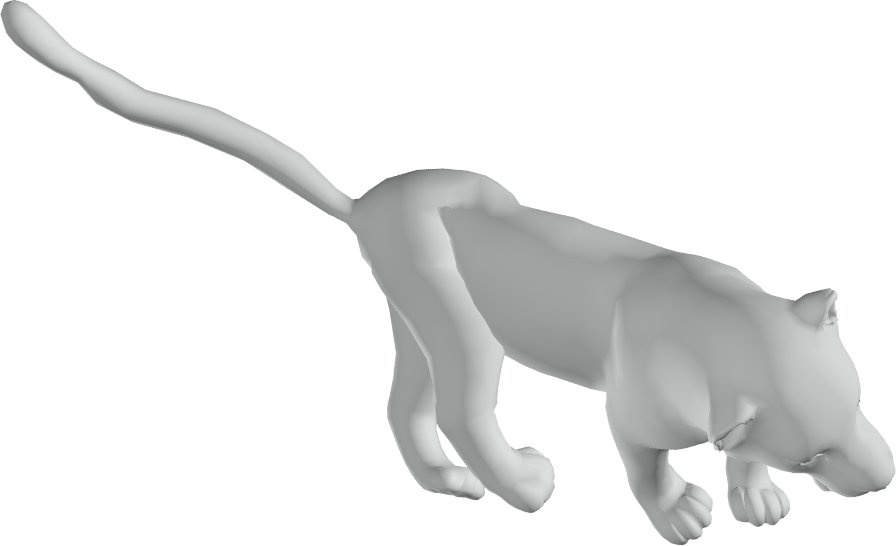}
& 
\includegraphics[width=0.19\textwidth]{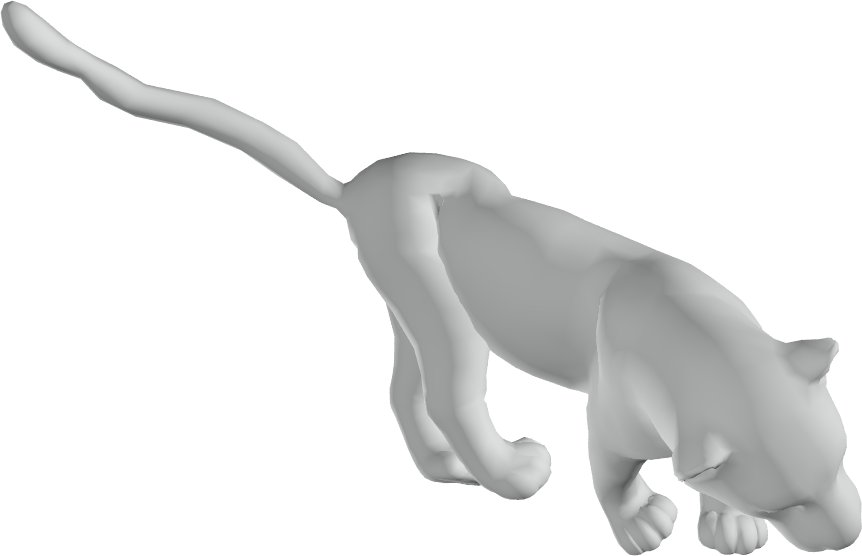}
& 
\includegraphics[width=0.19\textwidth]{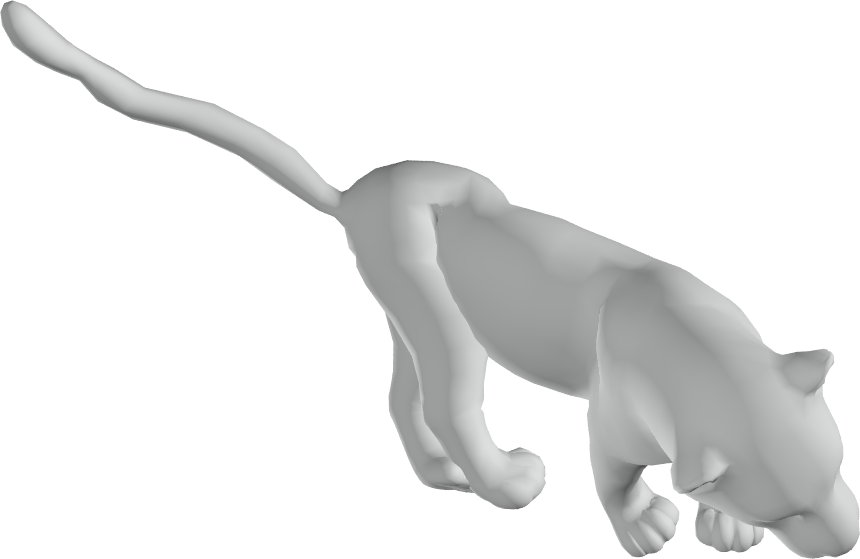}\\
\includegraphics[width=0.19\textwidth]{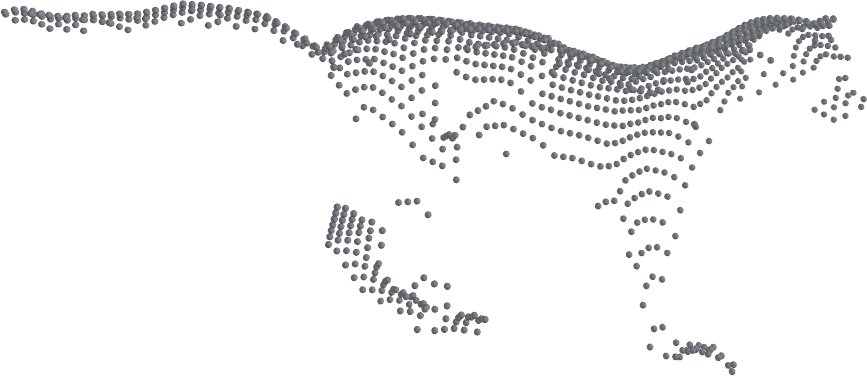}
&  
\includegraphics[width=0.19\textwidth]{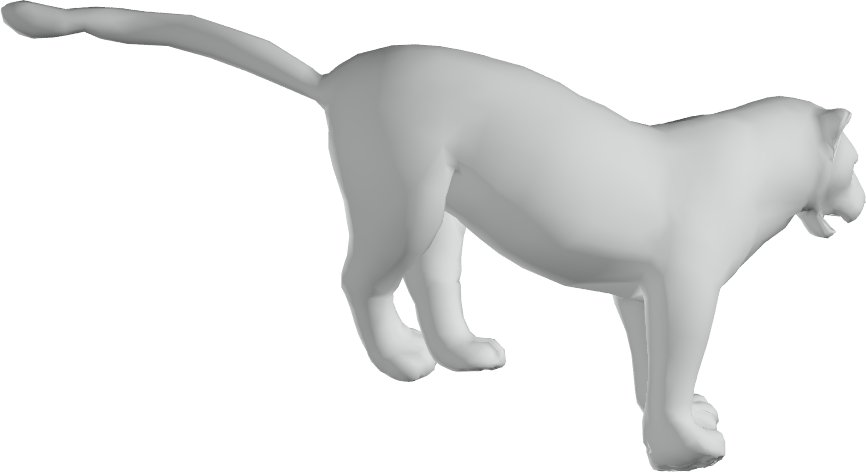}
&
\includegraphics[width=0.19\textwidth]{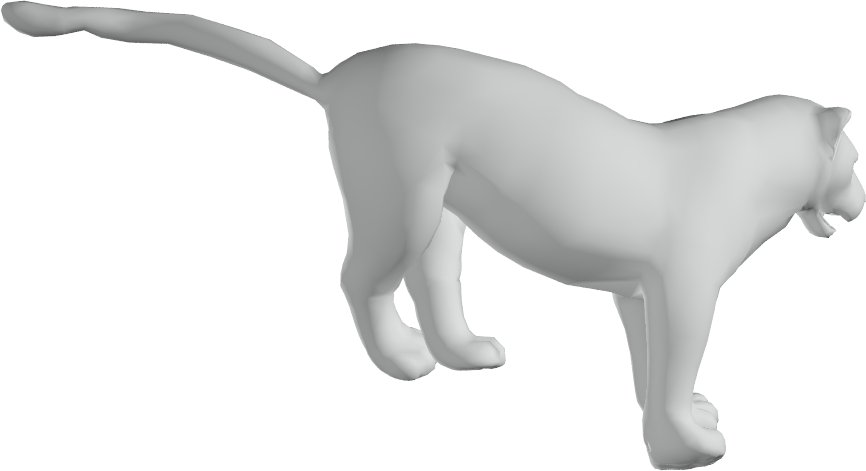}
& 
\includegraphics[width=0.19\textwidth]{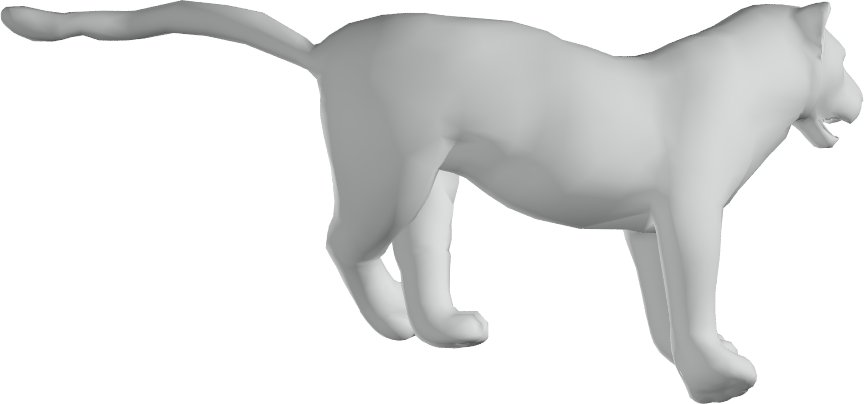}
& 
\includegraphics[width=0.19\textwidth]{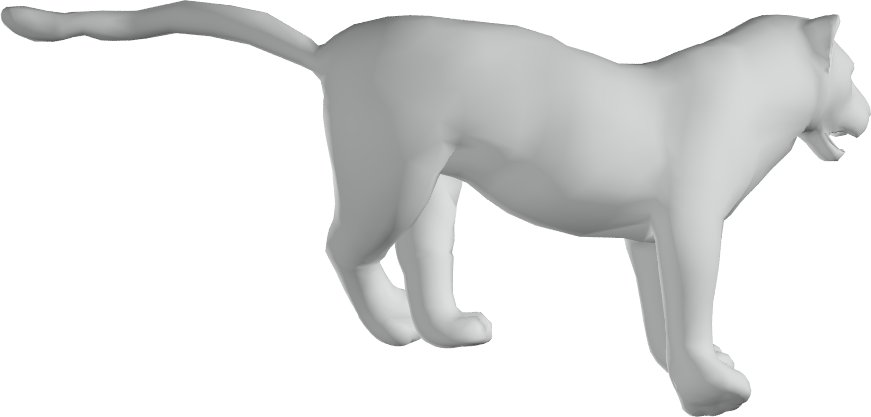}
\\
\includegraphics[width=0.19\textwidth]{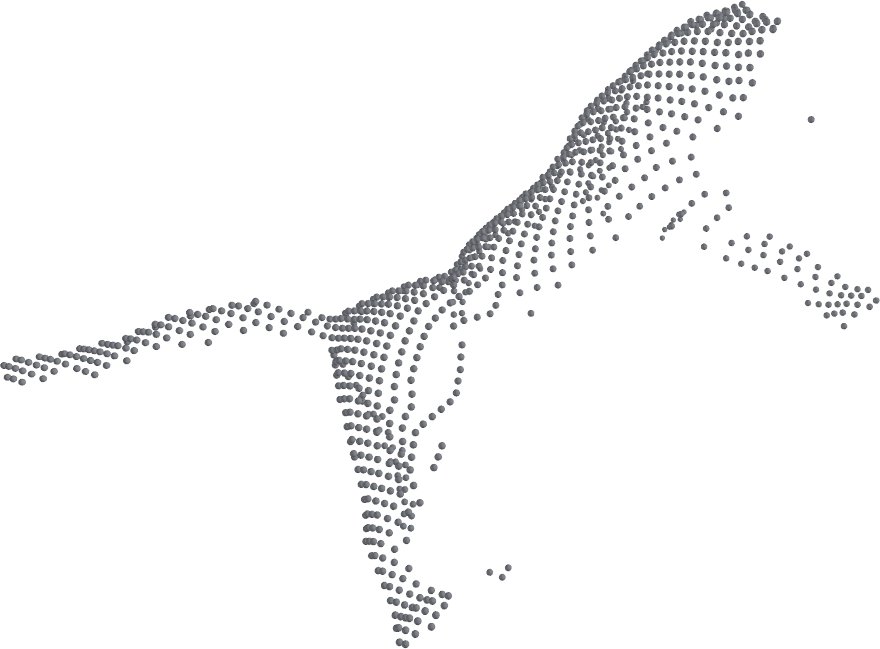}
&  
\includegraphics[width=0.19\textwidth]{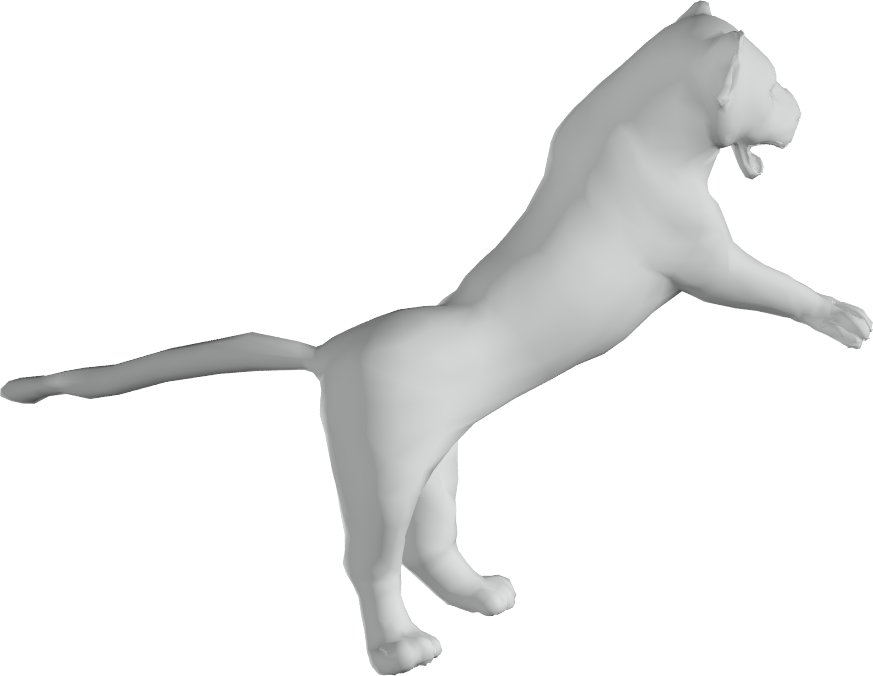}
&
\includegraphics[width=0.19\textwidth]{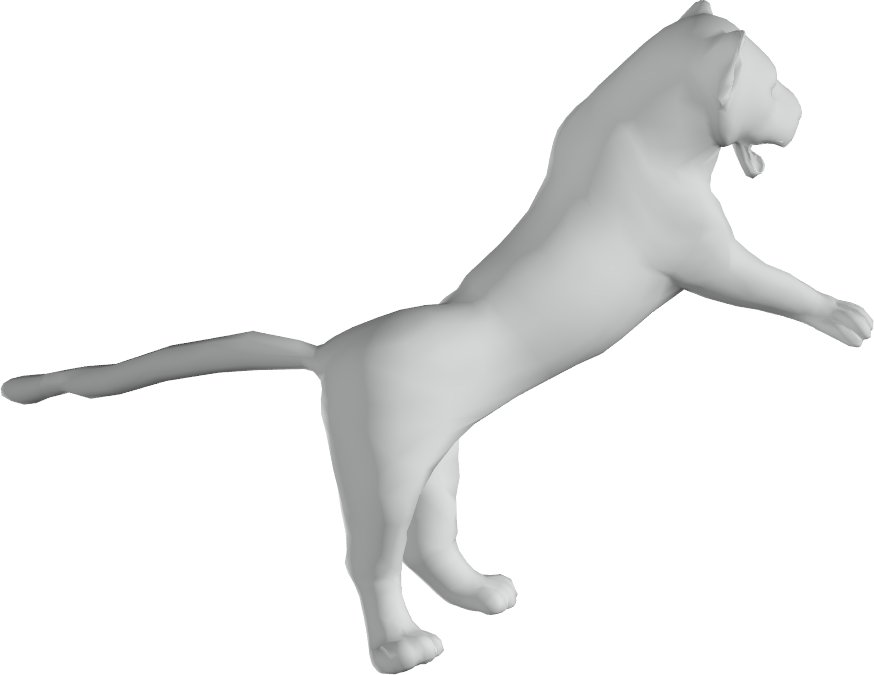}
& 
\includegraphics[width=0.19\textwidth]{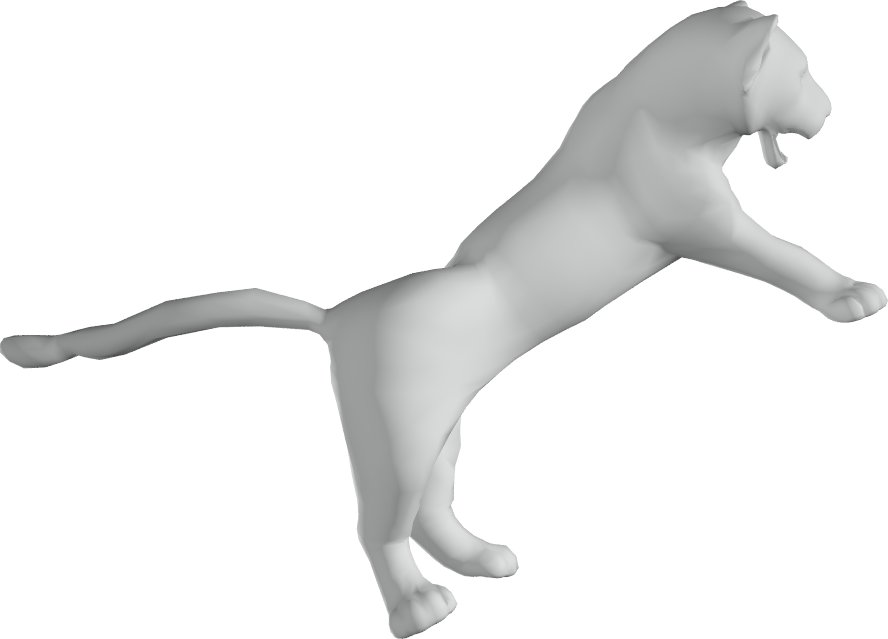}
& 
\includegraphics[width=0.19\textwidth]{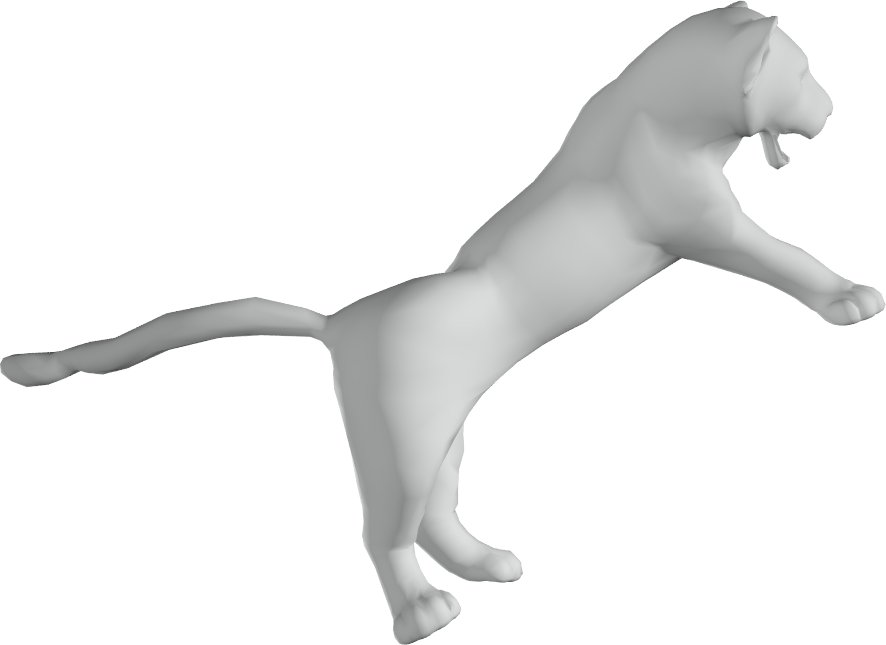}
\\
\includegraphics[width=0.19\textwidth]{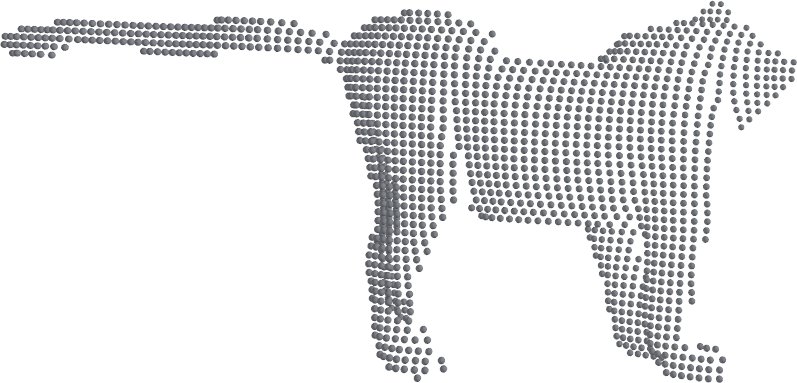}
&  
\includegraphics[width=0.19\textwidth]{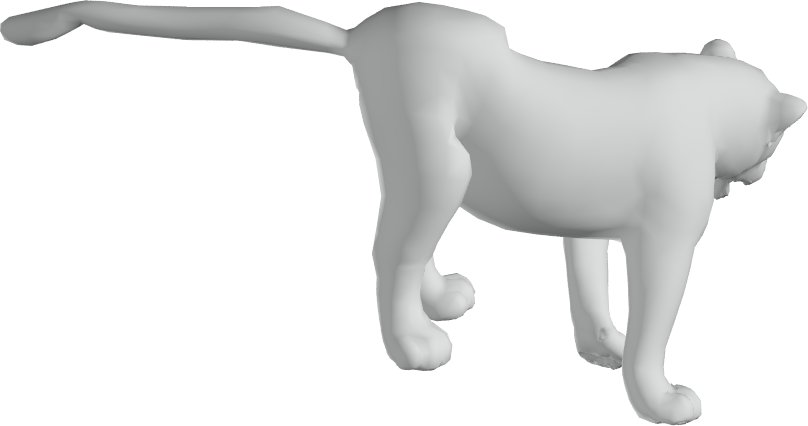}
&
\includegraphics[width=0.19\textwidth]{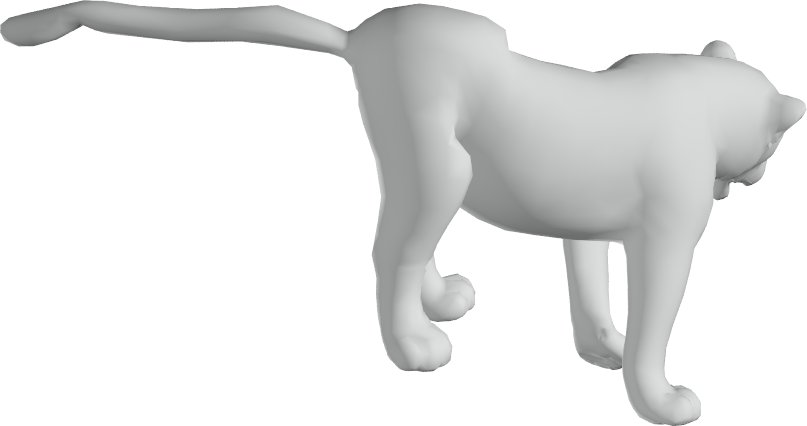}
& 
\includegraphics[width=0.19\textwidth]{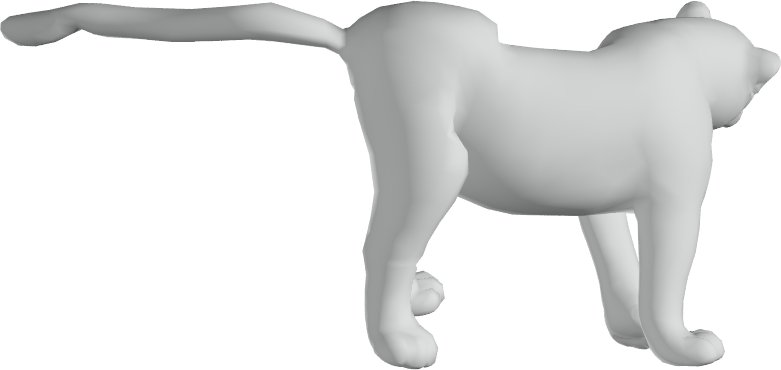}
& 
\includegraphics[width=0.19\textwidth]{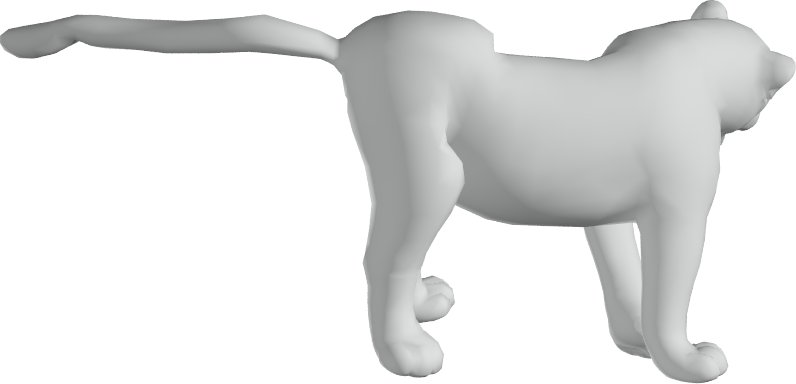}
\\
\includegraphics[width=0.19\textwidth]{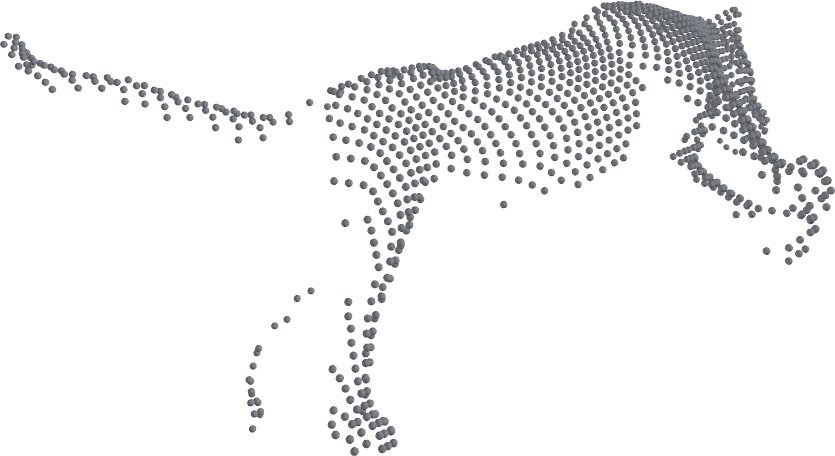}
&  
\includegraphics[width=0.19\textwidth]{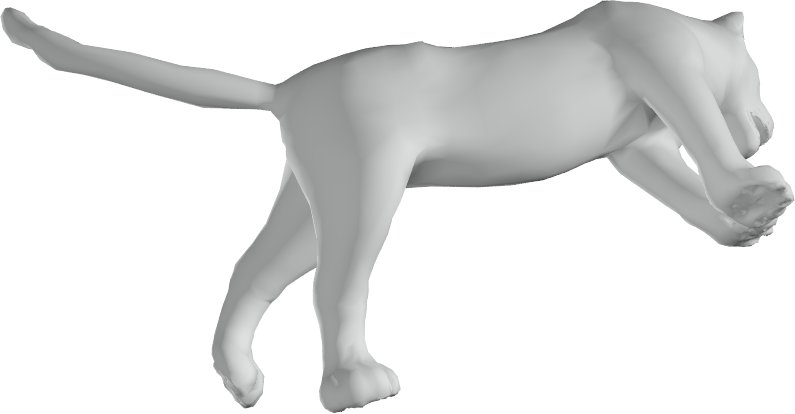}
&
\includegraphics[width=0.19\textwidth]{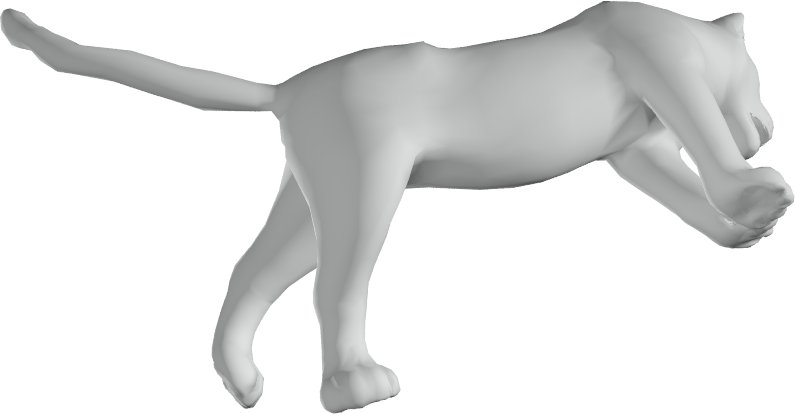}
& 
\includegraphics[width=0.19\textwidth]{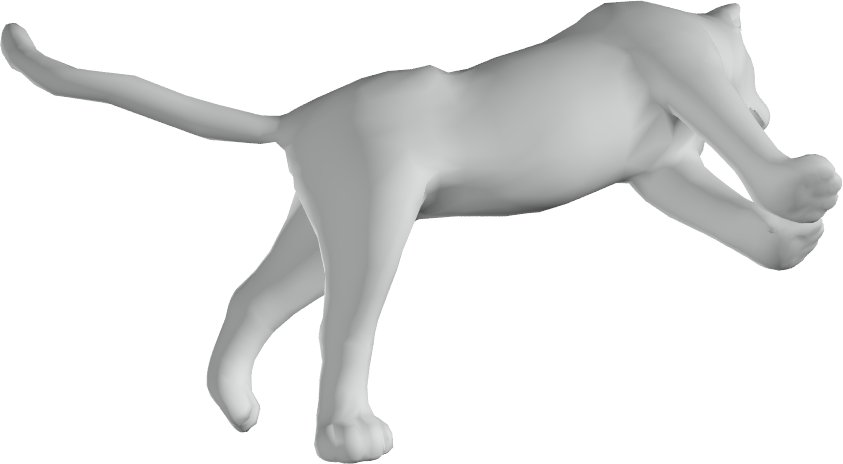}
& 
\includegraphics[width=0.19\textwidth]{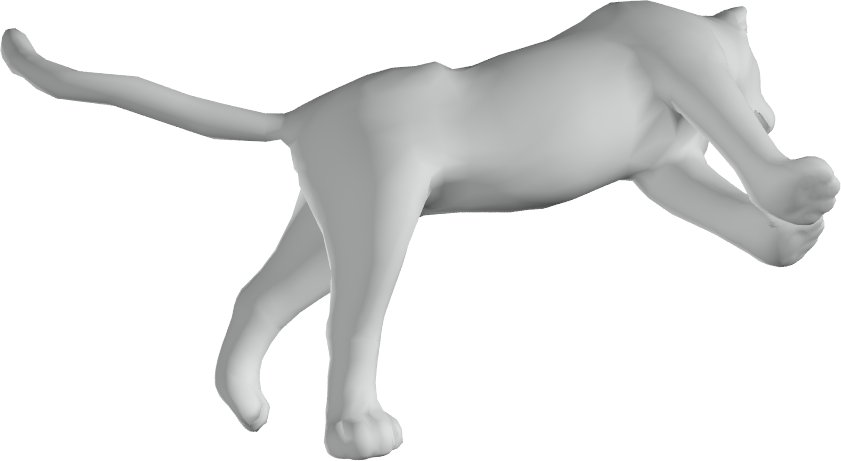}
\\
\includegraphics[width=0.19\textwidth]{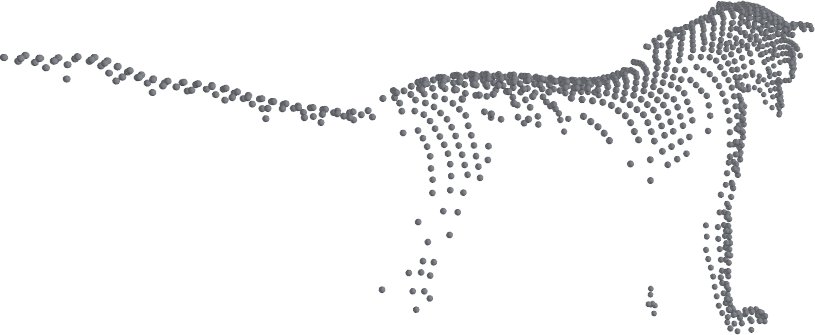}
&  
\includegraphics[width=0.19\textwidth]{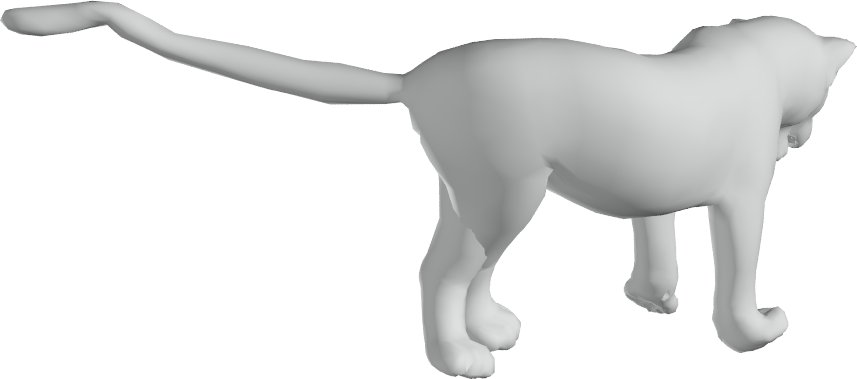}
&
\includegraphics[width=0.19\textwidth]{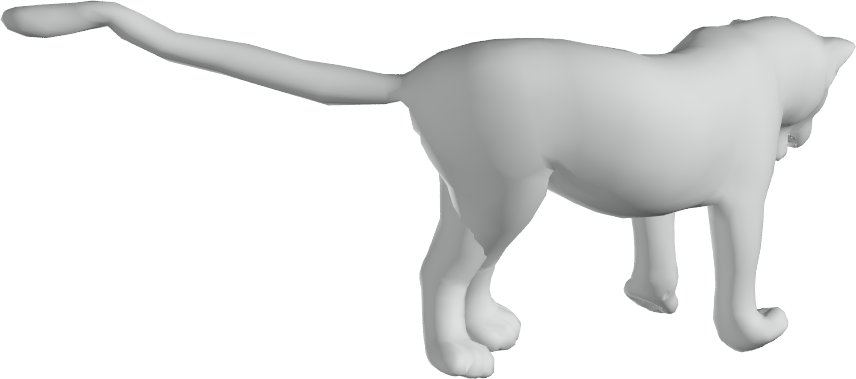}
& 
\includegraphics[width=0.19\textwidth]{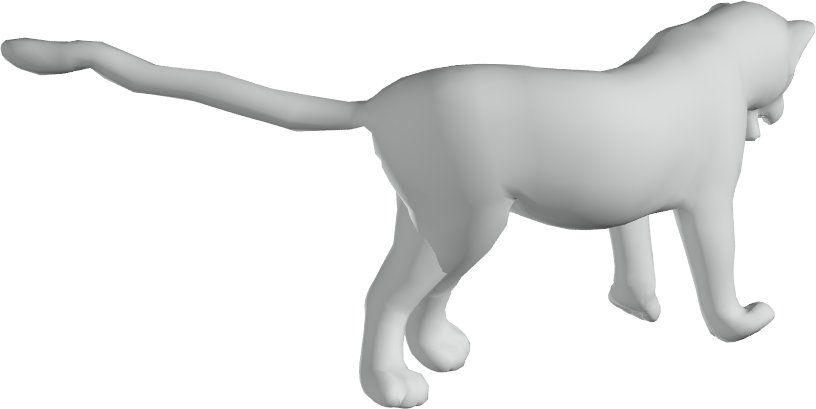}
& 
\includegraphics[width=0.19\textwidth]{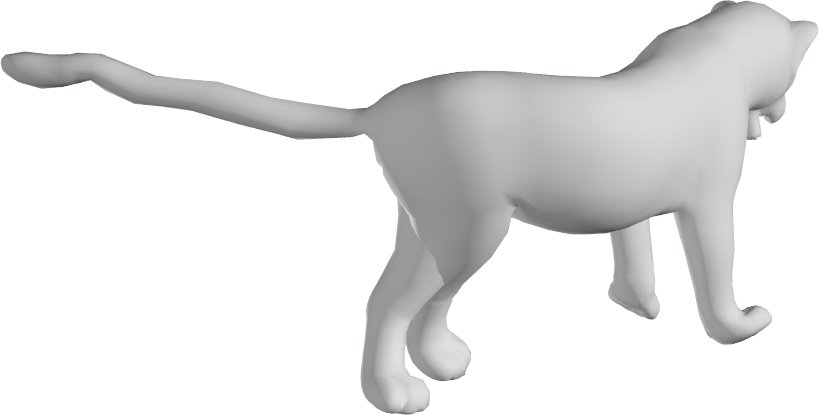}
\\
\includegraphics[width=0.19\textwidth]{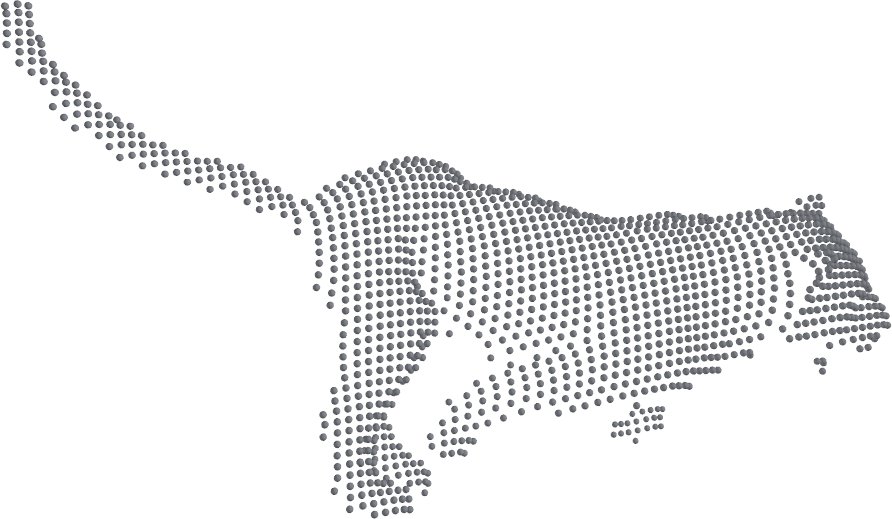}
&  
\includegraphics[width=0.19\textwidth]{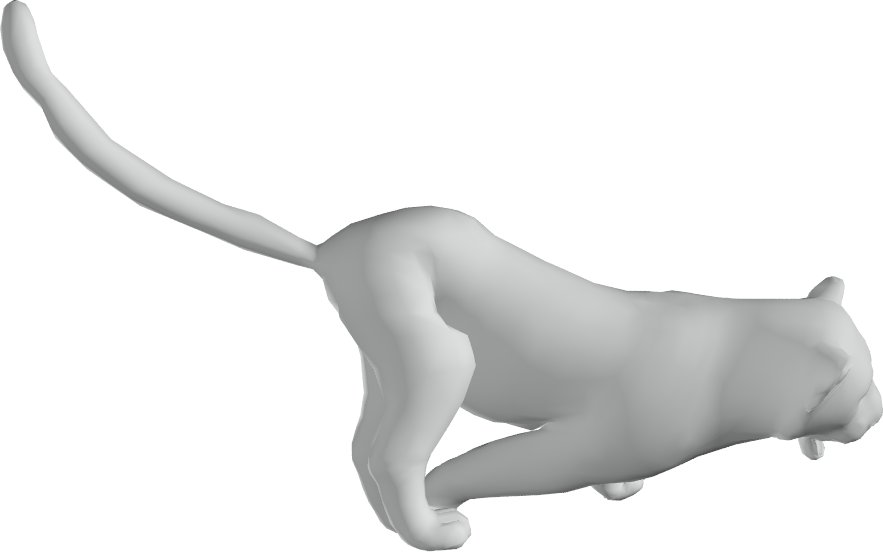}
&
\includegraphics[width=0.19\textwidth]{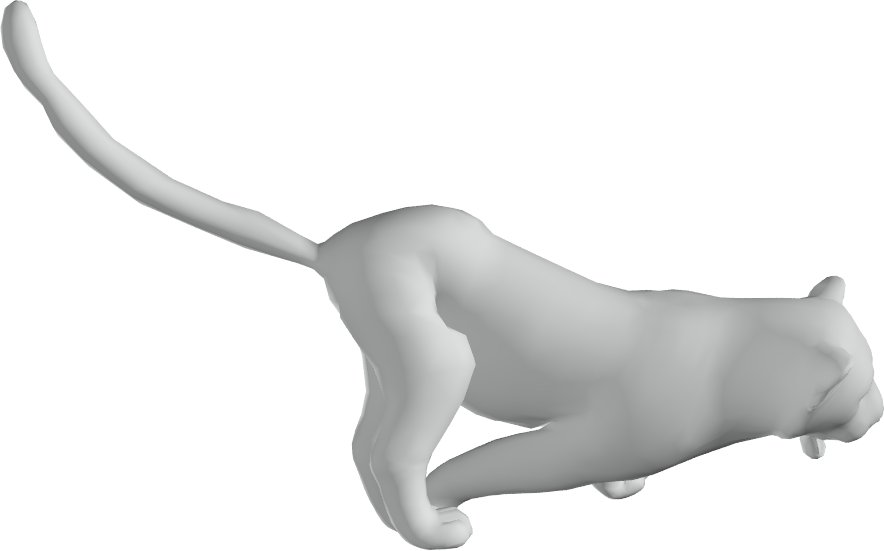}
& 
\includegraphics[width=0.19\textwidth]{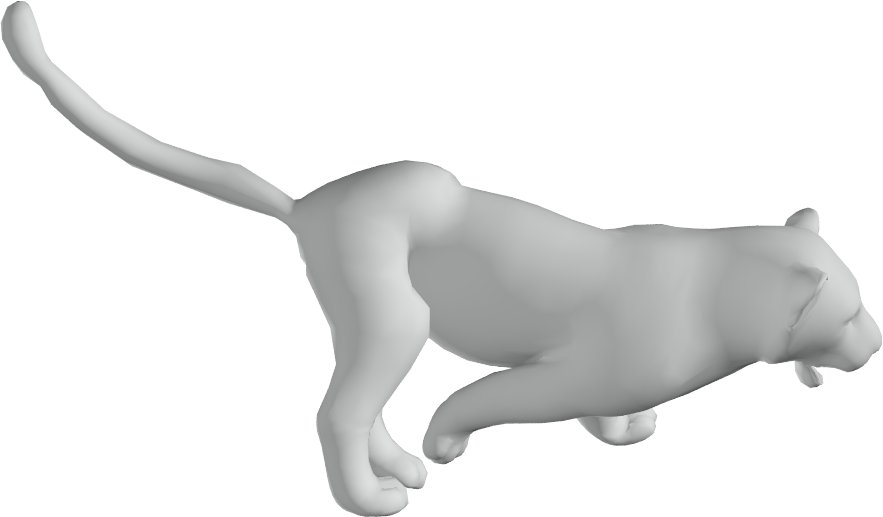}
& 
\includegraphics[width=0.19\textwidth]{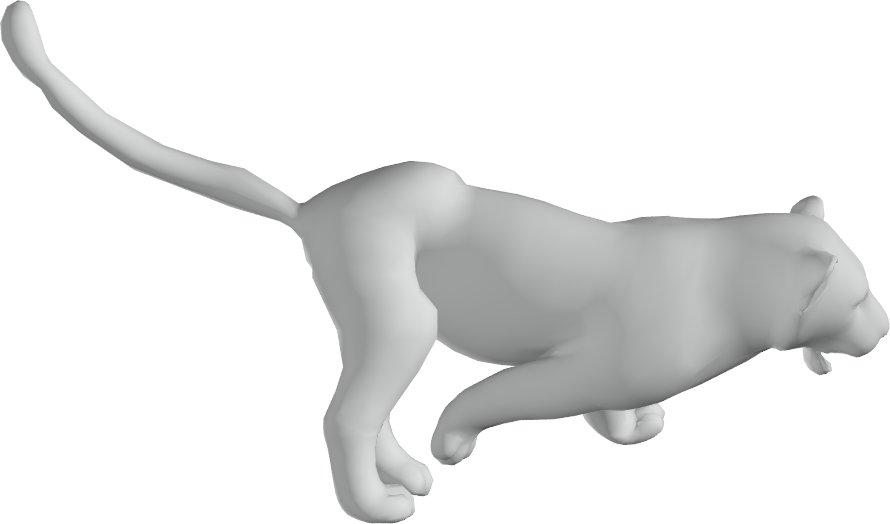}
\end{tabular}
\caption{Animal conditional generation results. Our results align with the ground-truth better than baseline approaches. }
\label{Fig:Animal:CG}
\end{figure*}

%\newpage
%\input{checklist.tex}

\end{document}